\documentclass{article}

\usepackage{microtype}
\usepackage{graphicx}
\usepackage{subcaption}
\usepackage{booktabs} 
\usepackage{mathrsfs}

\providecommand{\keywords}[1]
{
    \textbf{\textit{Keywords---}} #1
}

\usepackage{hyperref}

\usepackage[preprint]{icml2026}

\usepackage{amsmath}
\usepackage{amssymb}
\usepackage{mathtools}
\usepackage{amsthm}

\usepackage[capitalize,noabbrev]{cleveref}

\theoremstyle{plain}
\newtheorem{theorem}{Theorem}[section]

\theoremstyle{definition}

\theoremstyle{remark}

\usepackage[textsize=tiny]{todonotes}

\icmltitlerunning{Smoothing the Black-Box: Signed-Distance Supervision for Black-Box Model Copying}

\begin{document}

\twocolumn[
  \icmltitle{Smoothing the Black-Box: Signed-Distance Supervision for Black-Box Model Copying}

  \icmlsetsymbol{equal}{*}

  \begin{icmlauthorlist}
    \icmlauthor{Rub\'en Jim\'enez}{yyy}
    \icmlauthor{Oriol Pujol}{yyy}
  \end{icmlauthorlist}

  \icmlaffiliation{yyy}{Departament de Matemàtiques i Informàtica, Universitat de Barcelona, Barcelona, Catalonia, Spain}

  \icmlcorrespondingauthor{Oriol Pujol}{oriol\_pujol@ub.edu}

  \icmlkeywords{}

  \vskip 0.3in
]

\printAffiliationsAndNotice{}

\begin{abstract}
Deployed machine learning systems must continuously evolve as data, architectures, and regulations change, often without access to original training data or model internals. In such settings, black-box copying provides a practical refactoring mechanism, i.e. upgrading legacy models by learning replicas from input-output queries alone. When restricted to hard-label outputs, copying turns into a discontinuous surface reconstruction problem from pointwise queries, severely limiting the ability to recover boundary geometry efficiently. We propose a distance-based copying (distillation) framework that replaces hard-label supervision with signed distances to the teacher’s decision boundary, converting copying into a smooth regression problem that exploits local geometry. We develop an $\alpha$-governed smoothing and regularization scheme with Hölder/Lipschitz control over the induced target surface, and introduce two model-agnostic algorithms to estimate signed distances under label-only access. Experiments on synthetic problems and UCI benchmarks show consistent improvements in fidelity and generalization accuracy over hard-label baselines, while enabling distance outputs as uncertainty-related signals for black-box replicas.
\end{abstract}


\section{Introduction}
\label{sec1}

Deployed machine learning systems must evolve as data, architectures, and regulations change. Practitioners and companies are confronted with the problem of updating legacy models, migrating them to new architectures, or adapting them to satisfy transparency, privacy, or explainability requirements \cite{b3,b4,b5,b6,b7,b8,b9,b10}. This creates a growing need for systematic approaches to model maintenance and continuous evolution.

These challenges can be viewed as problems of environmental adaptation \cite{unceta2020environmental}, where a deployed model must be modified in response to shifts in data, features, computational budgets, institutional conditions, or regulation. Within this setting, differential replication denotes generating a new model that preserves salient properties of an existing one (e.g., decision boundary, predictive behaviour, performance) while satisfying updated constraints. Depending on access to data and internals, replication spans from straightforward retraining, through dataset editing or enrichment, to white-box transfer and distillation-based teacher-student strategies. In the most restrictive regime, when neither data nor internals are available, replication collapses to black-box copying, i.e. training a replica solely from input-output queries to approximate the original decision boundary \cite{b20}. This data- and model-agnostic regime is the focus of this work.

Black-box copying has been used to address feature unavailability, architectural migration, interpretability needs, and regulatory compliance \cite{b11,b12,b13,b14}. In practice, however, practitioners care not only about fidelity to the source model, but also about task accuracy and additional properties such as fairness, privacy, robustness, or computational efficiency. An important limitation in many realistic scenarios is that the oracle exposes only hard labels, providing a discontinuous supervisory signal that is especially problematic for non-smooth boundaries and high-dimensional inputs \cite{b15}. Most existing methods respond by designing better query datasets \cite{b18,b19}, but the structure of the hard-label signal itself has received comparatively little attention.

In this work, we study black-box copying in a non-adversarial, fully data-free setting where query access to deployed models is legitimate and inexpensive, and we frame copying and distillation as tools for model maintenance and responsible evolution, analogous to refactoring in software engineering. Departing from prior work focused on compression or adversarial extraction, we introduce a distance-based distillation framework that replaces hard-label supervision with signed distances to the teacher’s decision boundary, transforming copying from learning a discontinuous surface into regressing a smooth target that exploits local geometry. Building on this formulation, we propose a novel $\alpha$-parametrized, target-centric regularization mechanism that explicitly controls the regularity of the supervisory signal, with regimes that interpolate between replication and distance copying and provide provable H{\"o}lder/Lipschitz guarantees. This data-free regularization improves sample efficiency in low-data regimes and enables controlled post-hoc refinement of deployed models, while the distance outputs naturally endow replicas with uncertainty estimates useful for downstream decision-making.

Section~\ref{sec2} reviews related work. Section~\ref{sec3} introduces our framework and the proposed target-centric regularization scheme, as well as two model-agnostic algorithms for estimating signed distances from hard labels. Sections~\ref{sec4} and~\ref{sec:5} present the experimental setup and results on synthetic datasets and UCI benchmarks, respectively. We conclude by summarizing the main findings, discussing key limitations, and outlining directions for future research. The appendices provide pseudocode, the proof of Theorem~\ref{teo1}, and supplementary tables and figures supporting the reported results.

\section{Related Work: Distillation and Copying with Hard-Label Binary Teachers}
\label{sec2}

Knowledge distillation traditionally transfers knowledge from high-capacity teachers to compact students by matching soft class probabilities \cite{b30}, typically requiring access to model internals \cite{b32,b33}. In contrast, \emph{hard-label black-box copying} trains replicas solely from input--output queries \cite{b20,b21,Tramer2016Stealing}, addressing data-free model maintenance when neither training data nor internals are available.

Hard labels create discontinuous supervision that hinders decision boundary recovery \cite{b15}. Prior work mitigates this via boundary proxies. Wang et al. \cite{Wang2021DB3KD} estimate per-class boundary distances to construct soft labels; adversarial methods like boundary attacks \cite{Brendel2018BoundaryAttack} probe geometry via walks; and recent extensions pursue logit reconstruction \cite{Zhou2023DecisionLogitsKD}, generative boundary sampling \cite{Pei2025QEDG}, or active probing \cite{Pal2020ActiveThief}.

Unlike these, we regress signed distances directly with $\alpha$-controlled H\"older/Lipschitz regularity (Thm.~\ref{teo1}), providing a single-parameter family interpolating replication ($\alpha=0$) and distance copying ($\alpha=1$) while guaranteeing target smoothness without complex soft-label engineering.

\subsection{Differential replication through copying}\label{sec:diff}

Knowledge distillation typically assumes fixed datasets and soft targets (logits), even in data-free settings where synthetic samples query the teacher \cite{b30}. Black-box copying differs fundamentally \cite{b20} considering unknown training distributions and hard-label oracles. As a result, it becomes a \emph{dual optimization} over both synthetic queries $S$ and replica parameters $\theta$. For expressive student spaces $\mathcal{H}_t$, hard-label supervision is separable, so unconstrained fidelity minimization ($\theta^\dagger$) achieves perfect empirical fit. Copying thus reduces to constrained capacity ($\Omega(\theta)$) minimization:
\begin{flalign}\label{eq:capacity}
\underset{\theta,S}{\text{minimize}}  &\quad \Omega(\theta)\\
\text{subject to} & \quad \left\|R^{\mathcal{F}}_{emp}(f_{\mathcal{C}}(\theta),f_{\mathcal{O}})-R^{\mathcal{F}}_{emp}(f_{\mathcal{C}}(\theta^{\dagger}),f_{\mathcal{O}})\right\|<\varepsilon \nonumber
\end{flalign}
for a tolerance $\varepsilon$, where $\theta^{\dagger}$ is the solution of the unconstrained fidelity minimization problem\footnote{Single-pass unconstrained problem: $$\theta^{\dagger} = \underset{\theta,S}{\text{argmin}} \quad R^{\mathcal{F}}_{emp}(f_{\mathcal{C}}(\theta),f_{\mathcal{O}}).$$}, $R^{\mathcal{F}}_{emp}$ is the empirical fidelity risk, and $f_{\mathcal{C}}$ and $f_{\mathcal{O}}$ correspond to the copying and original models, respectively.

A simplified version of the former problem samples $S \sim P_\mathcal{Z}$ and then minimizes $R^\mathcal{F}_\text{emp}$. Although in practice this {\it single-pass} approach is the standard setup, achieving good fidelities require of large synthetic sets. Other approaches, such as alternating optimization jointly refines queries and parameters for memory-efficient approximations of \eqref{eq:capacity} \cite{b21}. Our signed-distance supervision can be layered on top of existing copying pipelines, providing a controllable smoothness–robustness regularizer.

\section{Target-centric Regularized Distance-based Copying}
\label{sec3}

Unlike prior decision-based distillation methods that reconstruct multi-class soft targets and optimize divergence-based objectives, we cast hard-label copying as a \emph{regression} problem by learning to approximate the full signed-distance surface induced by the teacher’s decision boundary. Moreover, we propose a novel target-centric regularization mechanism governed by a parameter $\alpha$ that converts discontinuous label feedback into a smooth metric target with provable regularity for which we provide explicit regularity guarantees (H\"older or Lipschitz continuity in relevant regimes), interpolating between pure replication ($\alpha=0$) and full distance copying ($\alpha=1$), and providing an interpretable control over target smoothness.

Specifically, we propose to build the copy $f_{\mathcal{C}}$ using signed distances to the decision boundary of the black-box teacher. Without loss of generality, we restrict our analysis to a binary teacher $f_{\mathcal{O}}:\mathbb{R}^d\rightarrow\{-1,+1\}$\footnote{Multiclass problems can be reduced to collections of binary classification tasks, for example via Error Correcting Output Coding.}. Then, for any synthetic query point $s_i \in S \subset \mathcal{Z}$, we replace the hard-label output $f_{\mathcal{O}}(s_i)$ by a signed-distance target $\ell(s_i)$ defined as
\begin{align}
\ell(s_i) &=  f_{\mathcal{O}}(s_i)\cdot\xi(s_i) ,\\
\xi(s_i) &= \inf_{\substack{z \in \mathcal{Z}\\ f_{\mathcal{O}}(z) \neq f_{\mathcal{O}}(s_i)}} d(s_i, z),
\end{align}
where $\xi(s_i)$ denotes the distance from $s_i$ to the closest point in the operational space with opposite predicted label.

Using these targets, which encode both class membership and distance information in a single scalar, we redefine the optimization goal in the single-pass by replacing the oracle outputs with signed-distance supervision in the empirical fidelity loss,
$R^{\mathcal{F}}_{emp}(f_{\mathcal{C}}(\cdot;\theta),\ell)$.

\subsection{Target-centric Regularization}

The framework of distillation implicitly exhibits a regularization effect by mimicking the original classifier using only the finite information obtained through the sampled query set $S$. This has been exploited in self-distillation settings \cite{b37,b38,b39,b40}, where the same architecture may show improved generalization performance. In the black-box setting, a similar effect emerges: although the copy aims to imitate the black-box, it rarely replicates all its irregularities, since it only observes the black-box’s outputs on a finite set. 

In this context, distance-based supervision can further strengthen this regularization effect. Signed distances vary smoothly across the input space, thereby imposing a continuity constraint on the target function, especially around the boundary. This property facilitates learning and reduces the influence of small-scale irregularities, which are naturally down-weighted in the loss and can be safely ignored.

To exert fine-grained control over the regularization induced by distance-based targets, we introduce an $\alpha$-governed, model-agnostic, target-centric regularization scheme. Specifically, the original signed-distance target $\ell$ is generalized by exponentiating the distance term as:
\begin{equation}
  \ell_{\alpha}(s_i) = f_{\mathcal{O}}(s_i)\,\xi(s_i)^{\alpha}.
\end{equation}
The parameter $\alpha$ allows the method to interpolate smoothly between pure replication ($\alpha = 0$) and full distance copying ($\alpha = 1$), directly controlling the regularity of the supervisory signal. The following theorem  characterizes the regularity of the resulting targets.

\begin{theorem}[Regularity of $\alpha$ signed distances]
\label{teo1}
Let $f: \mathcal{X} \longrightarrow \{-1,1\}$ be a function and let $\alpha > 0$, we consider $d$ a distance in $\mathcal{X}$ bounded by $D > 0$ and define $l_{\alpha}(x) = f(x)d(x, A_{x})^{\alpha}$ for all $x \in \mathcal{X}$, where $A_{x} = \{y \in \mathcal{X}\;\vert\; f(y) \neq f(x)\}$. Then, for all $x,y \in \mathcal{X}$:
\begin{itemize}
    \item If $\alpha \leq 1$, we have $\vert l_{\alpha}(x) - l_{\alpha}(y) \vert \leq 2d(x,y)^{\alpha}$.
    \item If $\alpha \geq 1$, we have $\vert l_{\alpha}(x) - l_{\alpha}(y)\vert \leq 2\alpha D^{\alpha-1}d(x,y)$.
\end{itemize}
\end{theorem}
\begin{proof}
See Appendix \ref{appxZ}.
\end{proof}

Theorem~\ref{teo1} shows that $\alpha$-signed distance targets are $\alpha$-H\"older continuous for $\alpha\leq 1$ and Lipschitz continuous for $\alpha\geq 1$. These guarantees explain the regularization effect of distance-based copying. Intuitively this implies that, within any ball of radius $\varepsilon$, function values vary by at most $O(\varepsilon^\alpha)$ ($\alpha\le 1$) or $O(\varepsilon)$ ($\alpha\ge 1$), ensuring the graph remains bounded in a vertical strip of width proportional to $\varepsilon^{\min(\alpha,1)}$. This encourages the copy to learn a smoother decision surface by progressively attenuating small-scale boundary irregularities as $\alpha$ increases.

\begin{figure}[htbp]
\centering
\begin{subfigure}[t]{0.488\columnwidth}
    \includegraphics[width=\linewidth]{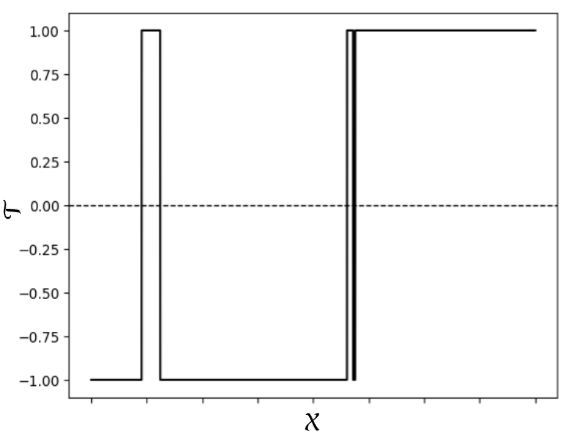}
    \caption{Black-box}
\end{subfigure}
\hfill
\begin{subfigure}[t]{0.492\columnwidth}
    \includegraphics[width=\linewidth]{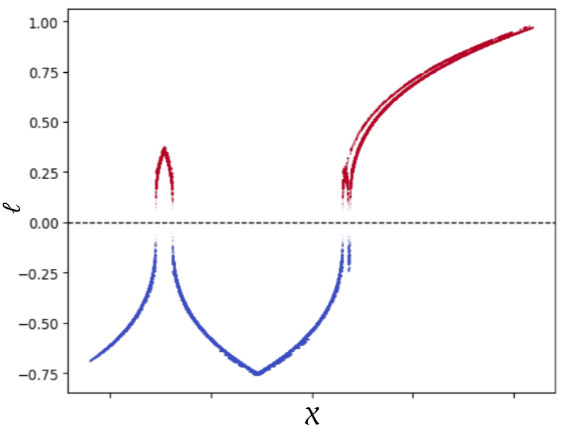}
    \caption{$\alpha = 0.33$}
\end{subfigure}

\begin{subfigure}[t]{0.488\columnwidth}
	\includegraphics[width=\linewidth]{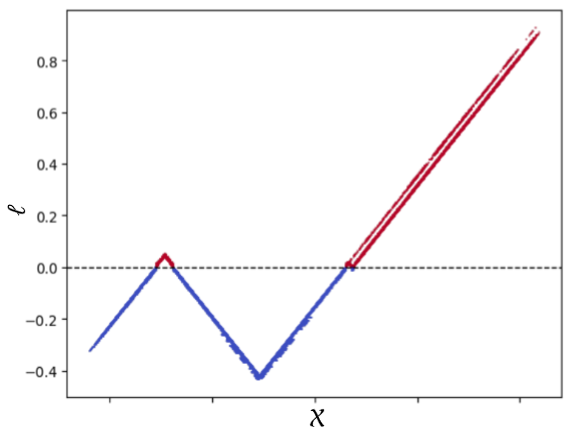}
	\caption{$\alpha = 1$}
\end{subfigure}
\hfill
\begin{subfigure}[t]{0.492\columnwidth}
    \includegraphics[width=\linewidth]{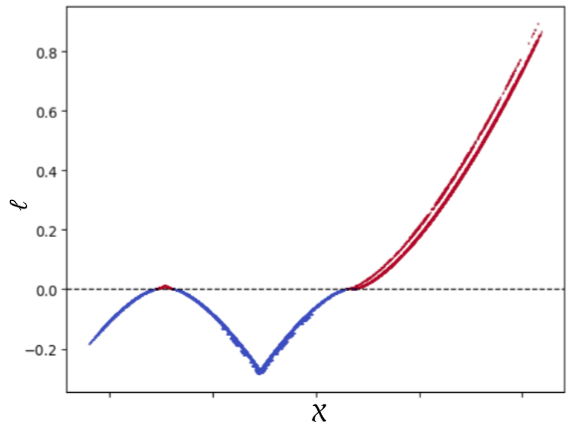}
    \caption{$\alpha = 1.5$}
\end{subfigure}
\caption{Comparison between a black-box and several distance-based labellings of a synthetic dataset $\mathscr{X}$.}
\label{fig:-123}
\end{figure}

The practical implications of this behaviour are illustrated in Fig.~\ref{fig:-123}. As $\alpha$ increases, the target surface becomes progressively smoother, especially around the desired target decision boundary. Small irregularities in the black-box have a limited influence on the target, which helps the copy to remove them. 

\subsection{Computing the distance to the decision boundary}
Once the proposal has been introduced, we discuss how these distances to the decision boundary, $\xi(s_{i}),\; s_i\in S \subset \mathcal{Z} $, can be computed, introducing and comparing two distance, sampling and model-agnostic approaches.

\begin{figure}[htbp]
\centering 
\includegraphics[width=\linewidth]{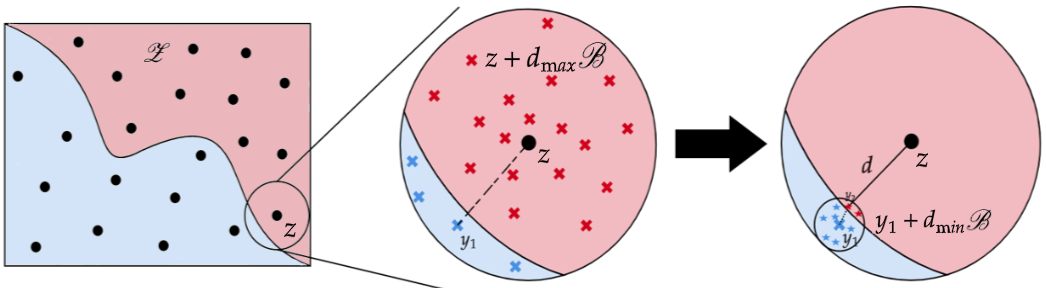}
\caption{Graphical representation of Algorithm 1}
\label{fig:2}
\end{figure}

To achieve high-quality distance estimates closely approximating the true boundary distances, we can apply Algorithm 1 (see Appendix \ref{appxA}), that iterates through each point in the synthetic dataset and computes the desired distance for each of them using a refining scheme. To construct these approximations, we previously compute and store an auxiliary set of points, $\mathscr{B}$, sampled from the unit ball. Given a query $z$, the center, $c$, is initialized to this value. The algorithm iteratively proceeds by finding the point $x \in c + d_{\mathrm{max}}\mathscr{B}$ such that minimizes $d(x, c)$ with the constraint $f_{\mathcal{O}}(x) \neq f_{\mathcal{O}}(c)$. Each iteration updates the center $c$ to the newly found $x$ and decreases the exploration radius $d_{\mathrm{max}}$ (see Fig. \ref{fig:2}).

Nevertheless, even though Alg. 1 can produce high quality distances, its complexity is considerable. Specifically, assuming that the maximum number of iterations $it_{\mathrm{max}}$ is always reached, that the cost of computing the distance $d$ is proportional to the dimension $dim$, and that $|\mathscr{B}| = m$, Alg. 1 requires $n\cdot (it_{\mathrm{max}}\cdot m + 1)$ black-box calls and $O(it_{\mathrm{max}}\cdot n\cdot m\cdot dim)$ operations. 

\begin{figure}[htbp]
\centering 
\includegraphics[width=\linewidth]{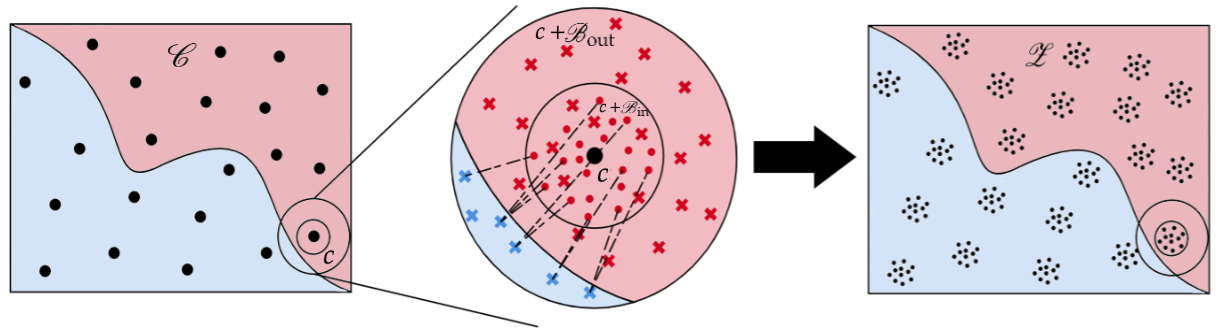}
\caption{Graphical representation of Algorithm 2}
\label{fig:-1}
\end{figure}

An alternative approach emphasizes maximizing the number of labelled points per query batch. Algorithm 2 (see Appendix \ref{appxA}) achieves this by sampling the synthetic dataset $\mathscr{Z}$ in small clusters that are labelled at the same time. In detail, we start by sampling a first dataset of centres, $\mathscr{C}$ of size $n_{c}$, that determines the positions of the clusters. Then, for each of these positions $c$, we center a small $\mathscr{B}_{\mathrm{in}}$ and a large $\mathscr{B}_{\mathrm{out}}$ cloud of points around it, labelling both with the black-box. Then, we use the outer cloud to compute the distances to the boundary of the points in the inner one, taking the closest point with a different label (see Fig. \ref{fig:-1}). 

Thanks to the use of these clusters, and denoting by $n_{\mathrm{in}}$ and $n_{\mathrm{out}}$ the number of points
in $\mathscr{B}_{\mathrm{in}}$ and $\mathscr{B}_{\mathrm{out}}$ respectively, this algorithm requires $n_{c}(n_{\mathrm{in}} + n_{\mathrm{out}})$ black-box evaluations and $O(n_{c}\cdot n_{\mathrm{in}}\cdot n_{\mathrm{out}}\cdot dim)$ additional operations. Since the roles of $m$ and $n_{\mathrm{out}}$ are comparable between Alg. 1 and 2, this represents a significant reduction in cost, i.e. the number of black-box calls decreases by a factor of $it_{\mathrm{max}}\cdot n_{\mathrm{in}}$ compared to Alg. 1.

\section{Experimental Setup}
\label{sec4}
\subsection{Datasets}
In this section, several datasets are introduced to validate the explained proposal in a series of experiments, using both synthetic data in dimension two as well as real datasets from the UCI machine learning repository \cite{b55}.

\begin{table}[!ht]
\centering
\caption{Information of the datasets used in this work.}
\begin{tabular}{lccc}
\toprule
Dataset & $\vert \mathscr{D}_{\mathrm{tr}}\vert$ & $\vert \mathscr{D}_{\mathrm{te}}\vert$ & Dim.\\
\midrule
Dat. 1: Colliding Gaussians & 800 & 200 & 2\\
Dat. 2: Two spirals & 8000 & 2000 & 2\\
Dat. 3: Irregular blobs & 8000 & 2000 & 2\\
Dat. 4: Breast Cancer & 455 & 114 & 30\\
Dat. 5: Rice & 3048 & 762 & 7\\
Dat. 6: C. b. (Mines vs Rocks) & 166 & 42 & 60\\
\bottomrule
\end{tabular}
\label{tab:1}
\end{table}

\begin{figure}[!ht]
\centering
\begin{subfigure}[t]{0.3\columnwidth}
    \includegraphics[width=\linewidth]{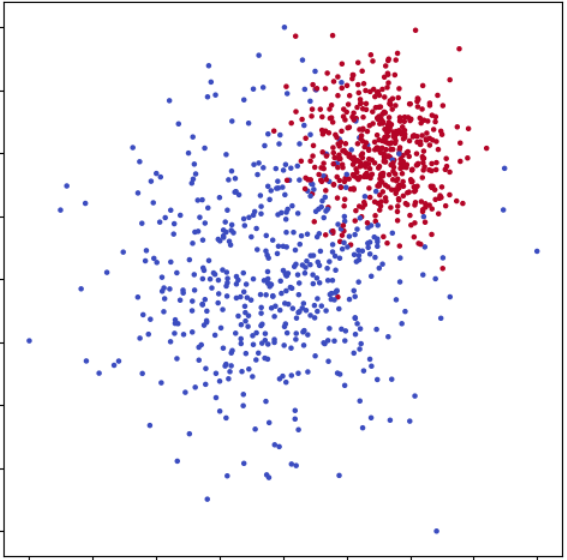}
    \caption{Dataset 1}
\end{subfigure}
\hfill
\begin{subfigure}[t]{0.3005\columnwidth}
    \includegraphics[width=\linewidth]{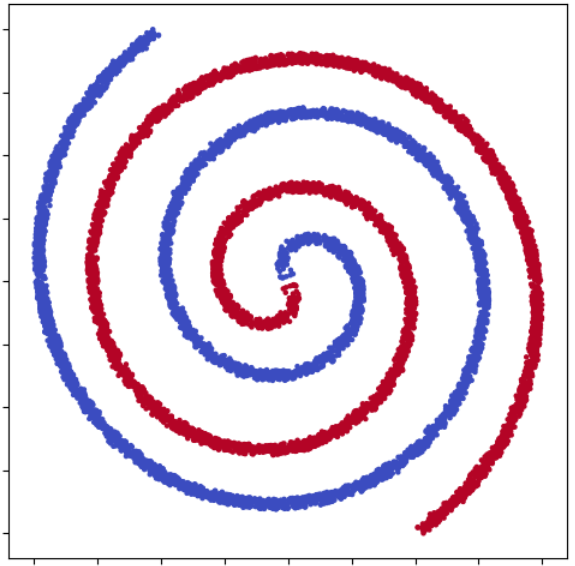}
    \caption{Dataset 2}
\end{subfigure}
\hfill
\begin{subfigure}[t]{0.3015\columnwidth}
    \includegraphics[width=\linewidth]{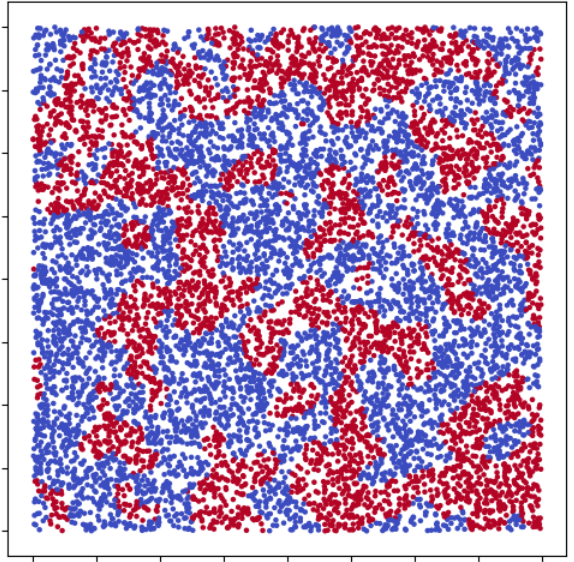}
    \caption{Dataset 3}
\end{subfigure}
\caption{Synthetic datasets considered in dimension 2.}
\label{fig:3}
\end{figure}

On the one hand, we consider three synthetic datasets (shown in Fig. \ref{fig:3}), aimed at visualizing the behaviours of the copies in different controlled but challenging scenarios. Dataset 1 is generated through two overlapping Gaussian distributions, which allows for a controllable class overlap. Dataset 2 displays two spirals shapes, which assess the ability of the copies to replicate highly non-linear manifolds. Dataset 3 is a convoluted dataset of irregular blobs, which aims at studying the performance in difficult scenarios with non-convex, clustered targets. 

On the other hand, to test this approach on real problems, we have also considered three datasets extracted from the UCI repository. In detail, Dataset 4, is Breast Cancer Wisconsin (Diagnostic) dataset; Dataset 5, Rice (Cammeo and Osmancik), and Dataset 6, Connectionist bench (Mines vs Rocks) datasets. These test how distance-based copies scale with feature dimensionality and boundary complexity beyond synthetic controls. Specific details of the datasets can be found in Table~\ref{tab:1}.

\subsection{Models and Model Parameters}
We have considered different types of models implemented with Scikit-learn \cite{b56} and Keras \cite{b57}. 

Specifically, as {\bf black-boxes}, we use:
\begin{itemize}
\item \textbf{Random Forest (RF):} 100 trees of maximum depth 10 and a minimum of 5 samples per leaf.

\item \textbf{Gradient Boosting Machines (GB):} trained with a 0.1 learning rate and using trees with a maximum of 31 leaves and a minimum of 20 samples in each of them.

\item\textbf{Neural Networks (NN):} following a 128-64-32-16-1 architecture and trained with a learning rate of 0.001 and batch size of 32 during 50 epochs.
\end{itemize}

The {\bf copies} are built using:
\begin{itemize}
    \item {\bf Neural Networks} with 0.001 learning rate, a batch size of 32 and $\mathrm{int}(100\cdot 20^{1 - \log_{1000}(\vert\mathscr{Z}\vert)})$ epochs\footnote{The number of epochs has been adjusted according to the size of the synthetic dataset, with larger datasets leading to fewer epochs. For example, the formula assigns 100 epochs to synthetic datasets of 1000 samples, while for datasets of 1,000,000 points the number of epochs is reduced to 5.}
    \begin{itemize}
        \item \textbf{Small (SNN):} 32-16-1 architecture.
        \item \textbf{Medium (MNN):} 128-64-32-16-1 architecture.
        \item \textbf{Large (LNN):} 512-256-256-128-64-32-16-1 architecture.
    \end{itemize}
    \item \textbf{Gradient Boosting Regressor:} with the same configuration as the black-box.
\end{itemize} 

\subsection{Performance Metrics}
As metrics, we have used the empirical fidelity error $R^{\mathcal{F}}_{emp}$ on an independent large uniform dataset $\mathcal{S}$ on the region of interest using 0-1 loss, and the accuracy $\mathcal{A}_{\mathcal{C}}$. The first measures how well the copy replicates the original black-box and the second shows the performance on the real test dataset. In addition, to test the quality of the distances predicted by these copies, we have used the MAE and the RMSE metrics.

\begin{figure*}[t]
\centering
\begin{subfigure}[t]{0.24\textwidth}
    \includegraphics[width=\linewidth]{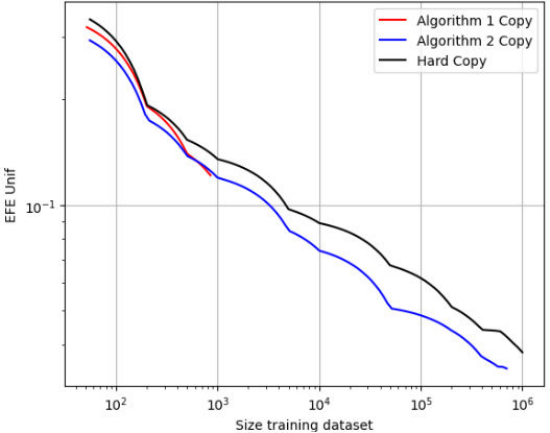}
    \caption{Dat. 3-$R^{\mathcal{F}}_{emp}$}
\end{subfigure}
\hfill
\begin{subfigure}[t]{0.244\textwidth}
    \includegraphics[width=\linewidth]{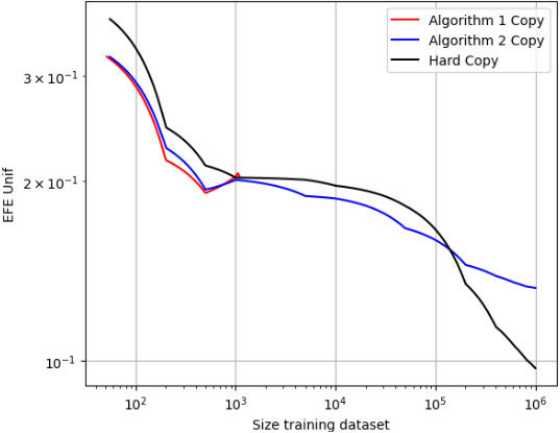}
    \caption{Dat. 6-$R^{\mathcal{F}}_{emp}$}
\end{subfigure}
\hfill
\begin{subfigure}[t]{0.248\textwidth}
    \includegraphics[width=\linewidth]{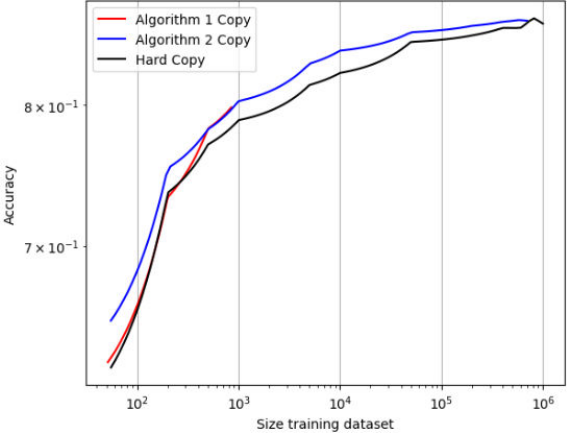}
    \caption{Dat. 3-$\mathcal{A}_{\mathcal{C}}$}
\end{subfigure}
\hfill
\begin{subfigure}[t]{0.248\textwidth}
    \includegraphics[width=\linewidth]{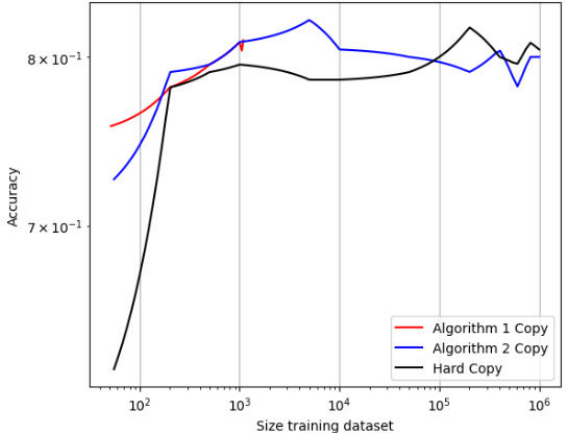}
    \caption{Dat. 6-$\mathcal{A}_{\mathcal{C}}$}
\end{subfigure}
\caption{Evolution of the metrics as a function of the number of training points. Plots made with a RF black-box and a LNN copy.}
\label{fig:4}
\end{figure*}

\begin{figure*}[t]
\centering
\begin{subfigure}[t]{0.489\textwidth}
    \includegraphics[width=\linewidth]{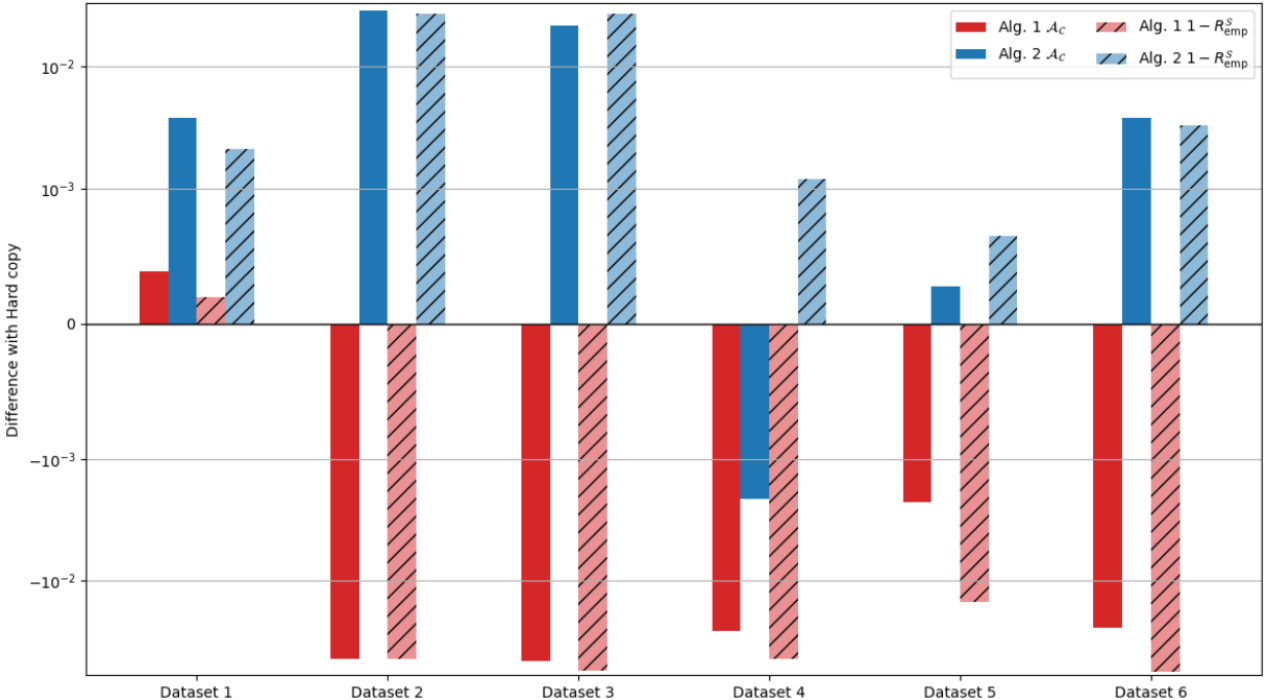}
    \caption{Average values in Fig. \ref{fig:4} - like line plots}
    \label{fig1.1}
\end{subfigure}
\hfill
\begin{subfigure}[t]{0.491\textwidth}
    \includegraphics[width=\linewidth]{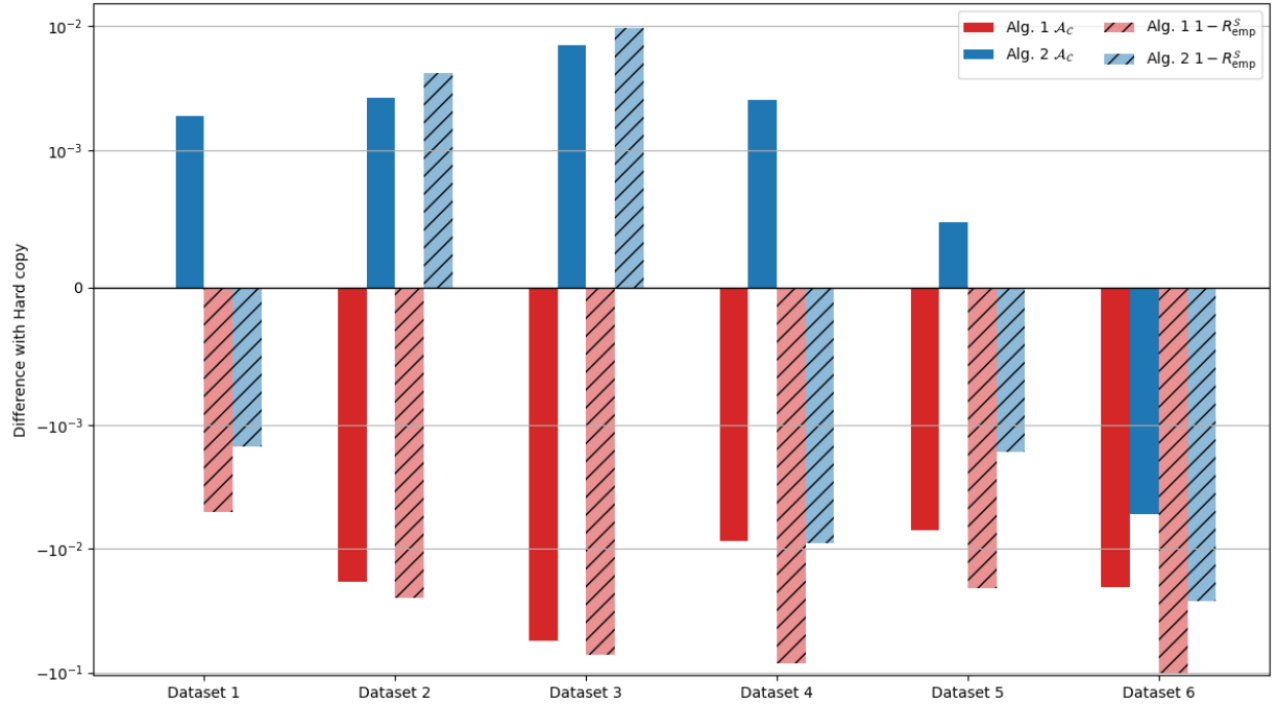}
    \caption{Final values in Fig. \ref{fig:4} - like line plots}
    \label{fig1.2}
\end{subfigure}
\caption{Algorithm comparison. Hard copies mark the horizontal axis. Results averaged over all black-box - NN copy combinations.}
\label{fig:-1314}
\end{figure*}

\subsection{Experimental Settings}
We design three experiments to analyse the performance of the copies, the regularization effect induced by distance-based copying, and the quality of the predicted distances.
\begin{itemize}
\item {\bf Experiment 1:} This experiment analyses how the performance of all pairwise combinations of black-box and copying models evolves as a function of the number of synthetic samples using both distance algorithms. We set a maximum budget of 1,000,000 synthetic points and a time limit of 240 seconds, fixing $\alpha = 1$. 
 
\item {\bf Experiment 2:} To show the regularization effect of the proposal, we have also trained these copies for multiple values of $\alpha$, and compared their results to the ones of the black-box and the hard copy. All copies have been trained on the same 1,000,000 synthetic points labelled with Alg. 2.

\item {\bf Experiment 3:} To assess the approximation error of distance-based copies, we analysed the quality of their predicted distances. We sampled a uniform dataset and a subset of the test dataset. Then, for each point in these datasets, we have computed their distance to the decision boundary using Alg. 1, treating these values as the ground-truth and comparing them to the predictions made by the copies of Experiment 1.
\end{itemize}

All results are obtained using a five-fold hold-out procedure, training the black-box models on the larger data split and evaluating them on the complementary set. In addition, we have used translated Sobol sequences \cite{b54} to sample the synthetic datasets.

\section{Results and Discussion}
\label{sec:5}
\subsection{Experiment 1. Global Comparison at $\alpha=1$.}
This experiment aims to show how black-box/copy combinations scale with synthetic dataset size under fixed computational budgets for the basic signed-distance, i.e. $\alpha = 1$. We track fidelity and accuracy evolution to reveal when distance-based supervision outperforms hard labels. Figure \ref{fig:4} shows representative curves (Datasets 3 and 6) of accuracy and fidelity errors with respect to the size of the sample set, while Figure \ref{fig:-1314} summarizes the average evolution and final comparisons across all datasets.

The curves in Figure \ref{fig:4} (extended version in Appendix \ref{appxC}) serve as example references for the general qualitative behaviour of the proposal. The figures show that signed distance copying generally display lower fidelity error in low and mid budget regimes, and higher accuracy on the test set. In this experiment with fixed alpha we observe the first consequences of the regularization effect. 

Figure~\ref{fig:-1314}(a) and (b) (based on tables in Appendix \ref{appxB}) show the average and final performance differences with respect to the hard-copy baseline, respectively, where bars above the horizontal axis indicate improvements over hard copying. The figure also compares both distance-based algorithms under a variable but limited synthetic data budget. The first observation is that Algorithm~2 outperforms Algorithm~1 in this operational setting. Recall that the scheme of Algorithm~1 asymptotically converges to the true distance, at the cost of larger point budgets.

\begin{table}[!ht]
\centering \scriptsize
\setlength{\tabcolsep}{4.8pt}
\caption{Relative differences in generalization and fidelity errors of distance copies compared to hard copies (negative values are preferred, indicating error reduction). Distance copies are built with Alg.~2 and trained using 1{,}000{,}000 synthetic points. Results are averaged over datasets and black-box models.}

\label{tab:nova}
\begin{tabular}{l
                *{2}{c}
                *{2}{c}}
\toprule
$f_{\mathcal{C}}$ model &
\multicolumn{2}{c}{$\alpha = 1$} &
\multicolumn{2}{c}{Best $\alpha > 0$} \\
\cmidrule(lr){2-5}
& $1-\mathcal{A}_{\mathcal{C}}$ (\%) & $R^{\mathcal{F}}_{emp}$ (\%) & $1-\mathcal{A}_{\mathcal{C}}$ (\%) & $R^{\mathcal{F}}_{emp}$ (\%) \\
\midrule
SNN & {\bf -7.52} & +3.76 & {\bf -12.44} & {\bf -4.24} \\
MNN & {\bf -6.99} & +7.71 & {\bf -15.11} & {\bf -1.57} \\
LNN & {\bf -16.61} & +5.96 & {\bf -19.85} & {\bf -4.14} \\
GB & +0.90 & +59.82 & {\bf -5.16} & +14.21 \\
\bottomrule
\end{tabular}
\end{table}

Focusing on the results obtained with Algorithm~2, we observe that, on average (Fig.~\ref{fig:-1314}(a)), it tends to achieve higher accuracies and improved fidelities compared to hard copies. However, in the final regime (Fig.~\ref{fig:-1314}(b)), a degradation in fidelity becomes apparent. These results suggest that signed-distance supervision provides an advantage in small and intermediate data-budget regimes. Due to the induced regularization effect, accuracy often improves (solid bars) at the expense of fidelity (dashed bars). This is further confirmed in the comprehensive analysis found in the first column of Table \ref{tab:nova} ($\alpha=1$). For clarity of interpretation, we report relative differences in both generalization and fidelity errors. Negative values indicate a percentage reduction in the corresponding error metric\footnote{The relative difference of the copy metric $m_c$ is computed with respect to the baseline metric, $m_b$, i.e. $(\%) = 100\cdot \frac{m_{c}-m_{b}}{m_{b}}.$} and are marked in bold font. Observe that all neural networks models reduce the generalization error at the cost of worsening the fidelity. Gradient boosting exhibits a large degradation in fidelity, most likely due to its limited ability to perform smooth regression.

\begin{figure*}[t]
	\centering
	\begin{subfigure}[t]{0.158\textwidth}
		\includegraphics[width=\linewidth]{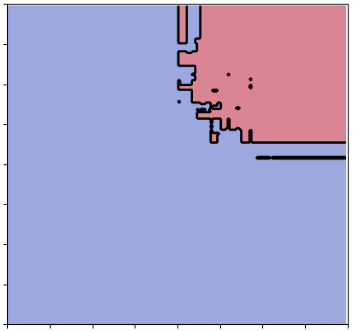}
		\caption{GB black-box}
	\end{subfigure}
	\hfill
	\begin{subfigure}[t]{0.158\textwidth}
		\includegraphics[width=\linewidth]{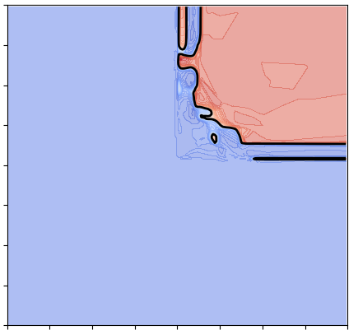}
		\caption{$\alpha = 0$}
	\end{subfigure}
	\hfill
	\begin{subfigure}[t]{0.158\textwidth}
		\includegraphics[width=\linewidth]{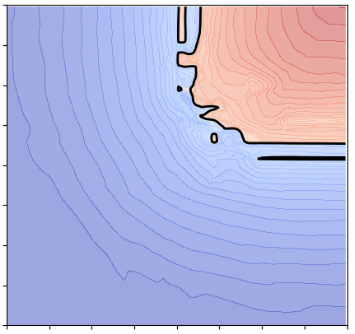}
		\caption{$\alpha = 0.25$}
	\end{subfigure}
	\hfill
	\begin{subfigure}[t]{0.158\textwidth}
		\includegraphics[width=\linewidth]{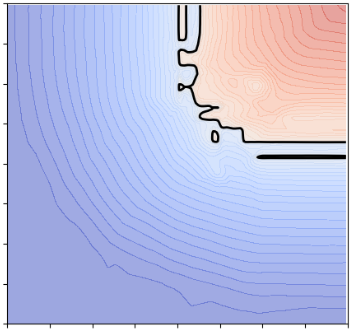}
		\caption{$\alpha = 0.5$}
	\end{subfigure}
    	\hfill
	\begin{subfigure}[t]{0.158\textwidth}
		\includegraphics[width=\linewidth]{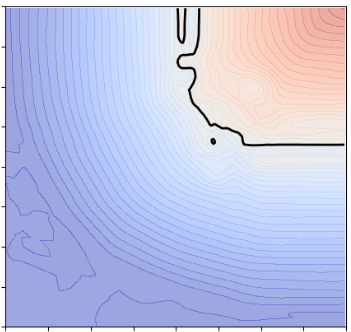}
		\caption{$\alpha = 0.75$}
	\end{subfigure}
	\hfill
	\begin{subfigure}[t]{0.158\textwidth}
		\includegraphics[width=\linewidth]{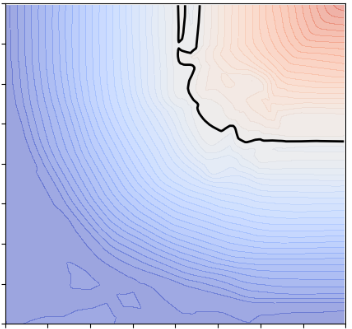}
		\caption{$\alpha = 1$}
	\end{subfigure}
	\caption{Regularization produced by distance copying on Dataset 1. MNN copies trained on the same 1,000,000 synthetic points.}
	\label{fig:5}
\end{figure*}

\begin{figure*}[t]
	\centering
	\begin{subfigure}[t]{0.49\textwidth}
		\includegraphics[width=\linewidth]{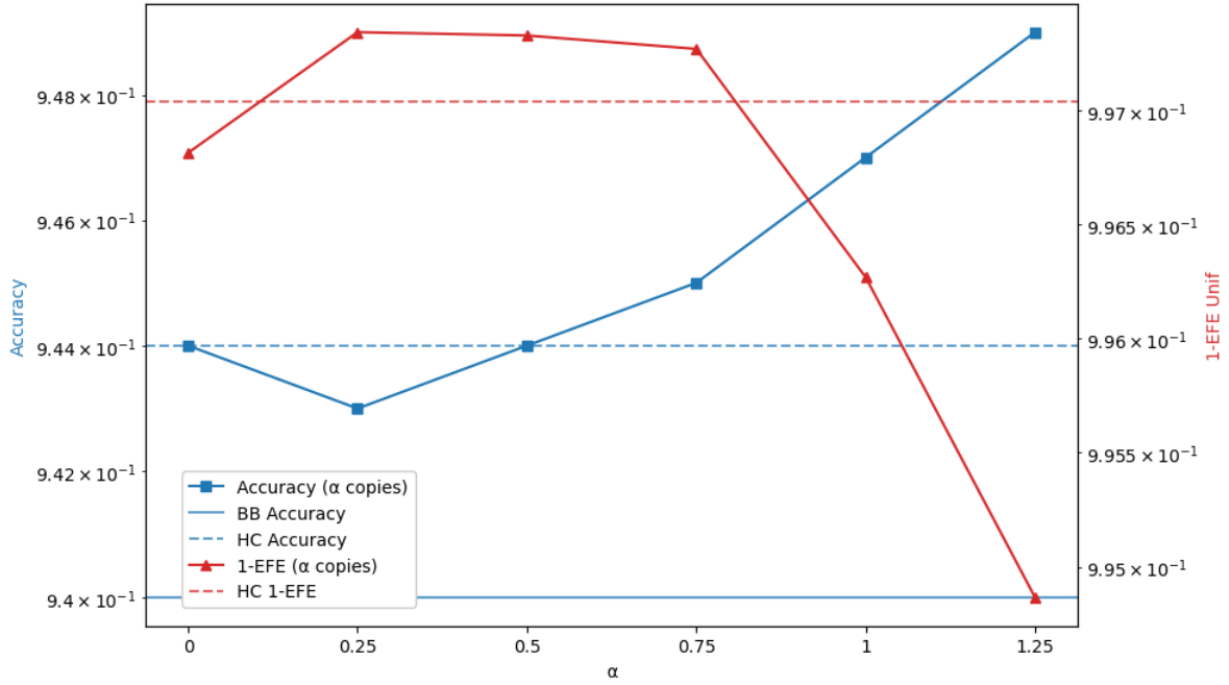}
		\caption{Black-box RF/ Copy MNN}
	\end{subfigure}
	\hfill
	\begin{subfigure}[t]{0.486\textwidth}
		\includegraphics[width=\linewidth]{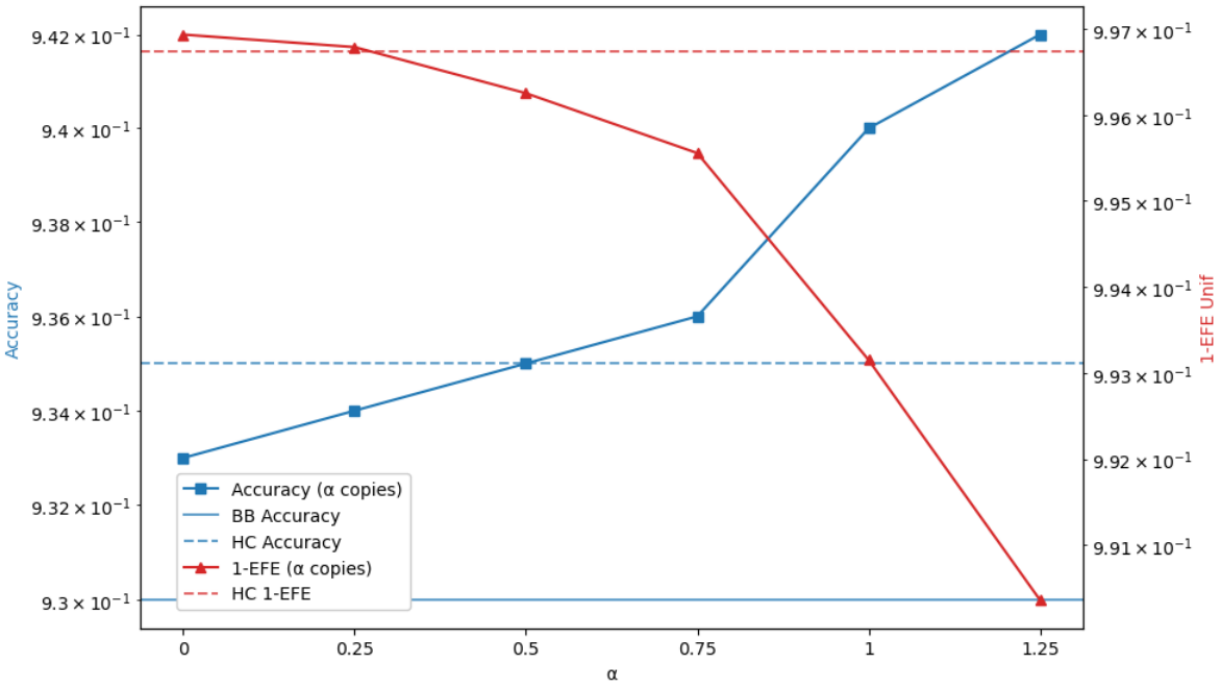}
		\caption{Black-box GB/ Copy GB}
	\end{subfigure}
	\caption{Plots showing the evolution of the two metrics as a function of the regularization parameter $\alpha$ in Dataset 1.}
	\label{fig:6}
\end{figure*}

\subsection{Experiment 2. Regularization Regimes}
The second experiment explicitly deals with the regularization effect. It is worth noting that traditional self-distillation\footnote{ Saving the differences in access to data and model internals, such as logits and soft targets.} is translated in our framework as the hard-label copying with $\alpha = 0$\footnote{Our approach with $\alpha = 0$ already displays all the benefits of self-distillation. This highlights that the proposed approach extends beyond this baseline.}. In order to visually illustrate the standard trending behaviour of the proposed approach, Figure \ref{fig:5}(a) shows the original black-box boundary for Dataset 1 using a gradient boosting model. Figure \ref{fig:5}(b) shows the corresponding hard-copy result ($\alpha = 0$). Figures \ref{fig:5} (c-f) show the decision boundaries for different values of $\alpha$. We observe that increasing $\alpha$ leads to progressively smoother decision boundaries. Complementing this qualitative illustration, Figure \ref{fig:6} (extended versions in Appendix \ref{appExp2Plots}) quantitatively shows two examples of the ideal behaviour of the method as the value of $\alpha$ increases (irregular or overfitted model). The horizontal dashed red and blue lines correspond to the fidelity and accuracy of the hard copy, respectively. As expected, the fidelity of the method (solid red line) tends to decrease as the value of $\alpha$ and the regularization becomes stronger. Additionally, in the cases in which the black-box displays some overfitting, the proposed regularization further improves generalization capability (solid blue line), even in the absence of the original data. Observe that in these examples, there are values of $\alpha$ in which fidelity is above the original hard copy performance, and that accuracy improves consistently over both the original black-box model and the corresponding hard copy. 

In practice, the parameter $\alpha$ enables explicit control over the trade-off between fidelity and accuracy, allowing practitioners to prioritize boundary replication (small $\alpha$), task accuracy (larger $\alpha$), or an intermediate compromise. At this point, it is important to further understand the implications of Theorem~\ref{teo1}, which identifies two regimes. For $\alpha \in (0,1]$ the supervisory signal is $\alpha$-H\"older continuous, and increasing $\alpha$ directly increases smoothness. However, when $\alpha\geq 1$, it changes to a Lipschitz continuous regime in which regularity saturates. Increasing $\alpha$ beyond this point emphasizes robustness by aggressively suppressing small-scale boundary noise and favouring the global geometry of the decision surface. This shift can be advantageous in highly noisy or overfitted teachers. This can be observed in Figure \ref{fig:6} at $\alpha=1.25$ where generalization further increases.  

\begin{table}[ht]
\centering
\scriptsize
\caption{Wins/ties/losses vs. hard-copy baseline (parentheses = conservative excluding ties). One-sided binomial tests (H$_0$: p=0.5). $\dagger$ p$<$0.001, $\bullet$ p$<$0.01,  $\circ$ p$<$0.05.}
\label{tab:wins-pvals}
\begin{tabular}{l cc | cc}
\hline
 & \multicolumn{2}{c|}{\textbf{Fidelity}} & \multicolumn{2}{c}{\textbf{Accuracy}} \\
\cline{2-5}
 & W/T/L & p-value & W/T/L  & p-value \\
\hline
SNN  & {\bf 13(12)/0/5}  & {\bf 0.041-0.118}$\circ$   & {\bf 14(12)/3/1}   & {\bf 0.002-0.02}$\bullet$  \\
MNN  & {\bf 16(15)/0/2}   & {\bf 0.0006-0.003}$\dagger$  & {\bf 17(15)/1/0}  & {\bf 7e-5-0.002}$\dagger$ \\
LNN  & {\bf 16(15)/0/2}  & {\bf 0.0006-0.003}$\dagger$  & {\bf 15(15)/2/0}   & {\bf 0.0006-0.0006}$\dagger$ \\
GB   & 7(4)/2/9     & 0.759-0.996    & 12(10)/3/3  & 0.048-0.154  \\
\hline
\end{tabular}
\end{table}

Table~\ref{tab:nova} (Best $\alpha>0$) summarizes the gains achieved at the optimal $\alpha$ (see Appendix~\ref{appxMax} for min/max/median). Negative values denote relative reductions in error. All copy models consistently reduce classification error (typically $10$–$20\%$), while neural students (SNN, MNN, LNN) also improve empirical fidelity (up to $4\%$), indicating that properly chosen $\alpha$ can simultaneously enhance generalization and boundary matching. In contrast, GB students improve accuracy at the expense of fidelity, exhibiting a stronger trade-off.

Table~\ref{tab:wins-pvals} reports win/tie/loss counts against hard copying over all evolution curves\footnote{Evolution curves are constructed by training copies on synthetic datasets generated using Algorithm~2, where the clustering effect restricts the effective number of samples.} (Appendix~\ref{appExp2Plots}). Distance-based copying yields statistically significant\footnote{Parenthesized values correspond to conservative counts excluding ties. Reported p-value ranges reflect tests using full vs.\ conservative counts; significance markers refer to values within this range.} improvements for neural students with MNN and LNN achieving strong gains in both accuracy and fidelity ($p<10^{-3}$), even under conservative counts, while SNN shows consistent accuracy gains with weaker but significant fidelity improvements. By contrast, GB students exhibit mixed outcomes, with marginal accuracy benefits and no significant fidelity gains, suggesting that tree-based models benefit less from distance supervision.

\begin{figure}[!h]
\centering
\begin{subfigure}[t]{0.481\columnwidth}
    \includegraphics[width=\linewidth]{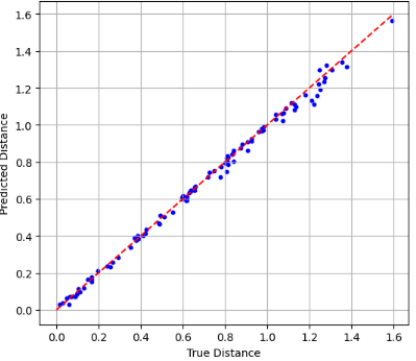}
    \caption{Dat. 1 - Alg. 1}
\end{subfigure}
\hfill
\begin{subfigure}[t]{0.495\columnwidth}
    \includegraphics[width=\linewidth]{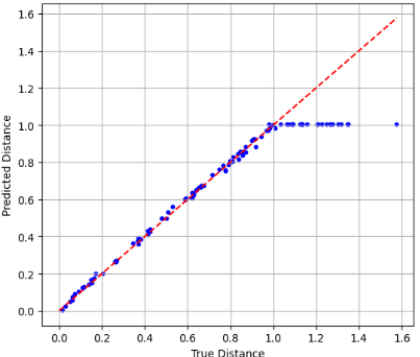}
    \caption{Dat. 1 - Alg. 2}
\end{subfigure}

\begin{subfigure}[t]{0.490\columnwidth}
    \includegraphics[width=\linewidth]{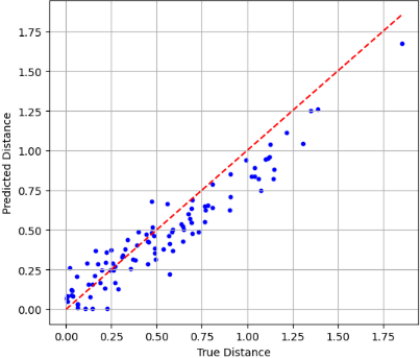}
    \caption{Dat. 5 - Alg. 1}
\end{subfigure}
\hfill
\begin{subfigure}[t]{0.486\columnwidth}
    \includegraphics[width=\linewidth]{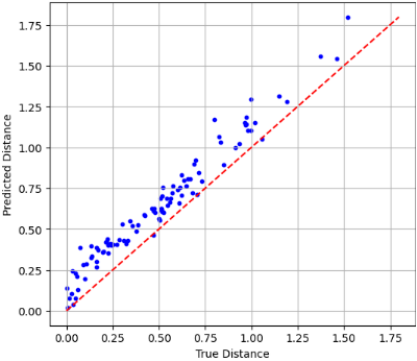}
    \caption{Dat. 5 - Alg. 2}
\end{subfigure}
\caption{Scatter plots of distance predictions against ground truth in $\mathcal{S}^{'}$. Results are shown for a RF black-box and MNN copies.}
\label{fig:9}
\end{figure}

\subsection{Experiment 3. Quality of the Distances}

Because signed distances act both as regression targets during copying and as uncertainty measures at inference time, their accurate approximation is critical. We evaluate the distances predicted by copies trained with both algorithms against the ground truth using MAE and RMSE in Table \ref{tab:4}, and complement these metrics with qualitative scatter plots in Figure \ref{fig:9} (extended versions available in Appendix \ref{appDistTabs} and \ref{appDistScatter}). Looking at the results, note that the predicted distances closely approximate the ground truth in the two-dimensional datasets. Nevertheless, the quality of these approximations decreases in higher-dimensional scenarios. This tendency is expected as the approximation of distances becomes more difficult, due to the inherent complexity of these spaces.

\begin{table}[!h]
\centering
\scriptsize
\caption{Error of the predicted distances averaged over every black-box, copy, and dataset combination.}
\label{tab:4}
\begin{tabular}{l c *{6}{c} c}
	\toprule
	Copy/Metric & Dat. 1 & Dat. 2 & Dat. 3 &
	Dat. 4 & Dat. 5 & Dat. 6  \\
	\midrule
	A1/ MAE & 
	0.022 & 0.013 & 0.020 & 0.283 & 0.083 & 0.499 \\
	A2/ MAE &  
	0.040 & 0.008 & 0.012 & 0.296 & 0.152 & 0.443 \\
	\midrule
	A1/ RMSE &  
	0.030 & 0.017 & 0.026 & 0.391 & 0.107 & 0.669 \\
	A2/ RMSE &  
	0.072 & 0.010 & 0.016 & 0.413 & 0.176 & 0.628 \\
	\bottomrule
\end{tabular}
\end{table}

Beyond dimensionality, performance also depends on the black-box model. When neural networks act as teachers, predictions generally align more closely with the target, even in high-dimensional settings. However, for Algorithm~2 this does not always translate into improved error metrics: more accurate distance predictions cause the model to more frequently saturate the maximum distance threshold in the synthetic data, which can paradoxically degrade aggregate scores. Nevertheless, this effect does not compromise the usefulness of the distances as uncertainty indicators. Consistent with this, Algorithm~2 typically outperforms Algorithm~1, likely due to its higher density of labelled samples.

\section{Conclusions}
\label{sec6}

We introduce a distance-based copying framework for hard-label black-box classifiers, reframing replication as regression over signed distances to the teacher’s decision boundary. Our $\alpha$-parameterized, target-centric regularization, supported by theoretical regularity guarantees, provides explicit control over the supervisory signal. While $\alpha\in[0,1]$ primarily control target smoothness, $\alpha>1$ emphasizes robustness to local boundary irregularities. Experiments show that appropriate choices of $\alpha$ can improve generalization accuracy when the original black-box exhibits overfitting, while often maintaining competitive empirical fidelity. Results indicate improved performance in low- and mid-data regimes, smoother learned decision boundaries, and distance outputs that correlate with predictive uncertainty. Although highly irregular teachers may induce a trade-off between accuracy and exact boundary fidelity, our fully data-free framework enables post-hoc refinement of models, offering a principled mechanism for controlled maintenance and evolution.

{\bf Limitations:} Our method depends on approximating signed distances to the teacher’s decision boundary, whose accuracy may deteriorate in high-dimensional spaces or for highly irregular models, potentially limiting the effectiveness of distance-based supervision. The framework is formulated for binary classifiers; while multi-class problems can be handled via standard reductions, efficient multi-class querying and distance estimation remain open directions. The regularization parameter $\alpha$ is selected empirically; automated $\alpha$-continuation or homotopy-based strategies are promising but left for future work. Moreover, although we provide regularity guarantees on the targets, gains in generalization and sample efficiency are demonstrated empirically rather than via formal learning-theoretic bounds. Finally, experiments are limited to synthetic and tabular datasets; extending validation to large-scale, high-dimensional domains is an important direction for future work.

\bibliography{paper}
\bibliographystyle{icml2026}

\newpage
\appendix
\onecolumn

\section{Pseudocode for black-box distance computation algorithms}
\label{appxA}

In this appendix, we present the pseudocode of the algorithms described in Sec. \ref{sec3}.

As a comment regarding both algorithms, it is possible that the way in which the points of the sets $\mathscr{B}$, $\mathscr{B}_{\mathrm{in}}$ and $\mathscr{B}_{\mathrm{out}}$ are sampled may seem unusual, because their distribution in the unit ball is neither uniform nor another distribution that evenly spaces the points in some sense. The reason behind this is that, with this choice, more points are generated around the centre of the ball rather than on its outskirts, something that reduces the error for small distances at the cost of increasing it for bigger ones. That controls the relative error of the computations and lowers the number of issues around the decision boundary, where misclassification is more likely. 

\subsection*{Algorithm 1}

\begin{algorithm}[h]
	\caption{Individual distance computation}
	\label{alg:1}
	\begin{algorithmic}[1]
		\REQUIRE The black-box model $f_{\mathcal{O}}$, the region of interest $R$, a distance $d$ and the parameters $n, m, d_{\mathrm{max}}, d_{\mathrm{min}}, it_{\mathrm{max}}$.
		\STATE Sample $n$ points $\mathscr{Z}$ in the region $R$ and label them with $f_{\mathcal{O}}$. 
		\STATE Sample $m$ points $\mathscr{B}$ in the unit ball with uniform directions and radius.
		\FOR{$z$ belonging to $\mathscr{Z}$}
		\STATE Store $z$ in a new variable $c$.
		\STATE Center $\mathscr{B}$ in $c$ and rescale it by $d_{\mathrm{max}}$, labelling the points in the resulting set with $f_{\mathcal{O}}$.
		\STATE Find the closest point of $c + d_{\mathrm{max}}\mathscr{B}$ to $c$ with a different label and store it in $c$. If there are none, store $d_{\mathrm{max}}f_{\mathcal{O}}(z)$ and continue the loop stated in 3.
		\FOR{$i$ belonging to $\{2,\dots, it_{\mathrm{max}}\}$}
		\STATE Repeat steps 5 and 6 using $d_{\mathrm{min}}$ instead of $d_{\mathrm{max}}$. 
		\ENDFOR
		\STATE Store the signed distance $d(c,z)f_{\mathcal{O}}(z)$ between the original $z$ and the current point in $c$.
		\ENDFOR
		\STATE \textbf{return} The set $\mathscr{Z} = \{z_{i}\}_{i}$ and the signed distances $\{\ell_{i}\}_{i}$.
	\end{algorithmic}
\end{algorithm}

\subsection*{Algorithm 2}

\begin{algorithm}[h]
	\caption{Grouped distance computation}
	\label{alg:2}
	\begin{algorithmic}[1]
		\REQUIRE The black-box model $f_{\mathcal{O}}$, the region of interest $R$, a distance $d$ and the parameters $n_{c}, n_{\mathrm{in}}, n_{\mathrm{out}}, d_{\mathrm{in}}, d_{\mathrm{out}}$.
		\STATE Sample $n_{c}$ points $\mathscr{C}$ in the region of interest $R$.
		\STATE Sample $n_{\mathrm{in}}$ and $n_{\mathrm{out}}$ points in the unit ball with uniform directions and radius. Rescale them by $d_{\mathrm{in}}$ and $d_{\mathrm{out}}$ to obtain the sets $\mathscr{B}_{\mathrm{in}}$ and $\mathscr{B}_{\mathrm{out}}$ respectively.
		\FOR{$c$ belonging to $\mathscr{C}$}
		\STATE Center $\mathscr{B}_{\mathrm{in}}$ and $\mathscr{B}_{\mathrm{out}}$ in $c$ and label them with $f_{\mathcal{O}}$.
		\FOR{$p$ belonging to $c + \mathscr{B}_{\mathrm{in}}$}
		\STATE Find $y$ the closest point of $c + \mathscr{B}_{\mathrm{out}}$ to $p$ with a different label and store $d(p,y)f_{\mathcal{O}}(p)$. If there are none, store $d_{\mathrm{out}}f_{\mathcal{O}}(p)$.
		\ENDFOR
		\ENDFOR
		\STATE \textbf{return} The set $\mathscr{Z} = \bigcup_{c \in \mathscr{C}}(c + \mathscr{B}_{\mathrm{in}}) = \{z_{i}\}_{i}$ and their corresponding signed distances to the boundary $\{\ell_{i}\}_{i}$.
	\end{algorithmic}
\end{algorithm}

\clearpage
\section{Proof of Theorem \ref{teo1}}
\label{appxZ}
\begin{theorem}[Regularity of $\alpha$ signed distances]
Let $f: \mathcal{X} \longrightarrow \{-1,1\}$ be a function and let $\alpha > 0$, we consider $d$ a distance in $\mathcal{X}$ bounded by $D > 0$ and define $l_{\alpha}(x) = f(x)d(x, A_{x})^{\alpha}$ for all $x \in \mathcal{X}$, where $A_{x} = \{y \in \mathcal{X}\;\vert\; f(y) \neq f(x)\}$. Then, for all $x,y \in \mathcal{X}$:
\begin{itemize}
    \item If $\alpha \leq 1$, we have $\vert l_{\alpha}(x) - l_{\alpha}(y) \vert \leq 2d(x,y)^{\alpha}$.
    \item If $\alpha \geq 1$, we have $\vert l_{\alpha}(x) - l_{\alpha}(y)\vert \leq 2\alpha D^{\alpha-1}d(x,y)$.
\end{itemize}
\end{theorem}
\begin{proof}
To show this result, given any $x,y \in \mathcal{X}$, we can start by distinguishing two different cases. On the one hand, if $f(x) \neq f(y)$, then we have by definition that $d(x, A_{x}) \leq d(x,y)$ and $d(y, A_{y}) \leq d(x,y)$, something that implies that $\vert l_{\alpha}(x) - l_{\alpha}(y) \vert = d(x, A_{x})^{\alpha} + d(y, A_{y})^{\alpha} \leq 2d(x,y)^{\alpha}$. In particular, if $\alpha \geq 1$, we can also deduce that $\vert l_{\alpha}(x) - l_{\alpha}(y) \vert \leq 2d(x,y)^{\alpha} \leq 2\alpha D^{\alpha - 1}d(x,y)$.

From here, assuming now that $f(x) = f(y)$, we can observe that $A_{x} = A_{y}$. As a consequence, for all $z \in A_{x}$, we have that $d(x, A_{x}) \leq d(x,z) \leq d(x,y) + d(y, z)$, so we can conclude that $d(x, A_{x}) \leq d(x,y) + d(y, A_{y})$. Then, the same argument exchanging the roles of $x$ and $y$ shows that $\vert d(x,A_{x}) - d(y, A_{y})\vert \leq d(x,y)$.  

Finally, we distinguish the two additional cases depending on the value of $\alpha$:

\textbf{Case 1:} On the one hand, assuming that $\alpha \leq 1$, we can consider without loss of generality that $d(y,A_{y}) \leq d(x,A_{x})$ to deduce that:
\begin{eqnarray}
 & \vert l_{\alpha}(x) - l_{\alpha}(y) \vert = \vert d(x, A_{x})^{\alpha} - d(y, A_{y})^{\alpha} \vert = \int_{d(y,A_{y})}^{d(x,A_{x})} \alpha t^{\alpha-1} dt \leq  \int_{d(y,A_{y})}^{d(x,A_{x})} \alpha (t - d(y, A_{y}))^{\alpha-1} dt = \nonumber\\
 & ((t - d(y, A_{y}))^{\alpha}\vert_{d(y,A_{y})}^{d(x,A_{x})} = \vert d(x, A_{x}) - d(y,A_{y})\vert^{\alpha} \leq d(x,y)^{\alpha}
\end{eqnarray}

\textbf{Case 2:} On the other hand, if $\alpha \geq 1$, we can conclude that:
\begin{eqnarray}
 & \vert l_{\alpha}(x) - l_{\alpha}(y) \vert = \vert\int_{d(y,A_{y})}^{d(x,A_{x})} \alpha t^{\alpha-1} dt\vert \leq  \vert\int_{d(y,A_{y})}^{d(x,A_{x})} \alpha D^{\alpha-1} dt\vert \leq \nonumber\\
 & \alpha D^{\alpha-1}\vert d(x,A_{x}) - d(y, A_{y})\vert \leq  2\alpha D^{\alpha-1}d(x,y) 
\end{eqnarray} something that finishes the proof.
\end{proof}

\clearpage
\section{Detailed final results for Experiment 1}
\label{appxB}
We present tables that summarize the final metrics achieved by the copies of Experiment 1. Results are shown for every combination of black-box, copy and dataset we analysed with the format mean $\pm$ std.

In these tables the last two columns display the average ranking of the metrics (1, 2 or 3) that appear in the corresponding row, together with their average across all datasets.

\begin{center}
\resizebox{!}{0.2\textheight}{%
\tiny
\begin{tabular}{l c
                *{3}{c}
                *{3}{c}
                *{3}{c}
              c c}
\toprule
Copy &
$f_{\mathcal{O}}$/$f_{\mathcal{C}}$ &
\multicolumn{3}{c}{Dataset 1} &
\multicolumn{3}{c}{Dataset 2} &
\multicolumn{3}{c}{Dataset 3} &
$\mathcal{A}_{\mathcal{C}}$/$R^{\mathcal{F}}_{emp}$ Dat. 1-3 &
$\mathcal{A}_{\mathcal{C}}$/$R^{\mathcal{F}}_{emp}$\\
\cmidrule(lr){3-11}
& & $\mathcal{A}_{\mathcal{O}}$ & $\mathcal{A}_{\mathcal{C}}$ & $R^{\mathcal{F}}_{emp}$ & 
$\mathcal{A}_{\mathcal{O}}$&
$\mathcal{A}_{\mathcal{C}}$ & $R^{\mathcal{F}}_{emp}$ & 
$\mathcal{A}_{\mathcal{O}}$&
$\mathcal{A}_{\mathcal{C}}$ & $R^{\mathcal{F}}_{emp}$ \\
\midrule
Algo. 1 copy & RF/SNN & 
0.94 & 0.945$\pm$0.000 & 0.0053$\pm$0.0007 &
0.99 & 0.883$\pm$0.047 & 0.1162$\pm$0.0360 &
0.87 & 0.720$\pm$0.020 & 0.2038$\pm$0.0160 & 3/3 & 2.5/3\\
Algo. 2 copy & RF/SNN & 
0.94 & \textbf{0.948$\pm$0.002} & 0.0047$\pm$0.0007 &
0.99 & \textbf{0.991$\pm$0.004} & \textbf{0.0248$\pm$0.0012} &
0.87 & \textbf{0.802$\pm$0.003} & 0.1147$\pm$0.0072 & \textbf{1}/1.67 & \textbf{1.33}/1.83\\
Hard copy  & RF/SNN & 
0.94 & 0.947$\pm$0.002 & \textbf{0.0037$\pm$0.0007} &
0.99 & 0.986$\pm$0.005 & 0.0273$\pm$0.0059 &
0.87 & 0.798$\pm$0.010 & \textbf{0.1142$\pm$0.0028} & 2/\textbf{1.33} & 1.83/\textbf{1.17}\\
\midrule
Algo. 1 copy & GB/SNN & 
0.93 & 0.941$\pm$0.009 & 0.0200$\pm$0.0032 &
1.00 & 0.949$\pm$0.022 & 0.0682$\pm$0.0167 &
0.90 & 0.713$\pm$0.016 & 0.2517$\pm$0.0108 & 3/3 & 2.67/3\\
Algo. 2 copy & GB/SNN &
0.93 & \textbf{0.947$\pm$0.005} & 0.0178$\pm$0.0029 &
1.00 & \textbf{0.995$\pm$0.002} & \textbf{0.0215$\pm$0.0027} &
0.90 & \textbf{0.798$\pm$0.009} & \textbf{0.1569$\pm$0.0086} & \textbf{1}/\textbf{1.33} & \textbf{1.17}/1.67\\
Hard copy  & GB/SNN & 
0.93 & 0.944$\pm$0.002 & \textbf{0.0112$\pm$0.0022} &
1.00 & 0.991$\pm$0.002 & 0.0233$\pm$0.0017 &
0.90 & 0.788$\pm$0.011 & 0.1680$\pm$0.0102 & 2/1.67 & 1.5/\textbf{1.33}\\
\midrule
Algo. 1 copy & NN/SNN & 
0.94 & \textbf{0.940$\pm$0.005} & \textbf{0.0014$\pm$0.0007} &
1.00 & 0.996$\pm$0.003 & 0.0398$\pm$0.0097 &
0.83 & 0.745$\pm$0.010 & 0.1442$\pm$0.0117 & 2/2.33 & 2/2.67\\
Algo. 2 copy & NN/SNN &
0.94 & \textbf{0.940$\pm$0.003} & 0.0015$\pm$0.0003 &
1.00 & \textbf{1.000$\pm$0.000} & \textbf{0.0173$\pm$0.0030} &
0.83 & \textbf{0.807$\pm$0.007} & \textbf{0.0576$\pm$0.0065} & \textbf{1}/\textbf{1.33} & \textbf{1.33}/1.67\\
Hard copy  & NN/SNN & 
0.94 & 0.939$\pm$0.002 & \textbf{0.0014$\pm$0.0008} &
1.00 & \textbf{1.000$\pm$0.001} & 0.0190$\pm$0.0037 &
0.83 & 0.786$\pm$0.010 & 0.0850$\pm$0.0119 & 1.67/1.67 & \textbf{1.33}/\textbf{1.33}\\
\midrule
Algo. 1 copy & RF/MNN &
0.94 & \textbf{0.946$\pm$0.002} & 0.0062$\pm$0.0003 &
0.99 & 0.983$\pm$0.007 & 0.0355$\pm$0.0080 &
0.87 & 0.783$\pm$0.009 & 0.1352$\pm$0.0108 & 2.33/3 & 2.33/3\\
Algo. 2 copy & RF/MNN &
0.94 & 0.945$\pm$0.003 & 0.0037$\pm$0.0004 &
0.99 & \textbf{0.993$\pm$0.002} & \textbf{0.0132$\pm$0.0013} &
0.87 & \textbf{0.862$\pm$0.003} & \textbf{0.0444$\pm$0.0031} & \textbf{1.33}/\textbf{1.33} & \textbf{1.17}/1.67\\
Hard copy  & RF/MNN & 
0.94 & 0.943$\pm$0.002 & \textbf{0.0030$\pm$0.0010} &
0.99 & 0.989$\pm$0.001 & 0.0173$\pm$0.0017 &
0.87 & 0.856$\pm$0.007 & 0.0475$\pm$0.0043 & 2.33/1.67 & 2.67/\textbf{1.33}\\
\midrule
Algo. 1 copy & GB/MNN & 
0.93 & 0.942$\pm$0.007 & 0.0216$\pm$0.0047 &
1.00 & 0.995$\pm$0.003 & 0.0277$\pm$0.0044 &
0.90 & 0.798$\pm$0.009 & 0.1615$\pm$0.0026 & 2.67/3 & 2.83/3\\
Algo. 2 copy & GB/MNN &
0.93 & \textbf{0.946$\pm$0.004} & 0.0094$\pm$0.0007 &
1.00 & \textbf{0.998$\pm$0.001} & \textbf{0.0125$\pm$0.0015} &
0.90 & \textbf{0.890$\pm$0.004} & \textbf{0.0496$\pm$0.0040} & \textbf{1}/\textbf{1.33} & \textbf{1}/1.67\\
Hard copy  & GB/MNN & 
0.93 & 0.937$\pm$0.002 & \textbf{0.0057$\pm$0.0008} &
1.00 & 0.996$\pm$0.001 & 0.0155$\pm$0.0032 &
0.90 & 0.880$\pm$0.009 & 0.0657$\pm$0.0065 & 2.33/1.67 & 2.17/\textbf{1.33}\\
\midrule
Algo. 1 copy & NN/MNN & 
0.94 & 0.939$\pm$0.004 & 0.0018$\pm$0.0009 &
1.00 & \textbf{1.000$\pm$0.000} & 0.0193$\pm$0.0057 &
0.82 & 0.803$\pm$0.010 & 0.0559$\pm$0.0070 & 2.33/3 & 2/3\\
Algo. 2 copy & NN/MNN &
0.94 & 0.941$\pm$0.004 & \textbf{0.0011$\pm$0.0002} &
1.00 & \textbf{1.000$\pm$0.000} & \textbf{0.0087$\pm$0.0023} &
0.82 & \textbf{0.819$\pm$0.008} & \textbf{0.0185$\pm$0.0020} & \textbf{1.33}/\textbf{1} & 1.83/\textbf{1.5}\\
Hard copy  & NN/MNN & 
0.94 & \textbf{0.942$\pm$0.002} & 0.0013$\pm$0.0005 &
1.00 & \textbf{1.000$\pm$0.000} & 0.0157$\pm$0.0032 &
0.82 & 0.817$\pm$0.009 & 0.0264$\pm$0.0041 & \textbf{1.33}/2 & \textbf{1.33}/\textbf{1.5}\\
\midrule
Algo. 1 copy & RF/LNN & 
0.94 & 0.945$\pm$0.003 & 0.0062$\pm$0.0009 &
0.99 & 0.973$\pm$0.008 & 0.0386$\pm$0.0058 &
0.87 & 0.802$\pm$0.013 & 0.1174$\pm$0.0111 & 2.67/3 & 2.33/3\\
Algo. 2 copy & RF/LNN &
0.94 & \textbf{0.948$\pm$0.004} & 0.0036$\pm$0.0007 &
0.99 & \textbf{0.988$\pm$0.008} & \textbf{0.0120$\pm$0.0014} &
0.87 & \textbf{0.865$\pm$0.002} & \textbf{0.0348$\pm$0.0036} & \textbf{1}/\textbf{1.33} & \textbf{1.33}/1.67\\
Hard copy  & RF/LNN & 
0.94 & 0.944$\pm$0.002 & \textbf{0.0035$\pm$0.0006} &
0.99 & 0.982$\pm$0.010 & 0.0214$\pm$0.0042 &
0.87 & 0.864$\pm$0.003 & 0.0381$\pm$0.0028 & 2.33/1.67 & 2/\textbf{1.33}\\
\midrule
Algo. 1 copy & GB/LNN & 
0.93 & 0.937$\pm$0.007 & 0.0181$\pm$0.0012 &
1.00 & 0.994$\pm$0.002 & 0.0287$\pm$0.0044 &
0.90 & 0.828$\pm$0.007 & 0.1327$\pm$0.0040 & 2.67/3 & 2.67/3\\
Algo. 2 copy & GB/LNN &
0.93 & \textbf{0.939$\pm$0.005} & 0.0086$\pm$0.0012 &
1.00 & \textbf{0.997$\pm$0.001} & \textbf{0.0123$\pm$0.0020} &
0.90 & \textbf{0.895$\pm$0.002} & \textbf{0.0355$\pm$0.0024} & \textbf{1}/\textbf{1.33} & \textbf{1.33}/1.67\\
Hard copy  & GB/LNN & 
0.93 & 0.937$\pm$0.002 & \textbf{0.0068$\pm$0.0006} &
1.00 & 0.995$\pm$0.001 & 0.0154$\pm$0.0035 &
0.90 & 0.889$\pm$0.004 & 0.0470$\pm$0.0022 & 2/1.67 & 1.83/\textbf{1.33}\\
\midrule
Algo. 1 copy & NN/LNN & 
0.94 & 0.939$\pm$0.006 & 0.0033$\pm$0.0040 &
1.00 & \textbf{1.000$\pm$0.000} & 0.0202$\pm$0.0036 &
0.83 & 0.814$\pm$0.010 & 0.0555$\pm$0.0130 & 2/3 & 2.17/3\\
Algo. 2 copy & NN/LNN &
0.94 & 0.937$\pm$0.005 & \textbf{0.0015$\pm$0.0008} &
1.00 & \textbf{1.000$\pm$0.000} & \textbf{0.0088$\pm$0.0017} &
0.83 & \textbf{0.823$\pm$0.011} & \textbf{0.0159$\pm$0.0024} & 1.67/\textbf{1} & 1.67/\textbf{1.5}\\
Hard copy  & NN/LNN & 
0.94 & \textbf{0.941$\pm$0.002} & 0.0019$\pm$0.0013 &
1.00 & \textbf{1.000$\pm$0.000} & 0.0143$\pm$0.0010 &
0.83 & 0.822$\pm$0.011 & 0.0238$\pm$0.0035 & \textbf{1.33}/2 & \textbf{1.33}/\textbf{1.5}\\
\midrule
Algo. 1 copy & RF/GB & 
0.94 & 0.944$\pm$0.006 & 0.0055$\pm$0.0005 &
0.99 & 0.981$\pm$0.007 & 0.0337$\pm$0.0037 &
0.87 & 0.801$\pm$0.004 & 0.1127$\pm$0.0040 & 2.33/3 & 1.83/3\\
Algo. 2 copy & RF/GB &
0.94 & \textbf{0.945$\pm$0.003} & 0.0026$\pm$0.0003 &
0.99 & \textbf{0.989$\pm$0.004} & 0.0130$\pm$0.0012 &
0.87 & 0.844$\pm$0.003 & 0.0540$\pm$0.0014 & \textbf{1.33}/2 & \textbf{1.5}/2\\
Hard copy  & RF/GB & 
0.94 & 0.941$\pm$0.002 & \textbf{0.0014$\pm$0.0002} &
0.99 & \textbf{0.989$\pm$0.004} & \textbf{0.0072$\pm$0.0011} &
0.87 & \textbf{0.856$\pm$0.004} & \textbf{0.0277$\pm$0.0005} & 1.67/\textbf{1} & 1.83/\textbf{1}\\
\midrule
Algo. 1 copy & GB/GB & 
0.93 & \textbf{0.940$\pm$0.003} & 0.0139$\pm$0.0004 &
1.00 & 0.993$\pm$0.001 & 0.0325$\pm$0.0020 &
0.90 & 0.829$\pm$0.004 & 0.1165$\pm$0.0047 & 2.33/3 & \textbf{1.67}/3\\
Algo. 2 copy & GB/GB &
0.93 & 0.939$\pm$0.002 & 0.0069$\pm$0.0005 &
1.00 & \textbf{0.998$\pm$0.001} & 0.0123$\pm$0.0005 &
0.90 & 0.861$\pm$0.002 & 0.0687$\pm$0.0027 & \textbf{1.67}/2 & 1.83/2\\
Hard copy  & GB/GB & 
0.93 & 0.933$\pm$0.002 & \textbf{0.0031$\pm$0.0003} &
1.00 & 0.997$\pm$0.001 & \textbf{0.0060$\pm$0.0005} &
0.90 & \textbf{0.882$\pm$0.003} & \textbf{0.0325$\pm$0.0020} & 2/\textbf{1} & \textbf{1.67}/\textbf{1}\\
\midrule
Algo. 1 copy & NN/GB & 
0.94 & 0.937$\pm$0.004 & 0.0027$\pm$0.0005 &
1.00 & \textbf{1.000$\pm$0.000} & 0.0340$\pm$0.0021 &
0.83 & 0.811$\pm$0.011 & 0.0649$\pm$0.0046 & 2/3 & 2.33/3\\
Algo. 2 copy & NN/GB &
0.94 & 0.937$\pm$0.008 & 0.0017$\pm$0.0001 &
1.00 & \textbf{1.000$\pm$0.000} & 0.0219$\pm$0.0010 &
0.83 & 0.812$\pm$0.012 & 0.0516$\pm$0.0068 & 1.67/2 & \textbf{1.5}/1.83\\
Hard copy  & NN/GB & 
0.94 & \textbf{0.939$\pm$0.002} & \textbf{0.0008$\pm$0.0001} &
1.00 & \textbf{1.000$\pm$0.000} & \textbf{0.0125$\pm$0.0006} &
0.83 & \textbf{0.826$\pm$0.011} & \textbf{0.0315$\pm$0.0035} & \textbf{1}/\textbf{1} & \textbf{1.5}/\textbf{1.17}\\
\bottomrule
\end{tabular}
}
\end{center}

\begin{center}
\resizebox{!}{0.2\textheight}{%
\centering
\tiny
\begin{tabular}{l c
                *{3}{c}
                *{3}{c}
                *{3}{c}
              c c}
\toprule
Copy &
$f_{\mathcal{O}}$/$f_{\mathcal{C}}$ &
\multicolumn{3}{c}{Dataset 4} &
\multicolumn{3}{c}{Dataset 5} &
\multicolumn{3}{c}{Dataset 6} &
$\mathcal{A}_{\mathcal{C}}$/$R^{\mathcal{F}}_{emp}$ Dat. 4-6 &
$\mathcal{A}_{\mathcal{C}}$/$R^{\mathcal{F}}_{emp}$\\
\cmidrule(lr){3-11}
& & $\mathcal{A}_{\mathcal{O}}$ & $\mathcal{A}_{\mathcal{C}}$ & $R^{\mathcal{F}}_{emp}$ & 
$\mathcal{A}_{\mathcal{O}}$&
$\mathcal{A}_{\mathcal{C}}$ & $R^{\mathcal{F}}_{emp}$ & 
$\mathcal{A}_{\mathcal{O}}$&
$\mathcal{A}_{\mathcal{C}}$ & $R^{\mathcal{F}}_{emp}$ \\
\midrule
Algo. 1 copy & RF/SNN & 
0.97 & 0.949$\pm$0.030 & 0.1678$\pm$0.0082 &
0.92 & 0.920$\pm$0.009 & 0.0523$\pm$0.0052 &
0.83 & \textbf{0.829$\pm$0.053} & 0.2218$\pm$0.0093 & 2/3 & 2.5/3\\
Algo. 2 copy & RF/SNN &
0.97 & 0.965$\pm$0.006 & 0.0762$\pm$0.0017 &
0.92 & \textbf{0.929$\pm$0.004} & 0.0306$\pm$0.0033 &
0.83 & 0.814$\pm$0.032 & 0.1244$\pm$0.0066 & \textbf{1.67}/2 & \textbf{1.33}/1.83\\
Hard copy  & RF/SNN & 
0.97 & \textbf{0.967$\pm$0.007} & \textbf{0.0613$\pm$0.0029} &
0.92 & \textbf{0.929$\pm$0.007} & \textbf{0.0272$\pm$0.0032} &
0.83 & 0.810$\pm$0.040 & \textbf{0.1154$\pm$0.0101} & \textbf{1.67}/\textbf{1} & 1.83/\textbf{1.17}\\
\midrule
Algo. 1 copy & GB/SNN & 
0.98 & 0.960$\pm$0.023 & 0.1606$\pm$0.0043 &
0.92 & 0.905$\pm$0.010 & 0.0541$\pm$0.0016 &
0.87 & 0.781$\pm$0.059 & 0.2484$\pm$0.0101 & 2.33/3 & 2.67/3\\
Algo. 2 copy & GB/SNN &
0.98 & 0.977$\pm$0.017 & 0.0769$\pm$0.0043 &
0.92 & \textbf{0.927$\pm$0.004} & 0.0349$\pm$0.0034 &
0.87 & \textbf{0.838$\pm$0.044} & 0.1264$\pm$0.0098 & 1.33/2 & \textbf{1.17}/1.67\\
Hard copy  & GB/SNN & 
0.98 & \textbf{0.979$\pm$0.012} & \textbf{0.0674$\pm$0.0048} &
0.92 & \textbf{0.927$\pm$0.004} & \textbf{0.0341$\pm$0.0021} &
0.87 & \textbf{0.838$\pm$0.041} & \textbf{0.1108$\pm$0.0107} & \textbf{1}/\textbf{1} & 1.5/\textbf{1.33}\\
\midrule
Algo. 1 copy & NN/SNN & 
0.98 & 0.979$\pm$0.015 & 0.0298$\pm$0.0024 &
0.92 & \textbf{0.925$\pm$0.004} & 0.0135$\pm$0.0010 &
0.89 & 0.838$\pm$0.032 & 0.0892$\pm$0.0041 & 2/3 & 2/2.67\\
Algo. 2 copy & NN/SNN &
0.98 & \textbf{0.981$\pm$0.010} & 0.0190$\pm$0.0012 &
0.92 & 0.924$\pm$0.005 & 0.0107$\pm$0.0013 &
0.89 & 0.848$\pm$0.039 & 0.0441$\pm$0.0029 & 1.67/2 & \textbf{1.33}/1.67\\
Hard copy  & NN/SNN & 
0.98 & \textbf{0.981$\pm$0.010} & \textbf{0.0176$\pm$0.0015} &
0.92 & \textbf{0.925$\pm$0.007} & \textbf{0.0083$\pm$0.0011} &
0.89 & \textbf{0.862$\pm$0.038} & \textbf{0.0417$\pm$0.0031} & \textbf{1}/\textbf{1} & \textbf{1.33}/\textbf{1.33}\\

\midrule
Algo. 1 copy & RF/MNN &
0.97 & 0.965$\pm$0.019 & 0.1771$\pm$0.0104 &
0.93 & 0.926$\pm$0.006 & 0.0501$\pm$0.0048 &
0.81 & 0.786$\pm$0.050 & 0.2105$\pm$0.0164 & 2.33/3 & 2.33/3\\
Algo. 2 copy & RF/MNN &
0.97 & \textbf{0.968$\pm$0.007} & 0.0613$\pm$0.0023 &
0.93 & \textbf{0.928$\pm$0.004} & 0.0230$\pm$0.0014 &
0.81 & \textbf{0.810$\pm$0.040} & 0.1440$\pm$0.0060 & \textbf{1}/2 & \textbf{1.17}/1.67\\
Hard copy  & RF/MNN & 
0.97 & \textbf{0.968$\pm$0.009} & \textbf{0.0493$\pm$0.0016} &
0.93 & \textbf{0.928$\pm$0.006} & \textbf{0.0210$\pm$0.0015} &
0.81 & 0.800$\pm$0.072 & \textbf{0.1021$\pm$0.0043} & 1.33/\textbf{1} & 2.67/\textbf{1.33}\\
\midrule
Algo. 1 copy & GB/MNN & 
0.98 & 0.958$\pm$0.039 & 0.1798$\pm$0.0092 &
0.92 & 0.919$\pm$0.006 & 0.0582$\pm$0.0025 &
0.87 & 0.786$\pm$0.026 & 0.2396$\pm$0.0124 & 3/3 & 2.83/3\\
Algo. 2 copy & GB/MNN &
0.98 & \textbf{0.981$\pm$0.012} & 0.0627$\pm$0.0062 &
0.92 & \textbf{0.927$\pm$0.003} & 0.0302$\pm$0.0028 &
0.87 & \textbf{0.833$\pm$0.026} & 0.1504$\pm$0.0144 & \textbf{1}/2 & \textbf{1}/1.67\\
Hard copy  & GB/MNN & 
0.98 & 0.979$\pm$0.012 & \textbf{0.0508$\pm$0.0039} &
0.92 & 0.926$\pm$0.004 & \textbf{0.0285$\pm$0.0016} &
0.87 & 0.824$\pm$0.024 & \textbf{0.0874$\pm$0.0097} & 2/\textbf{1} & 2.17/\textbf{1.33}\\
\midrule
Algo. 1 copy & NN/MNN & 
0.98 & \textbf{0.984$\pm$0.010} & 0.0351$\pm$0.0090 &
0.92 & \textbf{0.925$\pm$0.004} & 0.0129$\pm$0.0009 &
0.88 & 0.829$\pm$0.028 & 0.0791$\pm$0.0061 & 1.67/3 & 2/3\\
Algo. 2 copy & NN/MNN &
0.98 & 0.979$\pm$0.007 & 0.0160$\pm$0.0017 &
0.92 & 0.924$\pm$0.004 & 0.0078$\pm$0.0006 &
0.88 & 0.862$\pm$0.032 & 0.0384$\pm$0.0016 & 2.33/2 & 1.83/\textbf{1.5}\\
Hard copy  & NN/MNN & 
0.98 & 0.983$\pm$0.010 & \textbf{0.0138$\pm$0.0016} &
0.92 & \textbf{0.925$\pm$0.004} & \textbf{0.0065$\pm$0.0010} &
0.88 & \textbf{0.886$\pm$0.038} & \textbf{0.0313$\pm$0.0012} & \textbf{1.33}/\textbf{1} & \textbf{1.33}/\textbf{1.5}\\

\midrule
Algo. 1 copy & RF/LNN & 
0.97 & 0.953$\pm$0.026 & 0.1821$\pm$0.0107 &
0.93 & 0.921$\pm$0.005 & 0.0569$\pm$0.0128 &
0.80 & \textbf{0.810$\pm$0.072} & 0.2157$\pm$0.0163 & 2/3 & 2.33/3\\
Algo. 2 copy & RF/LNN &
0.97 & \textbf{0.970$\pm$0.009} & 0.0630$\pm$0.0018 &
0.93 & \textbf{0.928$\pm$0.004} & 0.0233$\pm$0.0032 &
0.80 & 0.800$\pm$0.039 & 0.1324$\pm$0.0074 & \textbf{1.67}/2 & \textbf{1.33}/1.67\\
Hard copy  & RF/LNN & 
0.97 & 0.960$\pm$0.012 & \textbf{0.0496$\pm$0.0018} &
0.93 & \textbf{0.928$\pm$0.006} & \textbf{0.0228$\pm$0.0032} &
0.80 & 0.805$\pm$0.044 & \textbf{0.0971$\pm$0.0042} & \textbf{1.67}/\textbf{1} & 2/\textbf{1.33}\\
\midrule
Algo. 1 copy & GB/LNN & 
0.98 & 0.956$\pm$0.029 & 0.1683$\pm$0.0126 &
0.92 & 0.907$\pm$0.027 & 0.0599$\pm$0.0115 &
0.87 & 0.819$\pm$0.044 & 0.2163$\pm$0.0097 & 2.67/3 & 2.67/3\\
Algo. 2 copy & GB/LNN &
0.98 & \textbf{0.979$\pm$0.009} & 0.0648$\pm$0.0050 &
0.92 & \textbf{0.929$\pm$0.004} & 0.0294$\pm$0.0020 &
0.87 & 0.814$\pm$0.038 & 0.1415$\pm$0.0096 & \textbf{1.67}/2 & \textbf{1.33}/1.67\\
Hard copy  & GB/LNN & 
0.98 & 0.961$\pm$0.015 & \textbf{0.0504$\pm$0.0037} &
0.92 & 0.926$\pm$0.004 & \textbf{0.0275$\pm$0.0012} &
0.87 & \textbf{0.833$\pm$0.040} & \textbf{0.0858$\pm$0.0067} & \textbf{1.67}/\textbf{1} & \textbf{1.83}/1.33\\
\midrule
Algo. 1 copy & NN/LNN & 
0.97 & 0.972$\pm$0.007 & 0.0342$\pm$0.0067 &
0.92 & 0.925$\pm$0.003 & 0.0137$\pm$0.0027 &
0.85 & 0.838$\pm$0.051 & 0.0822$\pm$0.0111 & 2.33/3 & 2.17/3\\
Algo. 2 copy & NN/LNN &
0.97 & \textbf{0.975$\pm$0.007} & 0.0156$\pm$0.0016 &
0.92 & \textbf{0.927$\pm$0.006} & 0.0077$\pm$0.0009 &
0.85 & 0.833$\pm$0.050 & 0.0371$\pm$0.0019 & 1.67/2 & 1.67/\textbf{1.5}\\
Hard copy  & NN/LNN & 
0.97 & \textbf{0.975$\pm$0.007} & \textbf{0.0140$\pm$0.0017} &
0.92 & 0.926$\pm$0.005 & \textbf{0.0070$\pm$0.0018} &
0.85 & \textbf{0.843$\pm$0.039} & \textbf{0.0296$\pm$0.0011} & \textbf{1.33}/\textbf{1} & 1.33/\textbf{1.5}\\

\midrule
Algo. 1 copy & RF/GB & 
0.97 & \textbf{0.970$\pm$0.009} & 0.0933$\pm$0.0028 &
0.93 & 0.927$\pm$0.004 & 0.0388$\pm$0.0018 &
0.82 & \textbf{0.805$\pm$0.071} & 0.1402$\pm$0.0079 & \textbf{1.33}/3 & 1.83/3\\
Algo. 2 copy & RF/GB &
0.97 & 0.967$\pm$0.007 & 0.0621$\pm$0.0017 &
0.93 & \textbf{0.928$\pm$0.005} & 0.0255$\pm$0.0022 &
0.82 & 0.800$\pm$0.070 & 0.1040$\pm$0.0043 & 1.67/2 & \textbf{1.5}/2\\
Hard copy  & RF/GB & 
0.97 & 0.967$\pm$0.007 & \textbf{0.0498$\pm$0.0024} &
0.93 & \textbf{0.928$\pm$0.003} & \textbf{0.0160$\pm$0.0011} &
0.82 & 0.791$\pm$0.076 & \textbf{0.0988$\pm$0.0047} & 2/\textbf{1} & 1.83/\textbf{1}\\
\midrule
Algo. 1 copy & GB/GB & 
0.98 & \textbf{0.975$\pm$0.015} & 0.0938$\pm$0.0023 &
0.92 & \textbf{0.927$\pm$0.005} & 0.0376$\pm$0.0025 &
0.87 & \textbf{0.848$\pm$0.019} & 0.1354$\pm$0.0061 & \textbf{1}/3 & \textbf{1.67}/3\\
Algo. 2 copy & GB/GB &
0.98 & 0.970$\pm$0.009 & 0.0594$\pm$0.0027 &
0.92 & 0.926$\pm$0.005 & 0.0278$\pm$0.0023 &
0.87 & \textbf{0.848$\pm$0.054} & 0.0838$\pm$0.0014 & 2/2 & 1.83/2\\
Hard copy  & GB/GB & 
0.98 & 0.972$\pm$0.009 & \textbf{0.0428$\pm$0.0042} &
0.92 & \textbf{0.927$\pm$0.005} & \textbf{0.0206$\pm$0.0024} &
0.87 & \textbf{0.848$\pm$0.044} & \textbf{0.0691$\pm$0.0045} & 1.33/\textbf{1} & \textbf{1.67}/\textbf{1}\\
\midrule
Algo. 1 copy & NN/GB & 
0.98 & 0.972$\pm$0.015 & 0.0887$\pm$0.0079 &
0.92 & 0.924$\pm$0.006 & 0.0298$\pm$0.0038 &
0.89 & 0.786$\pm$0.067 & 0.1543$\pm$0.0016 & 2.67/3 & 2.33/3\\
Algo. 2 copy & NN/GB &
0.98 & \textbf{0.984$\pm$0.010} & 0.0752$\pm$0.0062 &
0.92 & \textbf{0.925$\pm$0.006} & 0.0238$\pm$0.0033 &
0.89 & 0.848$\pm$0.036 & \textbf{0.1311$\pm$0.0009} & \textbf{1.33}/1.67 & \textbf{1.5}/1.83\\
Hard copy  & NN/GB & 
0.98 & 0.977$\pm$0.016 & \textbf{0.0744$\pm$0.0064} &
0.92 & 0.923$\pm$0.005 & \textbf{0.0206$\pm$0.0032} &
0.89 & \textbf{0.852$\pm$0.028} & 0.1339$\pm$0.0008 & 2/\textbf{1.33} & \textbf{1.5}/\textbf{1.17}\\
\bottomrule
\end{tabular}
}
\end{center}

\clearpage
\section{Detailed average statistics for Experiment 2}
\label{appxMax}

Global statistics of relative differences (\%) in classification error ($1-\mathcal{A}_{\mathcal{C}}$) and empirical fidelity error ($R^\mathcal{F}_\text{emp}$). The table reports aggregated minimum, maximum, mean, and median values. Negative values indicate percentage reductions in error.

\begin{table*}[!ht]
\centering
\label{tab:global_stats}
\setlength{\tabcolsep}{6pt}
\begin{tabular}{l l r r r r}
\toprule
$f_{\mathcal{C}}$ model & Stat & $(1-\mathcal{A}_{\mathcal{C}})$@$\alpha{=}1$ & $R^\mathcal{F}_\text{emp}$@$\alpha{=}1$ & $1-\mathcal{A}_{\mathcal{C}}$ \text{best} & $R^\mathcal{F}_\text{emp}$ \text{best} \\
\midrule
SNN & mean   &  -7.5241 &  +3.7612 & -12.4392 &  -4.2433 \\
SNN & median &  -3.0486 &  +5.6204 &  -6.3520 &  -1.7754 \\
SNN & min    & -67.7778 & -39.8993 & -67.7778 & -39.8993 \\
SNN & max    &  +8.3333 & +52.6226 &  +0.0000 & +26.8811 \\
\midrule
MNN & mean   &  -6.9863 &  +7.7095 & -15.1050 &  -1.5722 \\
MNN & median &  -3.1630 & +10.5833 &  -7.1485 &  -0.1927 \\
MNN & min    & -55.5556 & -42.0304 & -100.0000 & -42.0304 \\
MNN & max    & +20.0000 & +68.3114 & +10.0000 & +56.9783 \\
\midrule
LNN & mean   & -16.6116 &  +5.9607 & -19.8544 &  -4.1424 \\
LNN & median &  -5.2801 &  +6.4937 &  -6.4286 &  -6.4447 \\
LNN & min    & -100.0000 & -46.7362 & -100.0000 & -47.5401 \\
LNN & max    &  +8.5714 & +58.2688 &  +1.0135 & +55.4514 \\
\midrule
GB  & mean   &  +0.9017 & +59.8241 &  -5.1572 & +14.2069 \\
GB  & median &  +0.3663 & +67.9962 &  -1.4046 & +13.4212 \\
GB  & min    & -33.3333 &  -1.8922 & -33.3333 & -17.1805 \\
GB  & max    & +23.0769 & +123.8385 &  +3.2258 & +45.0774 \\
\bottomrule
\end{tabular}
\end{table*}

\clearpage
\newpage
\section{Evolution of the metrics as a function of the number of training points.}
\label{appxC}
\captionsetup[subfigure]{labelformat=empty}

\begin{figure}[!ht]
	\centering
	\begin{subfigure}[t]{0.15\textwidth}
		\includegraphics[width=\linewidth]{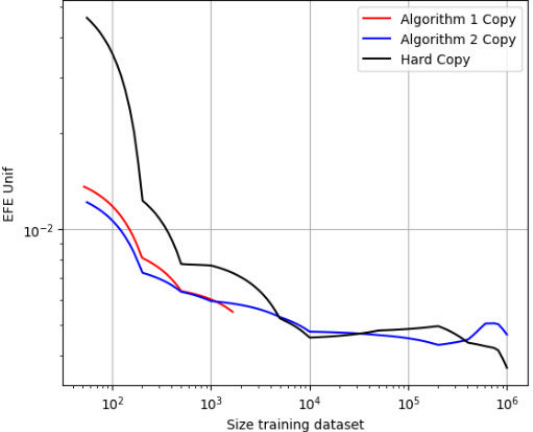}
		\caption{1-RF/SNN-$R^{\mathcal{F}}_{emp}$}
	\end{subfigure}
	\hfill
	\begin{subfigure}[t]{0.15\textwidth}
		\includegraphics[width=\linewidth]{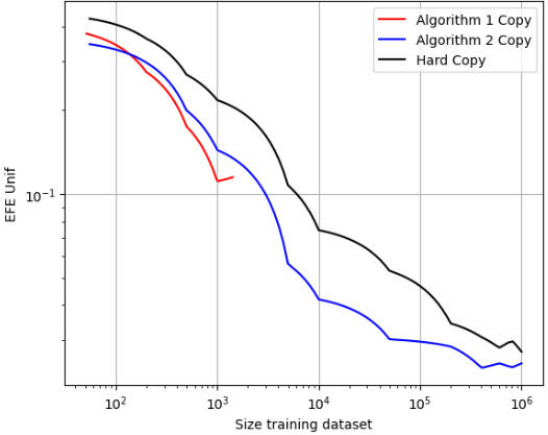}
		\caption{2-RF/SNN-$R^{\mathcal{F}}_{emp}$}
	\end{subfigure}
	\hfill
	\begin{subfigure}[t]{0.15\textwidth}
		\includegraphics[width=\linewidth]{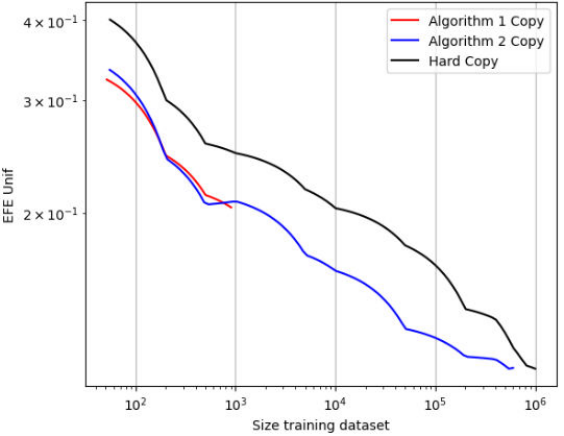}
		\caption{3-RF/SNN-$R^{\mathcal{F}}_{emp}$}
	\end{subfigure}
	\hfill
	\begin{subfigure}[t]{0.15\textwidth}
		\includegraphics[width=\linewidth]{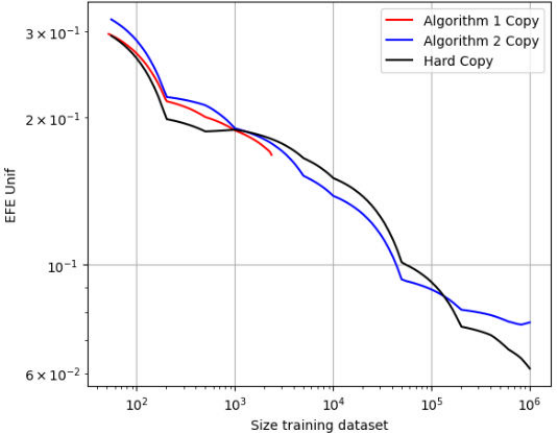}
		\caption{4-RF/SNN-$R^{\mathcal{F}}_{emp}$}
	\end{subfigure}
	\hfill
	\begin{subfigure}[t]{0.15\textwidth}
		\includegraphics[width=\linewidth]{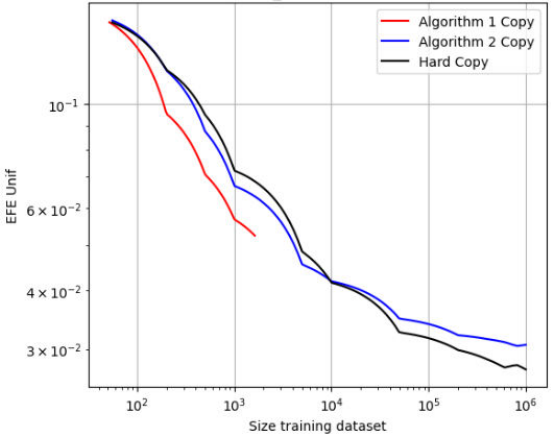}
		\caption{5-RF/SNN-$R^{\mathcal{F}}_{emp}$}
	\end{subfigure}
	\hfill
	\begin{subfigure}[t]{0.15\textwidth}
		\includegraphics[width=\linewidth]{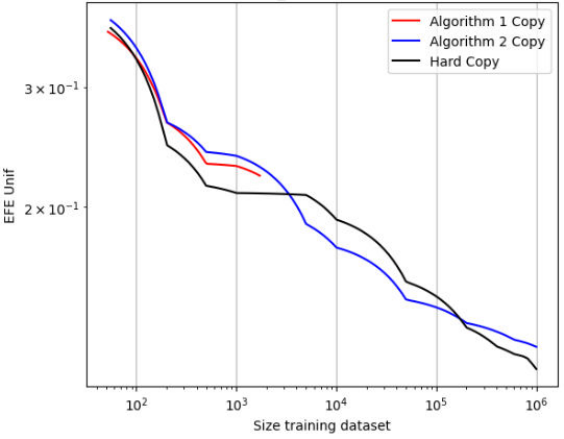}
		\caption{6-RF/SNN-$R^{\mathcal{F}}_{emp}$}
	\end{subfigure}

	\begin{subfigure}[t]{0.15\textwidth}
		\includegraphics[width=\linewidth]{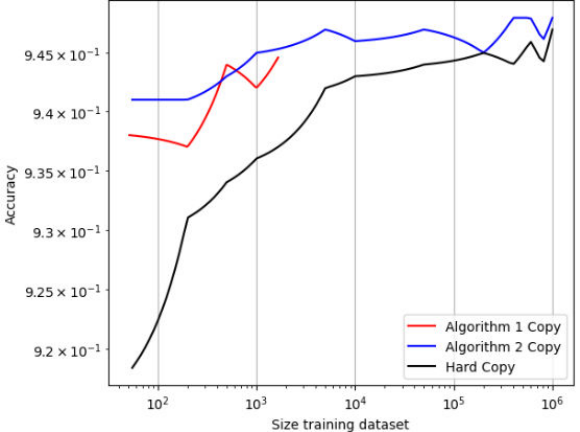}
		\caption{1-RF/SNN-$\mathcal{A}_{\mathcal{C}}$}
	\end{subfigure}
	\hfill
	\begin{subfigure}[t]{0.15\textwidth}
		\includegraphics[width=\linewidth]{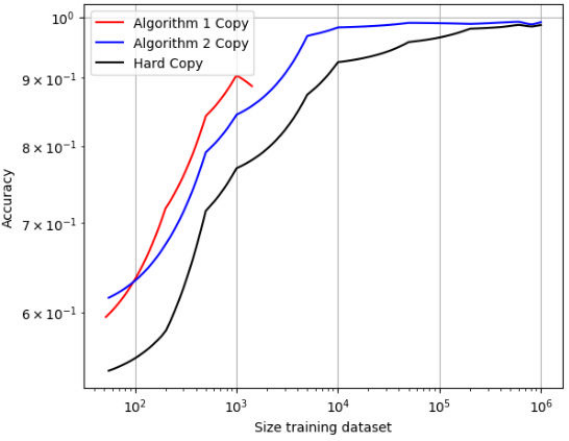}
		\caption{2-RF/SNN-$\mathcal{A}_{\mathcal{C}}$}
	\end{subfigure}
	\hfill
	\begin{subfigure}[t]{0.15\textwidth}
		\includegraphics[width=\linewidth]{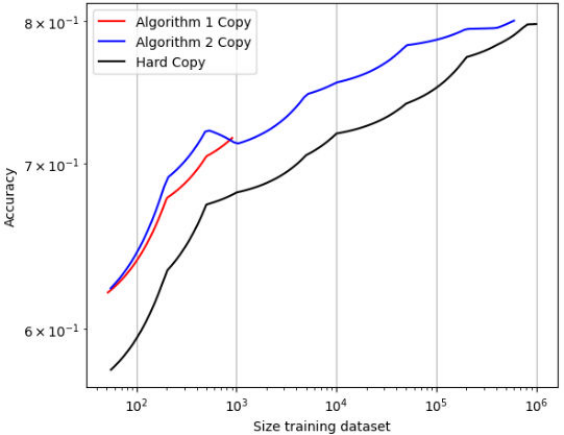}
		\caption{3-RF/SNN-$\mathcal{A}_{\mathcal{C}}$}
	\end{subfigure}
	\hfill
	\begin{subfigure}[t]{0.15\textwidth}
		\includegraphics[width=\linewidth]{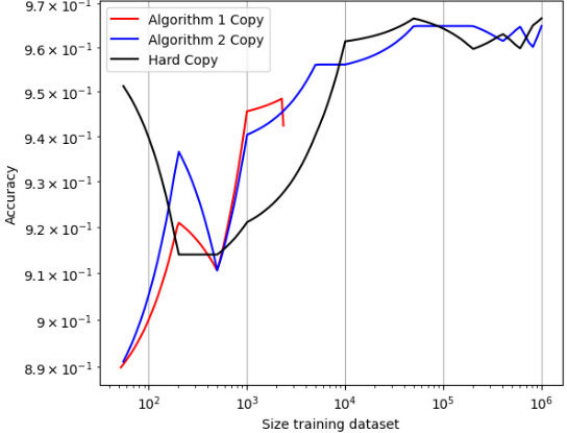}
		\caption{4-RF/SNN-$\mathcal{A}_{\mathcal{C}}$}
	\end{subfigure}
	\hfill
	\begin{subfigure}[t]{0.15\textwidth}
		\includegraphics[width=\linewidth]{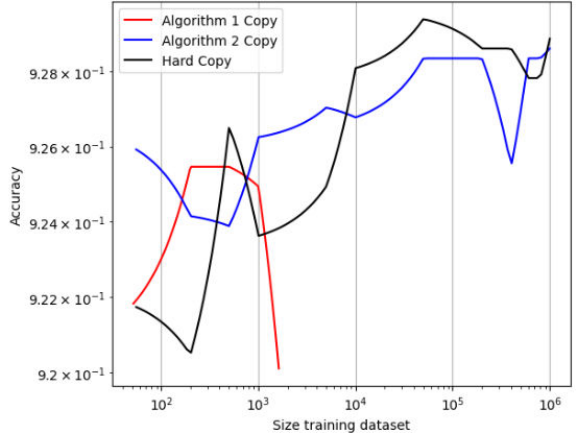}
		\caption{5-RF/SNN-$\mathcal{A}_{\mathcal{C}}$}
	\end{subfigure}
	\hfill
	\begin{subfigure}[t]{0.15\textwidth}
		\includegraphics[width=\linewidth]{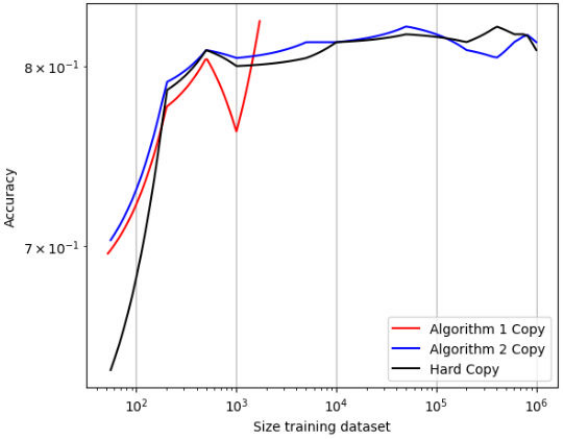}
		\caption{6-RF/SNN-$\mathcal{A}_{\mathcal{C}}$}
	\end{subfigure}

	\begin{subfigure}[t]{0.15\textwidth}
		\includegraphics[width=\linewidth]{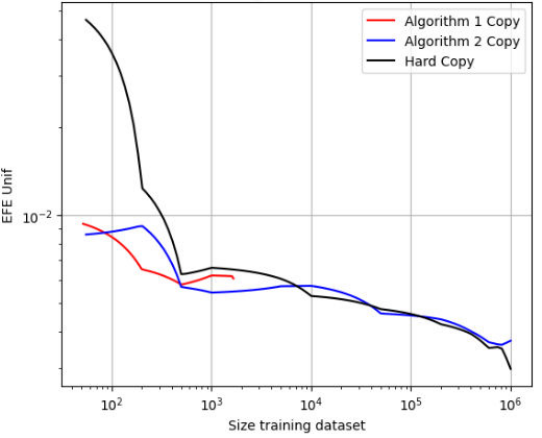}
		\caption{1-RF/MNN-$R^{\mathcal{F}}_{emp}$}
	\end{subfigure}
	\hfill
	\begin{subfigure}[t]{0.15\textwidth}
		\includegraphics[width=\linewidth]{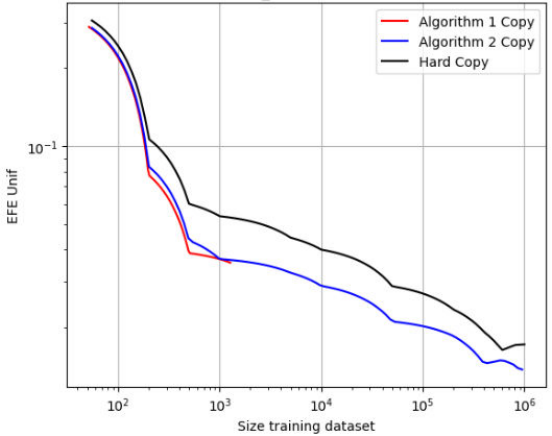}
		\caption{2-RF/MNN-$R^{\mathcal{F}}_{emp}$}
	\end{subfigure}
	\hfill
	\begin{subfigure}[t]{0.15\textwidth}
		\includegraphics[width=\linewidth]{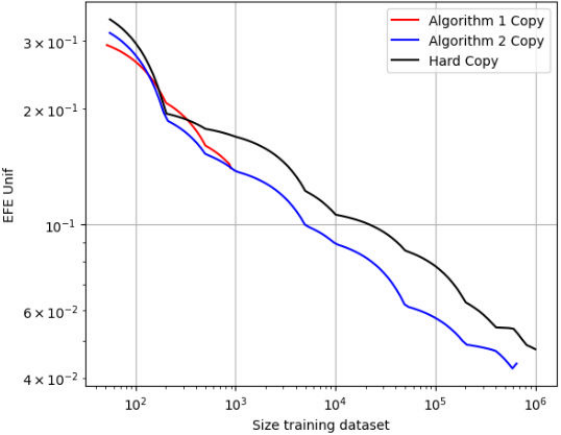}
		\caption{3-RF/MNN-$R^{\mathcal{F}}_{emp}$}
	\end{subfigure}
	\hfill
	\begin{subfigure}[t]{0.15\textwidth}
		\includegraphics[width=\linewidth]{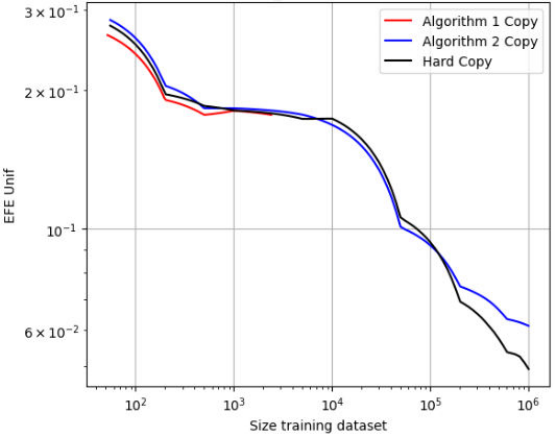}
		\caption{4-RF/MNN-$R^{\mathcal{F}}_{emp}$}
	\end{subfigure}
	\hfill
	\begin{subfigure}[t]{0.15\textwidth}
		\includegraphics[width=\linewidth]{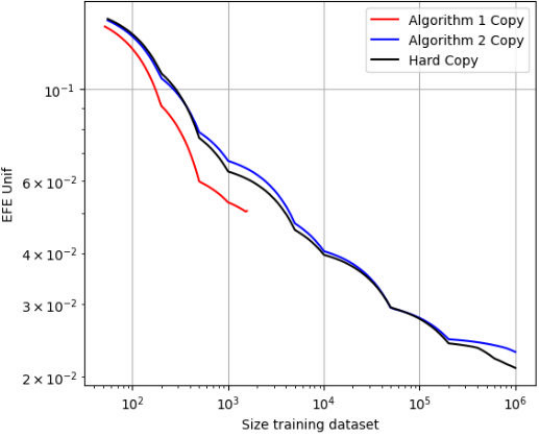}
		\caption{5-RF/MNN-$R^{\mathcal{F}}_{emp}$}
	\end{subfigure}
	\hfill
	\begin{subfigure}[t]{0.15\textwidth}
		\includegraphics[width=\linewidth]{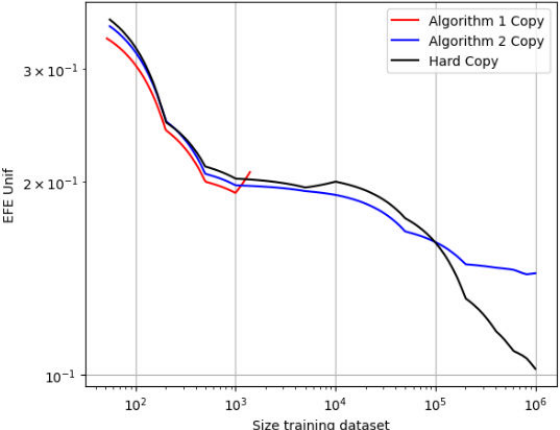}
		\caption{6-RF/MNN-$R^{\mathcal{F}}_{emp}$}
	\end{subfigure}

	\begin{subfigure}[t]{0.15\textwidth}
		\includegraphics[width=\linewidth]{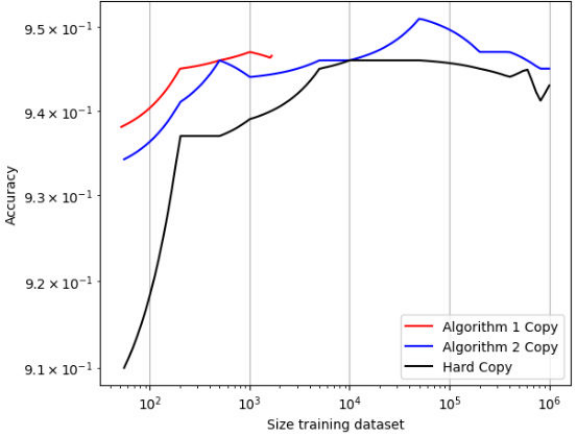}
		\caption{1-RF/MNN-$\mathcal{A}_{\mathcal{C}}$}
	\end{subfigure}
	\hfill
	\begin{subfigure}[t]{0.15\textwidth}
		\includegraphics[width=\linewidth]{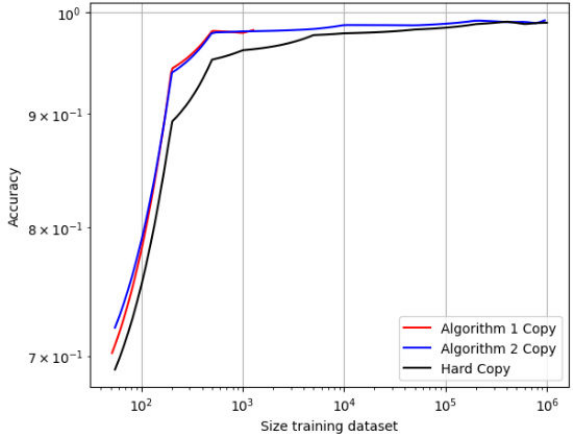}
		\caption{2-RF/MNN-$\mathcal{A}_{\mathcal{C}}$}
	\end{subfigure}
	\hfill
	\begin{subfigure}[t]{0.15\textwidth}
		\includegraphics[width=\linewidth]{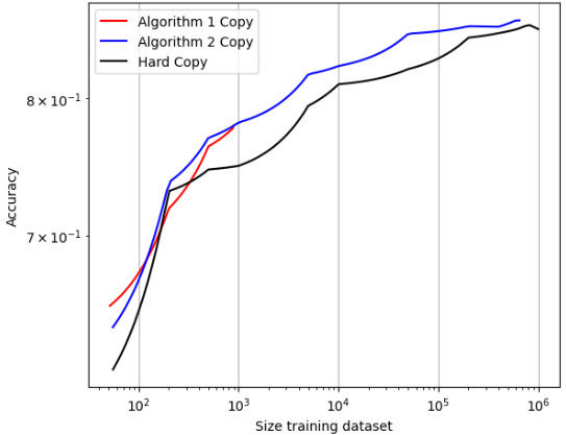}
		\caption{3-RF/MNN-$\mathcal{A}_{\mathcal{C}}$}
	\end{subfigure}
	\hfill
	\begin{subfigure}[t]{0.15\textwidth}
		\includegraphics[width=\linewidth]{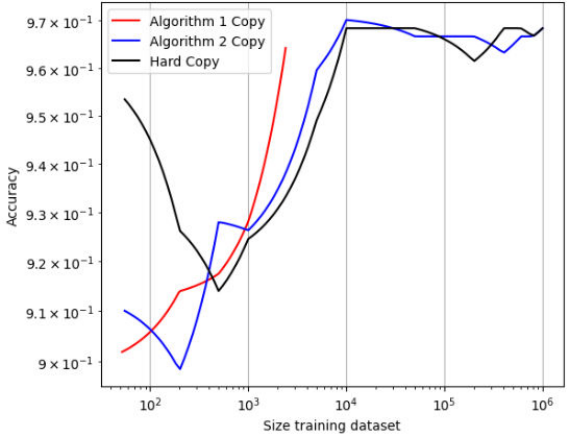}
		\caption{4-RF/MNN-$\mathcal{A}_{\mathcal{C}}$}
	\end{subfigure}
	\hfill
	\begin{subfigure}[t]{0.15\textwidth}
		\includegraphics[width=\linewidth]{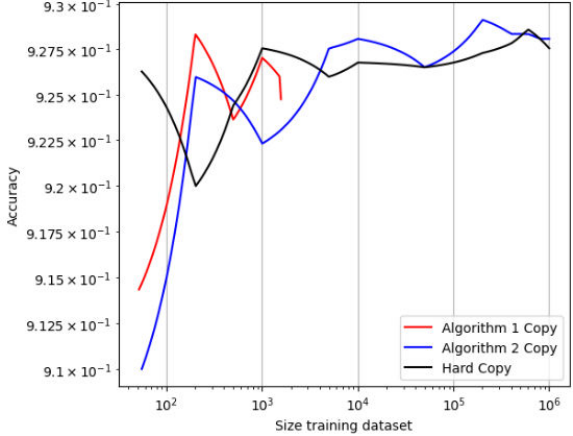}
		\caption{5-RF/MNN-$\mathcal{A}_{\mathcal{C}}$}
	\end{subfigure}
	\hfill
	\begin{subfigure}[t]{0.15\textwidth}
		\includegraphics[width=\linewidth]{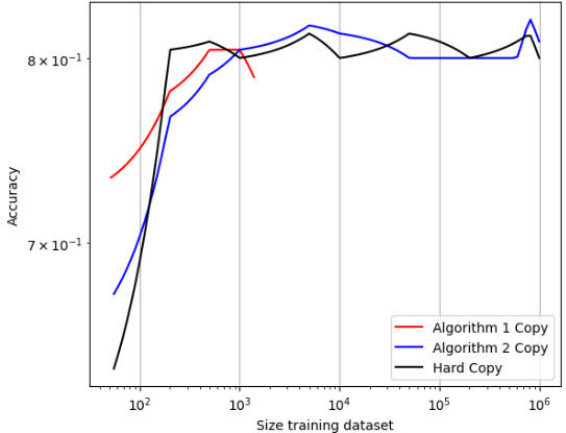}
		\caption{6-RF/MNN-$\mathcal{A}_{\mathcal{C}}$}
	\end{subfigure}

	\begin{subfigure}[t]{0.15\textwidth}
		\includegraphics[width=\linewidth]{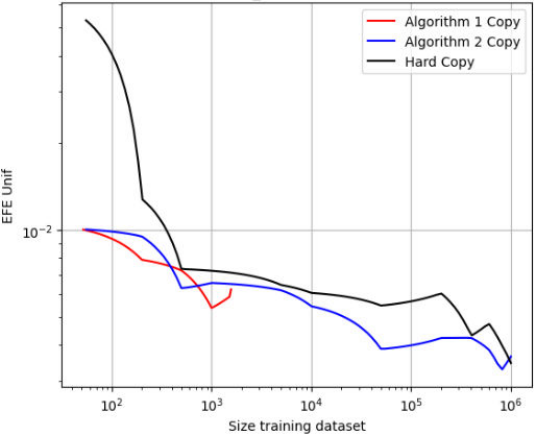}
		\caption{1-RF/LNN-$R^{\mathcal{F}}_{emp}$}
	\end{subfigure}
	\hfill
	\begin{subfigure}[t]{0.15\textwidth}
		\includegraphics[width=\linewidth]{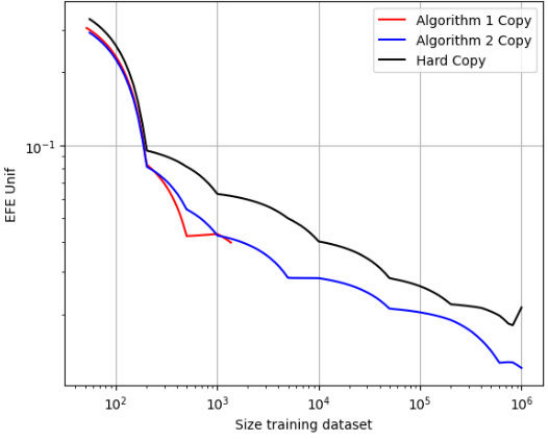}
		\caption{2-RF/LNN-$R^{\mathcal{F}}_{emp}$}
	\end{subfigure}
	\hfill
	\begin{subfigure}[t]{0.15\textwidth}
		\includegraphics[width=\linewidth]{Figures/Figure_47/3-RF-LNN-R.pdf}
		\caption{3-RF/LNN-$R^{\mathcal{F}}_{emp}$}
	\end{subfigure}
	\hfill
	\begin{subfigure}[t]{0.15\textwidth}
		\includegraphics[width=\linewidth]{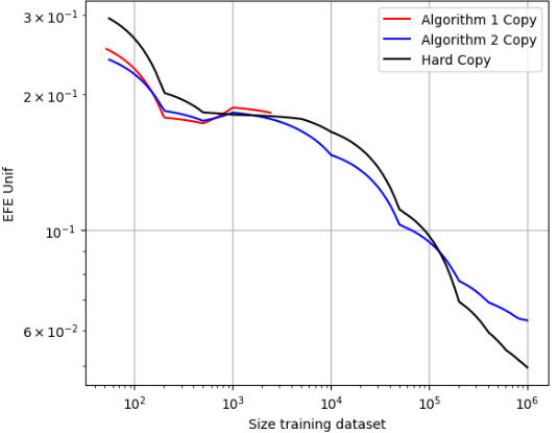}
		\caption{4-RF/LNN-$R^{\mathcal{F}}_{emp}$}
	\end{subfigure}
	\hfill
	\begin{subfigure}[t]{0.15\textwidth}
		\includegraphics[width=\linewidth]{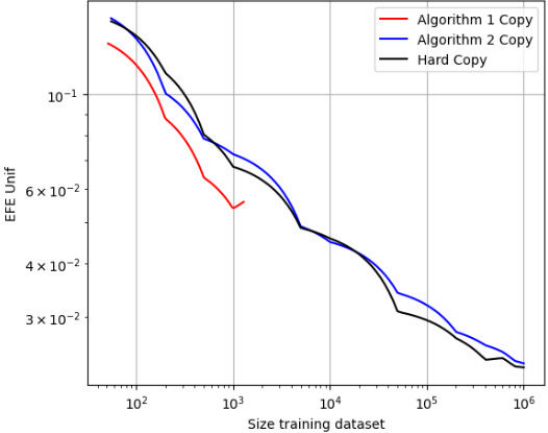}
		\caption{5-RF/LNN-$R^{\mathcal{F}}_{emp}$}
	\end{subfigure}
	\hfill
	\begin{subfigure}[t]{0.15\textwidth}
		\includegraphics[width=\linewidth]{Figures/Figure_47/6-RF-LNN-R.pdf}
		\caption{6-RF/LNN-$R^{\mathcal{F}}_{emp}$}
	\end{subfigure}

	\begin{subfigure}[t]{0.15\textwidth}
		\includegraphics[width=\linewidth]{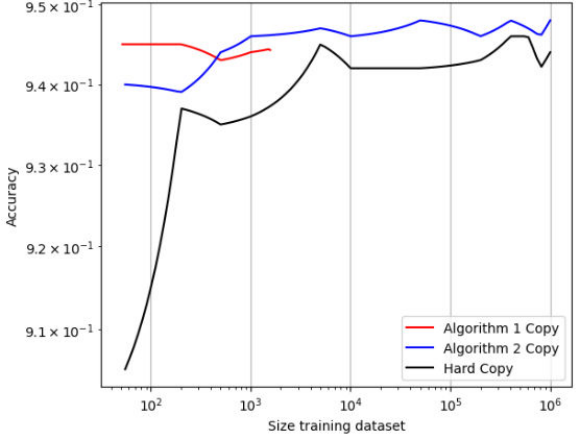}
		\caption{1-RF/LNN-$\mathcal{A}_{\mathcal{C}}$}
	\end{subfigure}
	\hfill
	\begin{subfigure}[t]{0.15\textwidth}
		\includegraphics[width=\linewidth]{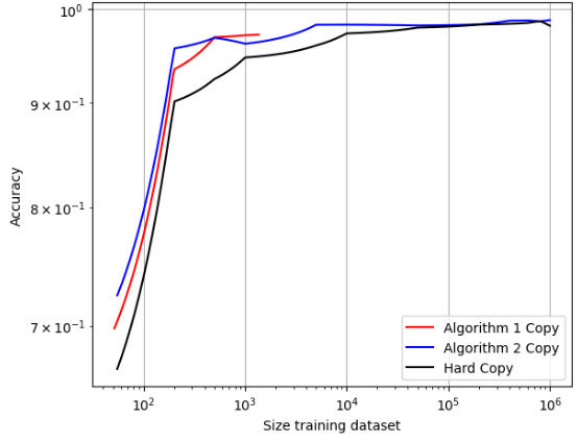}
		\caption{2-RF/LNN-$\mathcal{A}_{\mathcal{C}}$}
	\end{subfigure}
	\hfill
	\begin{subfigure}[t]{0.15\textwidth}
		\includegraphics[width=\linewidth]{Figures/Figure_47/3-RF-LNN-A.pdf}
		\caption{3-RF/LNN-$\mathcal{A}_{\mathcal{C}}$}
	\end{subfigure}
	\hfill
	\begin{subfigure}[t]{0.15\textwidth}
		\includegraphics[width=\linewidth]{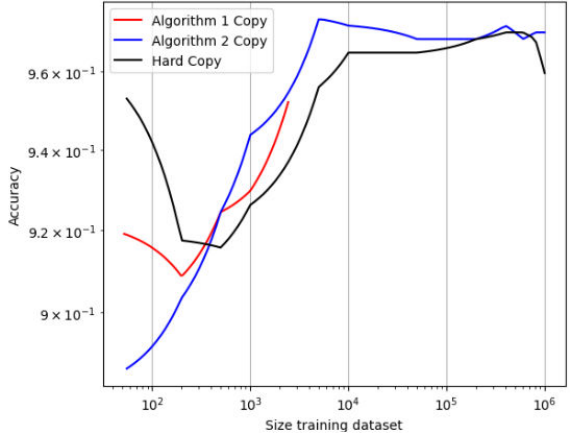}
		\caption{4-RF/LNN-$\mathcal{A}_{\mathcal{C}}$}
	\end{subfigure}
	\hfill
	\begin{subfigure}[t]{0.15\textwidth}
		\includegraphics[width=\linewidth]{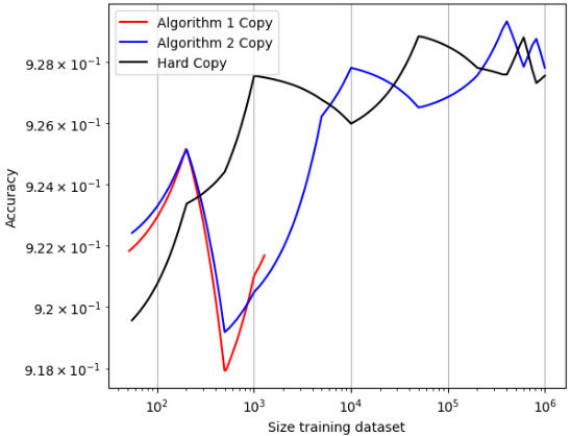}
		\caption{5-RF/LNN-$\mathcal{A}_{\mathcal{C}}$}
	\end{subfigure}
	\hfill
	\begin{subfigure}[t]{0.15\textwidth}
		\includegraphics[width=\linewidth]{Figures/Figure_47/6-RF-LNN-A.pdf}
		\caption{6-RF/LNN-$\mathcal{A}_{\mathcal{C}}$}
	\end{subfigure}

	\begin{subfigure}[t]{0.15\textwidth}
		\includegraphics[width=\linewidth]{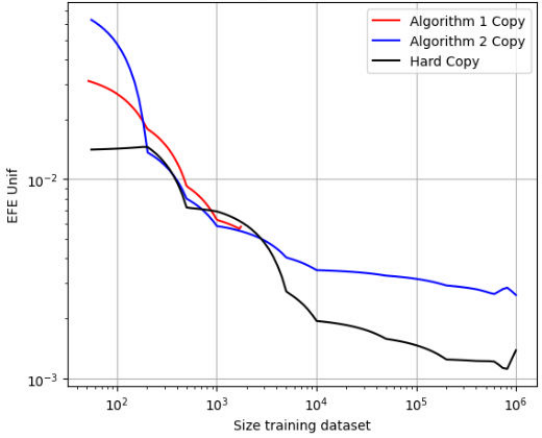}
		\caption{1-RF/GB-$R^{\mathcal{F}}_{emp}$}
	\end{subfigure}
	\hfill
	\begin{subfigure}[t]{0.15\textwidth}
		\includegraphics[width=\linewidth]{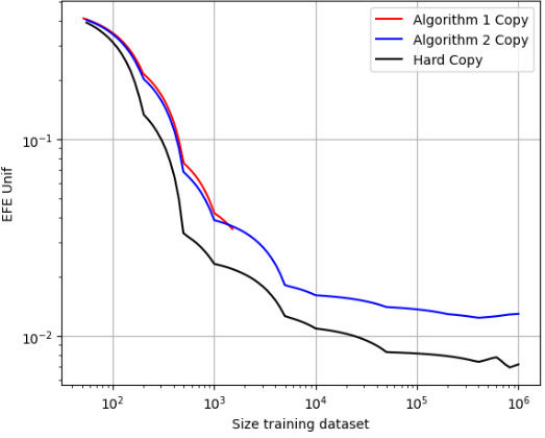}
		\caption{2-RF/GB-$R^{\mathcal{F}}_{emp}$}
	\end{subfigure}
	\hfill
	\begin{subfigure}[t]{0.15\textwidth}
		\includegraphics[width=\linewidth]{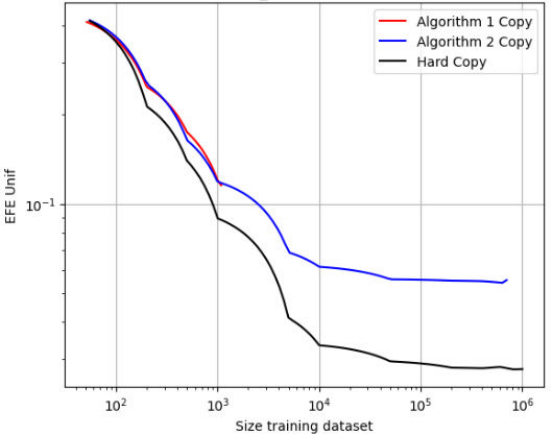}
		\caption{3-RF/GB-$R^{\mathcal{F}}_{emp}$}
	\end{subfigure}
	\hfill
	\begin{subfigure}[t]{0.15\textwidth}
		\includegraphics[width=\linewidth]{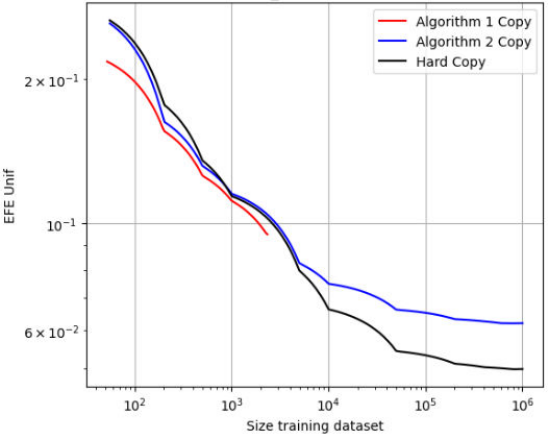}
		\caption{4-RF/GB-$R^{\mathcal{F}}_{emp}$}
	\end{subfigure}
	\hfill
	\begin{subfigure}[t]{0.15\textwidth}
		\includegraphics[width=\linewidth]{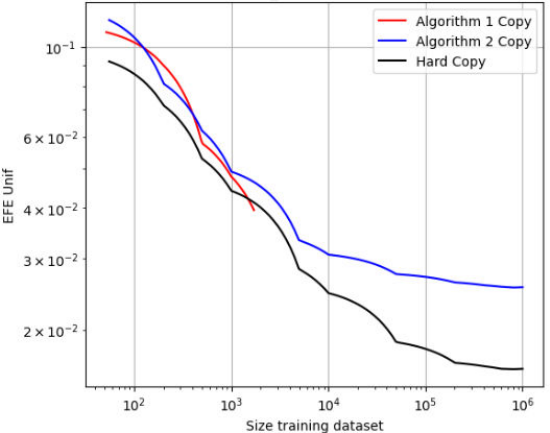}
		\caption{5-RF/GB-$R^{\mathcal{F}}_{emp}$}
	\end{subfigure}
	\hfill
	\begin{subfigure}[t]{0.15\textwidth}
		\includegraphics[width=\linewidth]{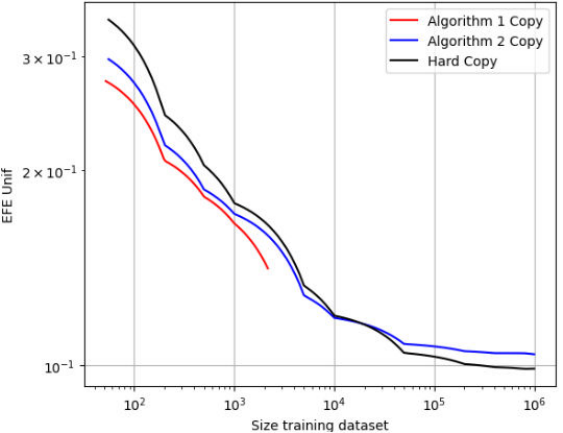}
		\caption{6-RF/GB-$R^{\mathcal{F}}_{emp}$}
	\end{subfigure}

	\begin{subfigure}[t]{0.15\textwidth}
		\includegraphics[width=\linewidth]{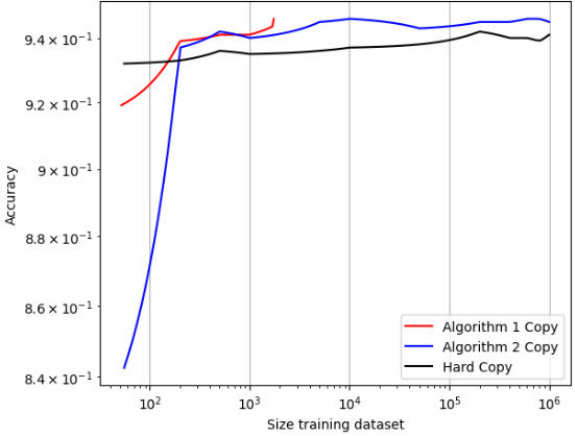}
		\caption{1-RF/GB-$\mathcal{A}_{\mathcal{C}}$}
	\end{subfigure}
	\hfill
	\begin{subfigure}[t]{0.15\textwidth}
		\includegraphics[width=\linewidth]{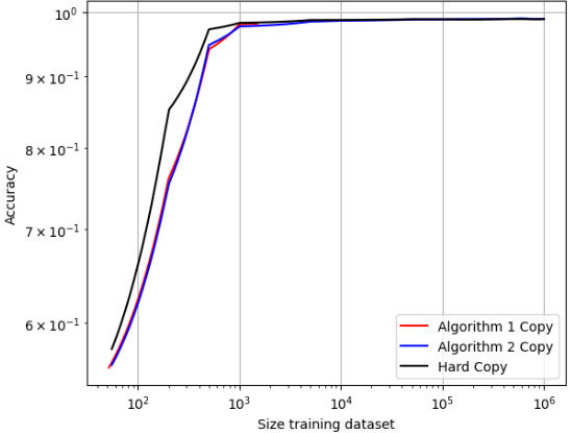}
		\caption{2-RF/GB-$\mathcal{A}_{\mathcal{C}}$}
	\end{subfigure}
	\hfill
	\begin{subfigure}[t]{0.15\textwidth}
		\includegraphics[width=\linewidth]{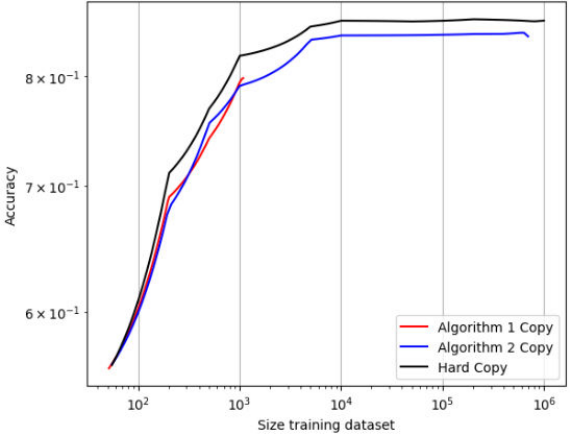}
		\caption{3-RF/GB-$\mathcal{A}_{\mathcal{C}}$}
	\end{subfigure}
	\hfill
	\begin{subfigure}[t]{0.15\textwidth}
		\includegraphics[width=\linewidth]{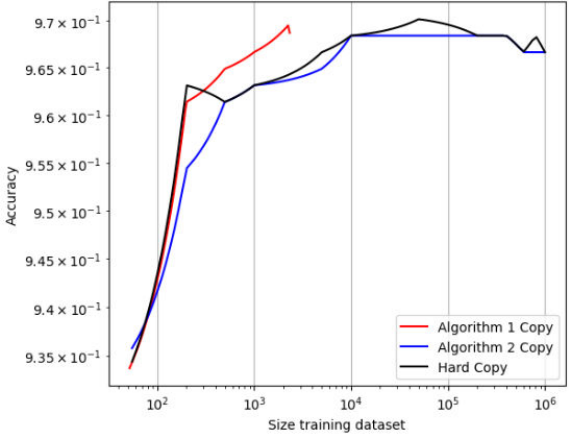}
		\caption{4-RF/GB-$\mathcal{A}_{\mathcal{C}}$}
	\end{subfigure}
	\hfill
	\begin{subfigure}[t]{0.15\textwidth}
		\includegraphics[width=\linewidth]{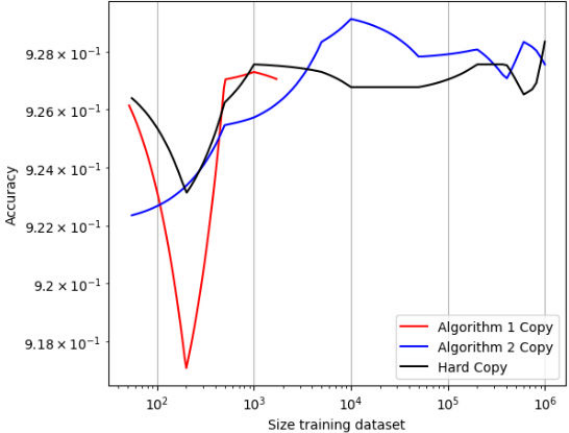}
		\caption{5-RF/GB-$\mathcal{A}_{\mathcal{C}}$}
	\end{subfigure}
	\hfill
	\begin{subfigure}[t]{0.15\textwidth}
		\includegraphics[width=\linewidth]{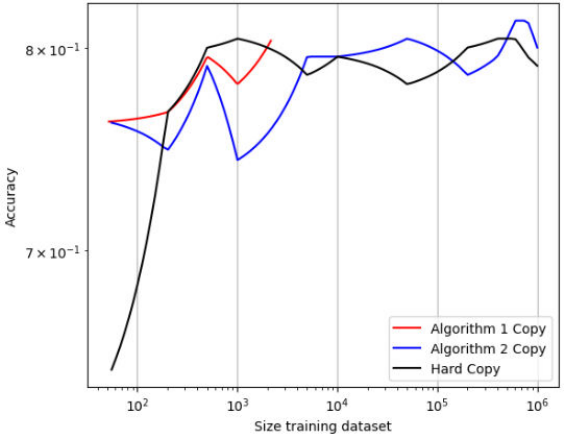}
		\caption{6-RF/GB-$\mathcal{A}_{\mathcal{C}}$}
	\end{subfigure}
\end{figure}

\begin{figure}[!ht]
	\centering
	\begin{subfigure}[t]{0.15\textwidth}
		\includegraphics[width=\linewidth]{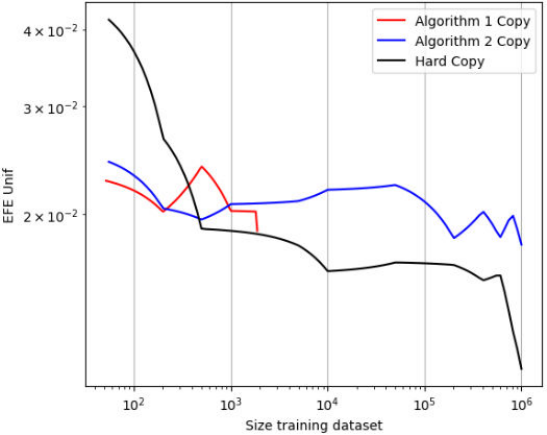}
		\caption{1-GB/SNN-$R^{\mathcal{F}}_{emp}$}
	\end{subfigure}
	\hfill
	\begin{subfigure}[t]{0.15\textwidth}
		\includegraphics[width=\linewidth]{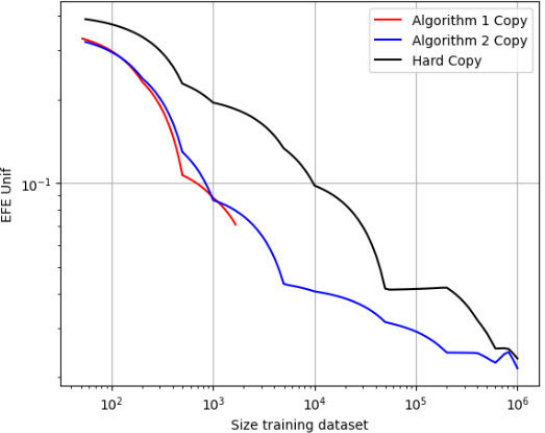}
		\caption{2-GB/SNN-$R^{\mathcal{F}}_{emp}$}
	\end{subfigure}
	\hfill
	\begin{subfigure}[t]{0.15\textwidth}
		\includegraphics[width=\linewidth]{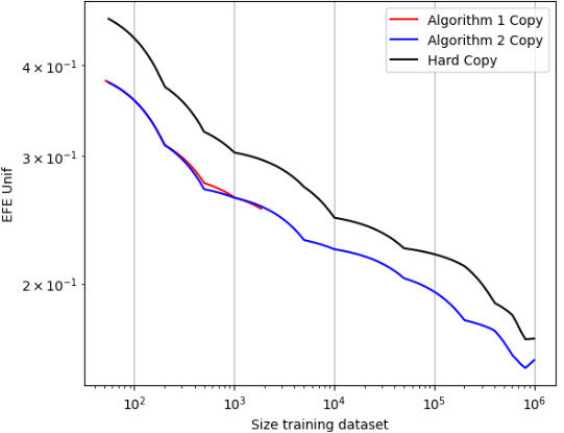}
		\caption{3-GB/SNN-$R^{\mathcal{F}}_{emp}$}
	\end{subfigure}
	\hfill
	\begin{subfigure}[t]{0.15\textwidth}
		\includegraphics[width=\linewidth]{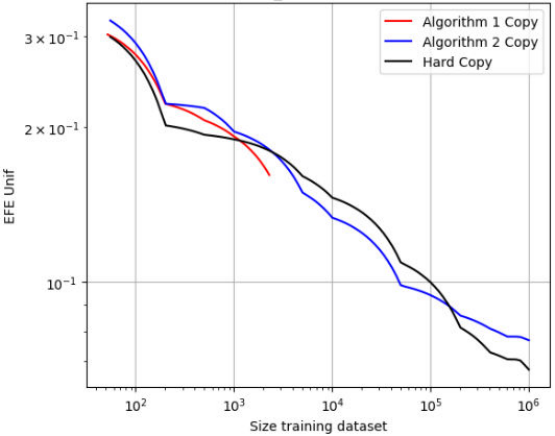}
		\caption{4-GB/SNN-$R^{\mathcal{F}}_{emp}$}
	\end{subfigure}
	\hfill
	\begin{subfigure}[t]{0.15\textwidth}
		\includegraphics[width=\linewidth]{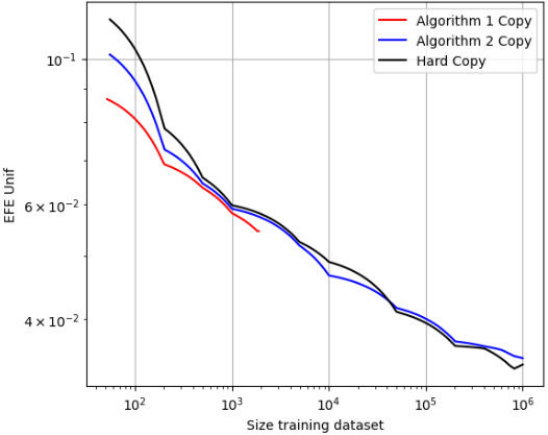}
		\caption{5-GB/SNN-$R^{\mathcal{F}}_{emp}$}
	\end{subfigure}
	\hfill
	\begin{subfigure}[t]{0.15\textwidth}
		\includegraphics[width=\linewidth]{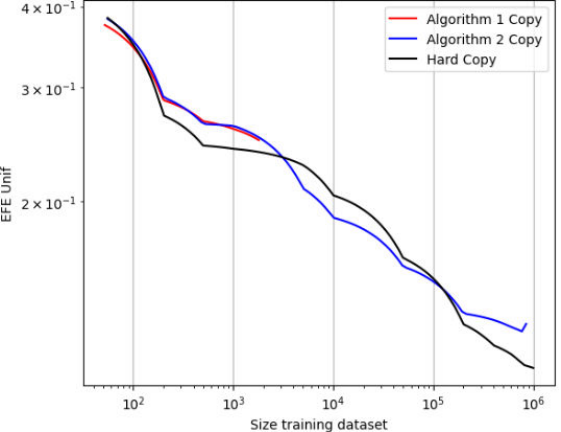}
		\caption{6-GB/SNN-$R^{\mathcal{F}}_{emp}$}
	\end{subfigure}
	
	\vspace{0.5mm}
	
	\begin{subfigure}[t]{0.15\textwidth}
		\includegraphics[width=\linewidth]{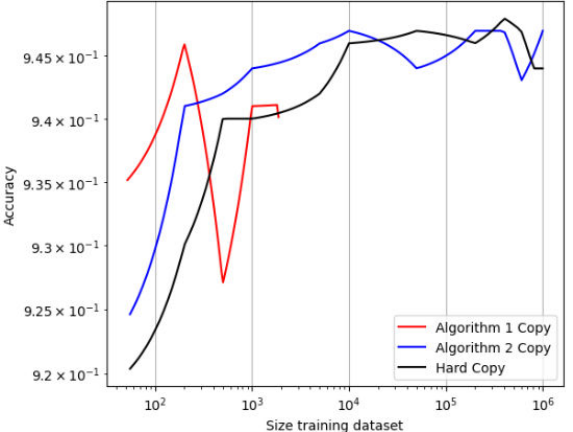}
		\caption{1-GB/SNN-$\mathcal{A}_{\mathcal{C}}$}
	\end{subfigure}
	\hfill
	\begin{subfigure}[t]{0.15\textwidth}
		\includegraphics[width=\linewidth]{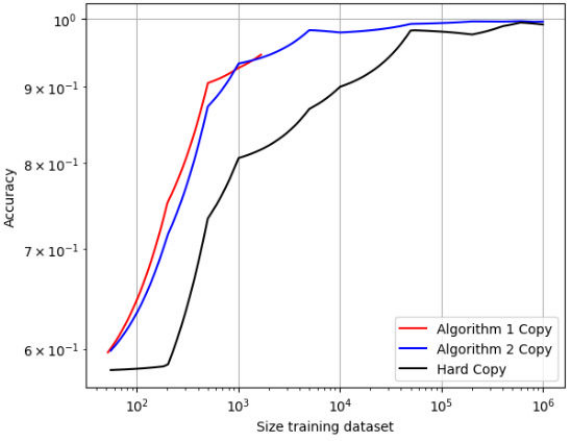}
		\caption{2-GB/SNN-$\mathcal{A}_{\mathcal{C}}$}
	\end{subfigure}
	\hfill
	\begin{subfigure}[t]{0.15\textwidth}
		\includegraphics[width=\linewidth]{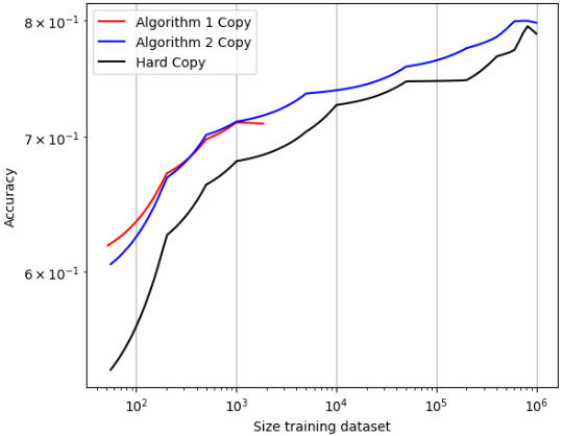}
		\caption{3-GB/SNN-$\mathcal{A}_{\mathcal{C}}$}
	\end{subfigure}
	\hfill
	\begin{subfigure}[t]{0.15\textwidth}
		\includegraphics[width=\linewidth]{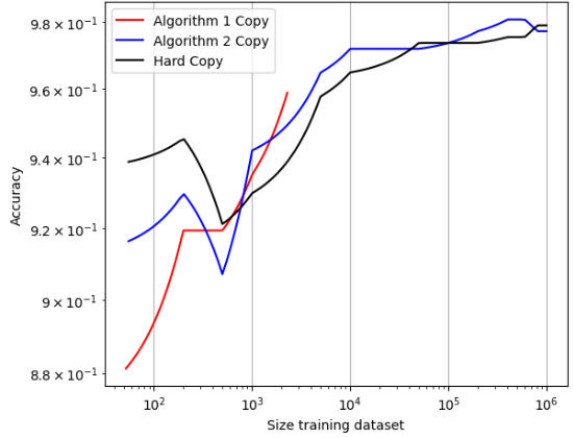}
		\caption{4-GB/SNN-$\mathcal{A}_{\mathcal{C}}$}
	\end{subfigure}
	\hfill
	\begin{subfigure}[t]{0.15\textwidth}
		\includegraphics[width=\linewidth]{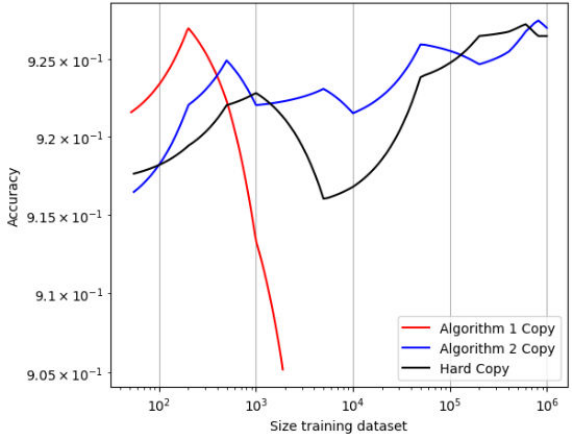}
		\caption{5-GB/SNN-$\mathcal{A}_{\mathcal{C}}$}
	\end{subfigure}
	\hfill
	\begin{subfigure}[t]{0.15\textwidth}
		\includegraphics[width=\linewidth]{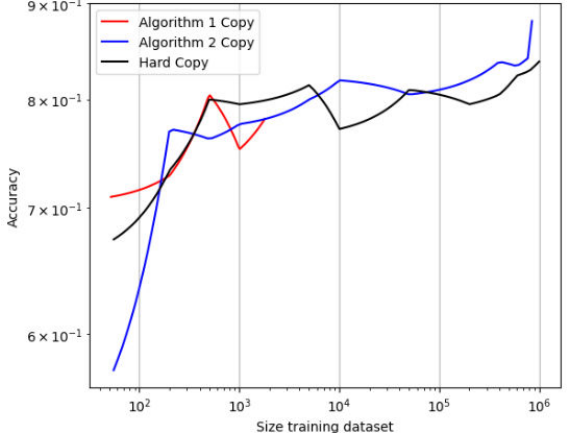}
		\caption{6-GB/SNN-$\mathcal{A}_{\mathcal{C}}$}
	\end{subfigure}
	
	\vspace{5mm}
	
	\begin{subfigure}[t]{0.15\textwidth}
		\includegraphics[width=\linewidth]{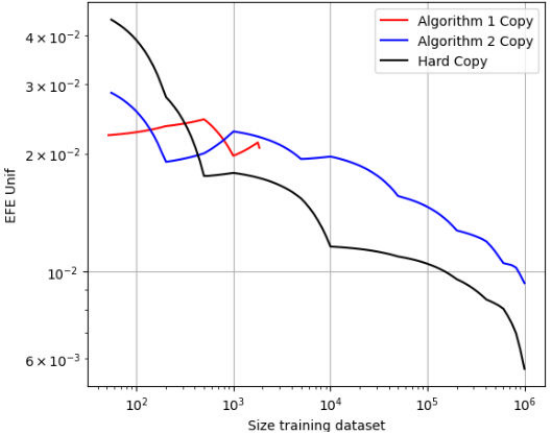}
		\caption{1-GB/MNN-$R^{\mathcal{F}}_{emp}$}
	\end{subfigure}
	\hfill
	\begin{subfigure}[t]{0.15\textwidth}
		\includegraphics[width=\linewidth]{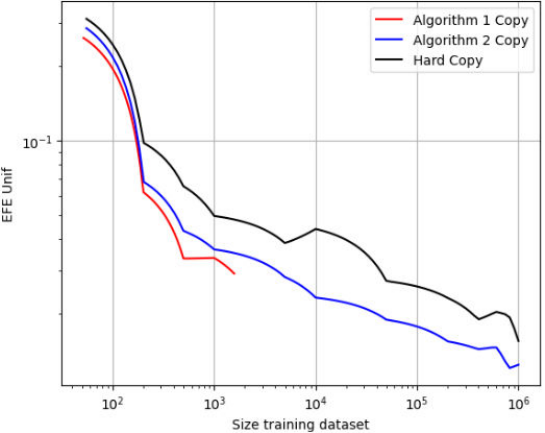}
		\caption{2-GB/MNN-$R^{\mathcal{F}}_{emp}$}
	\end{subfigure}
	\hfill
	\begin{subfigure}[t]{0.15\textwidth}
		\includegraphics[width=\linewidth]{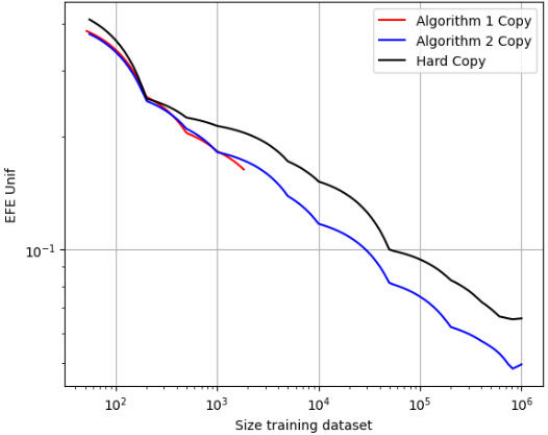}
		\caption{3-GB/MNN-$R^{\mathcal{F}}_{emp}$}
	\end{subfigure}
	\hfill
	\begin{subfigure}[t]{0.15\textwidth}
		\includegraphics[width=\linewidth]{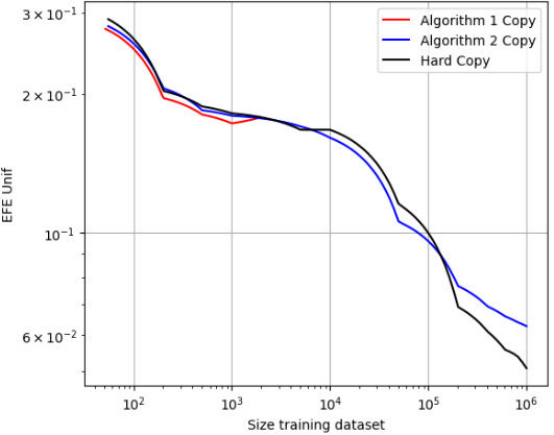}
		\caption{4-GB/MNN-$R^{\mathcal{F}}_{emp}$}
	\end{subfigure}
	\hfill
	\begin{subfigure}[t]{0.15\textwidth}
		\includegraphics[width=\linewidth]{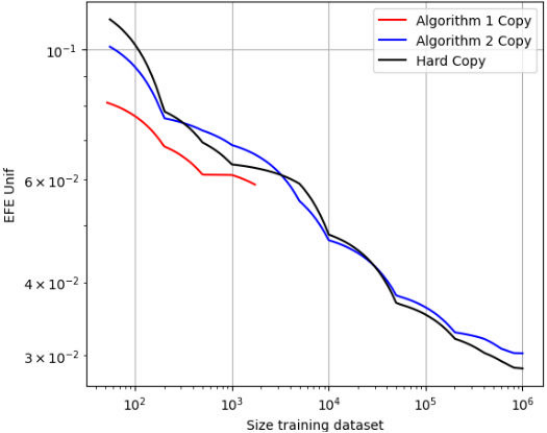}
		\caption{5-GB/MNN-$R^{\mathcal{F}}_{emp}$}
	\end{subfigure}
	\hfill
	\begin{subfigure}[t]{0.15\textwidth}
		\includegraphics[width=\linewidth]{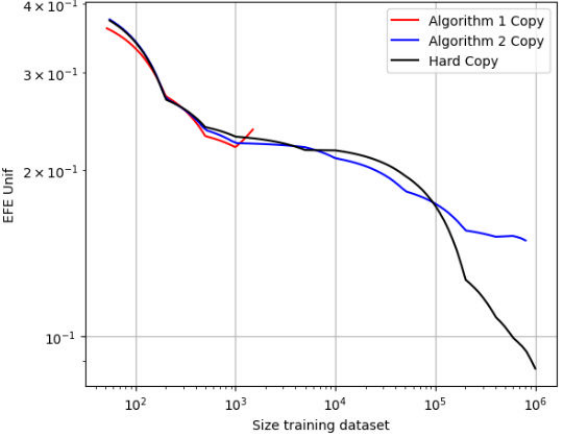}
		\caption{6-GB/MNN-$R^{\mathcal{F}}_{emp}$}
	\end{subfigure}
	
	\vspace{0.5mm}
	
	\begin{subfigure}[t]{0.15\textwidth}
		\includegraphics[width=\linewidth]{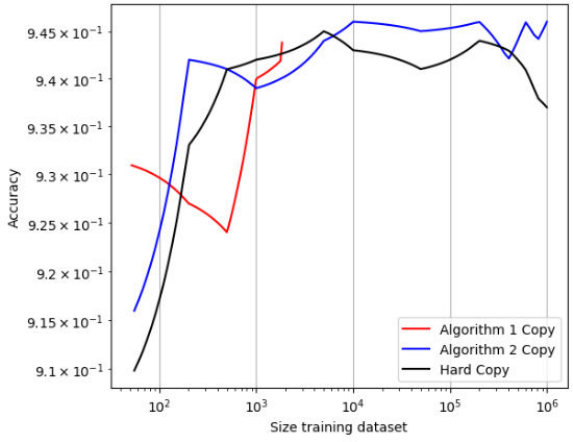}
		\caption{1-GB/MNN-$\mathcal{A}_{\mathcal{C}}$}
	\end{subfigure}
	\hfill
	\begin{subfigure}[t]{0.15\textwidth}
		\includegraphics[width=\linewidth]{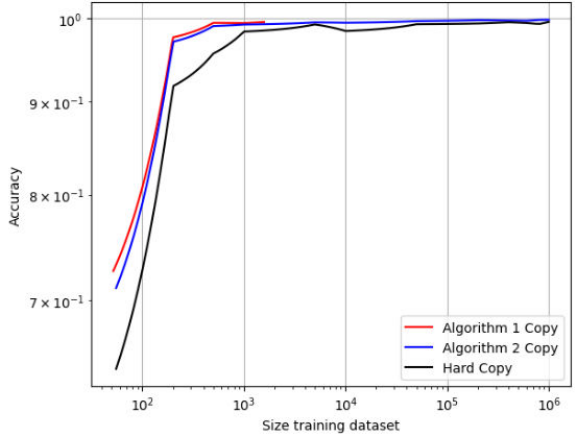}
		\caption{2-GB/MNN-$\mathcal{A}_{\mathcal{C}}$}
	\end{subfigure}
	\hfill
	\begin{subfigure}[t]{0.15\textwidth}
		\includegraphics[width=\linewidth]{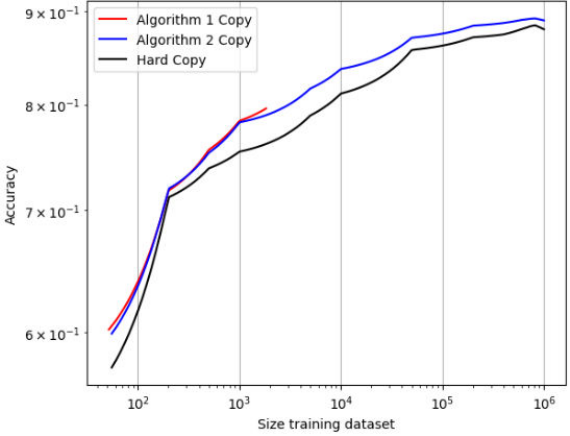}
		\caption{3-GB/MNN-$\mathcal{A}_{\mathcal{C}}$}
	\end{subfigure}
	\hfill
	\begin{subfigure}[t]{0.15\textwidth}
		\includegraphics[width=\linewidth]{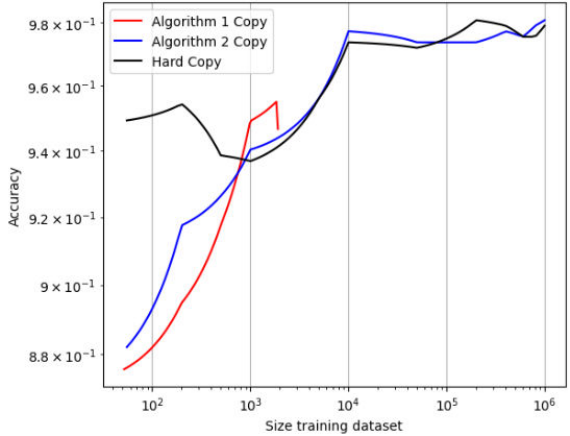}
		\caption{4-GB/MNN-$\mathcal{A}_{\mathcal{C}}$}
	\end{subfigure}
	\hfill
	\begin{subfigure}[t]{0.15\textwidth}
		\includegraphics[width=\linewidth]{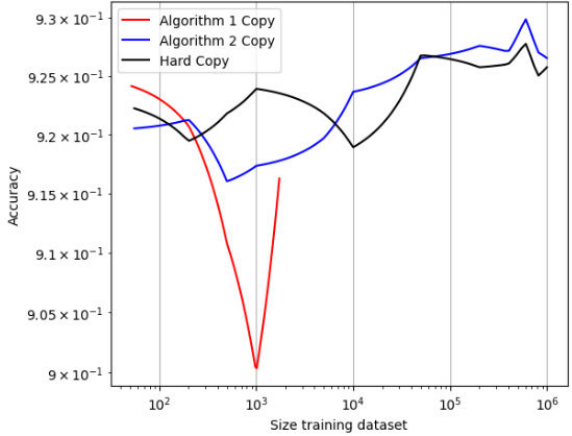}
		\caption{5-GB/MNN-$\mathcal{A}_{\mathcal{C}}$}
	\end{subfigure}
	\hfill
	\begin{subfigure}[t]{0.15\textwidth}
		\includegraphics[width=\linewidth]{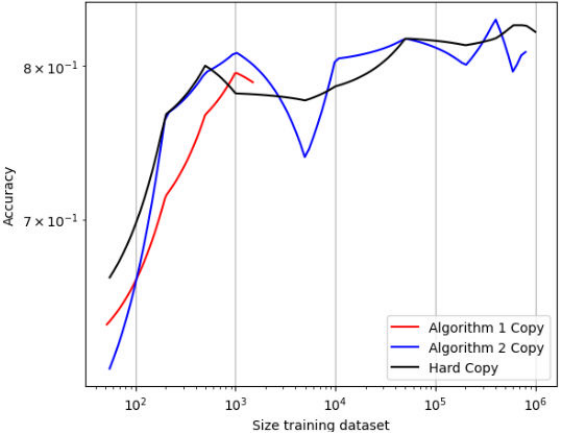}
		\caption{6-GB/MNN-$\mathcal{A}_{\mathcal{C}}$}
	\end{subfigure}
	
	\vspace{5mm} 
	
	\begin{subfigure}[t]{0.15\textwidth}
		\includegraphics[width=\linewidth]{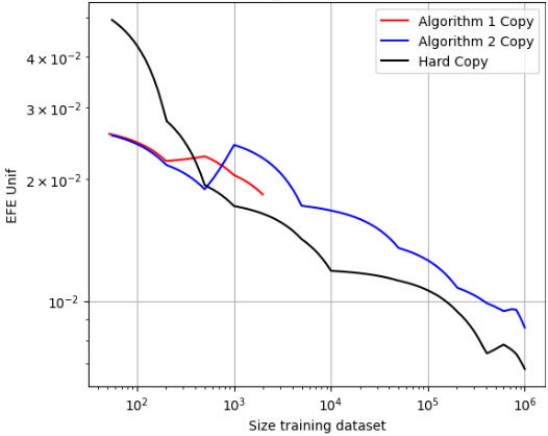}
		\caption{1-GB/LNN-$R^{\mathcal{F}}_{emp}$}
	\end{subfigure}
	\hfill
	\begin{subfigure}[t]{0.15\textwidth}
		\includegraphics[width=\linewidth]{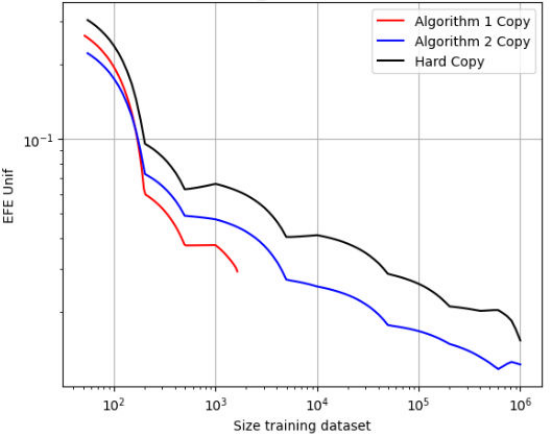}
		\caption{2-GB/LNN-$R^{\mathcal{F}}_{emp}$}
	\end{subfigure}
	\hfill
	\begin{subfigure}[t]{0.15\textwidth}
		\includegraphics[width=\linewidth]{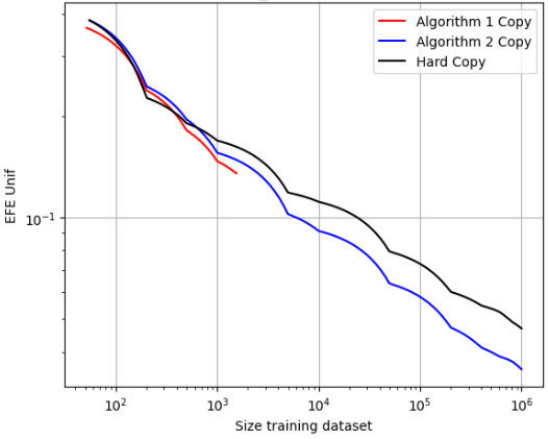}
		\caption{3-GB/LNN-$R^{\mathcal{F}}_{emp}$}
	\end{subfigure}
	\hfill
	\begin{subfigure}[t]{0.15\textwidth}
		\includegraphics[width=\linewidth]{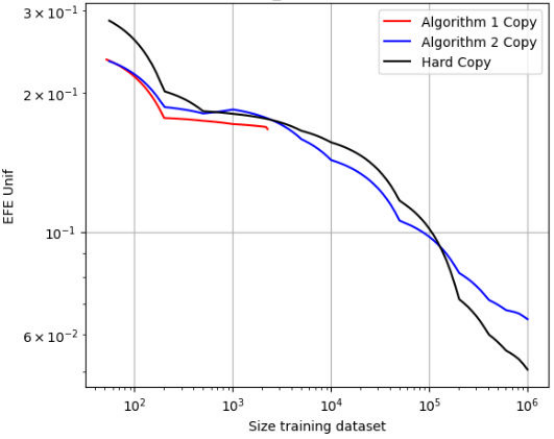}
		\caption{4-GB/LNN-$R^{\mathcal{F}}_{emp}$}
	\end{subfigure}
	\hfill
	\begin{subfigure}[t]{0.15\textwidth}
		\includegraphics[width=\linewidth]{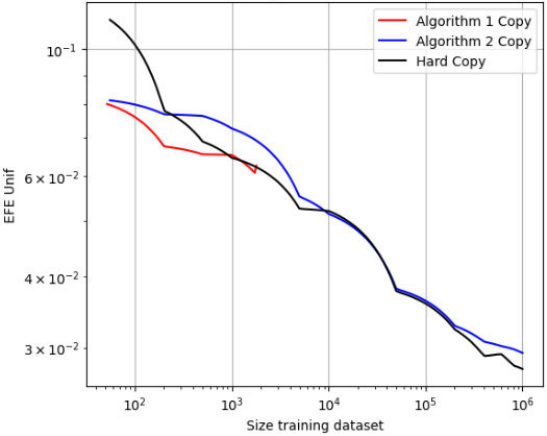}
		\caption{5-GB/LNN-$R^{\mathcal{F}}_{emp}$}
	\end{subfigure}
	\hfill
	\begin{subfigure}[t]{0.15\textwidth}
		\includegraphics[width=\linewidth]{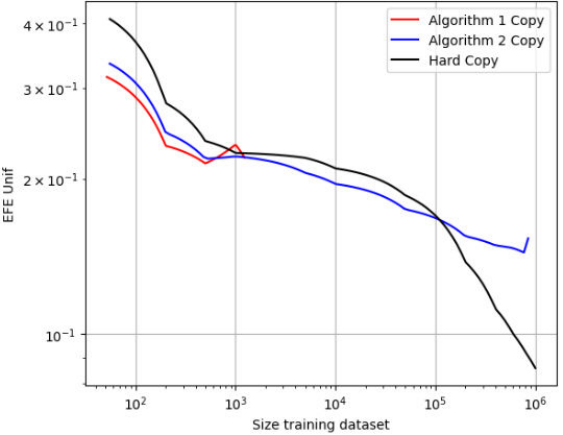}
		\caption{6-GB/LNN-$R^{\mathcal{F}}_{emp}$}
	\end{subfigure}
	
	\vspace{0.5mm}
	
	\begin{subfigure}[t]{0.15\textwidth}
		\includegraphics[width=\linewidth]{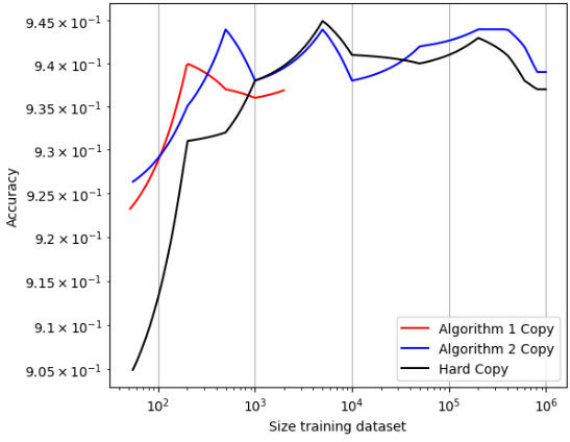}
		\caption{1-GB/LNN-$\mathcal{A}_{\mathcal{C}}$}
	\end{subfigure}
	\hfill
	\begin{subfigure}[t]{0.15\textwidth}
		\includegraphics[width=\linewidth]{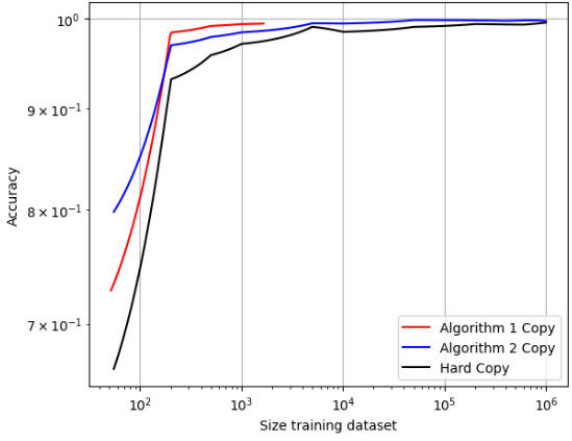}
		\caption{2-GB/LNN-$\mathcal{A}_{\mathcal{C}}$}
	\end{subfigure}
	\hfill
	\begin{subfigure}[t]{0.15\textwidth}
		\includegraphics[width=\linewidth]{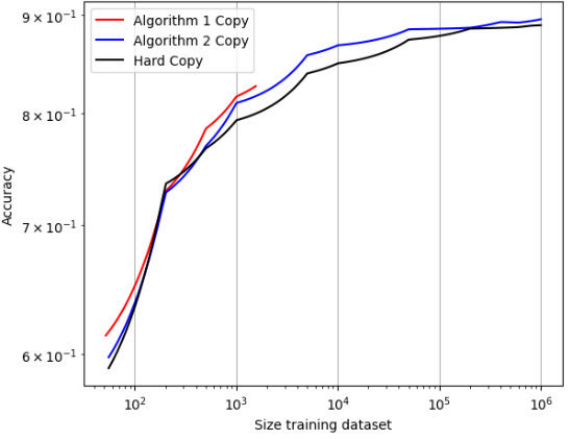}
		\caption{3-GB/LNN-$\mathcal{A}_{\mathcal{C}}$}
	\end{subfigure}
	\hfill
	\begin{subfigure}[t]{0.15\textwidth}
		\includegraphics[width=\linewidth]{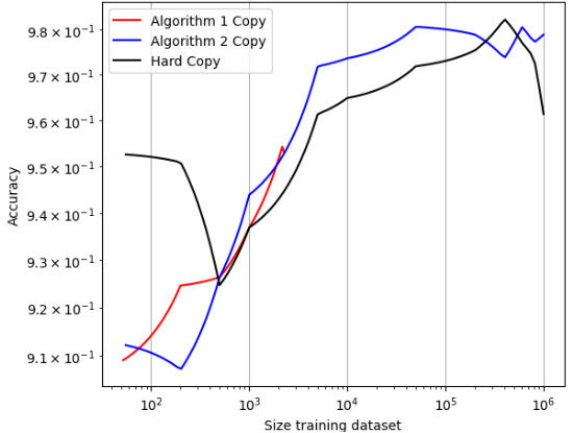}
		\caption{4-GB/LNN-$\mathcal{A}_{\mathcal{C}}$}
	\end{subfigure}
	\hfill
	\begin{subfigure}[t]{0.15\textwidth}
		\includegraphics[width=\linewidth]{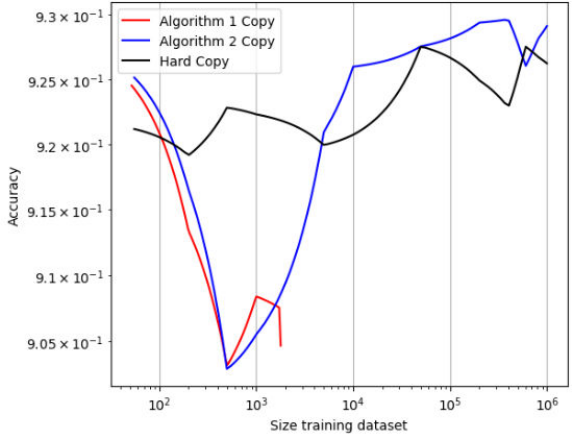}
		\caption{5-GB/LNN-$\mathcal{A}_{\mathcal{C}}$}
	\end{subfigure}
	\hfill
	\begin{subfigure}[t]{0.15\textwidth}
		\includegraphics[width=\linewidth]{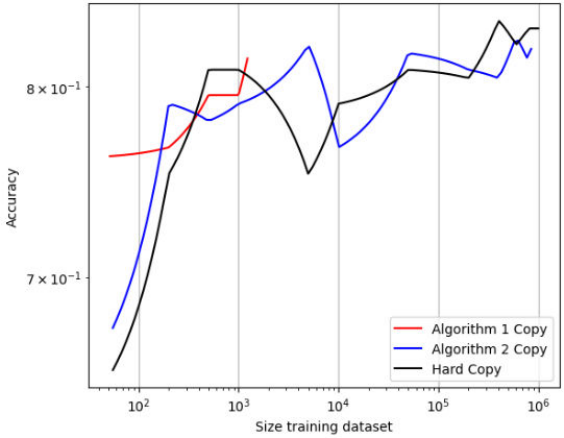}
		\caption{6-GB/LNN-$\mathcal{A}_{\mathcal{C}}$}
	\end{subfigure}
	
	\vspace{5mm}
	
	\begin{subfigure}[t]{0.15\textwidth}
		\includegraphics[width=\linewidth]{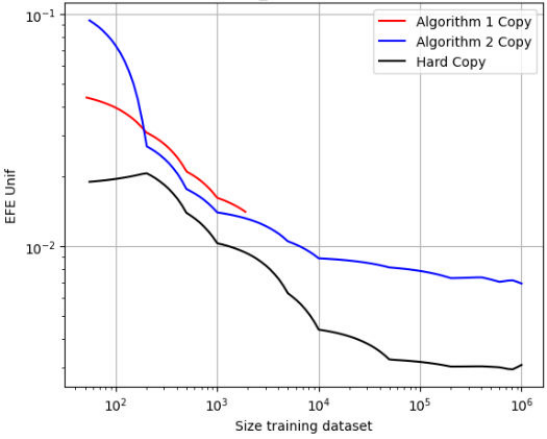}
		\caption{1-GB/GB-$R^{\mathcal{F}}_{emp}$}
	\end{subfigure}
	\hfill
	\begin{subfigure}[t]{0.15\textwidth}
		\includegraphics[width=\linewidth]{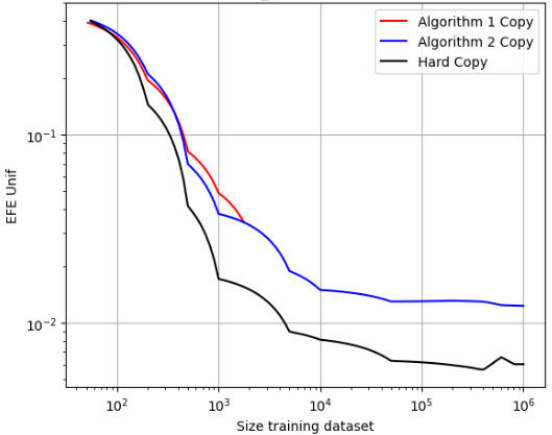}
		\caption{2-GB/GB-$R^{\mathcal{F}}_{emp}$}
	\end{subfigure}
	\hfill
	\begin{subfigure}[t]{0.15\textwidth}
		\includegraphics[width=\linewidth]{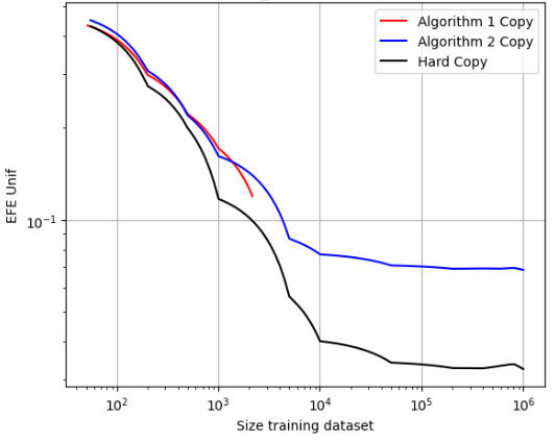}
		\caption{3-GB/GB-$R^{\mathcal{F}}_{emp}$}
	\end{subfigure}
	\hfill
	\begin{subfigure}[t]{0.15\textwidth}
		\includegraphics[width=\linewidth]{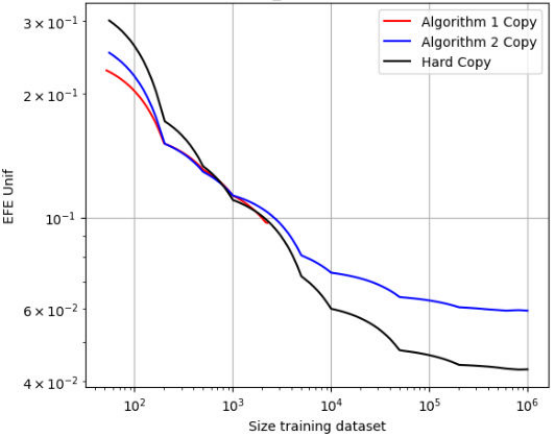}
		\caption{4-GB/GB-$R^{\mathcal{F}}_{emp}$}
	\end{subfigure}
	\hfill
	\begin{subfigure}[t]{0.15\textwidth}
		\includegraphics[width=\linewidth]{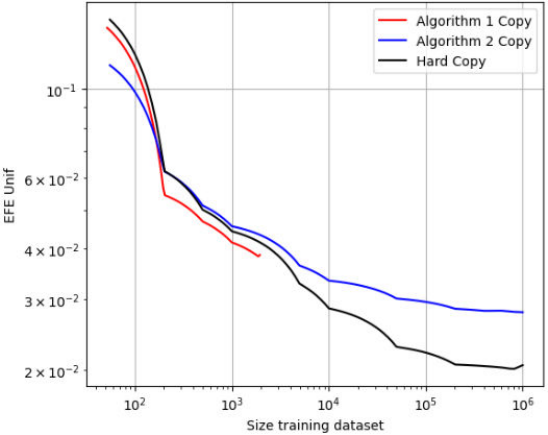}
		\caption{5-GB/GB-$R^{\mathcal{F}}_{emp}$}
	\end{subfigure}
	\hfill
	\begin{subfigure}[t]{0.15\textwidth}
		\includegraphics[width=\linewidth]{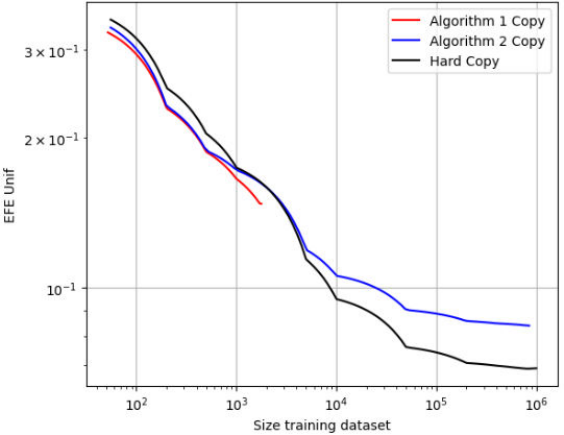}
		\caption{6-GB/GB-$R^{\mathcal{F}}_{emp}$}
	\end{subfigure}
	
	\vspace{0.5mm}
	
	\begin{subfigure}[t]{0.15\textwidth}
		\includegraphics[width=\linewidth]{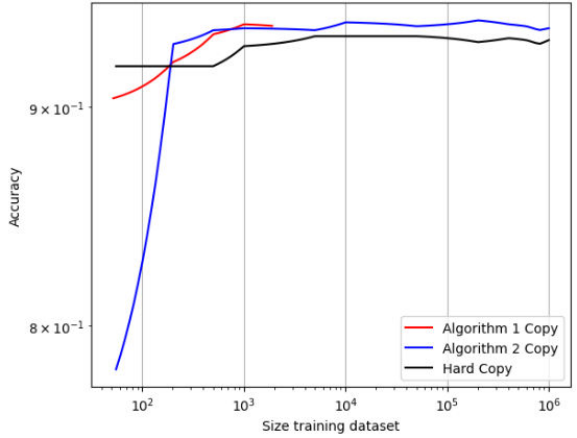}
		\caption{1-GB/GB-$\mathcal{A}_{\mathcal{C}}$}
	\end{subfigure}
	\hfill
	\begin{subfigure}[t]{0.15\textwidth}
		\includegraphics[width=\linewidth]{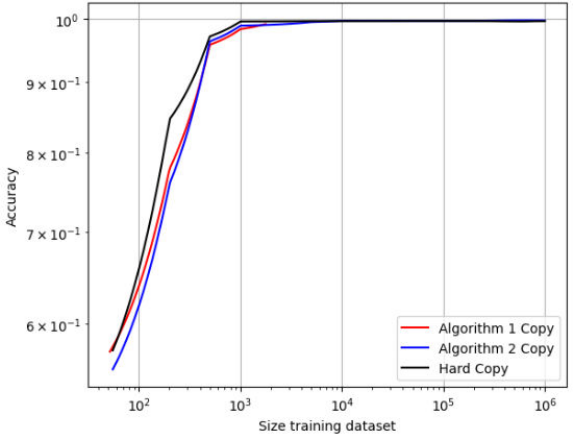}
		\caption{2-GB/GB-$\mathcal{A}_{\mathcal{C}}$}
	\end{subfigure}
	\hfill
	\begin{subfigure}[t]{0.15\textwidth}
		\includegraphics[width=\linewidth]{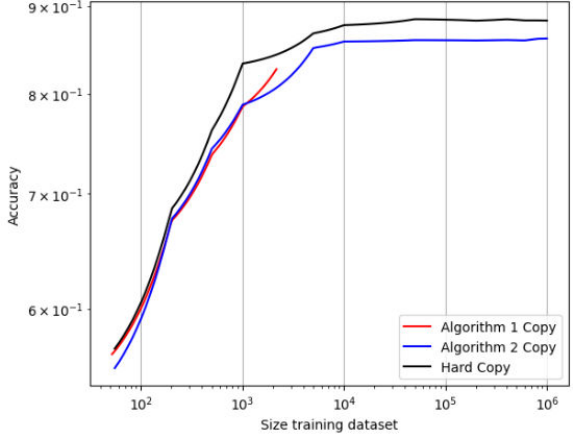}
		\caption{3-GB/GB-$\mathcal{A}_{\mathcal{C}}$}
	\end{subfigure}
	\hfill
	\begin{subfigure}[t]{0.15\textwidth}
		\includegraphics[width=\linewidth]{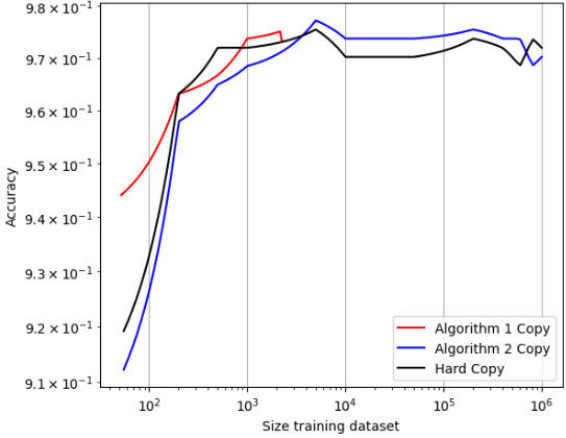}
		\caption{4-GB/GB-$\mathcal{A}_{\mathcal{C}}$}
	\end{subfigure}
	\hfill
	\begin{subfigure}[t]{0.15\textwidth}
		\includegraphics[width=\linewidth]{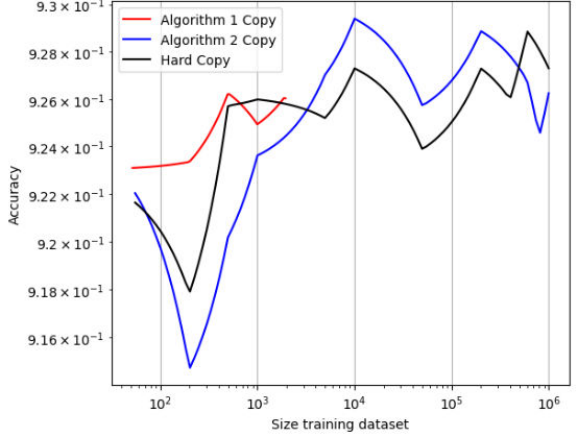}
		\caption{5-GB/GB-$\mathcal{A}_{\mathcal{C}}$}
	\end{subfigure}
	\hfill
	\begin{subfigure}[t]{0.15\textwidth}
		\includegraphics[width=\linewidth]{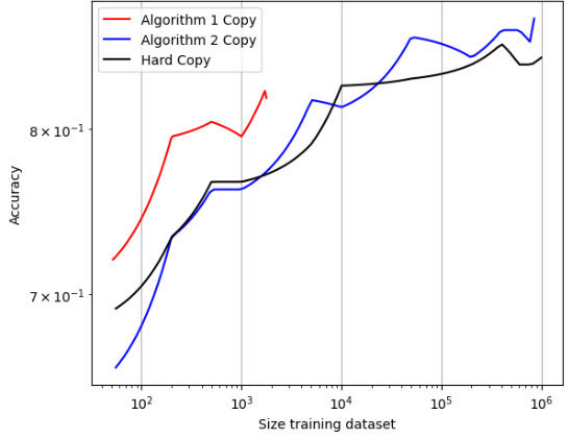}
		\caption{6-GB/GB-$\mathcal{A}_{\mathcal{C}}$}
	\end{subfigure}
\end{figure}


\begin{figure}[!ht]
	\centering
	\begin{subfigure}[t]{0.15\textwidth}
		\includegraphics[width=\linewidth]{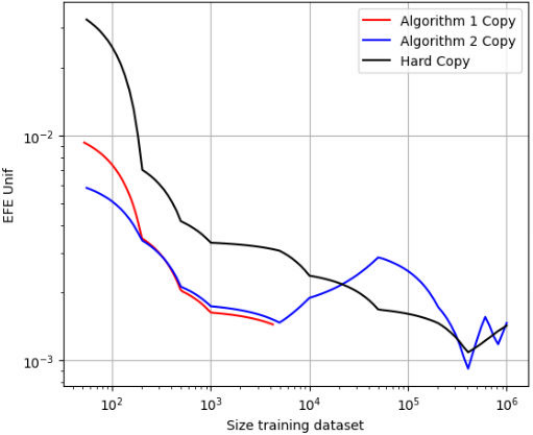}
		\caption{1-NN/SNN-$R^{\mathcal{F}}_{emp}$}
	\end{subfigure}
	\hfill
	\begin{subfigure}[t]{0.15\textwidth}
		\includegraphics[width=\linewidth]{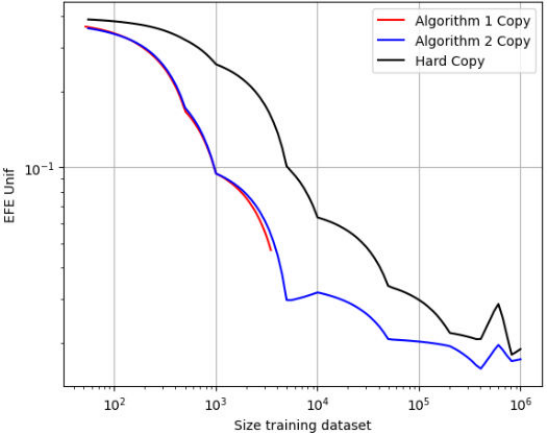}
		\caption{2-NN/SNN-$R^{\mathcal{F}}_{emp}$}
	\end{subfigure}
	\hfill
	\begin{subfigure}[t]{0.15\textwidth}
		\includegraphics[width=\linewidth]{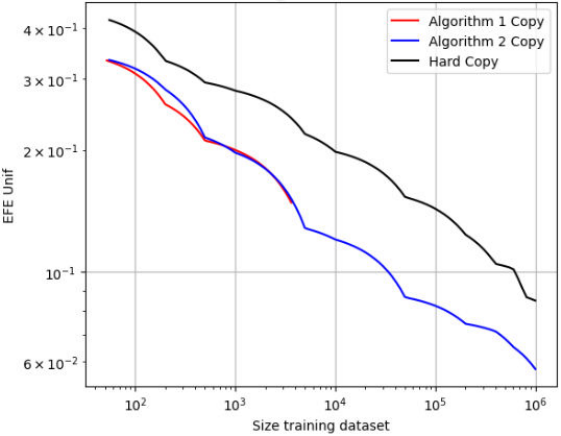}
		\caption{3-NN/SNN-$R^{\mathcal{F}}_{emp}$}
	\end{subfigure}
	\hfill
	\begin{subfigure}[t]{0.15\textwidth}
		\includegraphics[width=\linewidth]{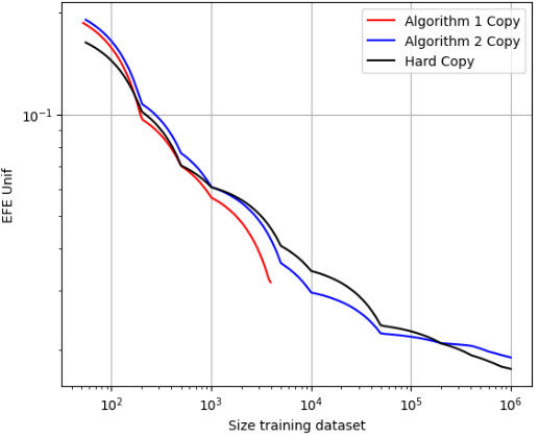}
		\caption{4-NN/SNN-$R^{\mathcal{F}}_{emp}$}
	\end{subfigure}
	\hfill
	\begin{subfigure}[t]{0.15\textwidth}
		\includegraphics[width=\linewidth]{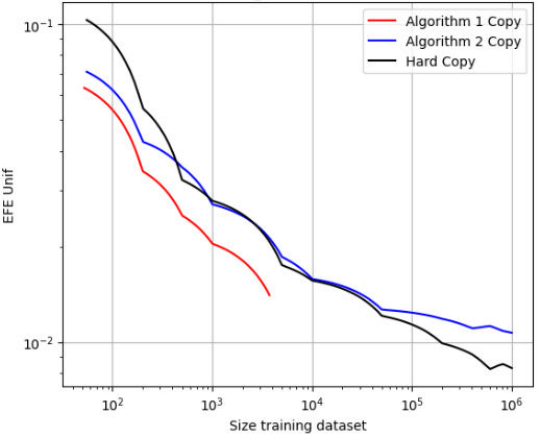}
		\caption{5-NN/SNN-$R^{\mathcal{F}}_{emp}$}
	\end{subfigure}
	\hfill
	\begin{subfigure}[t]{0.15\textwidth}
		\includegraphics[width=\linewidth]{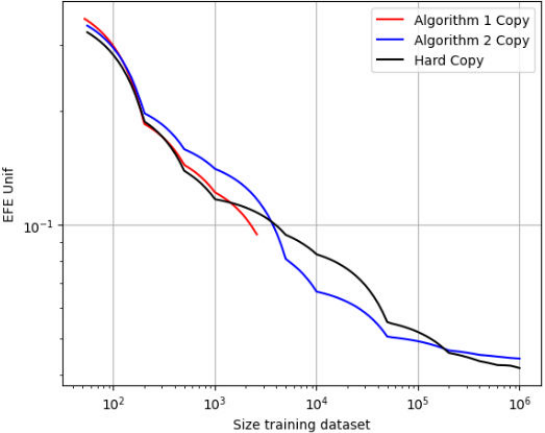}
		\caption{6-NN/SNN-$R^{\mathcal{F}}_{emp}$}
	\end{subfigure}
	
	\vspace{0.5mm}
	
	\begin{subfigure}[t]{0.15\textwidth}
		\includegraphics[width=\linewidth]{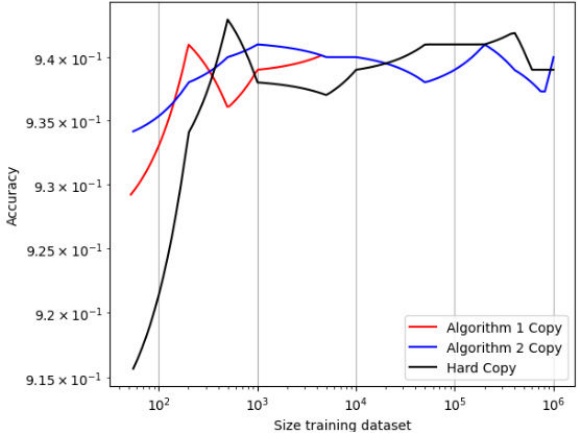}
		\caption{1-NN/SNN-$\mathcal{A}_{\mathcal{C}}$}
	\end{subfigure}
	\hfill
	\begin{subfigure}[t]{0.15\textwidth}
		\includegraphics[width=\linewidth]{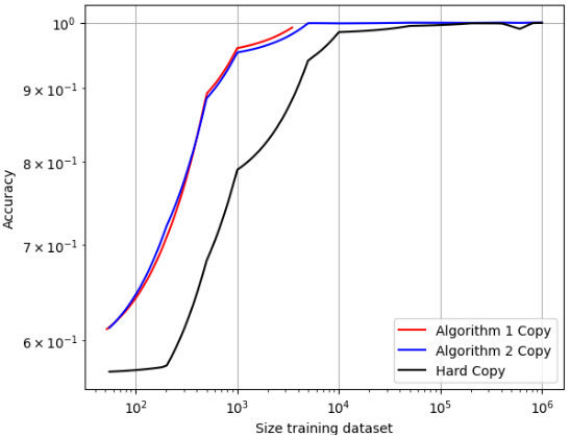}
		\caption{2-NN/SNN-$\mathcal{A}_{\mathcal{C}}$}
	\end{subfigure}
	\hfill
	\begin{subfigure}[t]{0.15\textwidth}
		\includegraphics[width=\linewidth]{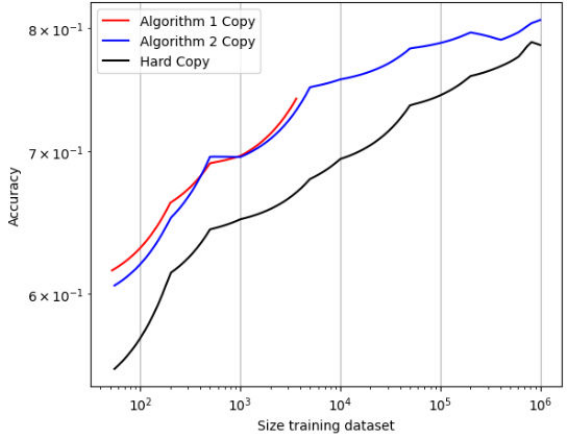}
		\caption{3-NN/SNN-$\mathcal{A}_{\mathcal{C}}$}
	\end{subfigure}
	\hfill
	\begin{subfigure}[t]{0.15\textwidth}
		\includegraphics[width=\linewidth]{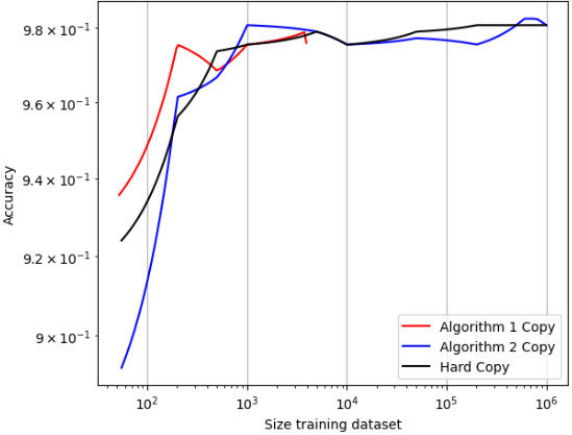}
		\caption{4-NN/SNN-$\mathcal{A}_{\mathcal{C}}$}
	\end{subfigure}
	\hfill
	\begin{subfigure}[t]{0.15\textwidth}
		\includegraphics[width=\linewidth]{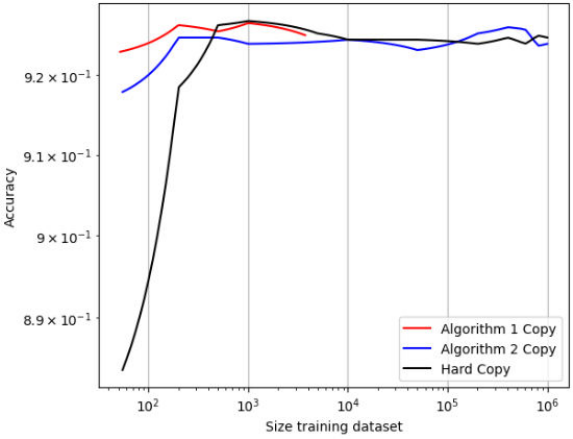}
		\caption{5-NN/SNN-$\mathcal{A}_{\mathcal{C}}$}
	\end{subfigure}
	\hfill
	\begin{subfigure}[t]{0.15\textwidth}
		\includegraphics[width=\linewidth]{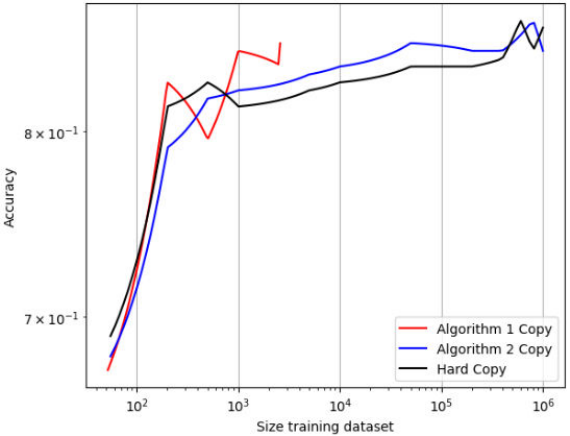}
		\caption{6-NN/SNN-$\mathcal{A}_{\mathcal{C}}$}
	\end{subfigure}
	
	\vspace{5mm}
	
	\begin{subfigure}[t]{0.15\textwidth}
		\includegraphics[width=\linewidth]{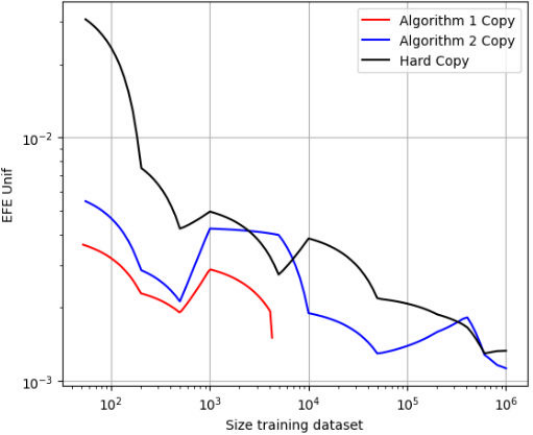}
		\caption{1-NN/MNN-$R^{\mathcal{F}}_{emp}$}
	\end{subfigure}
	\hfill
	\begin{subfigure}[t]{0.15\textwidth}
		\includegraphics[width=\linewidth]{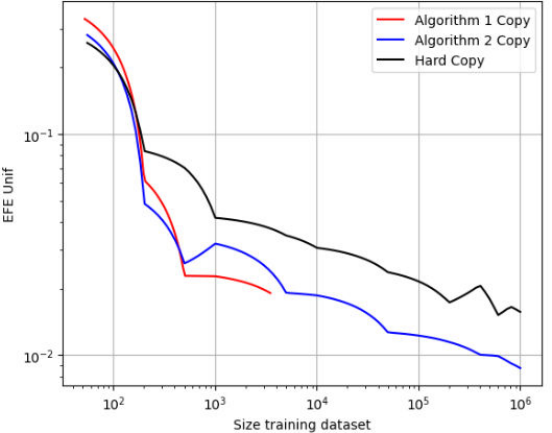}
		\caption{2-NN/MNN-$R^{\mathcal{F}}_{emp}$}
	\end{subfigure}
	\hfill
	\begin{subfigure}[t]{0.15\textwidth}
		\includegraphics[width=\linewidth]{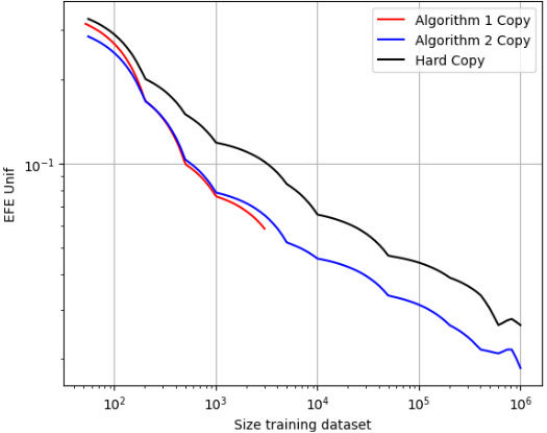}
		\caption{3-NN/MNN-$R^{\mathcal{F}}_{emp}$}
	\end{subfigure}
	\hfill
	\begin{subfigure}[t]{0.15\textwidth}
		\includegraphics[width=\linewidth]{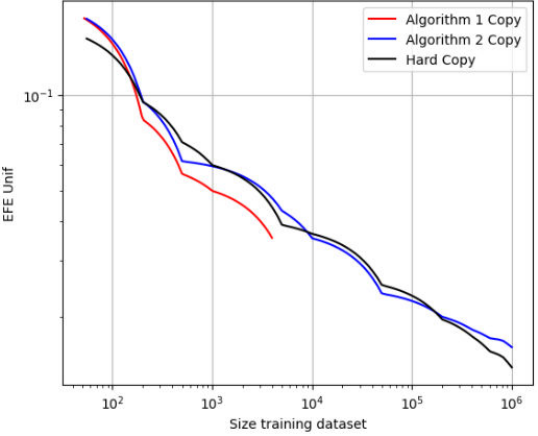}
		\caption{4-NN/MNN-$R^{\mathcal{F}}_{emp}$}
	\end{subfigure}
	\hfill
	\begin{subfigure}[t]{0.15\textwidth}
		\includegraphics[width=\linewidth]{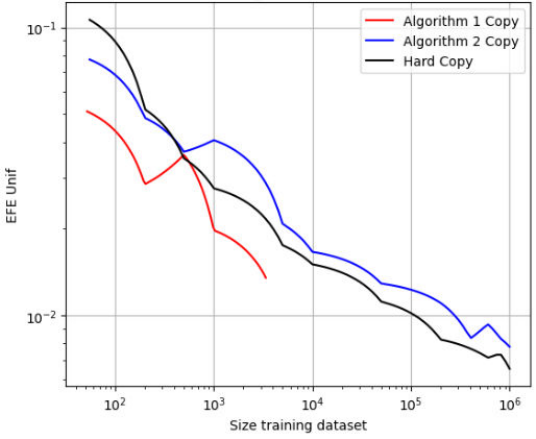}
		\caption{5-NN/MNN-$R^{\mathcal{F}}_{emp}$}
	\end{subfigure}
	\hfill
	\begin{subfigure}[t]{0.15\textwidth}
		\includegraphics[width=\linewidth]{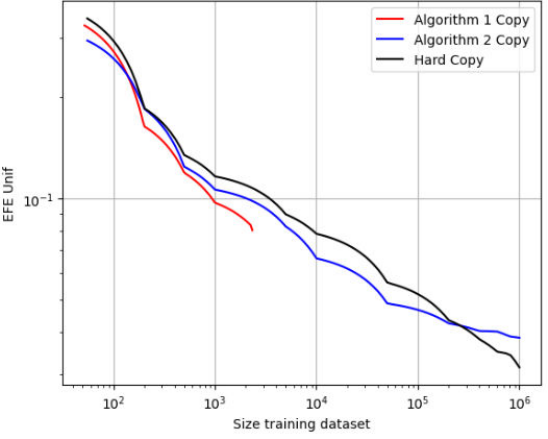}
		\caption{6-NN/MNN-$R^{\mathcal{F}}_{emp}$}
	\end{subfigure}
	
	\vspace{0.5mm}
	
	\begin{subfigure}[t]{0.15\textwidth}
		\includegraphics[width=\linewidth]{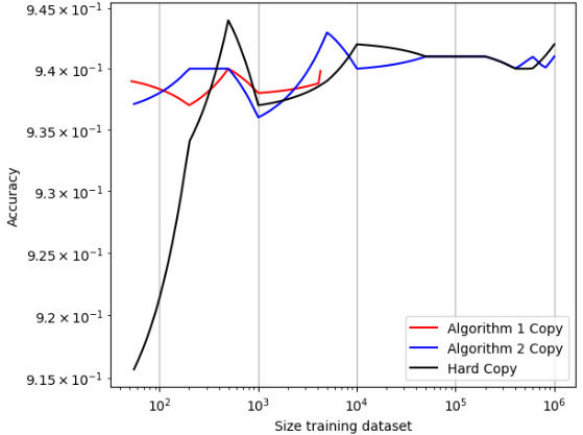}
		\caption{1-NN/MNN-$\mathcal{A}_{\mathcal{C}}$}
	\end{subfigure}
	\hfill
	\begin{subfigure}[t]{0.15\textwidth}
		\includegraphics[width=\linewidth]{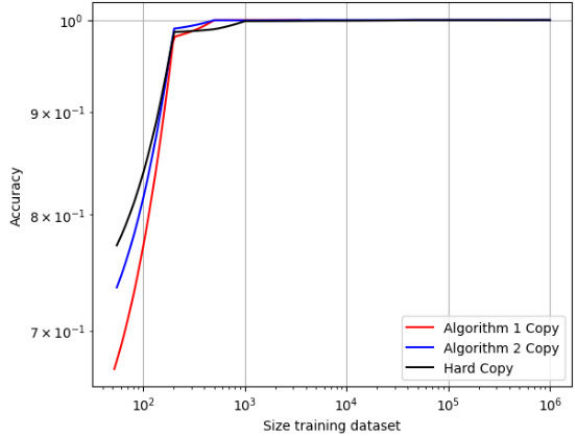}
		\caption{2-NN/MNN-$\mathcal{A}_{\mathcal{C}}$}
	\end{subfigure}
	\hfill
	\begin{subfigure}[t]{0.15\textwidth}
		\includegraphics[width=\linewidth]{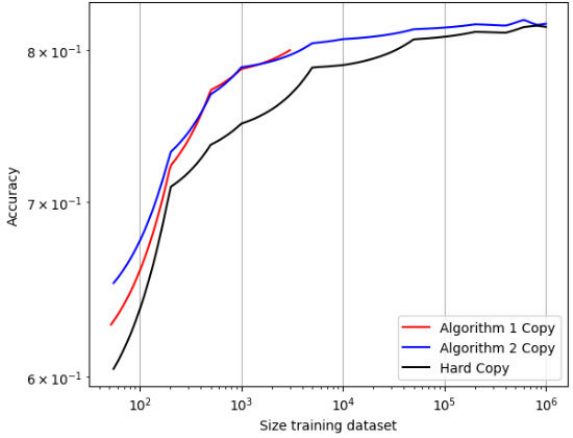}
		\caption{3-NN/MNN-$\mathcal{A}_{\mathcal{C}}$}
	\end{subfigure}
	\hfill
	\begin{subfigure}[t]{0.15\textwidth}
		\includegraphics[width=\linewidth]{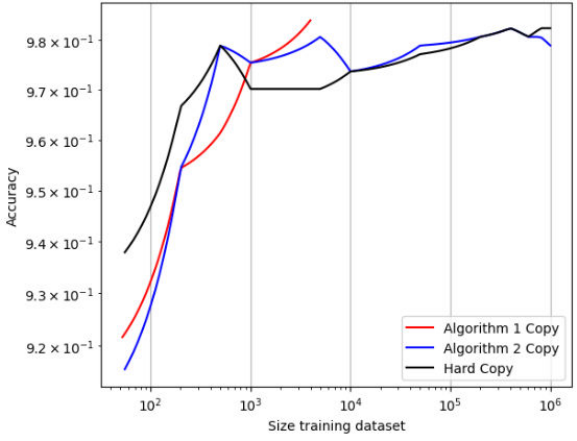}
		\caption{4-NN/MNN-$\mathcal{A}_{\mathcal{C}}$}
	\end{subfigure}
	\hfill
	\begin{subfigure}[t]{0.15\textwidth}
		\includegraphics[width=\linewidth]{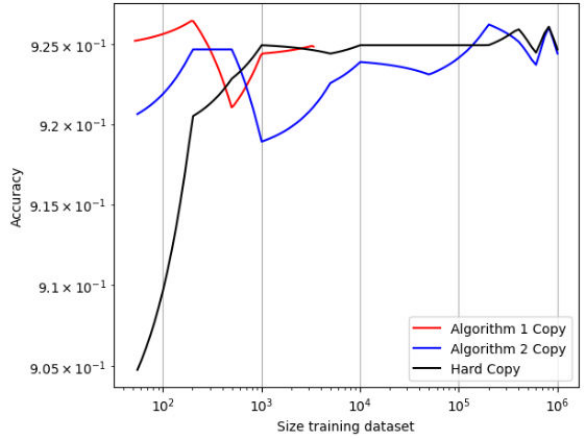}
		\caption{5-NN/MNN-$\mathcal{A}_{\mathcal{C}}$}
	\end{subfigure}
	\hfill
	\begin{subfigure}[t]{0.15\textwidth}
		\includegraphics[width=\linewidth]{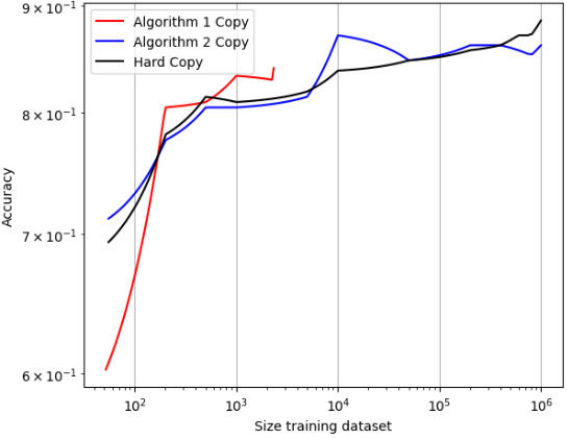}
		\caption{6-NN/MNN-$\mathcal{A}_{\mathcal{C}}$}
	\end{subfigure}
	
	\vspace{5mm} 
	
	\begin{subfigure}[t]{0.15\textwidth}
		\includegraphics[width=\linewidth]{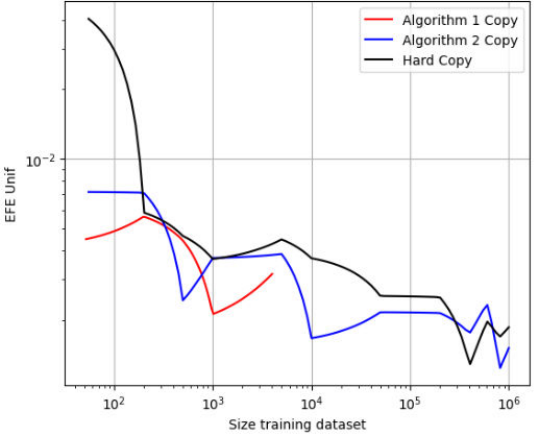}
		\caption{1-NN/LNN-$R^{\mathcal{F}}_{emp}$}
	\end{subfigure}
	\hfill
	\begin{subfigure}[t]{0.15\textwidth}
		\includegraphics[width=\linewidth]{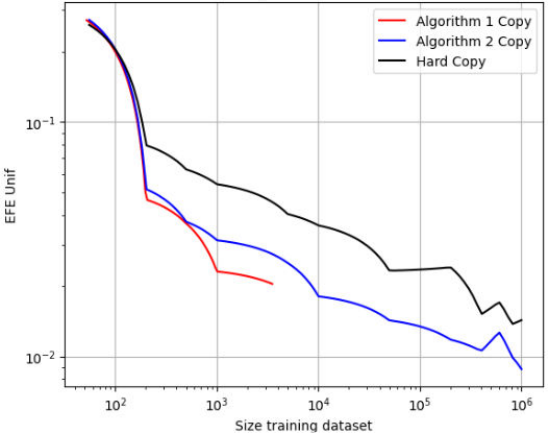}
		\caption{2-NN/LNN-$R^{\mathcal{F}}_{emp}$}
	\end{subfigure}
	\hfill
	\begin{subfigure}[t]{0.15\textwidth}
		\includegraphics[width=\linewidth]{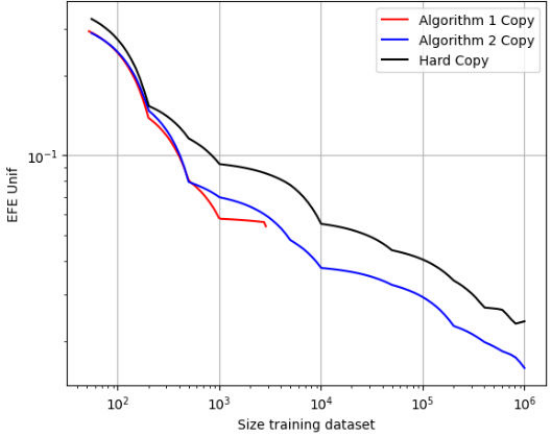}
		\caption{3-NN/LNN-$R^{\mathcal{F}}_{emp}$}
	\end{subfigure}
	\hfill
	\begin{subfigure}[t]{0.15\textwidth}
		\includegraphics[width=\linewidth]{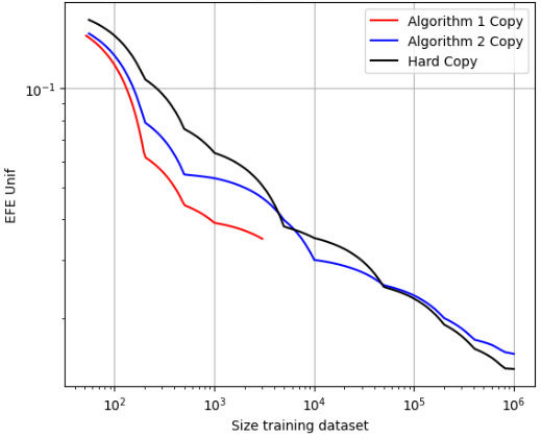}
		\caption{4-NN/LNN-$R^{\mathcal{F}}_{emp}$}
	\end{subfigure}
	\hfill
	\begin{subfigure}[t]{0.15\textwidth}
		\includegraphics[width=\linewidth]{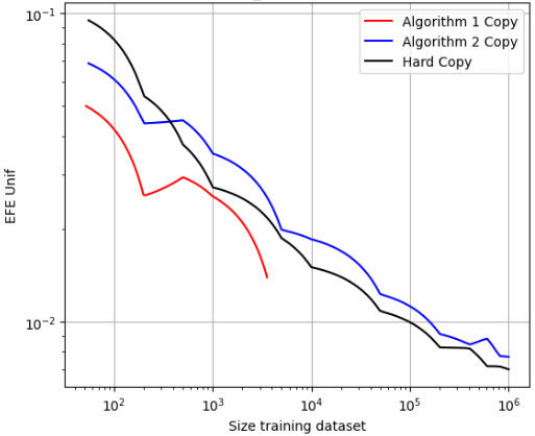}
		\caption{5-NN/LNN-$R^{\mathcal{F}}_{emp}$}
	\end{subfigure}
	\hfill
	\begin{subfigure}[t]{0.15\textwidth}
		\includegraphics[width=\linewidth]{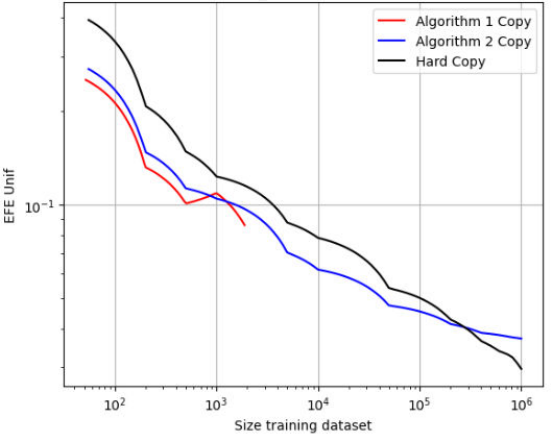}
		\caption{6-NN/LNN-$R^{\mathcal{F}}_{emp}$}
	\end{subfigure}
	
	\vspace{0.4mm}
	
	\begin{subfigure}[t]{0.15\textwidth}
		\includegraphics[width=\linewidth]{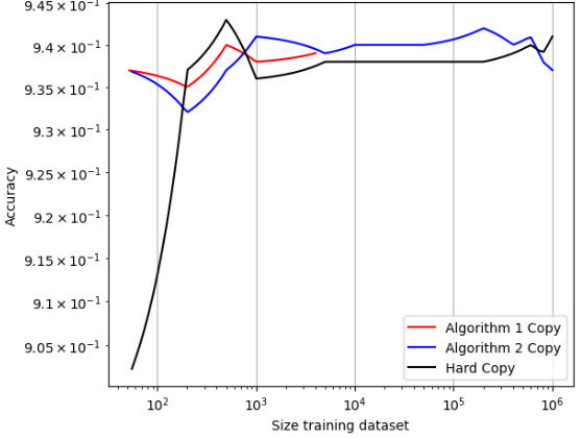}
		\caption{1-NN/LNN-$\mathcal{A}_{\mathcal{C}}$}
	\end{subfigure}
	\hfill
	\begin{subfigure}[t]{0.15\textwidth}
		\includegraphics[width=\linewidth]{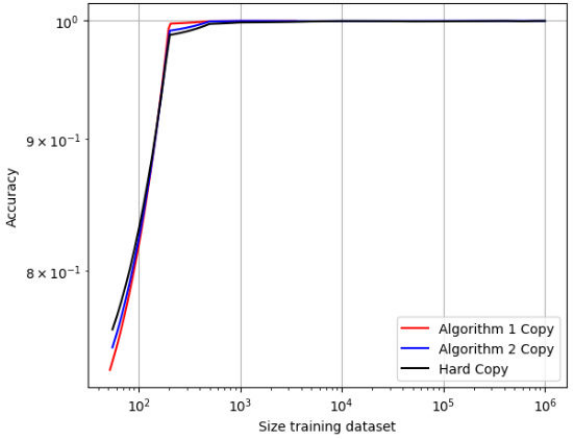}
		\caption{2-NN/LNN-$\mathcal{A}_{\mathcal{C}}$}
	\end{subfigure}
	\hfill
	\begin{subfigure}[t]{0.15\textwidth}
		\includegraphics[width=\linewidth]{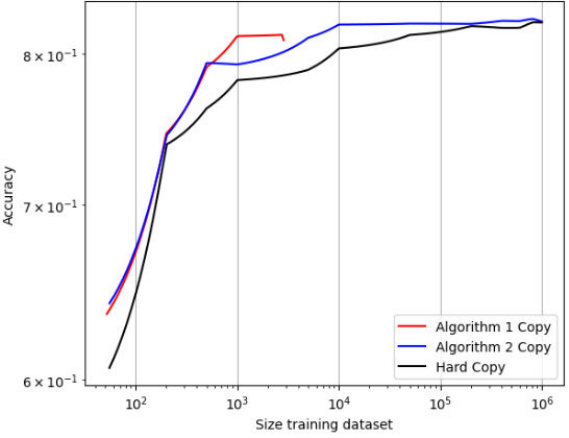}
		\caption{3-NN/LNN-$\mathcal{A}_{\mathcal{C}}$}
	\end{subfigure}
	\hfill
	\begin{subfigure}[t]{0.15\textwidth}
		\includegraphics[width=\linewidth]{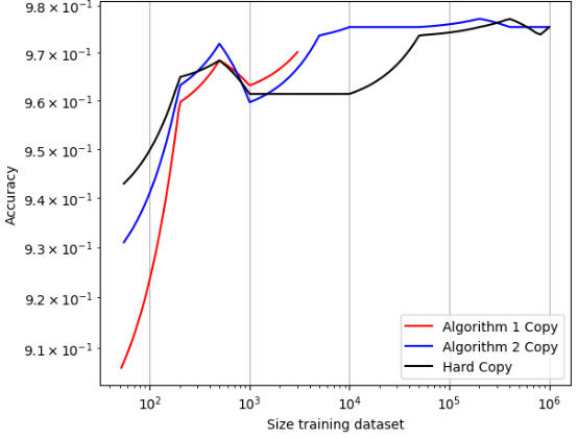}
		\caption{4-NN/LNN-$\mathcal{A}_{\mathcal{C}}$}
	\end{subfigure}
	\hfill
	\begin{subfigure}[t]{0.15\textwidth}
		\includegraphics[width=\linewidth]{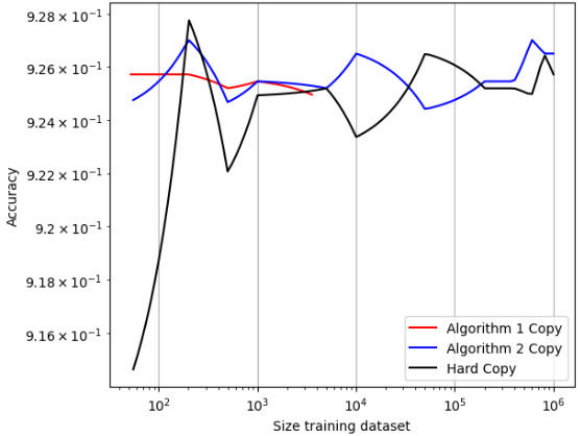}
		\caption{5-NN/LNN-$\mathcal{A}_{\mathcal{C}}$}
	\end{subfigure}
	\hfill
	\begin{subfigure}[t]{0.15\textwidth}
		\includegraphics[width=\linewidth]{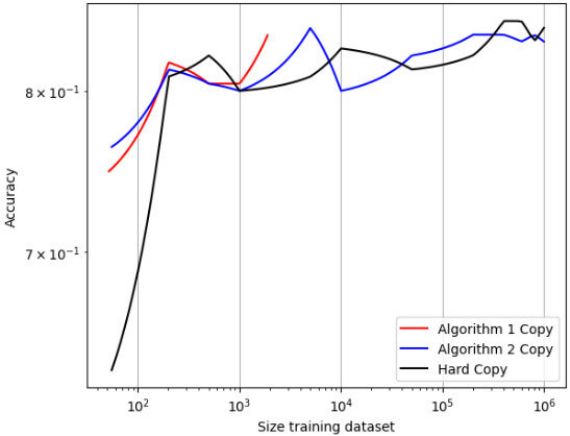}
		\caption{6-NN/LNN-$\mathcal{A}_{\mathcal{C}}$}
	\end{subfigure}
	
	\vspace{4.5mm}
	
	\begin{subfigure}[t]{0.15\textwidth}
		\includegraphics[width=\linewidth]{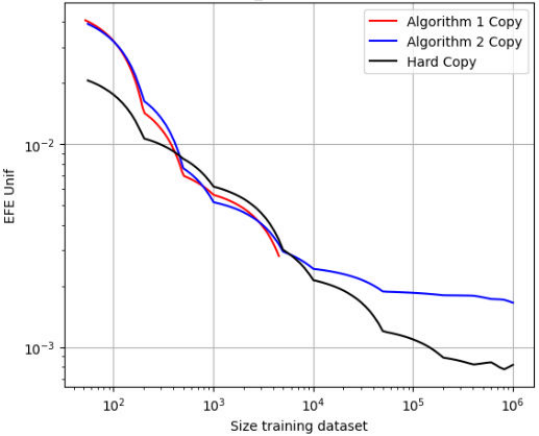}
		\caption{1-NN/GB-$R^{\mathcal{F}}_{emp}$}
	\end{subfigure}
	\hfill
	\begin{subfigure}[t]{0.15\textwidth}
		\includegraphics[width=\linewidth]{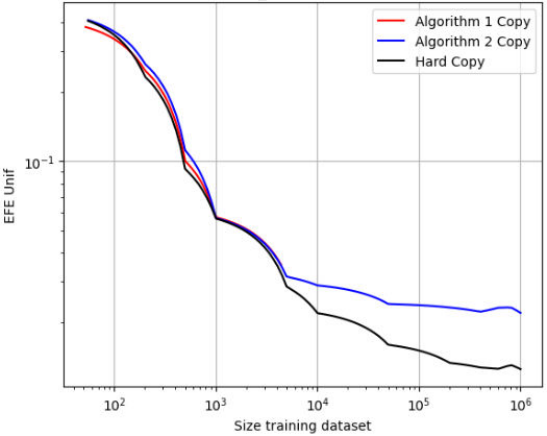}
		\caption{2-NN/GB-$R^{\mathcal{F}}_{emp}$}
	\end{subfigure}
	\hfill
	\begin{subfigure}[t]{0.15\textwidth}
		\includegraphics[width=\linewidth]{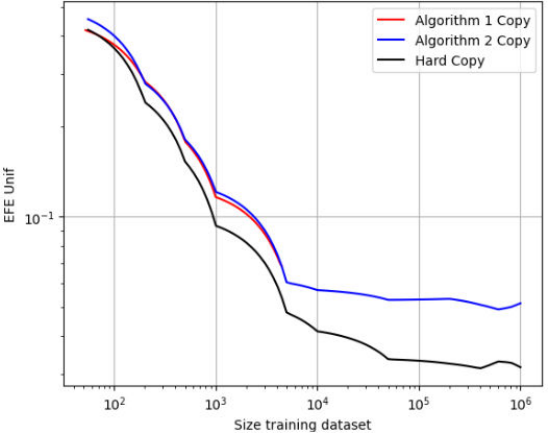}
		\caption{3-NN/GB-$R^{\mathcal{F}}_{emp}$}
	\end{subfigure}
	\hfill
	\begin{subfigure}[t]{0.15\textwidth}
		\includegraphics[width=\linewidth]{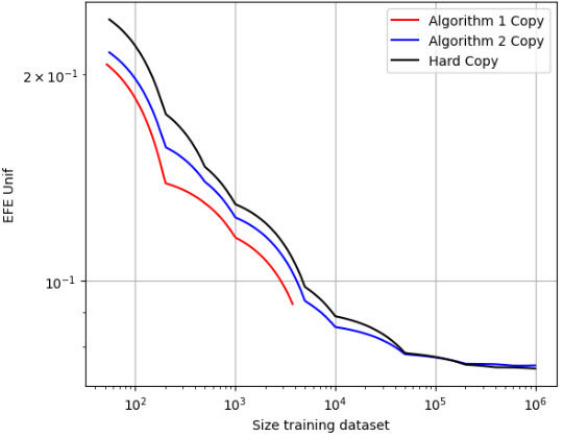}
		\caption{4-NN/GB-$R^{\mathcal{F}}_{emp}$}
	\end{subfigure}
	\hfill
	\begin{subfigure}[t]{0.15\textwidth}
		\includegraphics[width=\linewidth]{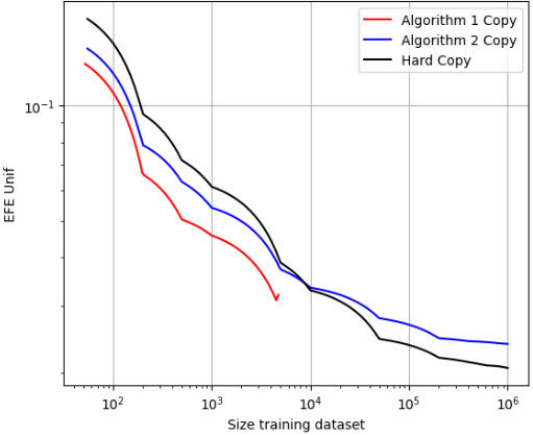}
		\caption{5-NN/GB-$R^{\mathcal{F}}_{emp}$}
	\end{subfigure}
	\hfill
	\begin{subfigure}[t]{0.15\textwidth}
		\includegraphics[width=\linewidth]{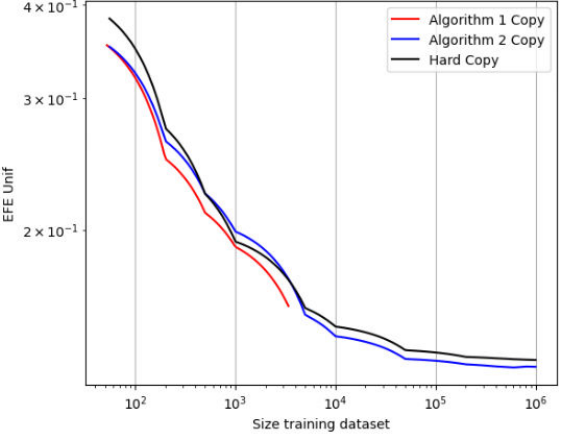}
		\caption{6-NN/GB-$R^{\mathcal{F}}_{emp}$}
	\end{subfigure}
	
	\vspace{0.4mm}
	
	\begin{subfigure}[t]{0.15\textwidth}
		\includegraphics[width=\linewidth]{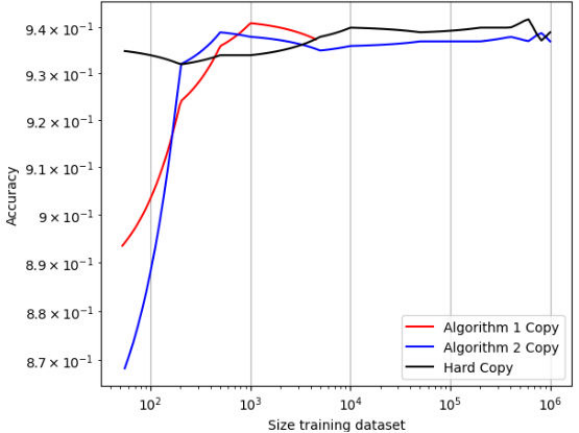}
		\caption{1-NN/GB-$\mathcal{A}_{\mathcal{C}}$}
	\end{subfigure}
	\hfill
	\begin{subfigure}[t]{0.15\textwidth}
		\includegraphics[width=\linewidth]{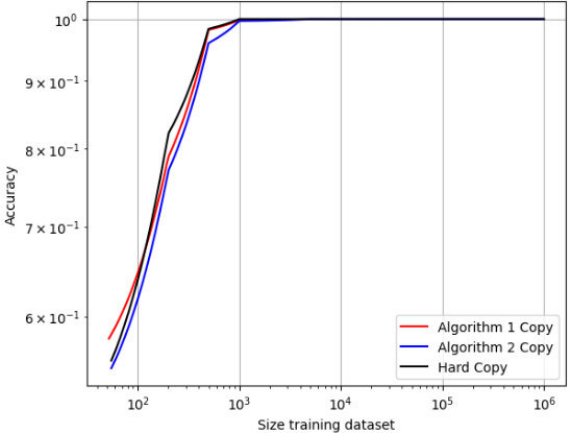}
		\caption{2-NN/GB-$\mathcal{A}_{\mathcal{C}}$}
	\end{subfigure}
	\hfill
	\begin{subfigure}[t]{0.15\textwidth}
		\includegraphics[width=\linewidth]{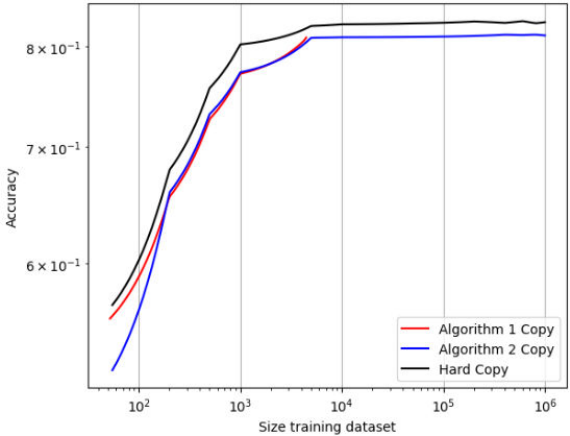}
		\caption{3-NN/GB-$\mathcal{A}_{\mathcal{C}}$}
	\end{subfigure}
	\hfill
	\begin{subfigure}[t]{0.15\textwidth}
		\includegraphics[width=\linewidth]{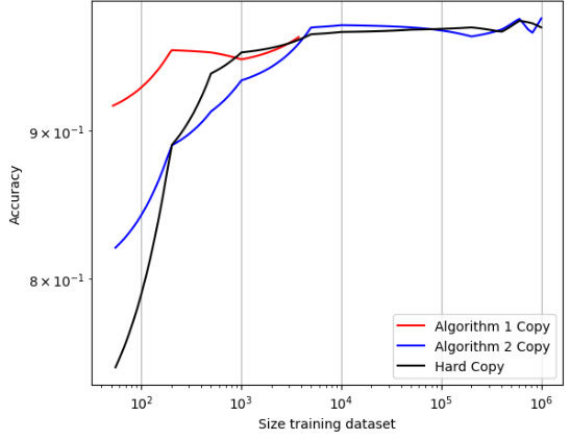}
		\caption{-NN/GB-$\mathcal{A}_{\mathcal{C}}$}
	\end{subfigure}
	\hfill
	\begin{subfigure}[t]{0.15\textwidth}
		\includegraphics[width=\linewidth]{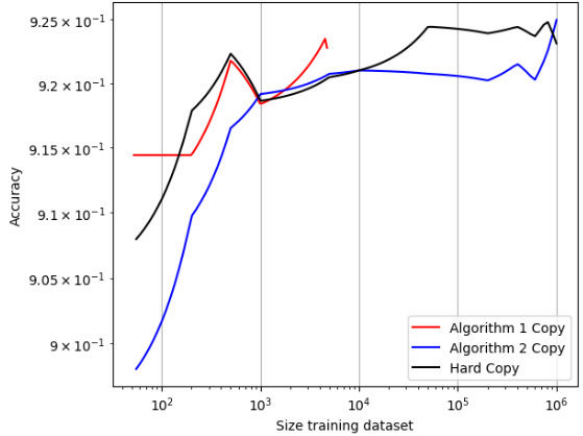}
		\caption{5-NN/GB-$\mathcal{A}_{\mathcal{C}}$}
	\end{subfigure}
	\hfill
	\begin{subfigure}[t]{0.15\textwidth}
		\includegraphics[width=\linewidth]{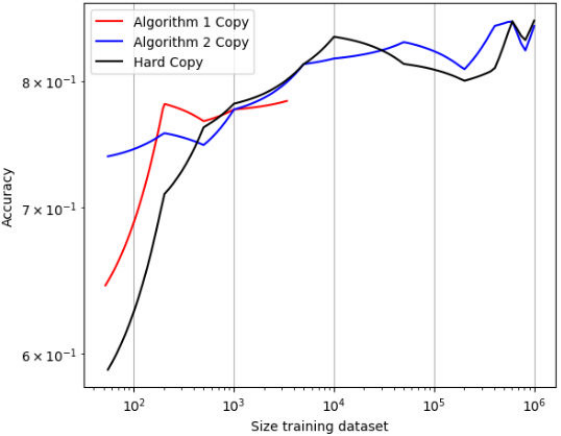}
		\caption{6-NN/GB-$\mathcal{A}_{\mathcal{C}}$}
	\end{subfigure}
\end{figure}

\clearpage
\newpage
\section{Evolution of Fidelity and Accuracy Across Different Values of $\alpha$}
\label{appExp2Plots}

\begin{figure}[!ht]
	\centering
	\begin{subfigure}[t]{0.14\textheight}
		\includegraphics[width=\linewidth]{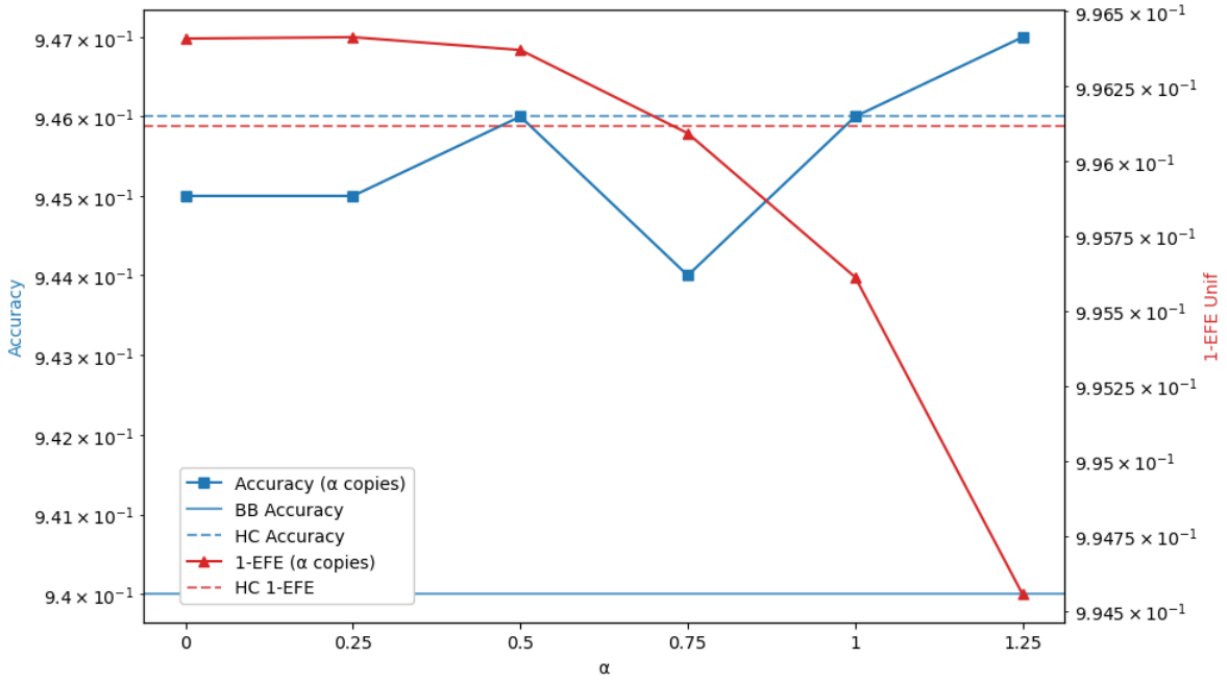}
		\caption{1-RF/SNN}
	\end{subfigure}
	\hfill
	\begin{subfigure}[t]{0.14\textheight}
		\includegraphics[width=\linewidth]{Figures/Figure_6/1-RF-MNN.pdf}
		\caption{1-RF/MNN}
	\end{subfigure}
	\hfill
	\begin{subfigure}[t]{0.14\textheight}
		\includegraphics[width=\linewidth]{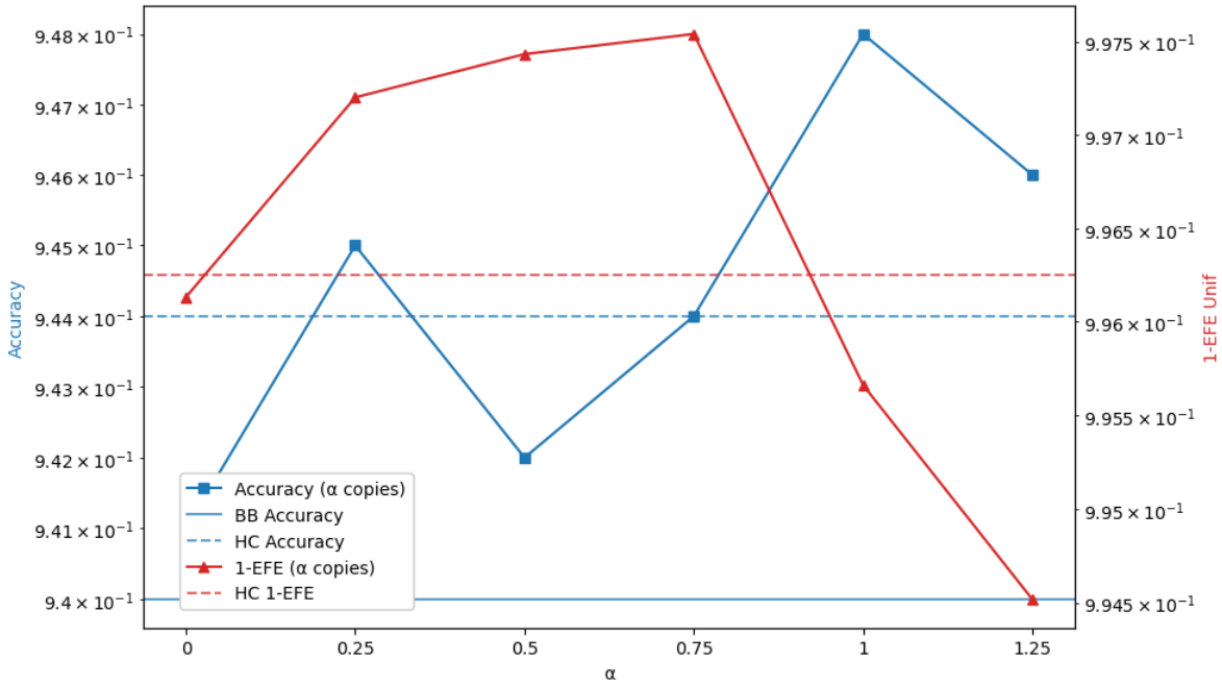}
		\caption{1-RF/LNN}
	\end{subfigure}
	\hfill
	\begin{subfigure}[t]{0.14\textheight}
		\includegraphics[width=\linewidth]{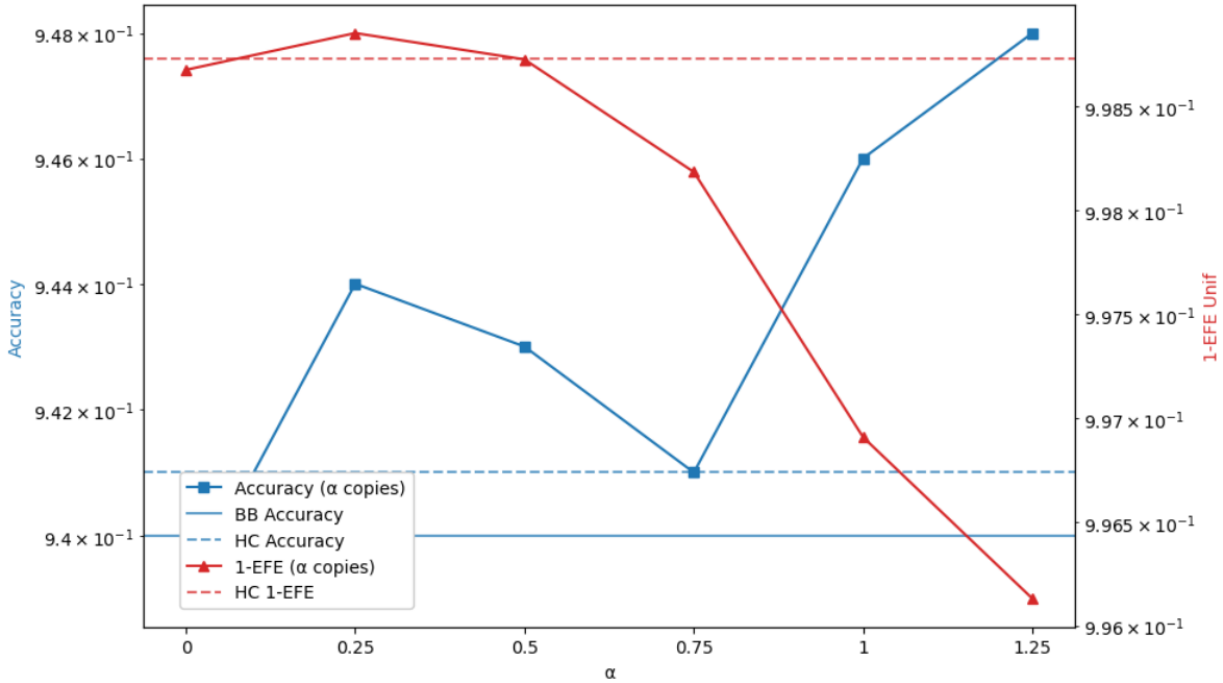}
		\caption{1-RF/GB}
	\end{subfigure}

	\begin{subfigure}[t]{0.14\textheight}
		\includegraphics[width=\linewidth]{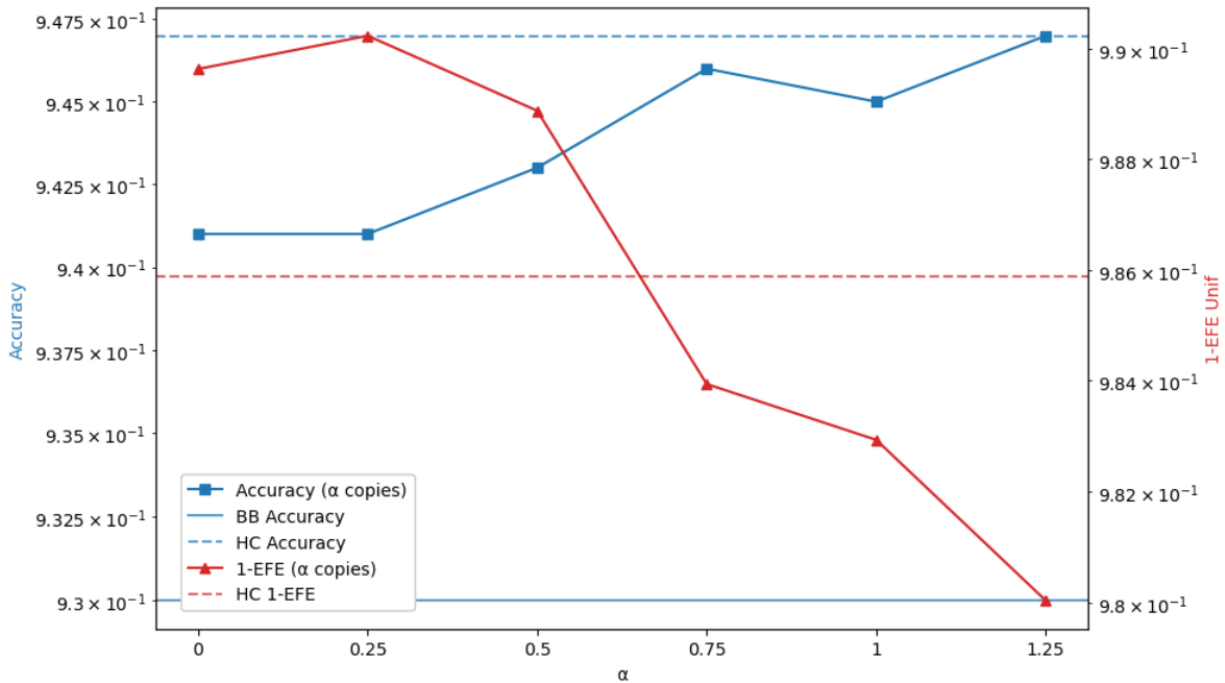}
		\caption{1-GB/SNN}
	\end{subfigure}
	\hfill
	\begin{subfigure}[t]{0.14\textheight}
		\includegraphics[width=\linewidth]{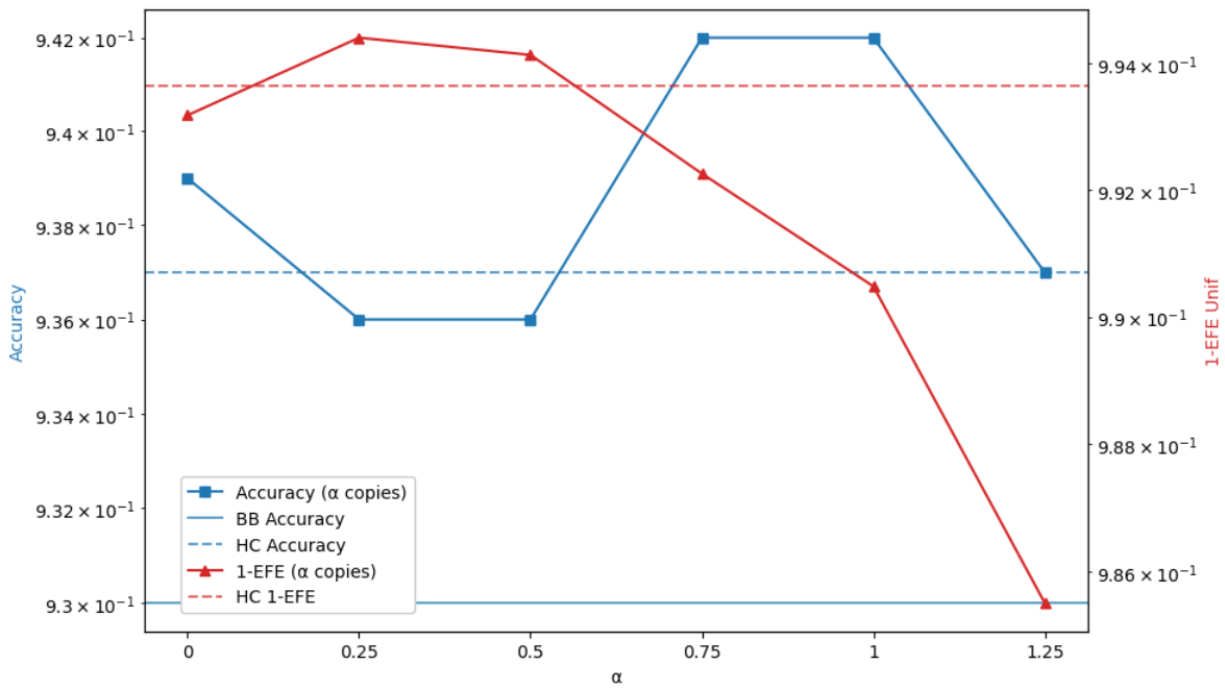}
		\caption{1-GB/MNN}
	\end{subfigure}
	\hfill
	\begin{subfigure}[t]{0.14\textheight}
		\includegraphics[width=\linewidth]{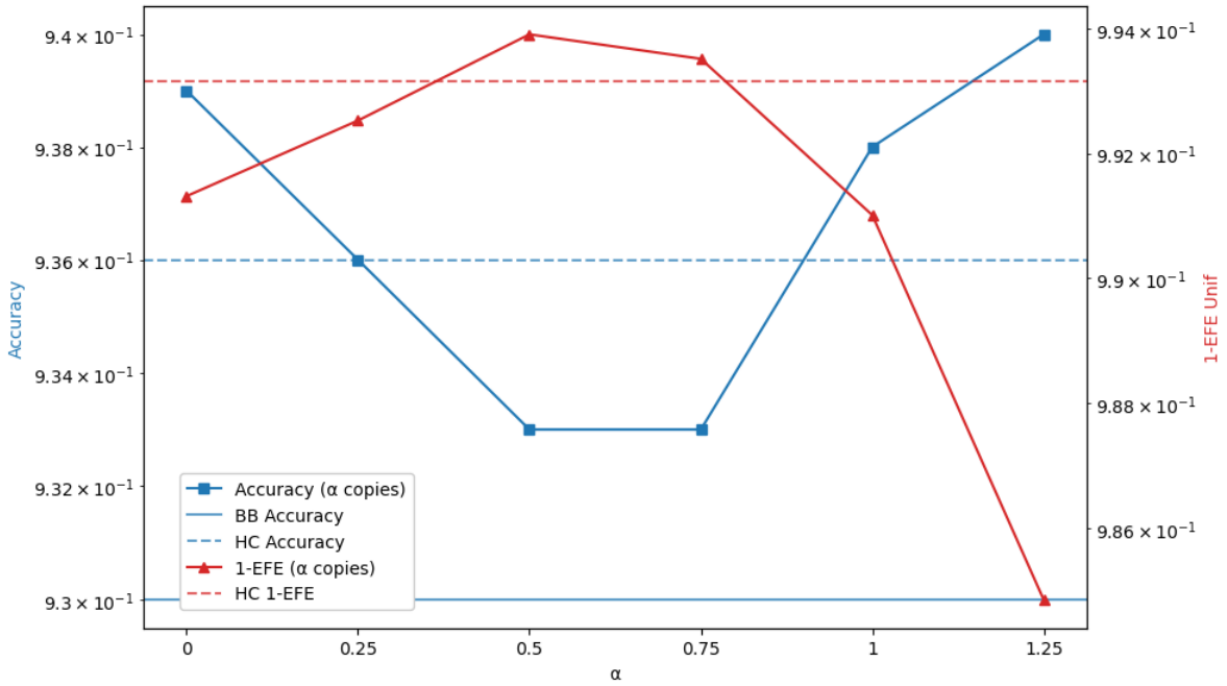}
		\caption{1-GB/LNN}
	\end{subfigure}
	\hfill
	\begin{subfigure}[t]{0.14\textheight}
		\includegraphics[width=\linewidth]{Figures/Figure_6/1-GB-GB.pdf}
		\caption{1-GB/GB}
	\end{subfigure}

	\begin{subfigure}[t]{0.14\textheight}
		\includegraphics[width=\linewidth]{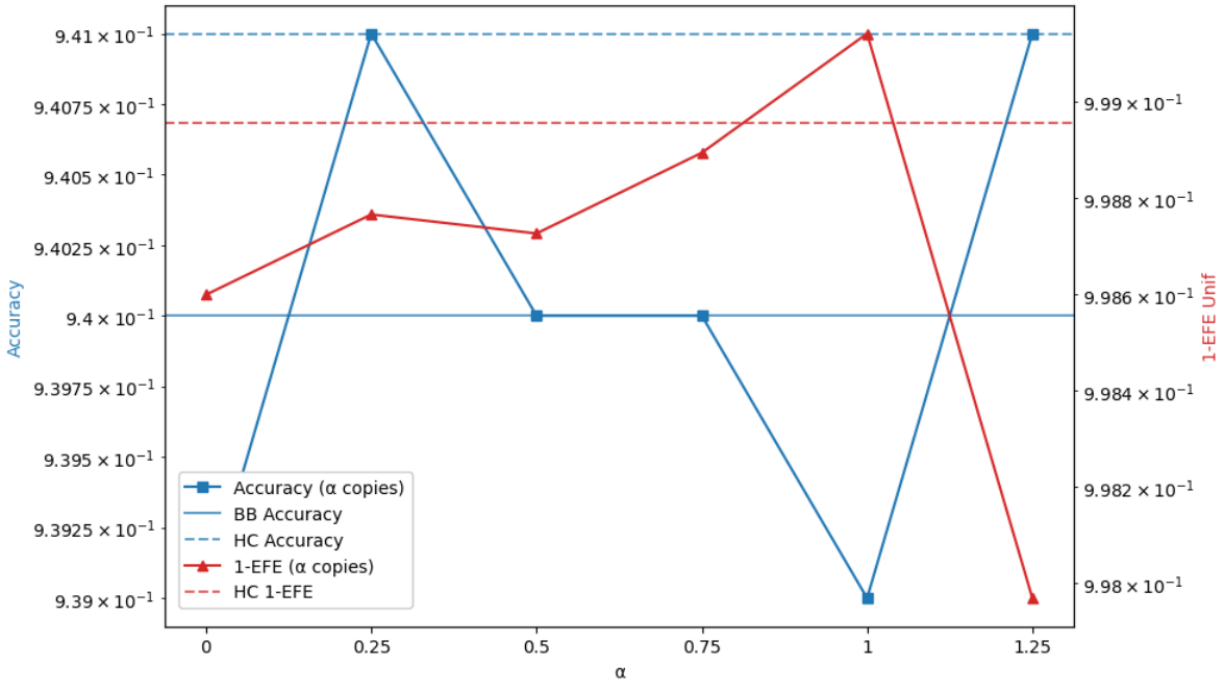}
		\caption{1-NN/SNN}
	\end{subfigure}
	\hfill
	\begin{subfigure}[t]{0.14\textheight}
		\includegraphics[width=\linewidth]{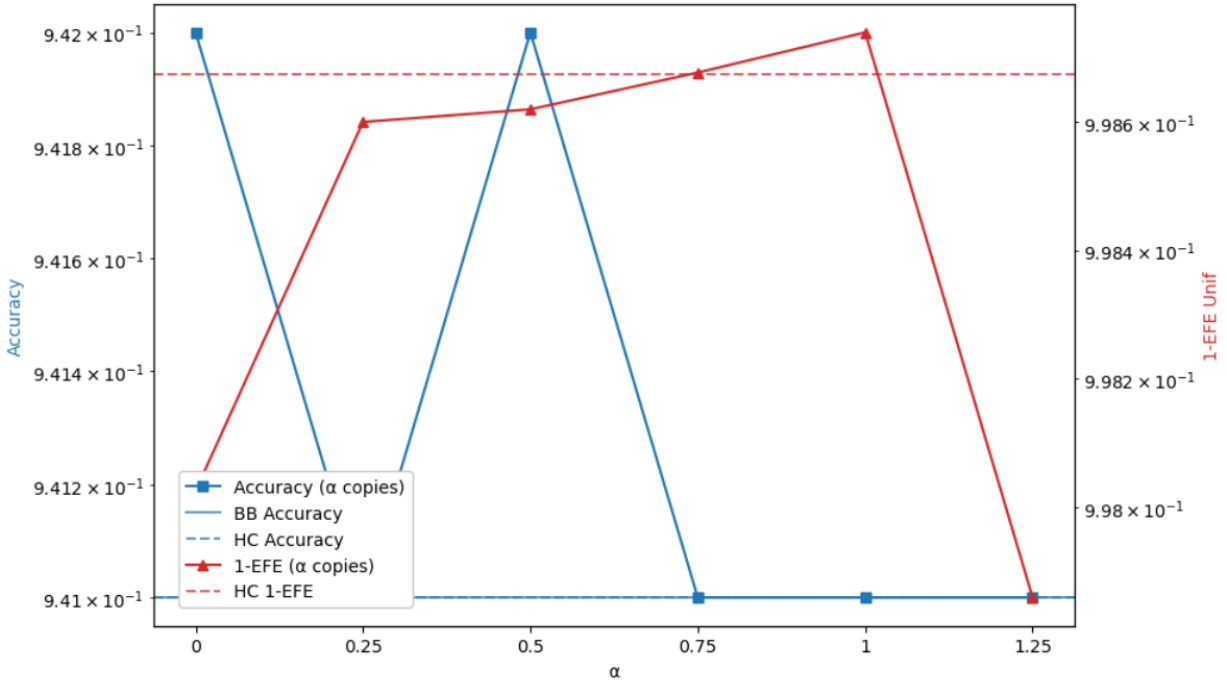}
		\caption{1-NN/MNN}
	\end{subfigure}
	\hfill
	\begin{subfigure}[t]{0.14\textheight}
		\includegraphics[width=\linewidth]{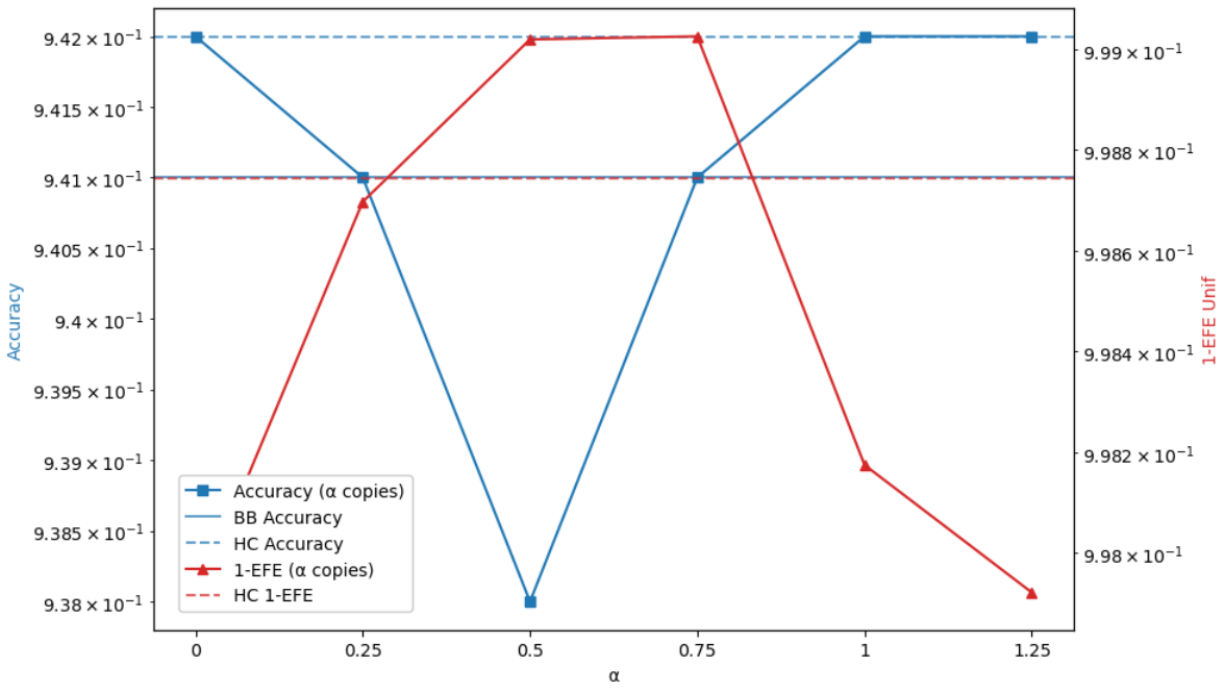}
		\caption{1-NN/LNN}
	\end{subfigure}
	\hfill
	\begin{subfigure}[t]{0.14\textheight}
		\includegraphics[width=\linewidth]{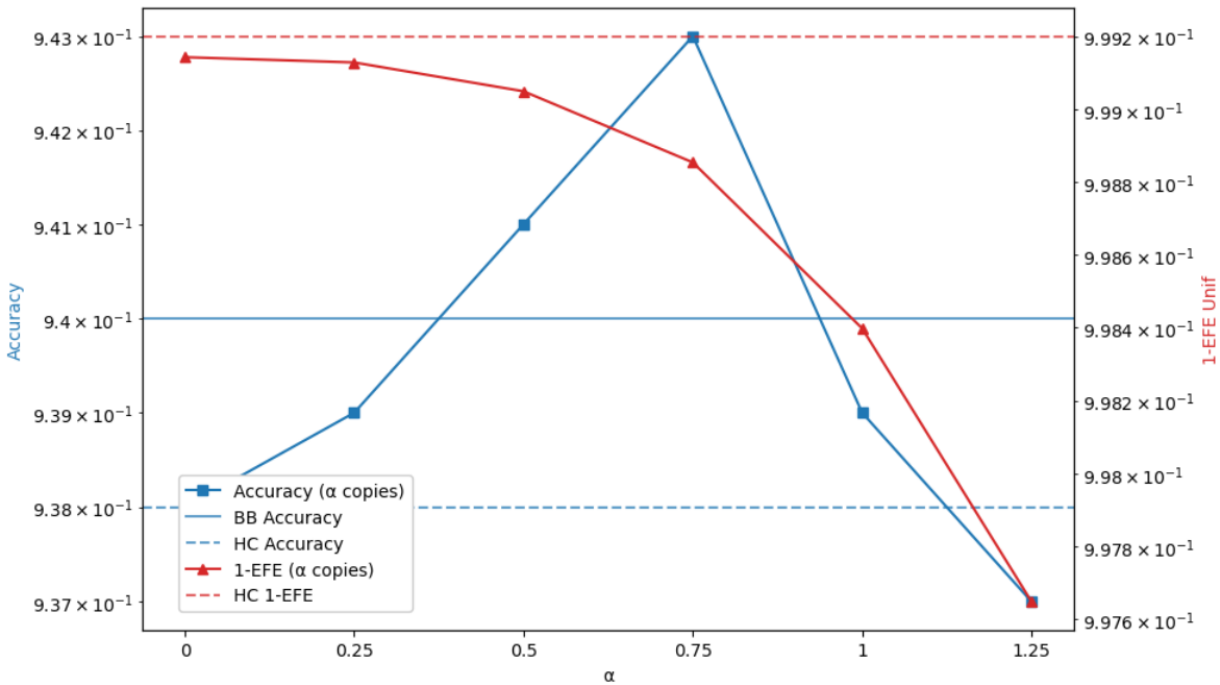}
		\caption{1-NN/GB}
	\end{subfigure}

	\begin{subfigure}[t]{0.14\textheight}
		\includegraphics[width=\linewidth]{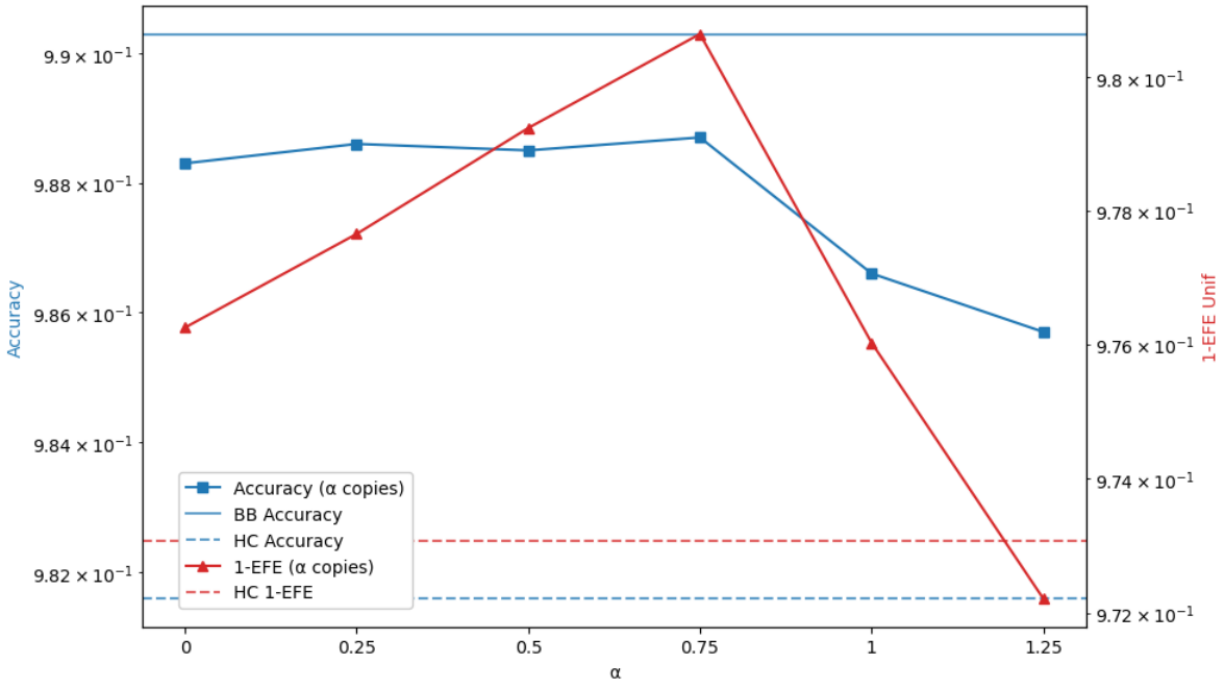}
		\caption{2-RF/SNN}
	\end{subfigure}
	\hfill
	\begin{subfigure}[t]{0.14\textheight}
		\includegraphics[width=\linewidth]{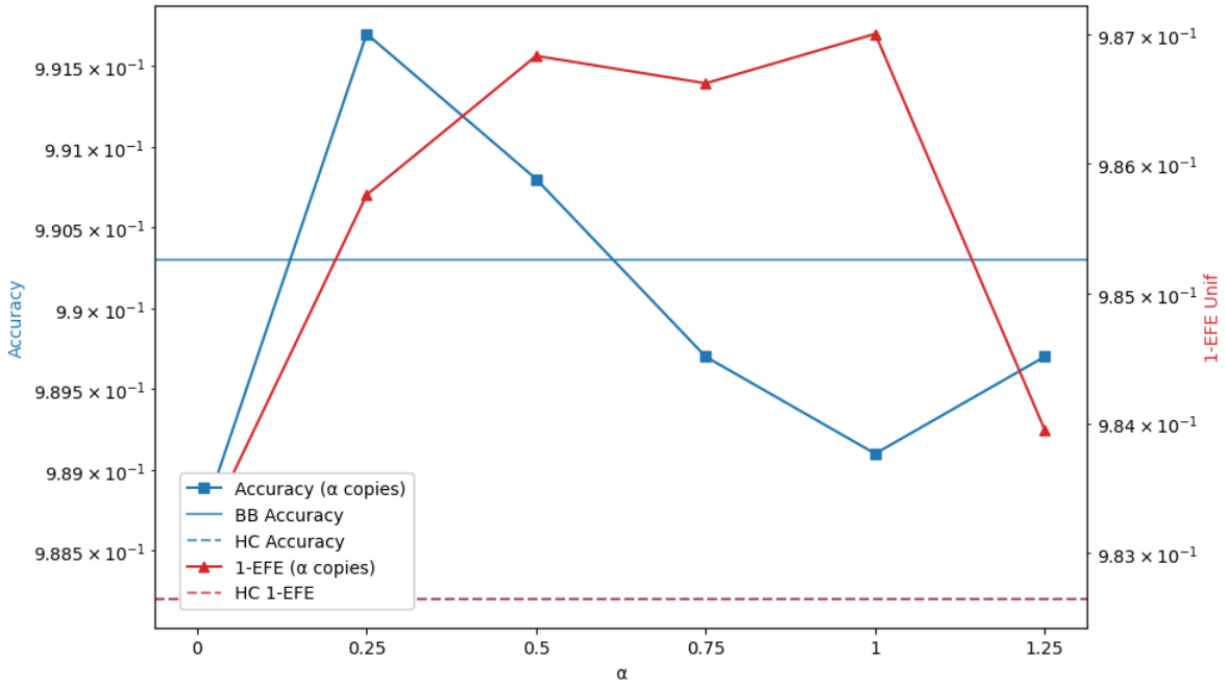}
		\caption{2-RF/MNN}
	\end{subfigure}
	\hfill
	\begin{subfigure}[t]{0.14\textheight}
		\includegraphics[width=\linewidth]{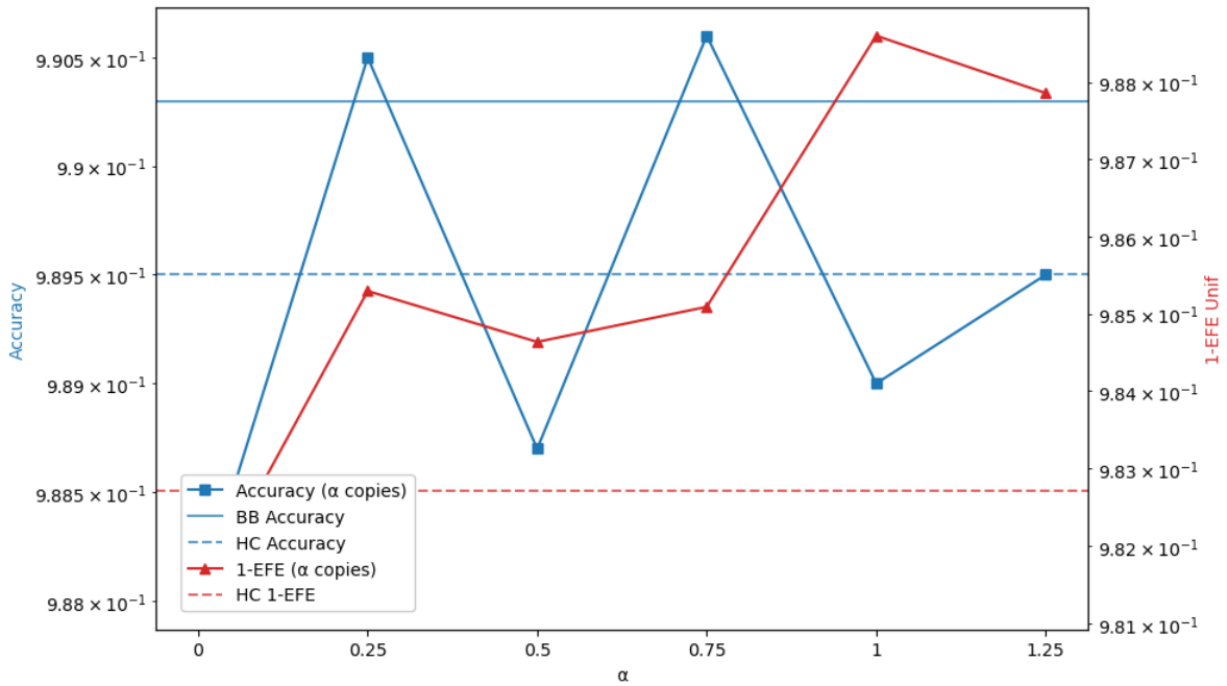}
		\caption{2-RF/LNN}
	\end{subfigure}
	\hfill
	\begin{subfigure}[t]{0.14\textheight}
		\includegraphics[width=\linewidth]{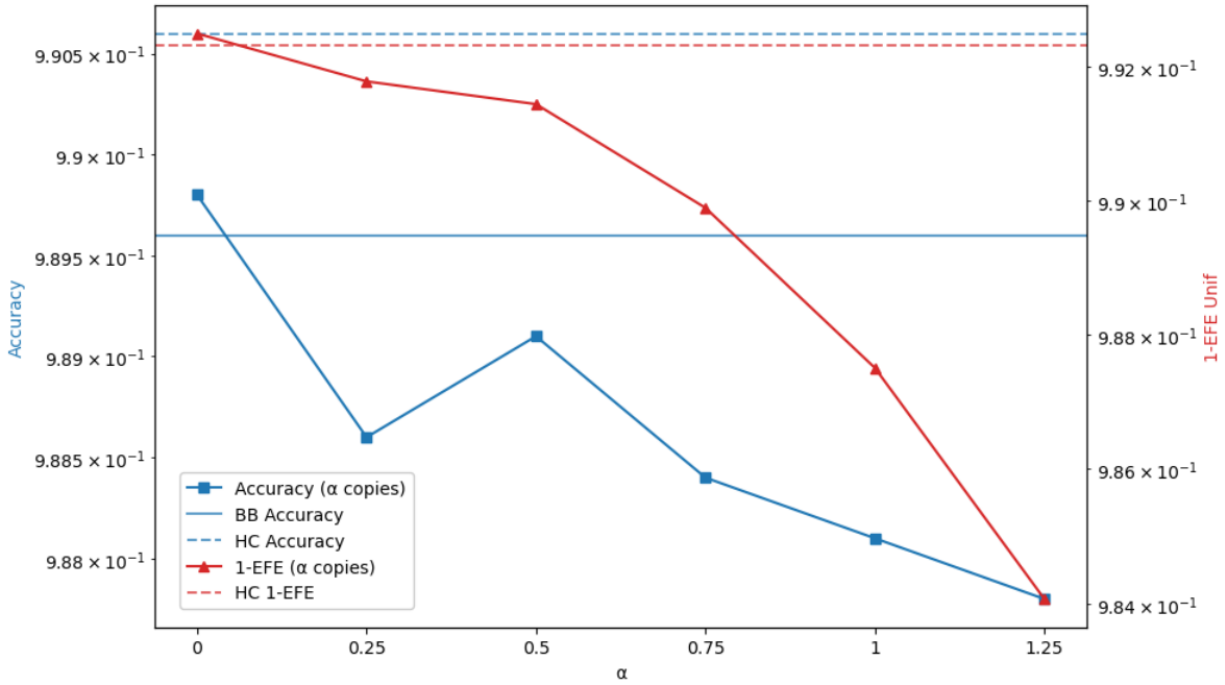}
		\caption{2-RF/GB}
	\end{subfigure}

	\begin{subfigure}[t]{0.14\textheight}
		\includegraphics[width=\linewidth]{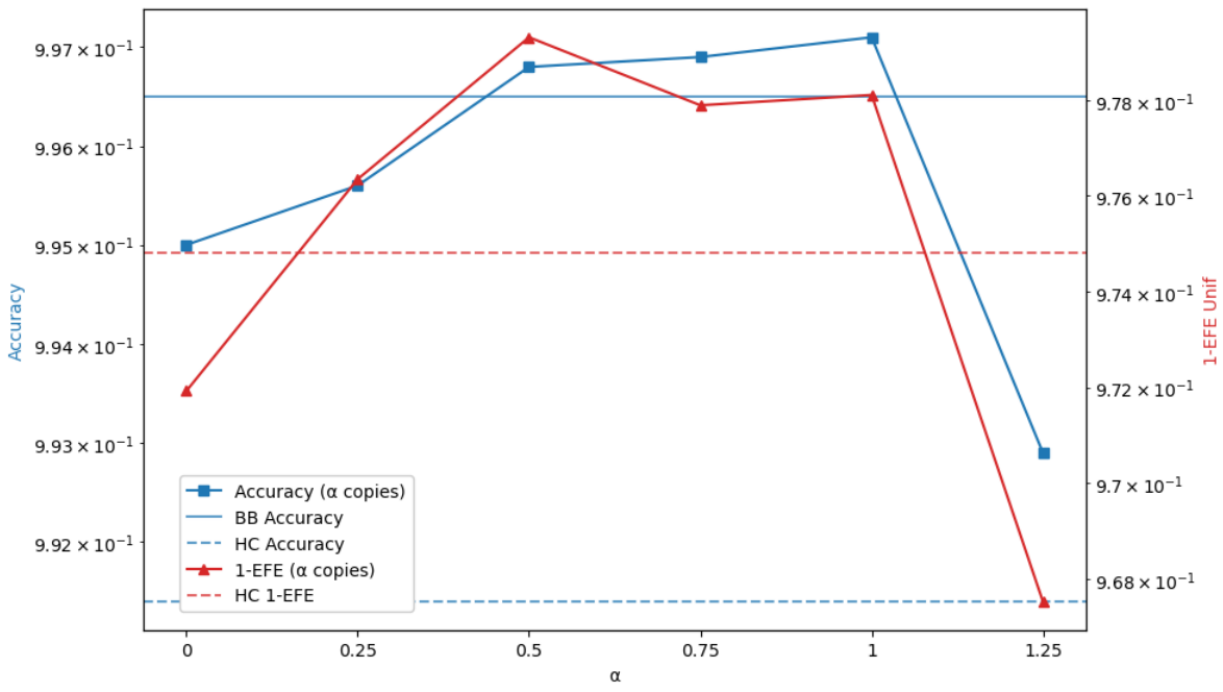}
		\caption{2-GB/SNN}
	\end{subfigure}
	\hfill
	\begin{subfigure}[t]{0.14\textheight}
		\includegraphics[width=\linewidth]{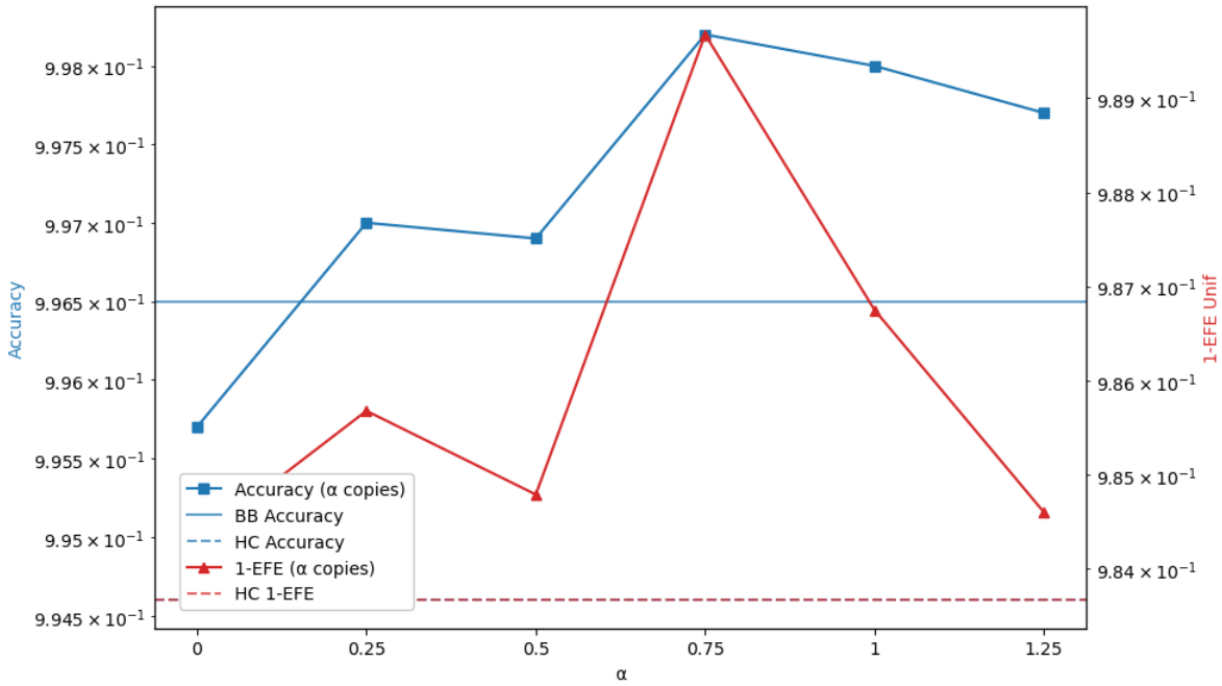}
		\caption{2-GB/MNN}
	\end{subfigure}
	\hfill
	\begin{subfigure}[t]{0.14\textheight}
		\includegraphics[width=\linewidth]{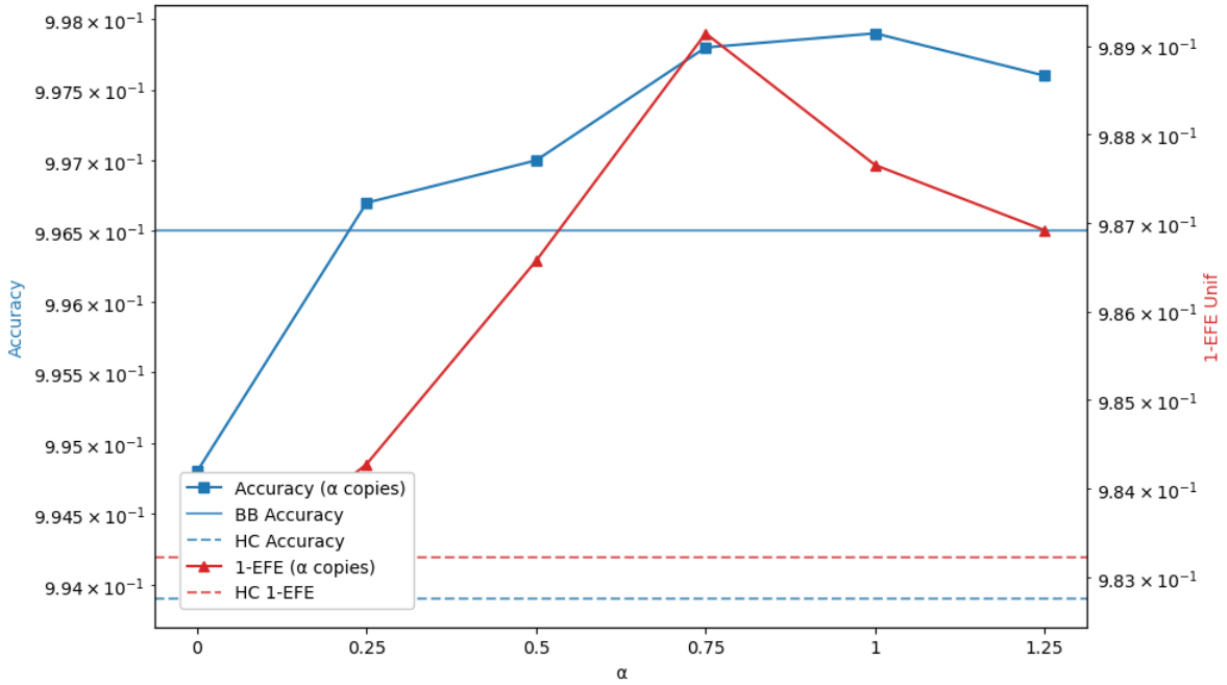}
		\caption{2-GB/LNN}
	\end{subfigure}
	\hfill
	\begin{subfigure}[t]{0.14\textheight}
		\includegraphics[width=\linewidth]{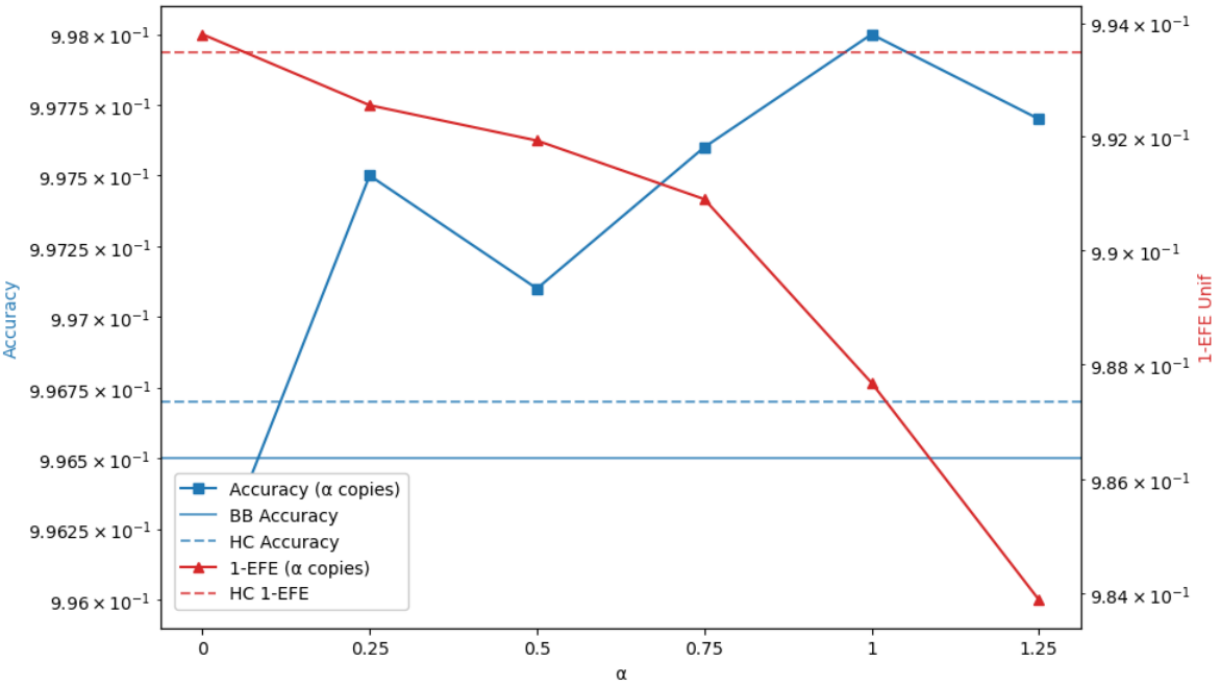}
		\caption{2-GB/GB}
	\end{subfigure}

	\begin{subfigure}[t]{0.14\textheight}
		\includegraphics[width=\linewidth]{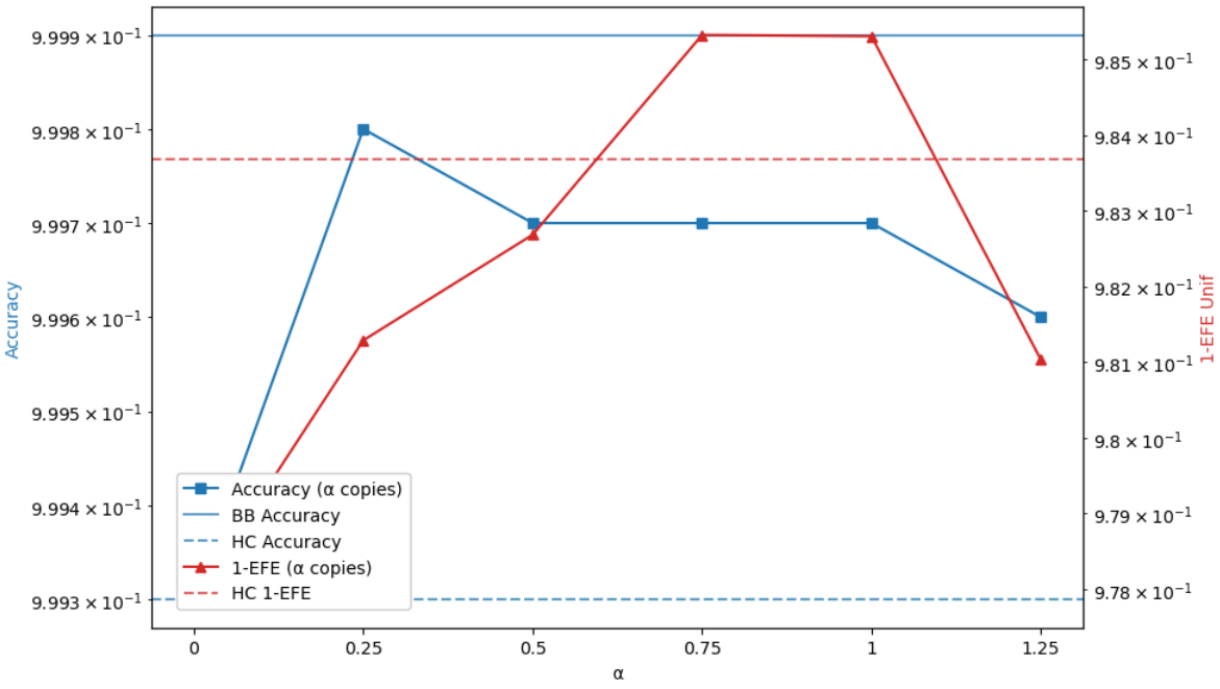}
		\caption{2-NN/SNN}
	\end{subfigure}
	\hfill
	\begin{subfigure}[t]{0.14\textheight}
		\includegraphics[width=\linewidth]{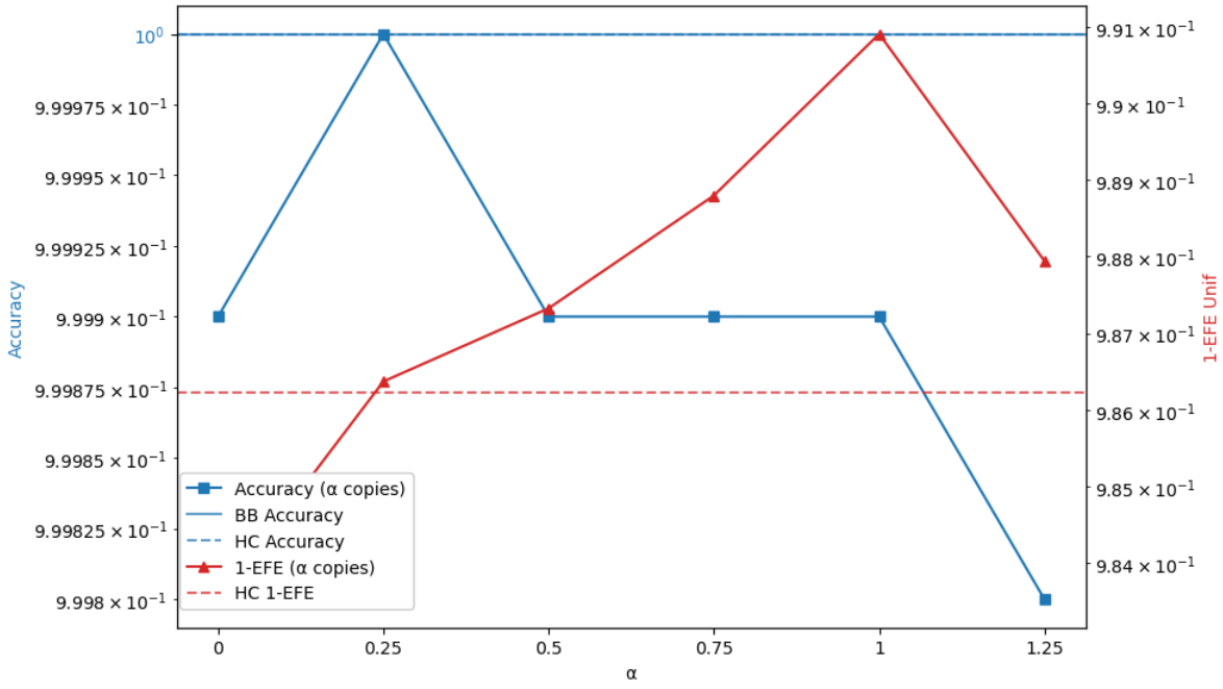}
		\caption{2-NN/MNN}
	\end{subfigure}
	\hfill
	\begin{subfigure}[t]{0.14\textheight}
		\includegraphics[width=\linewidth]{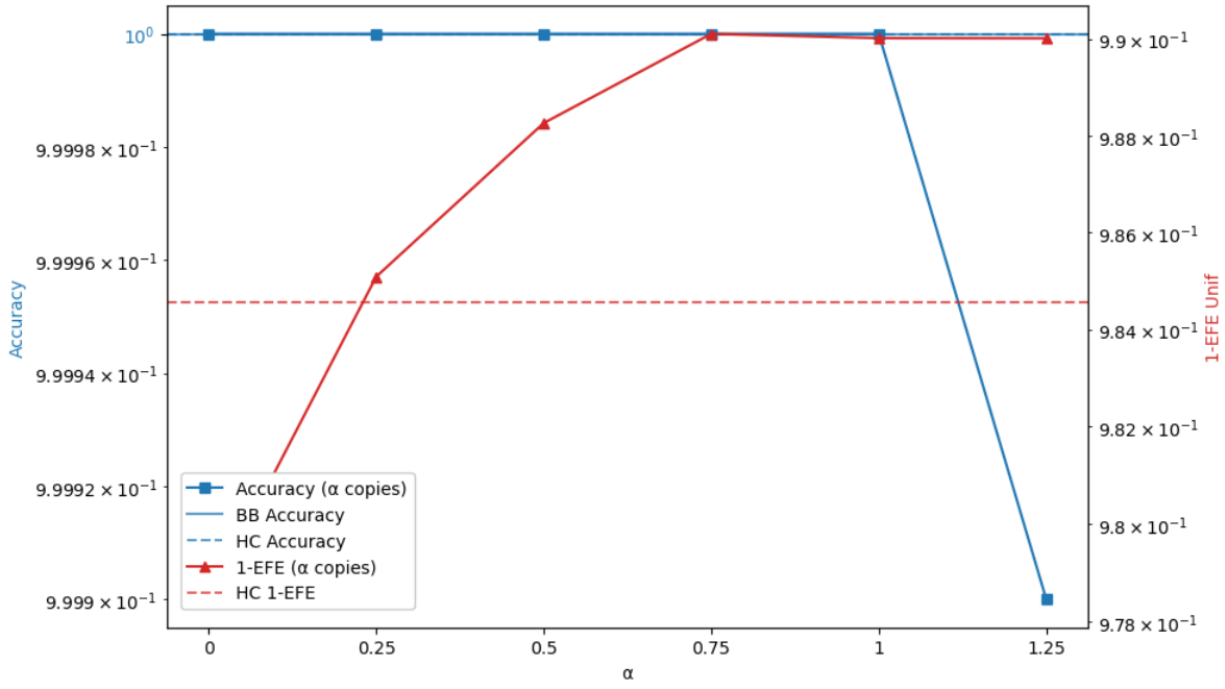}
		\caption{2-NN/LNN}
	\end{subfigure}
	\hfill
	\begin{subfigure}[t]{0.14\textheight}
		\includegraphics[width=\linewidth]{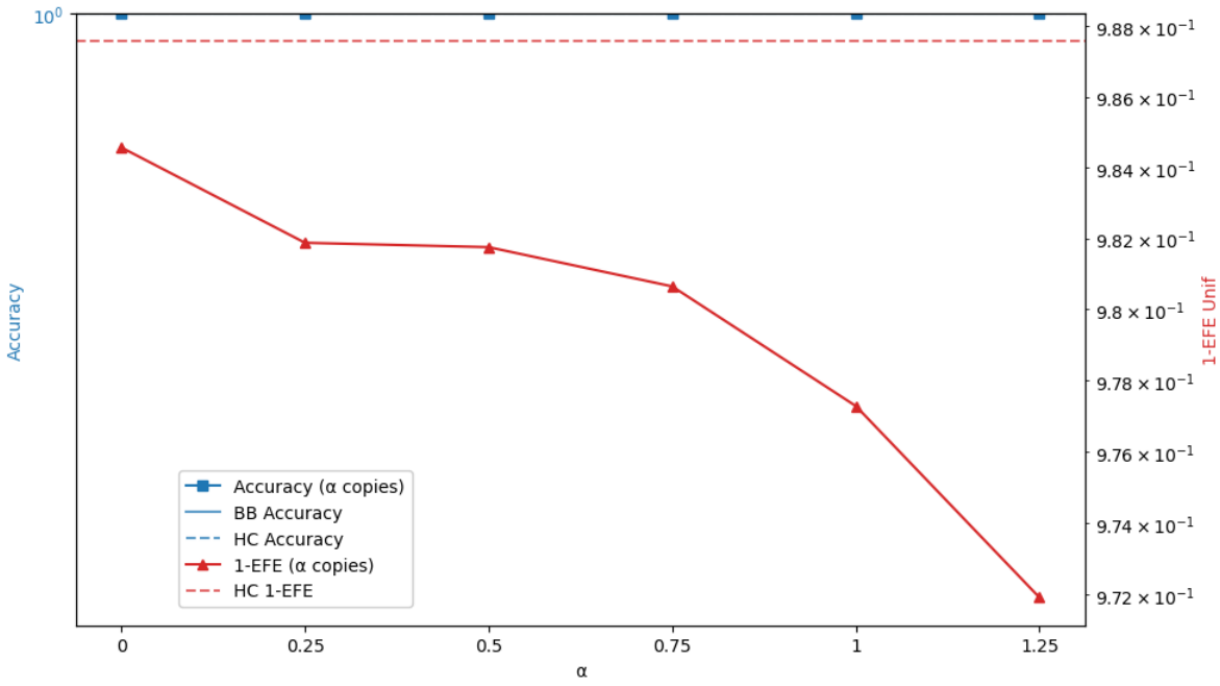}
		\caption{2-NN/GB}
	\end{subfigure}

	\begin{subfigure}[t]{0.14\textheight}
		\includegraphics[width=\linewidth]{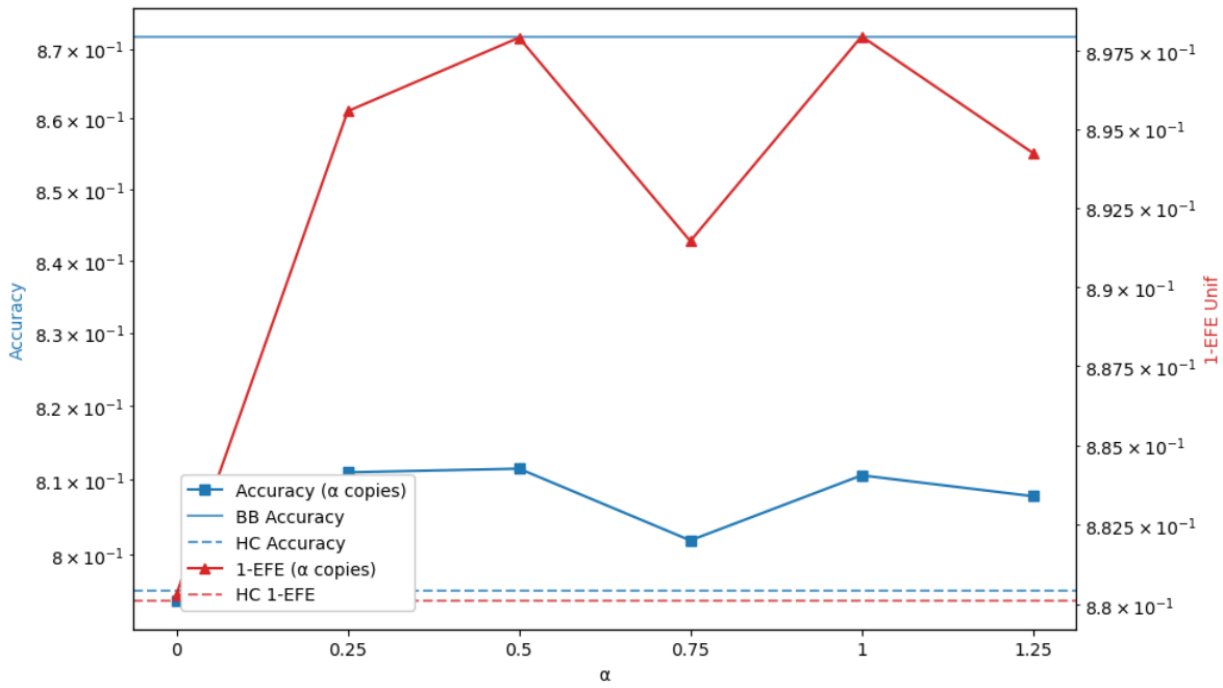}
		\caption{3-RF/SNN}
	\end{subfigure}
	\hfill
	\begin{subfigure}[t]{0.14\textheight}
		\includegraphics[width=\linewidth]{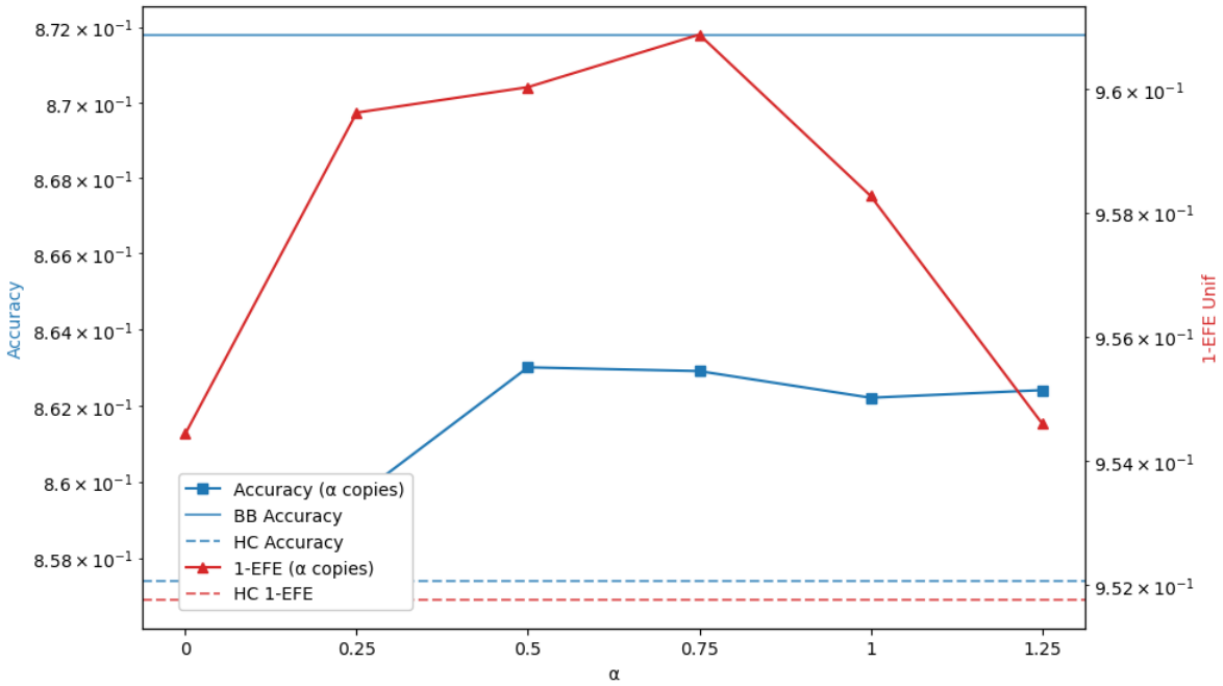}
		\caption{3-RF/MNN}
	\end{subfigure}
	\hfill
	\begin{subfigure}[t]{0.14\textheight}
		\includegraphics[width=\linewidth]{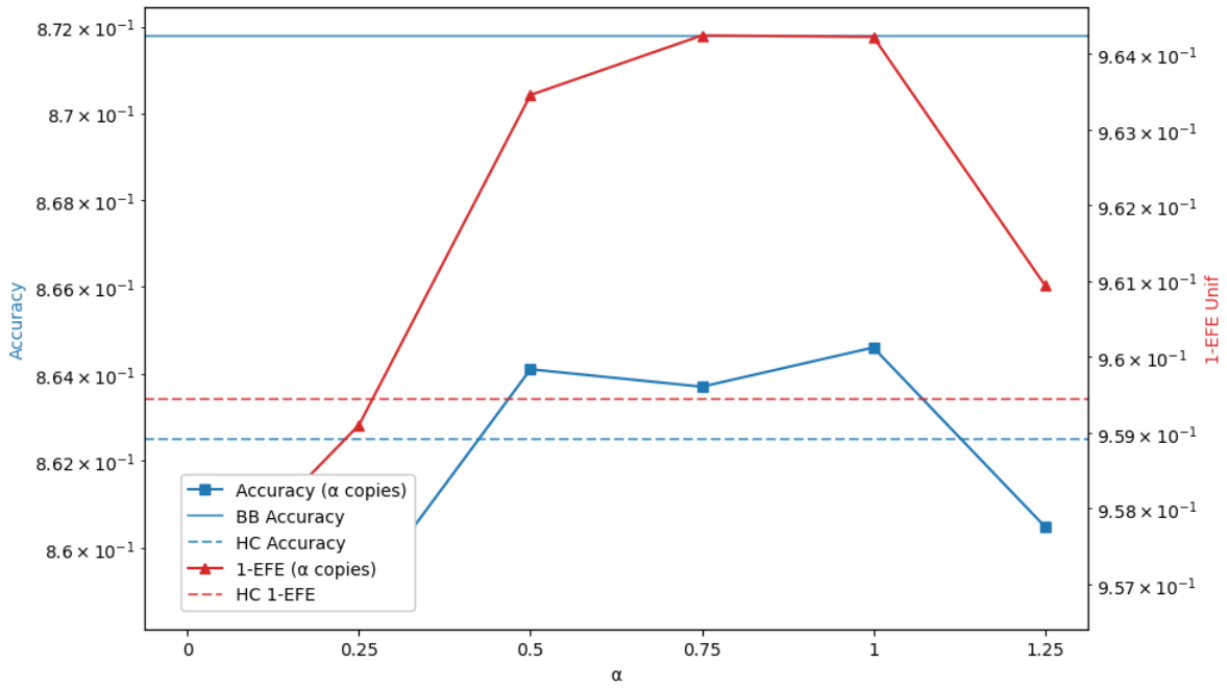}
		\caption{3-RF/LNN}
	\end{subfigure}
	\hfill
	\begin{subfigure}[t]{0.14\textheight}
		\includegraphics[width=\linewidth]{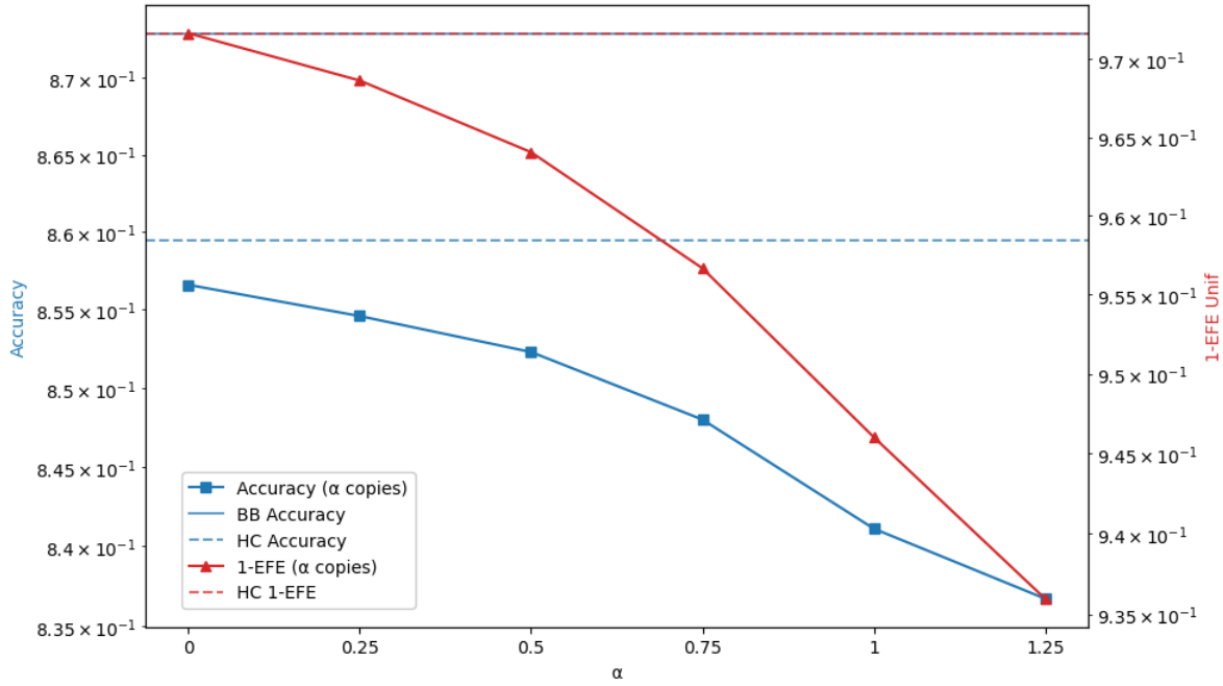}
		\caption{3-RF/GB}
	\end{subfigure}

	\begin{subfigure}[t]{0.14\textheight}
		\includegraphics[width=\linewidth]{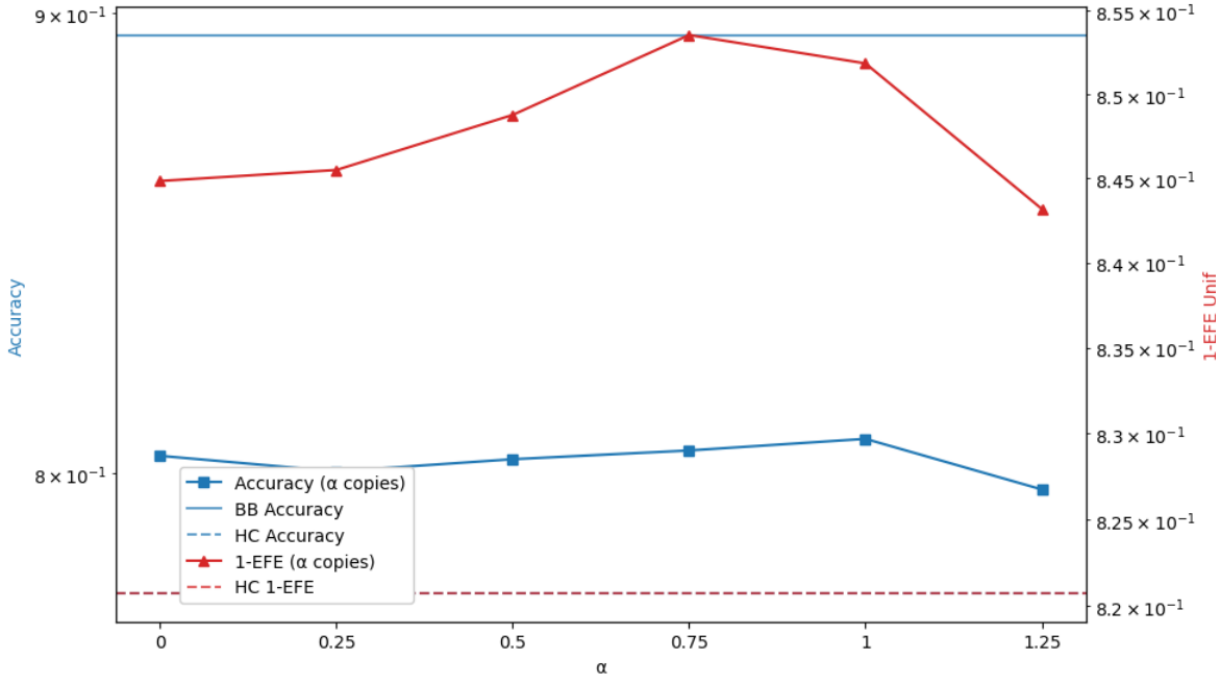}
		\caption{3-GB/SNN}
	\end{subfigure}
	\hfill
	\begin{subfigure}[t]{0.14\textheight}
		\includegraphics[width=\linewidth]{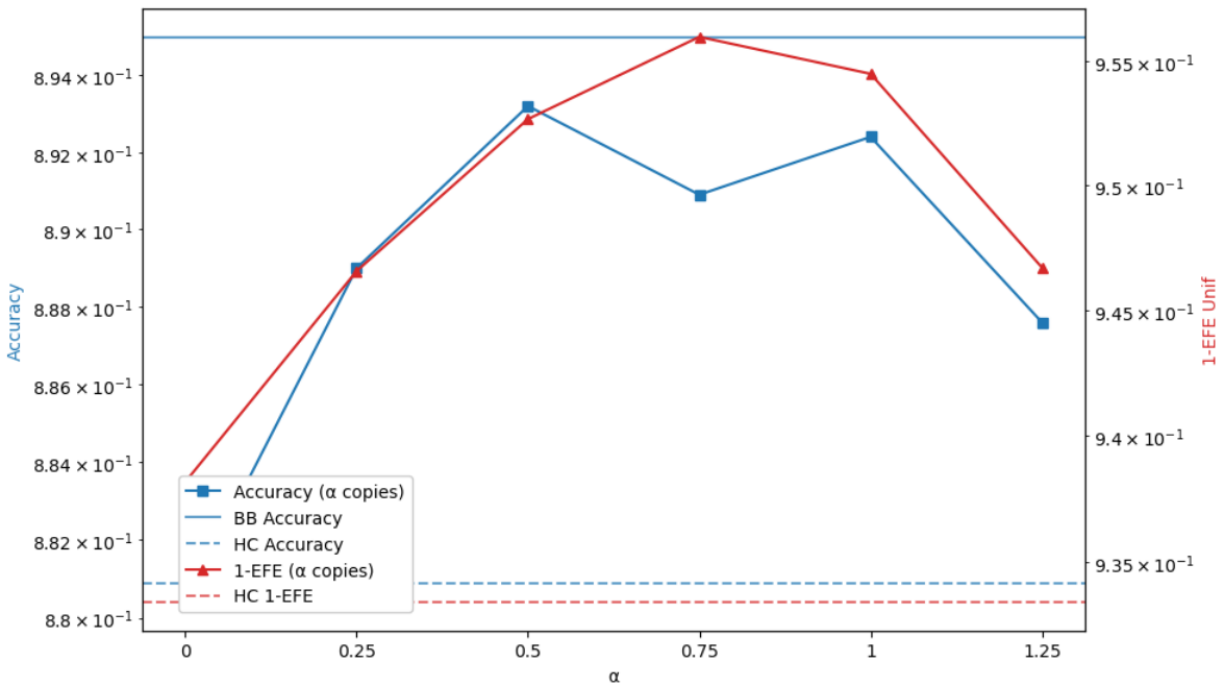}
		\caption{3-GB/MNN}
	\end{subfigure}
	\hfill
	\begin{subfigure}[t]{0.14\textheight}
		\includegraphics[width=\linewidth]{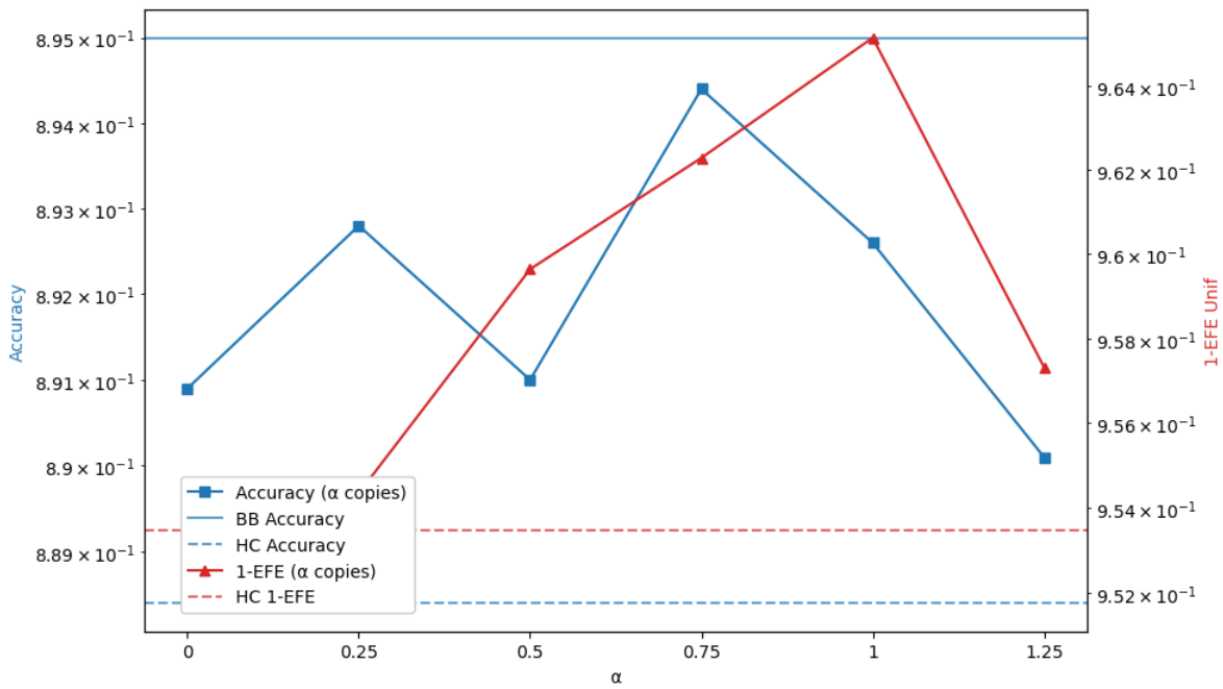}
		\caption{3-GB/LNN}
	\end{subfigure}
	\hfill
	\begin{subfigure}[t]{0.14\textheight}
		\includegraphics[width=\linewidth]{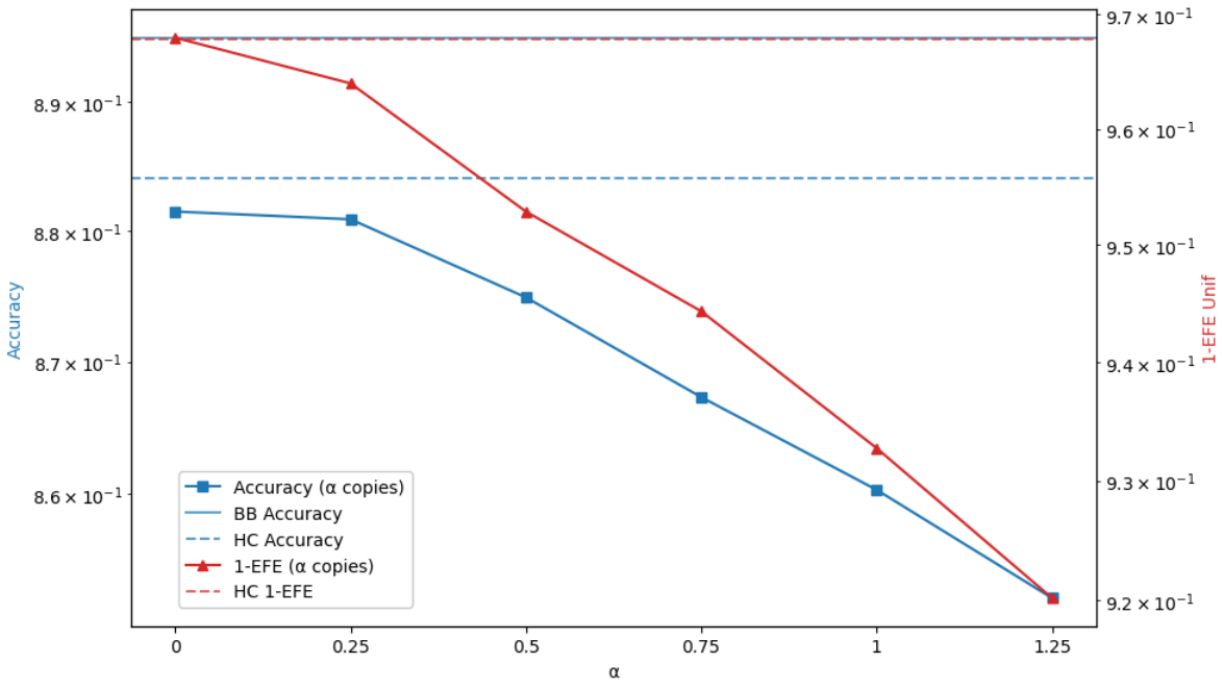}
		\caption{3-GB/GB}
	\end{subfigure}

	\begin{subfigure}[t]{0.14\textheight}
		\includegraphics[width=\linewidth]{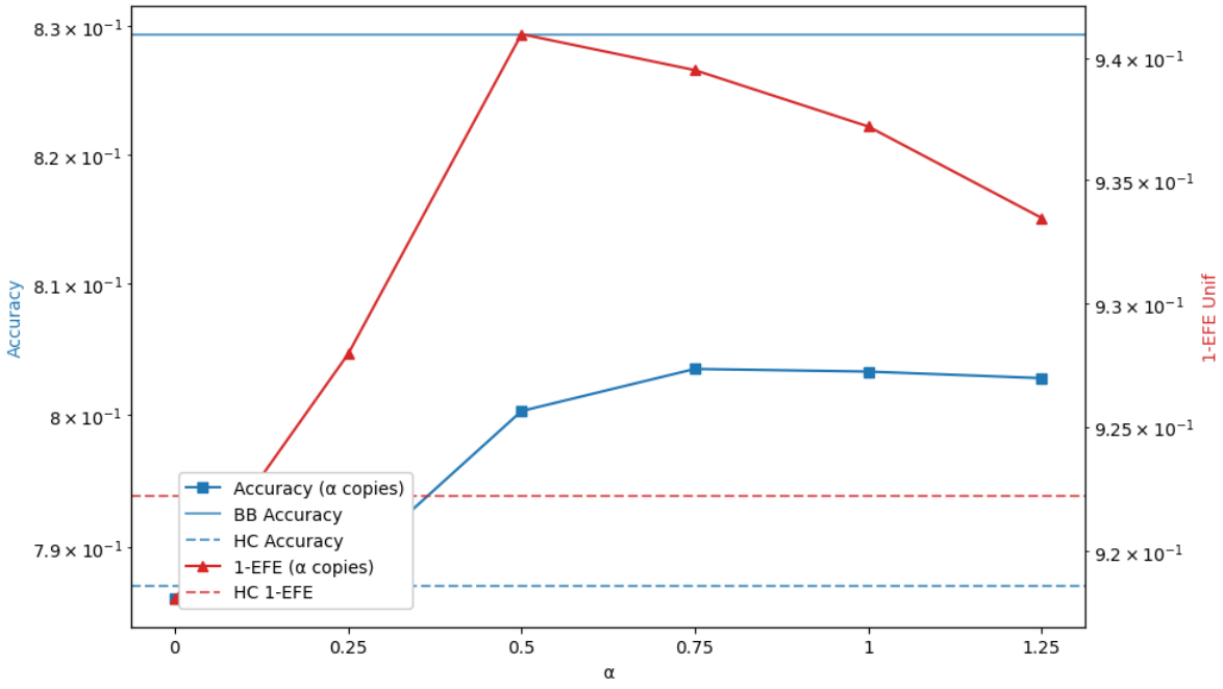}
		\caption{3-NN/SNN}
	\end{subfigure}
	\hfill
	\begin{subfigure}[t]{0.14\textheight}
		\includegraphics[width=\linewidth]{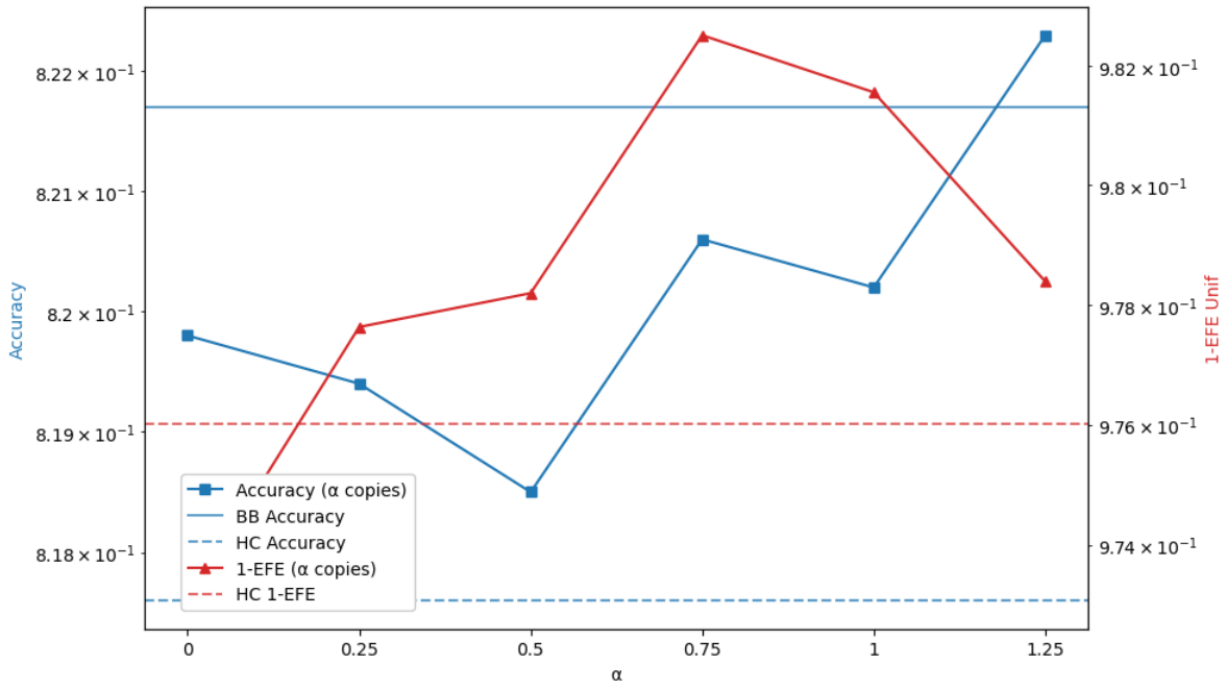}
		\caption{3-NN/MNN}
	\end{subfigure}
	\hfill
	\begin{subfigure}[t]{0.14\textheight}
		\includegraphics[width=\linewidth]{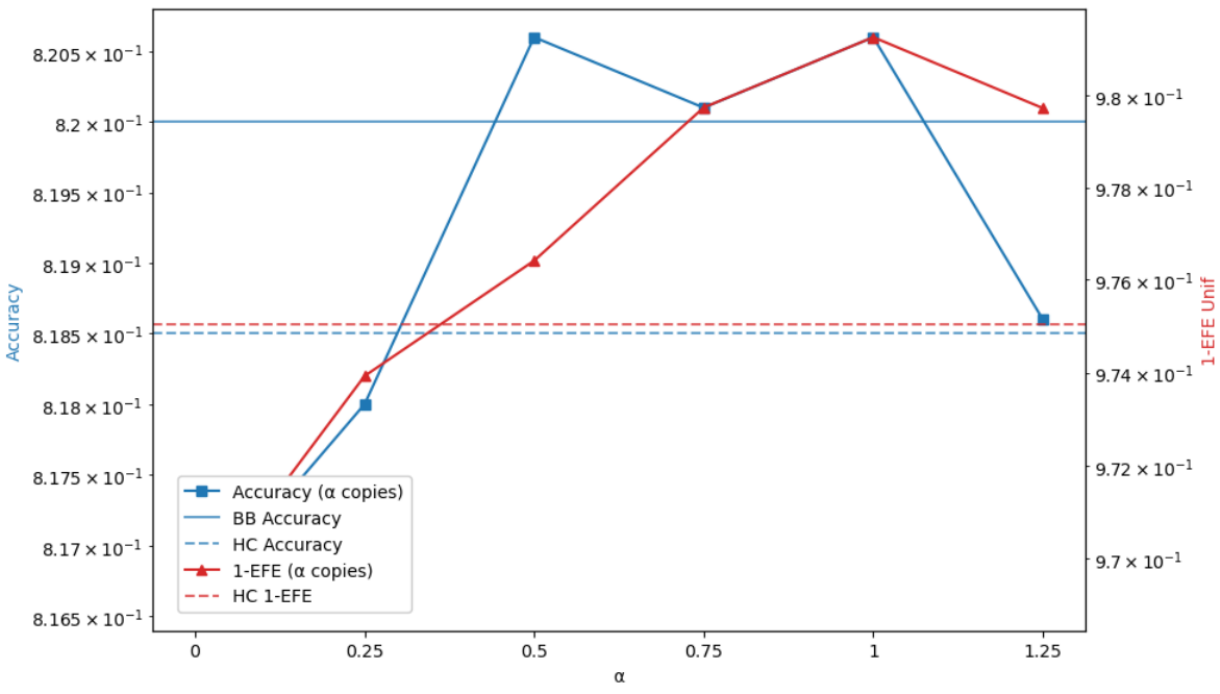}
		\caption{3-NN/LNN}
	\end{subfigure}
	\hfill
	\begin{subfigure}[t]{0.14\textheight}
		\includegraphics[width=\linewidth]{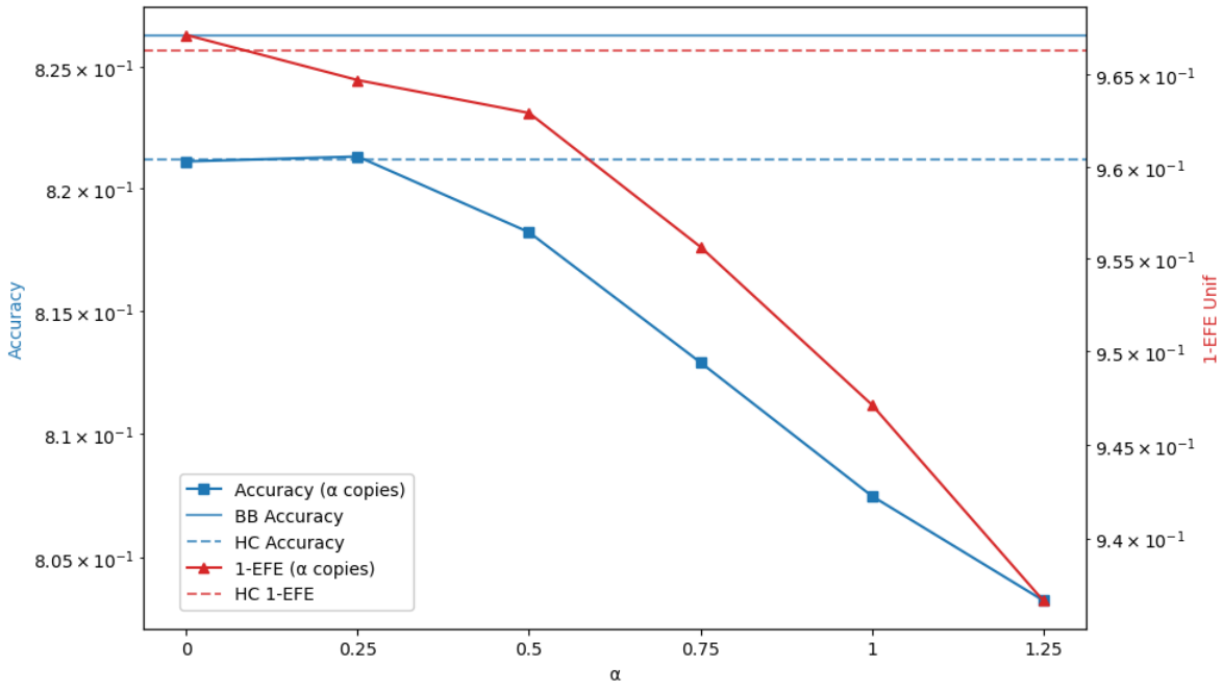}
		\caption{3-NN/GB}
	\end{subfigure}
\end{figure}

\begin{figure}[!ht]
	\centering
	\begin{subfigure}[t]{0.14\textheight}
		\includegraphics[width=\linewidth]{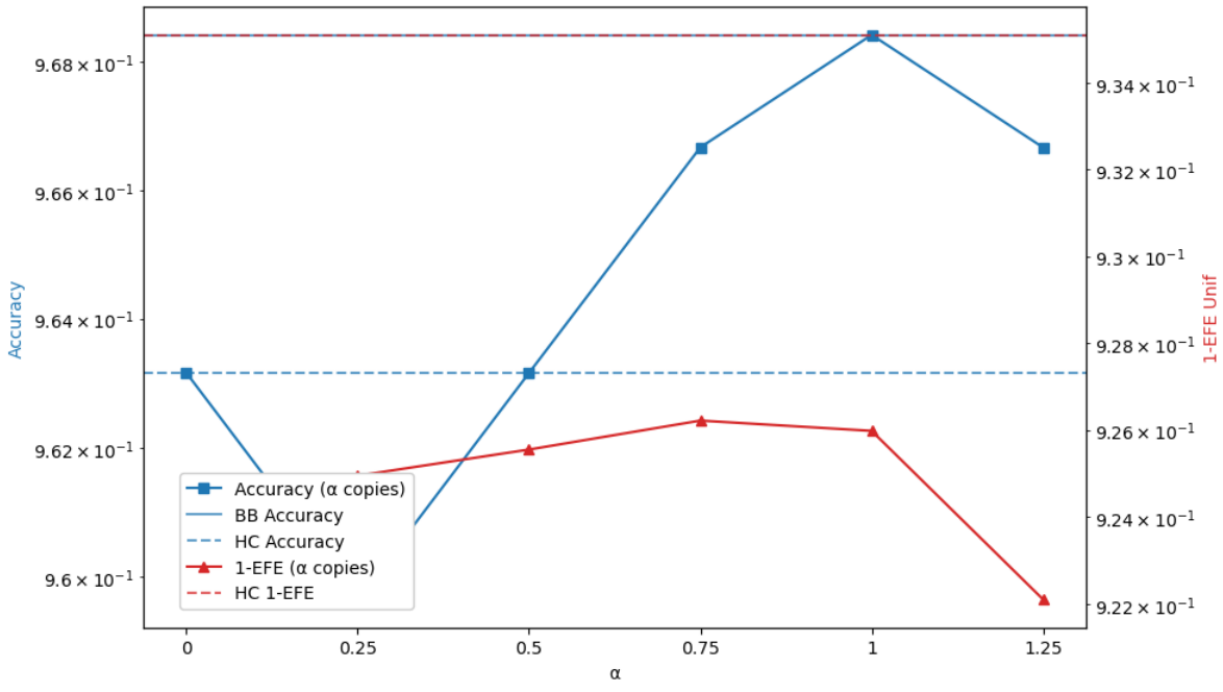}
		\caption{4-RF/SNN}
	\end{subfigure}
	\hfill
	\begin{subfigure}[t]{0.14\textheight}
		\includegraphics[width=\linewidth]{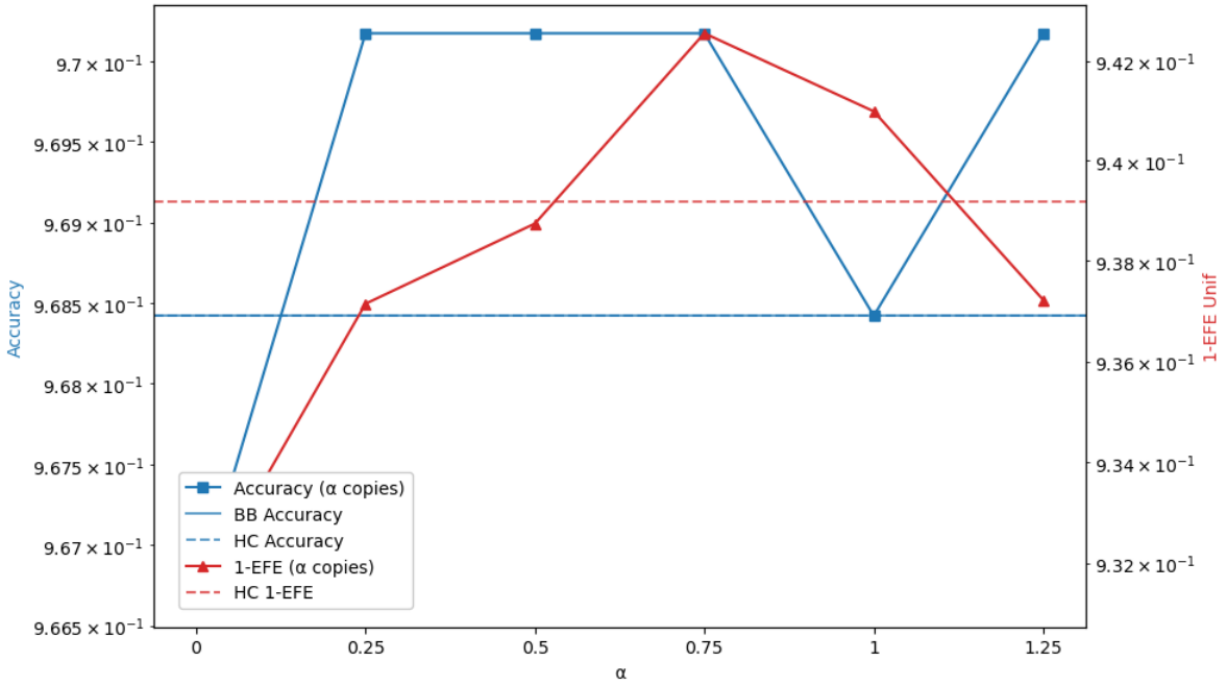}
		\caption{4-RF/MNN}
	\end{subfigure}
	\hfill
	\begin{subfigure}[t]{0.14\textheight}
		\includegraphics[width=\linewidth]{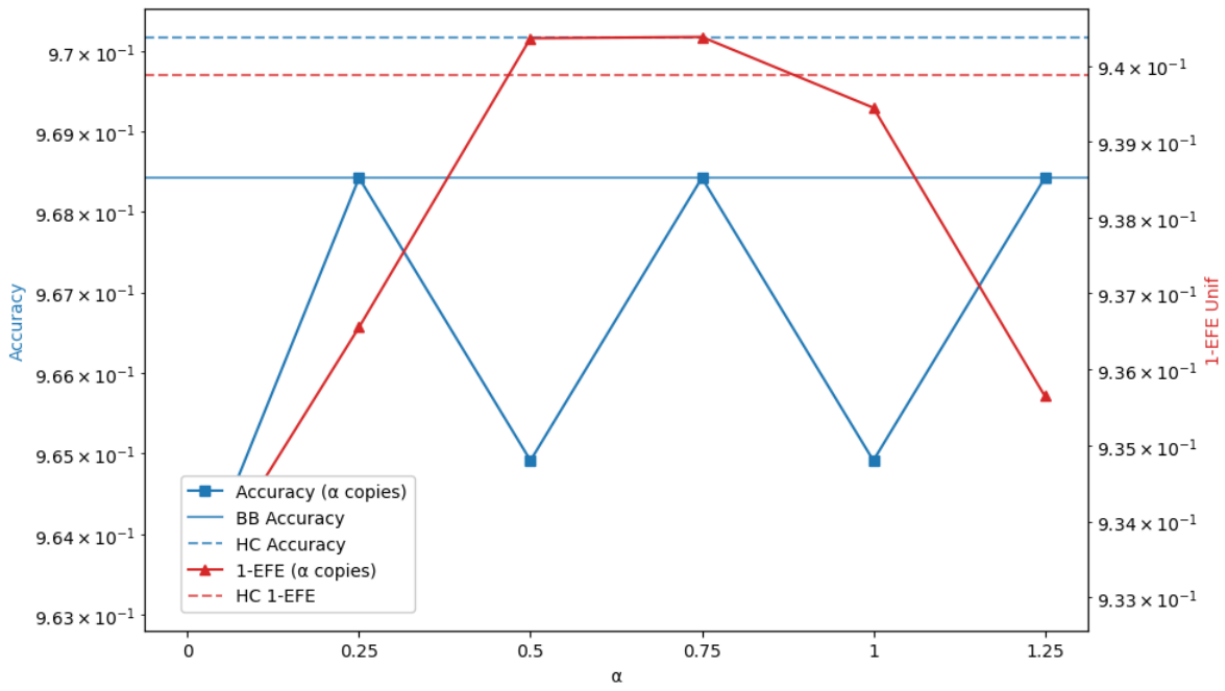}
		\caption{4-RF/LNN}
	\end{subfigure}
	\hfill
	\begin{subfigure}[t]{0.14\textheight}
		\includegraphics[width=\linewidth]{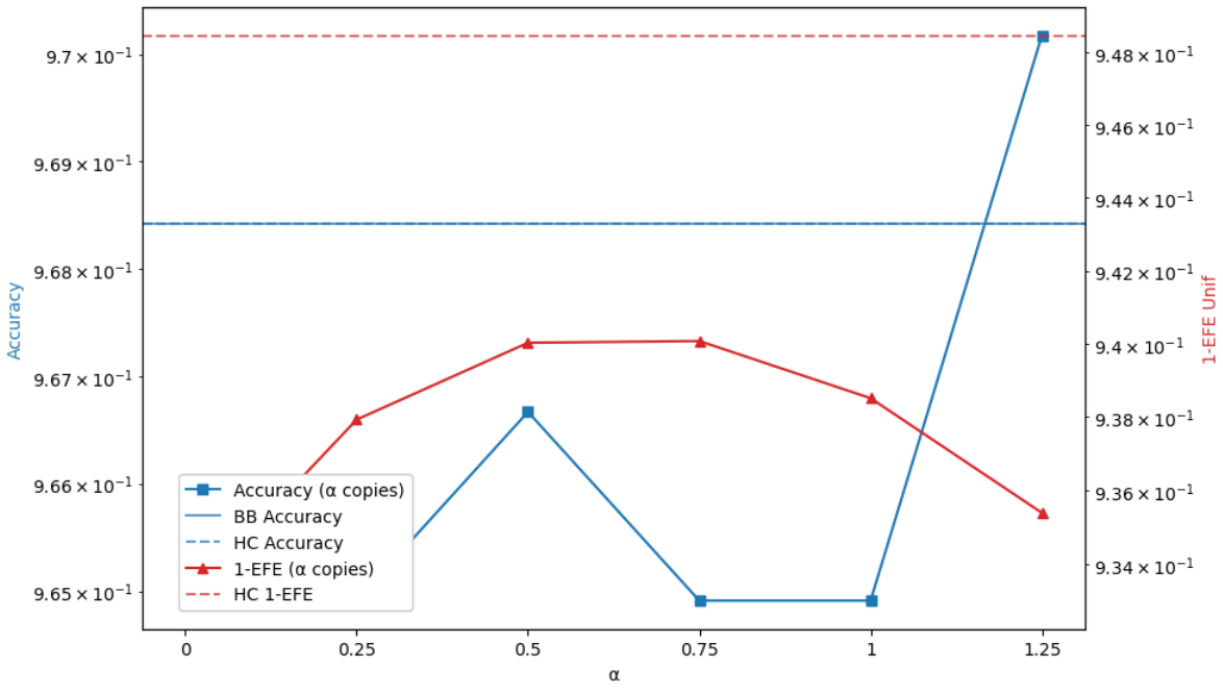}
		\caption{4-RF/GB}
	\end{subfigure}
	
	
	\begin{subfigure}[t]{0.14\textheight}
		\includegraphics[width=\linewidth]{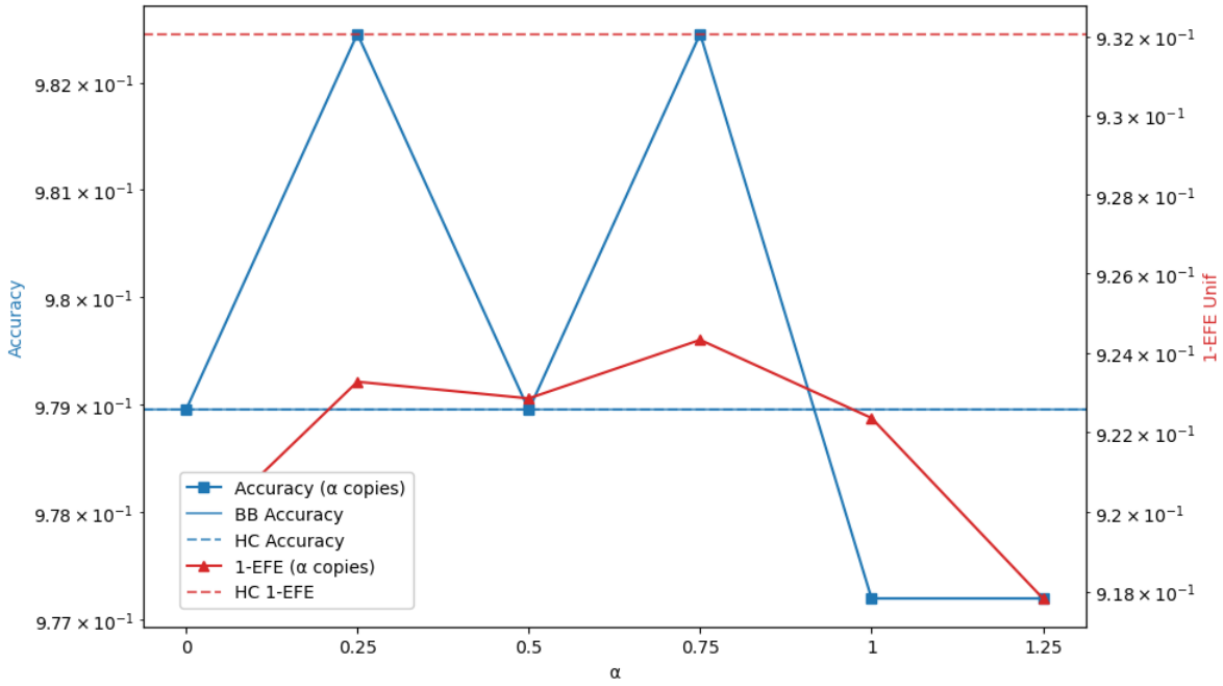}
		\caption{4-GB/SNN}
	\end{subfigure}
	\hfill
	\begin{subfigure}[t]{0.14\textheight}
		\includegraphics[width=\linewidth]{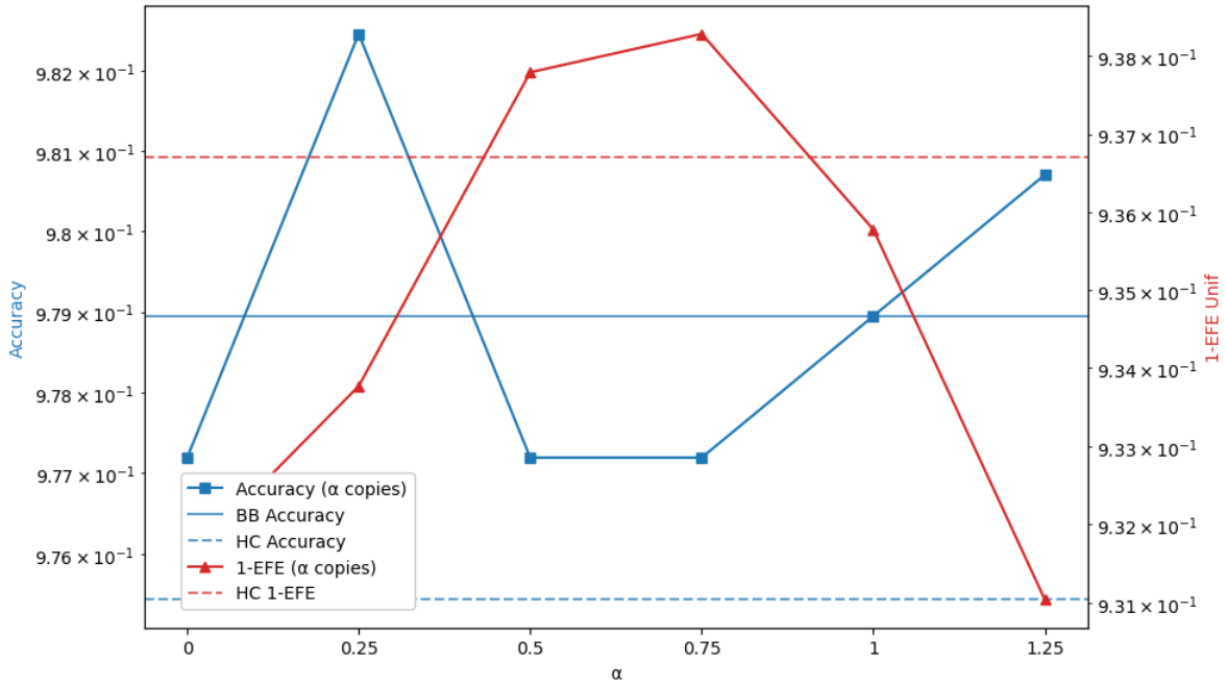}
		\caption{4-GB/MNN}
	\end{subfigure}
	\hfill
	\begin{subfigure}[t]{0.14\textheight}
		\includegraphics[width=\linewidth]{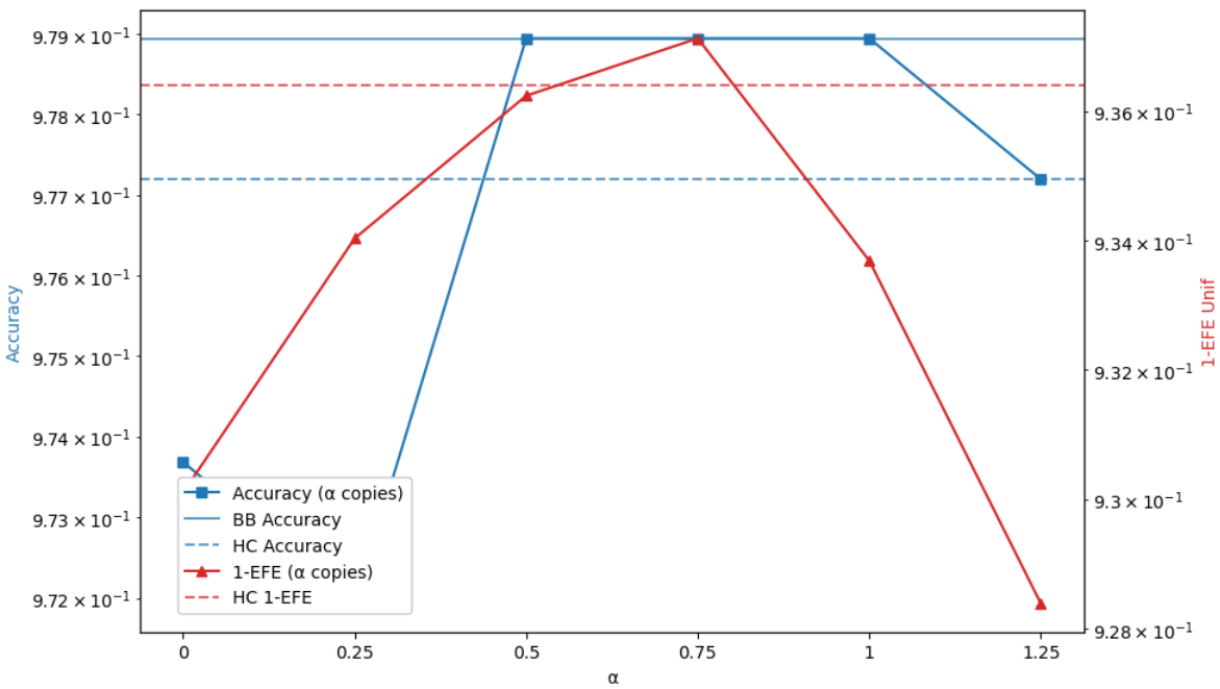}
		\caption{4-GB/LNN}
	\end{subfigure}
	\hfill
	\begin{subfigure}[t]{0.14\textheight}
		\includegraphics[width=\linewidth]{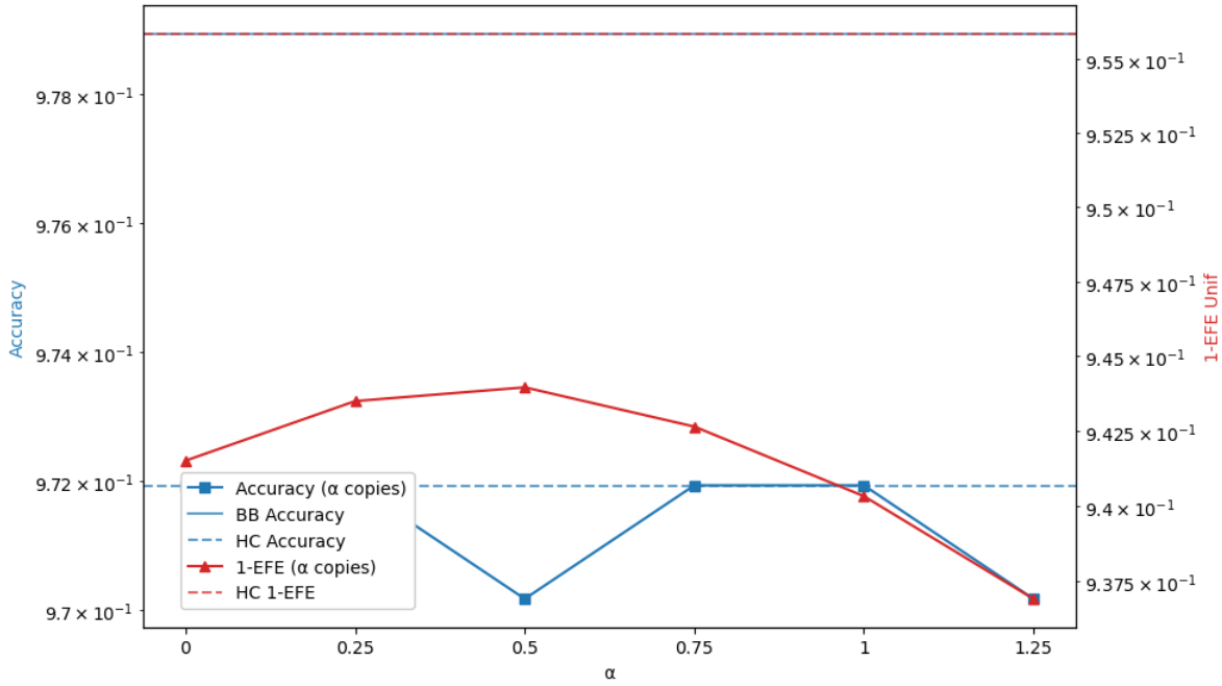}
		\caption{4-GB/GB}
	\end{subfigure}
	
	
	\begin{subfigure}[t]{0.14\textheight}
		\includegraphics[width=\linewidth]{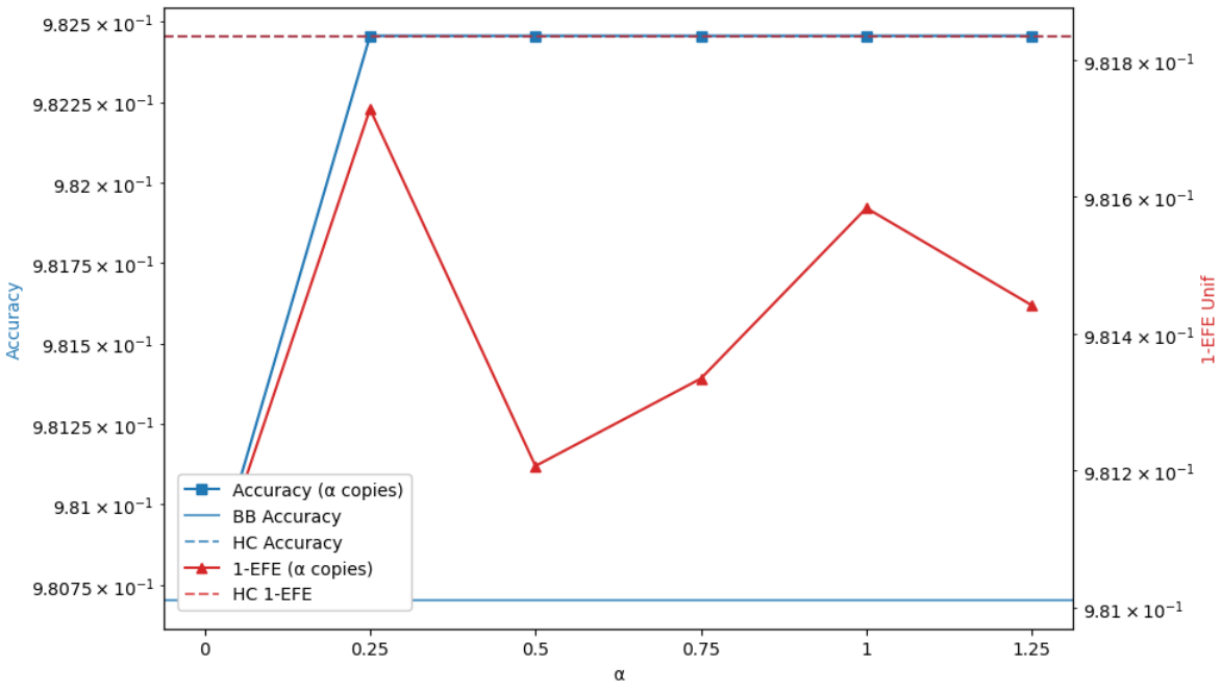}
		\caption{4-NN/SNN}
	\end{subfigure}
	\hfill
	\begin{subfigure}[t]{0.14\textheight}
		\includegraphics[width=\linewidth]{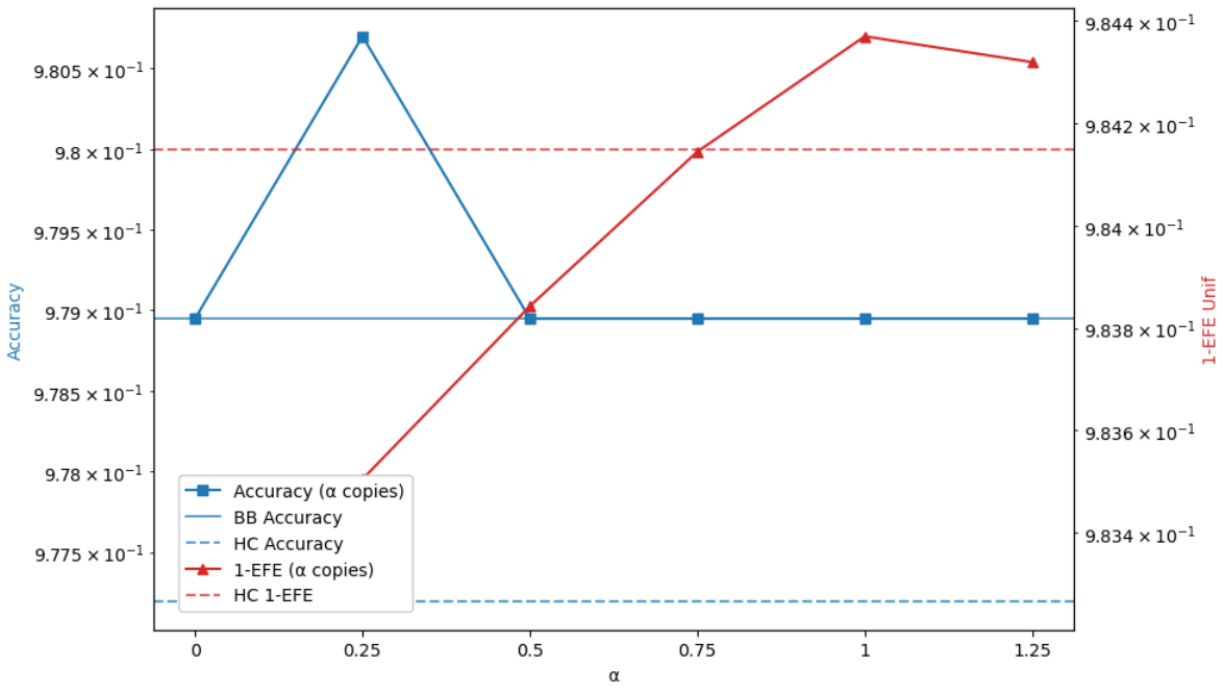}
		\caption{4-NN/MNN}
	\end{subfigure}
	\hfill
	\begin{subfigure}[t]{0.14\textheight}
		\includegraphics[width=\linewidth]{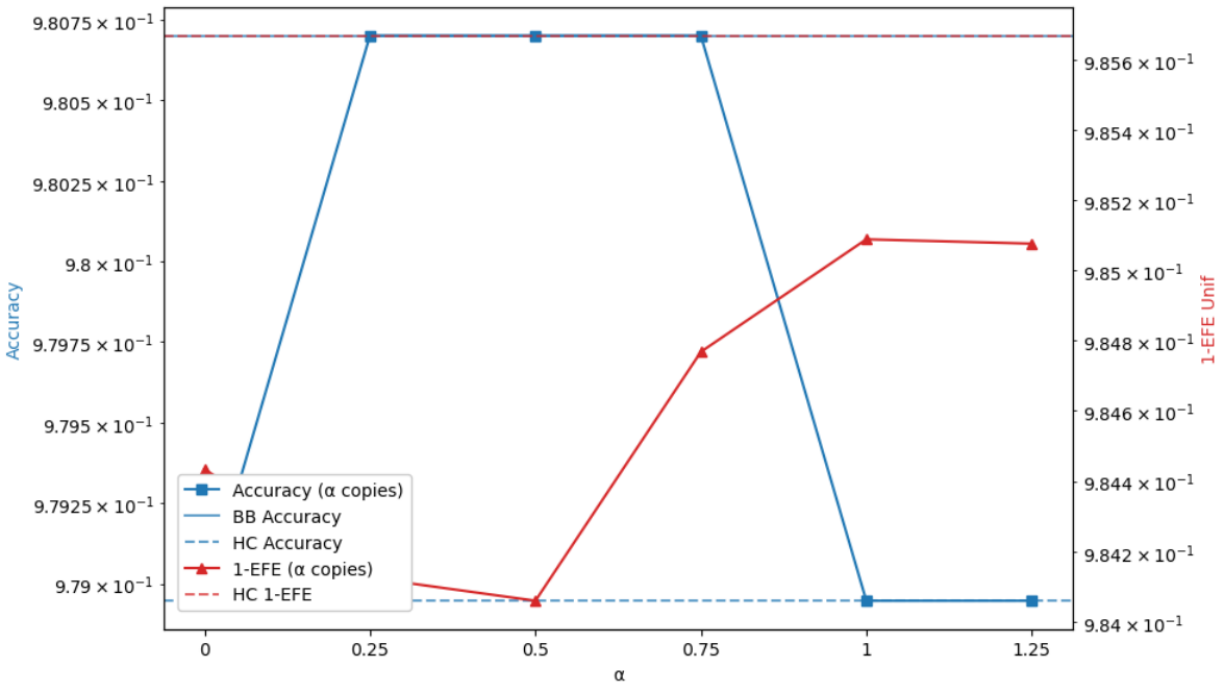}
		\caption{4-NN/LNN}
	\end{subfigure}
	\hfill
	\begin{subfigure}[t]{0.14\textheight}
		\includegraphics[width=\linewidth]{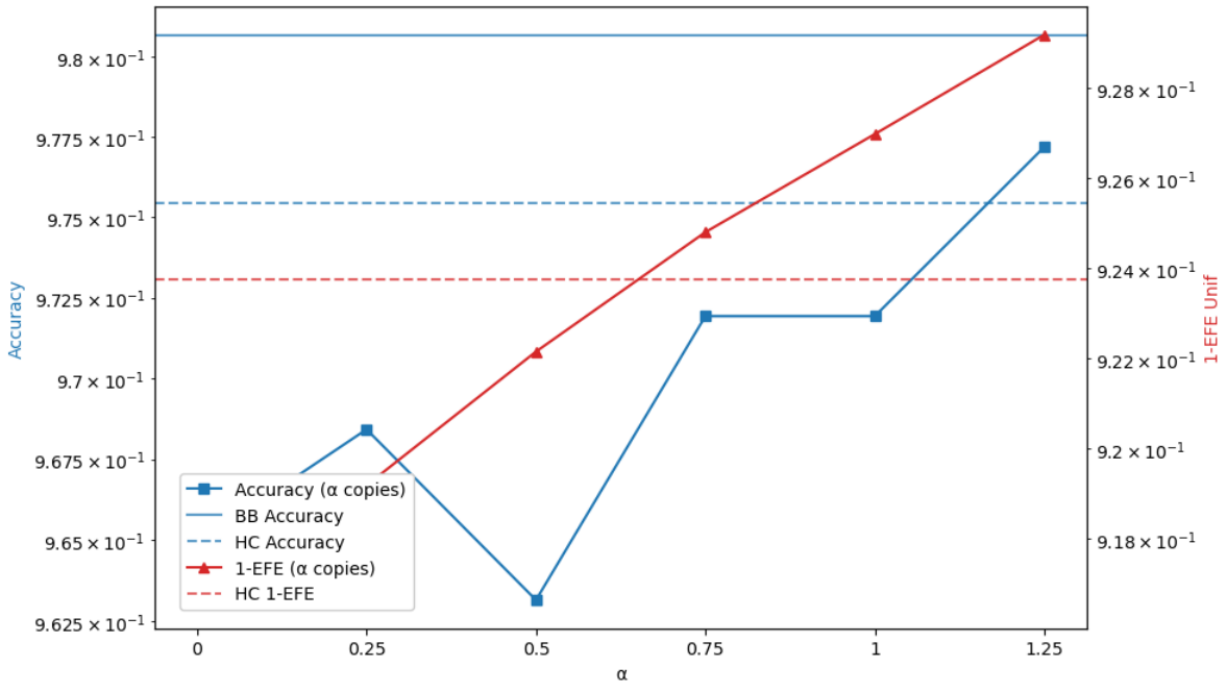}
		\caption{4-NN/GB}
	\end{subfigure}
	
	
	\begin{subfigure}[t]{0.14\textheight}
		\includegraphics[width=\linewidth]{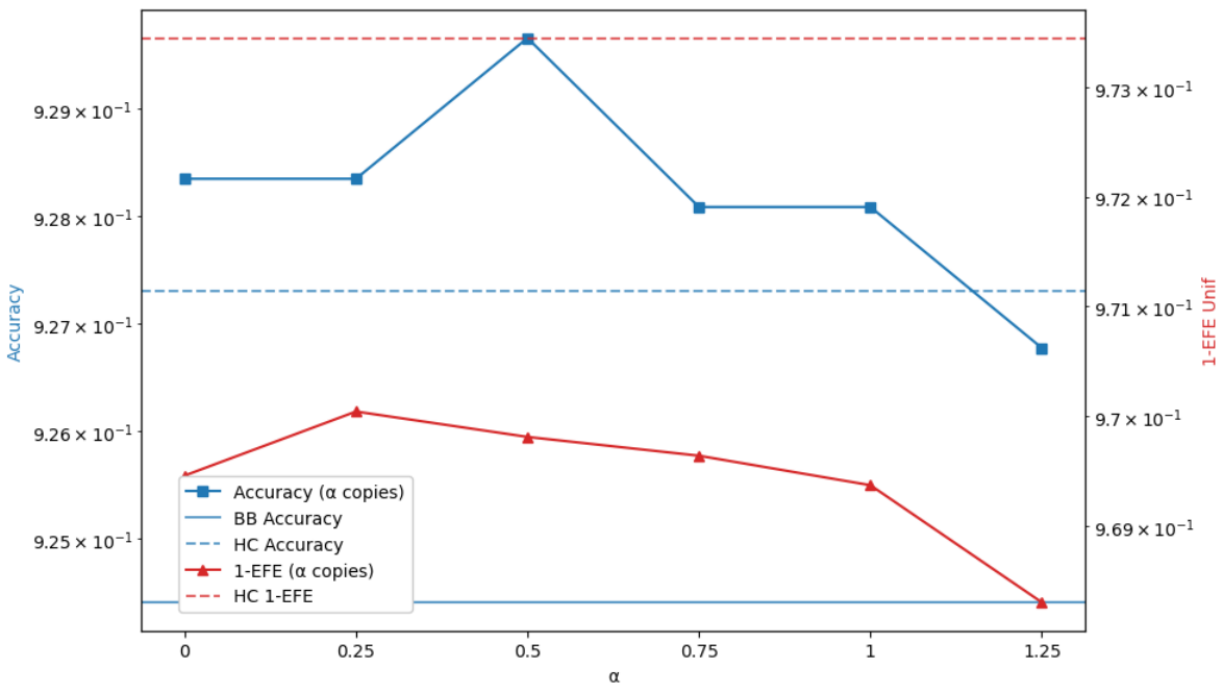}
		\caption{5-RF/SNN}
	\end{subfigure}
	\hfill
	\begin{subfigure}[t]{0.14\textheight}
		\includegraphics[width=\linewidth]{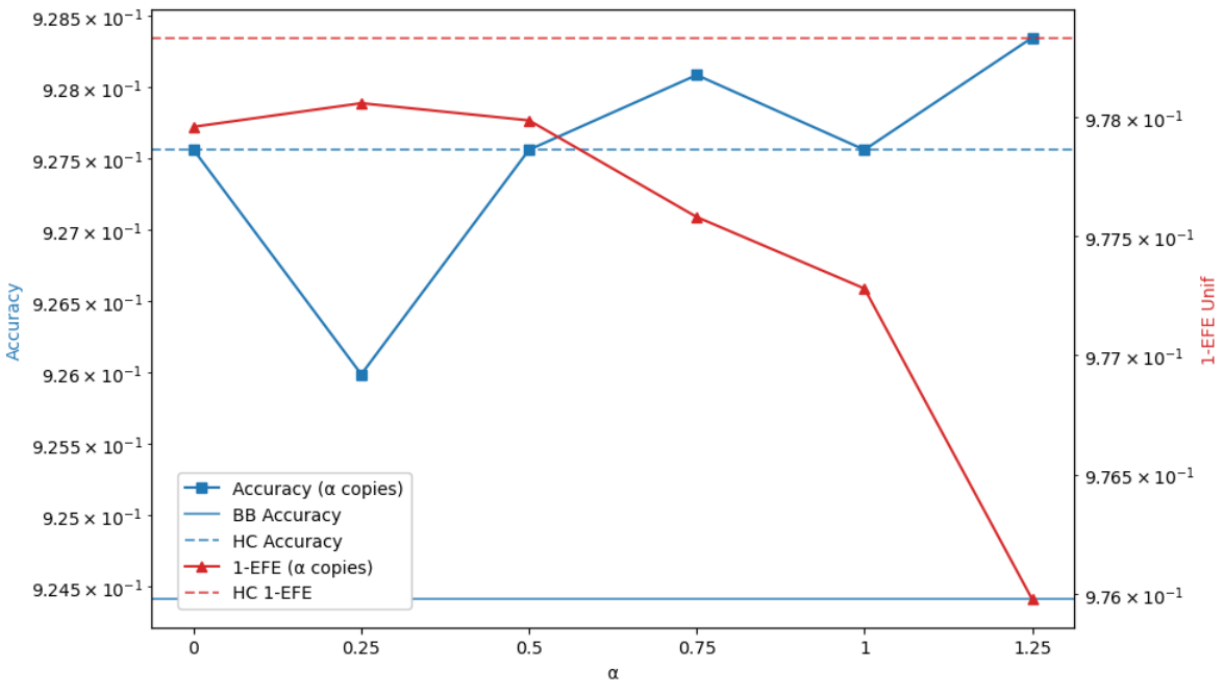}
		\caption{5-RF/MNN}
	\end{subfigure}
	\hfill
	\begin{subfigure}[t]{0.14\textheight}
		\includegraphics[width=\linewidth]{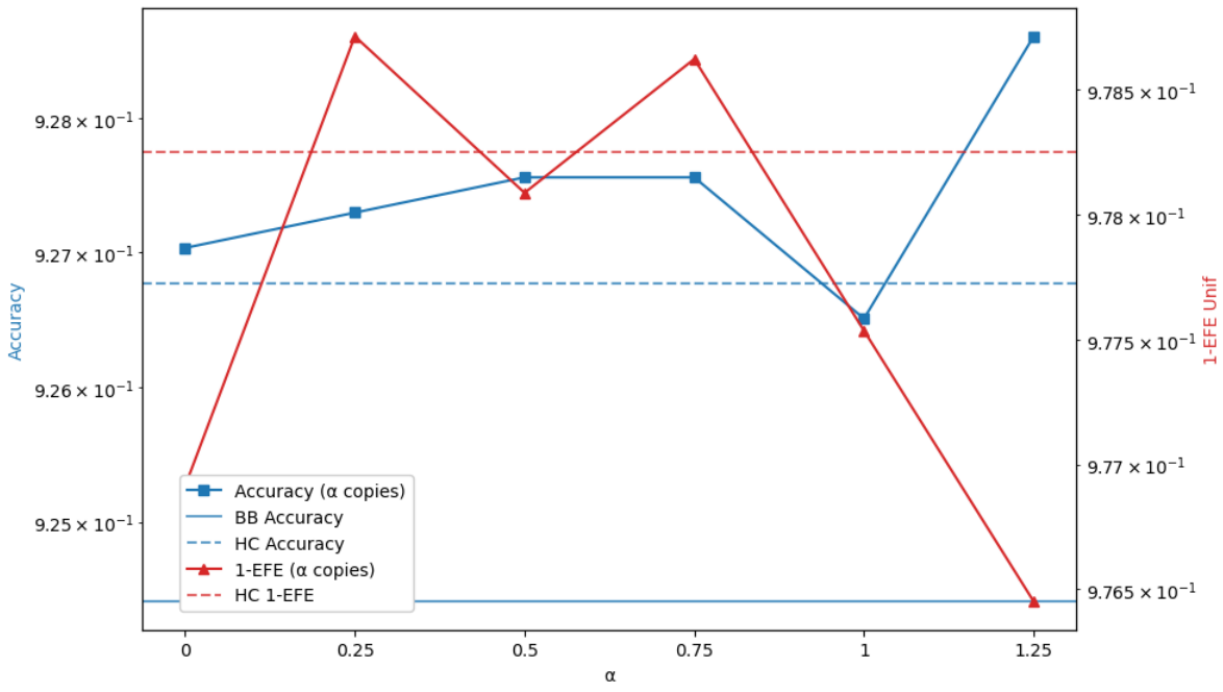}
		\caption{5-RF/LNN}
	\end{subfigure}
	\hfill
	\begin{subfigure}[t]{0.14\textheight}
		\includegraphics[width=\linewidth]{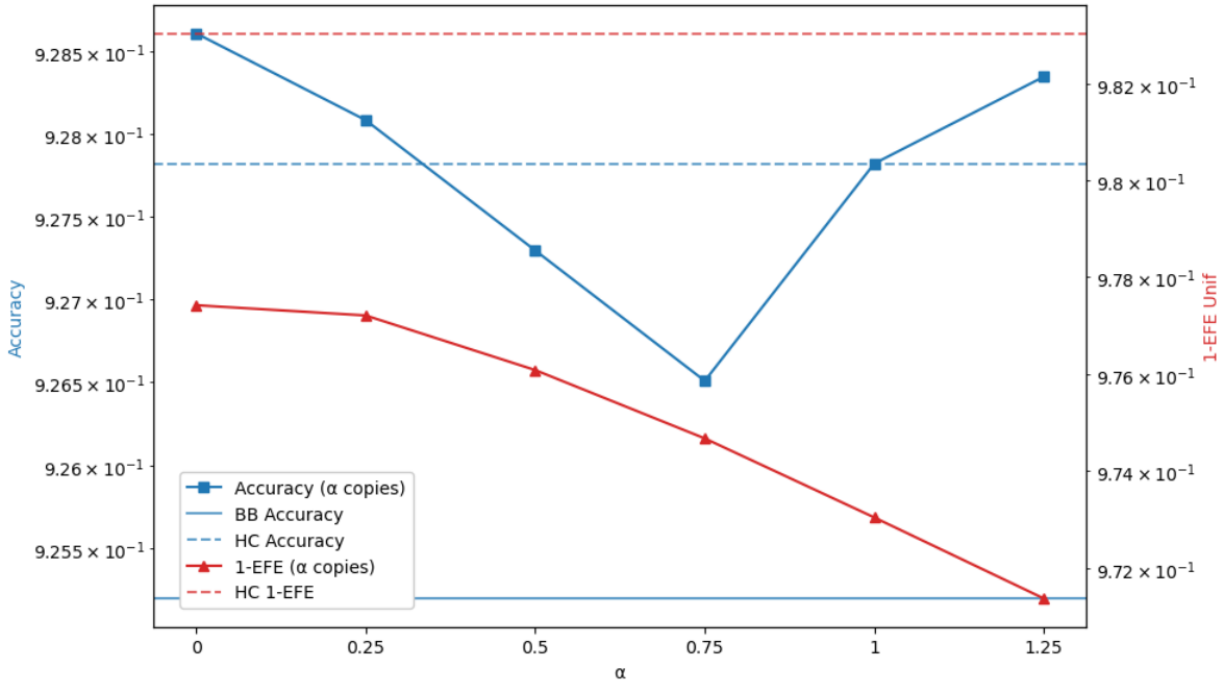}
		\caption{5-RF/GB}
	\end{subfigure}
	
	
	\begin{subfigure}[t]{0.14\textheight}
		\includegraphics[width=\linewidth]{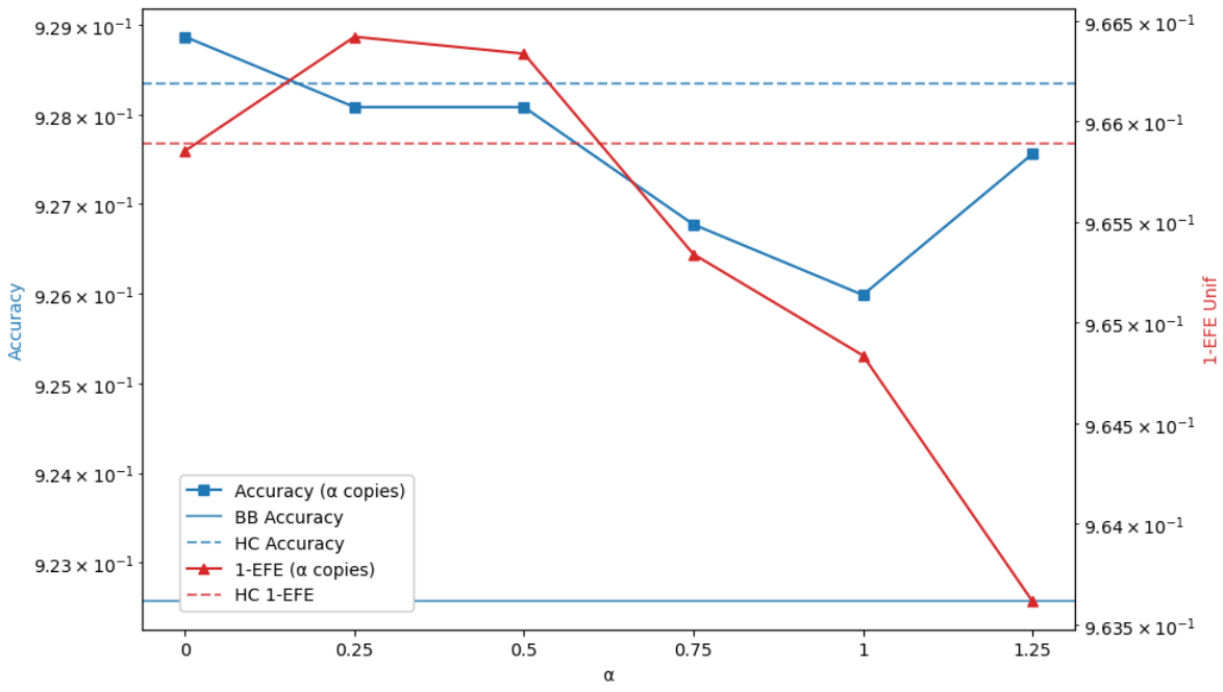}
		\caption{5-GB/SNN}
	\end{subfigure}
	\hfill
	\begin{subfigure}[t]{0.14\textheight}
		\includegraphics[width=\linewidth]{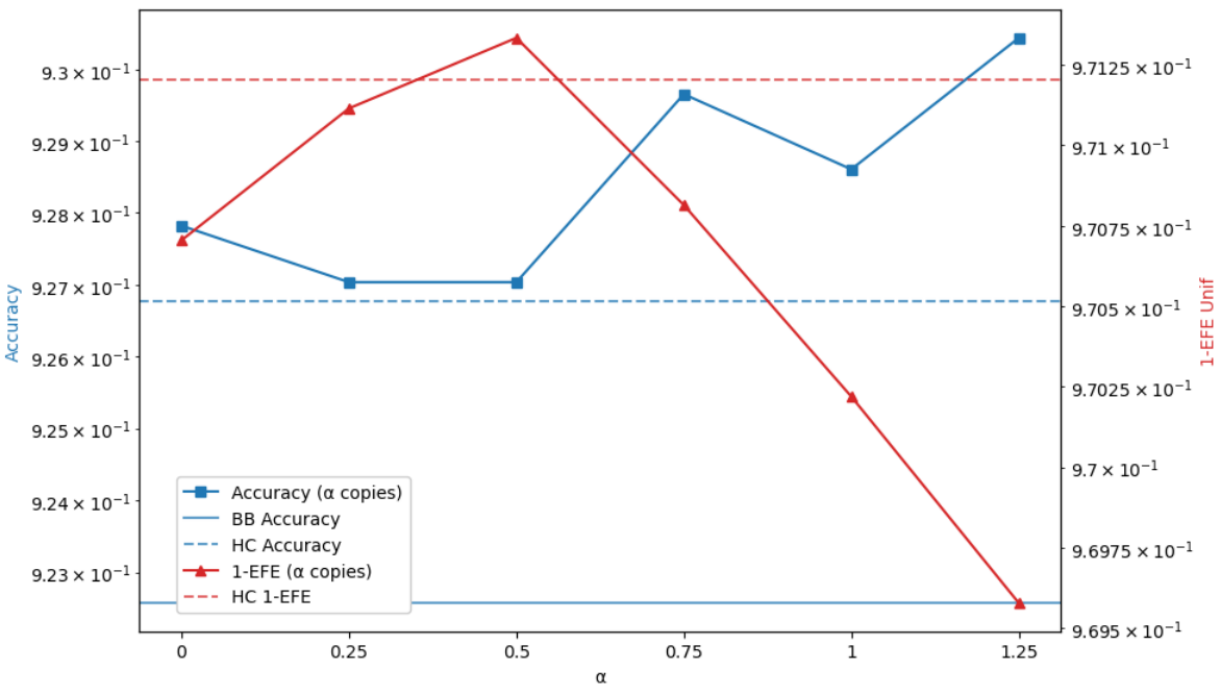}
		\caption{5-GB/MNN}
	\end{subfigure}
	\hfill
	\begin{subfigure}[t]{0.14\textheight}
		\includegraphics[width=\linewidth]{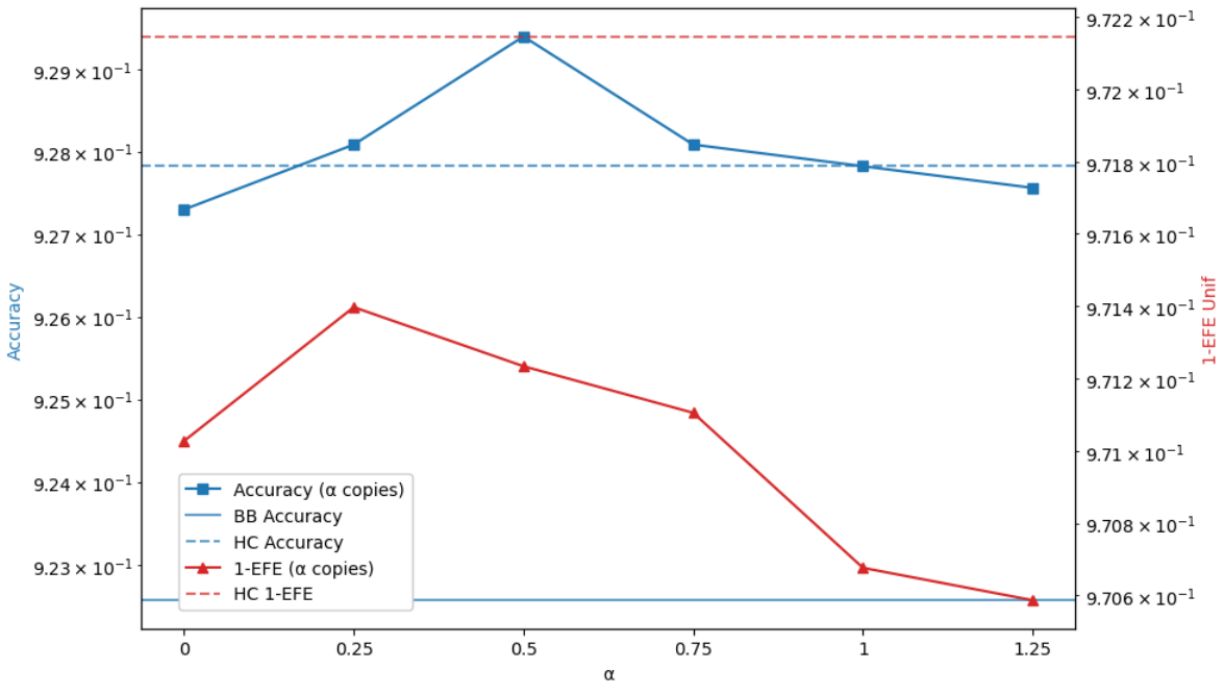}
		\caption{5-GB/LNN}
	\end{subfigure}
	\hfill
	\begin{subfigure}[t]{0.14\textheight}
		\includegraphics[width=\linewidth]{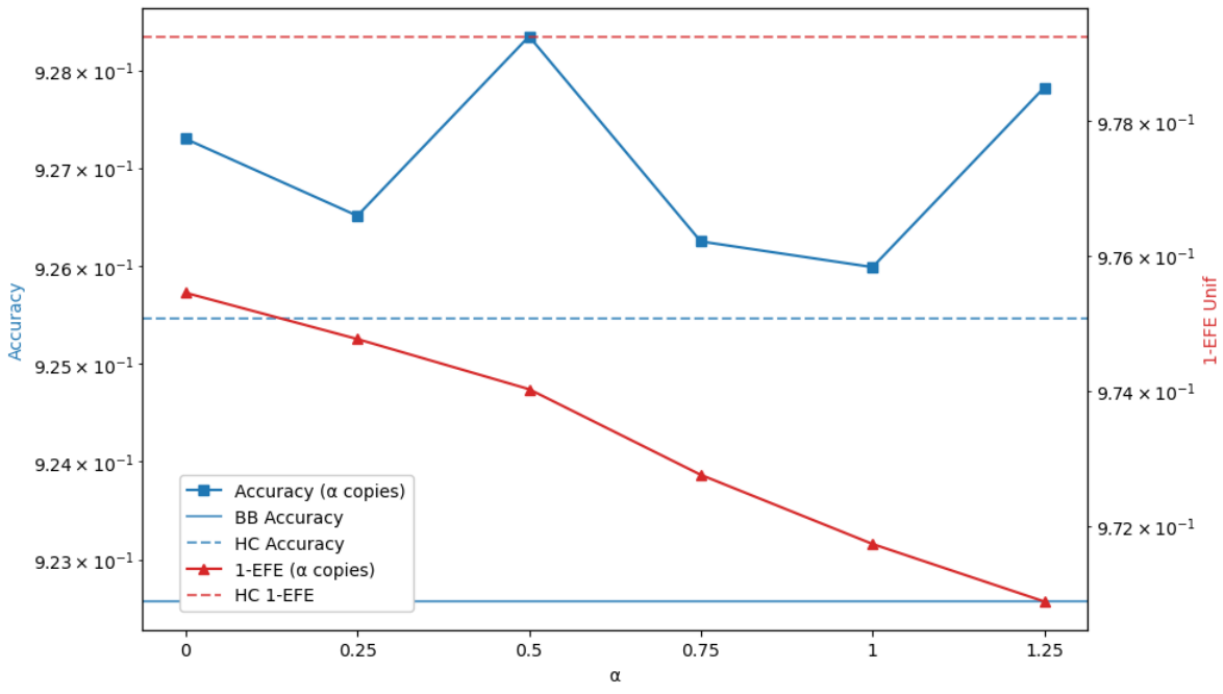}
		\caption{5-GB/GB}
	\end{subfigure}
	
	
	\begin{subfigure}[t]{0.14\textheight}
		\includegraphics[width=\linewidth]{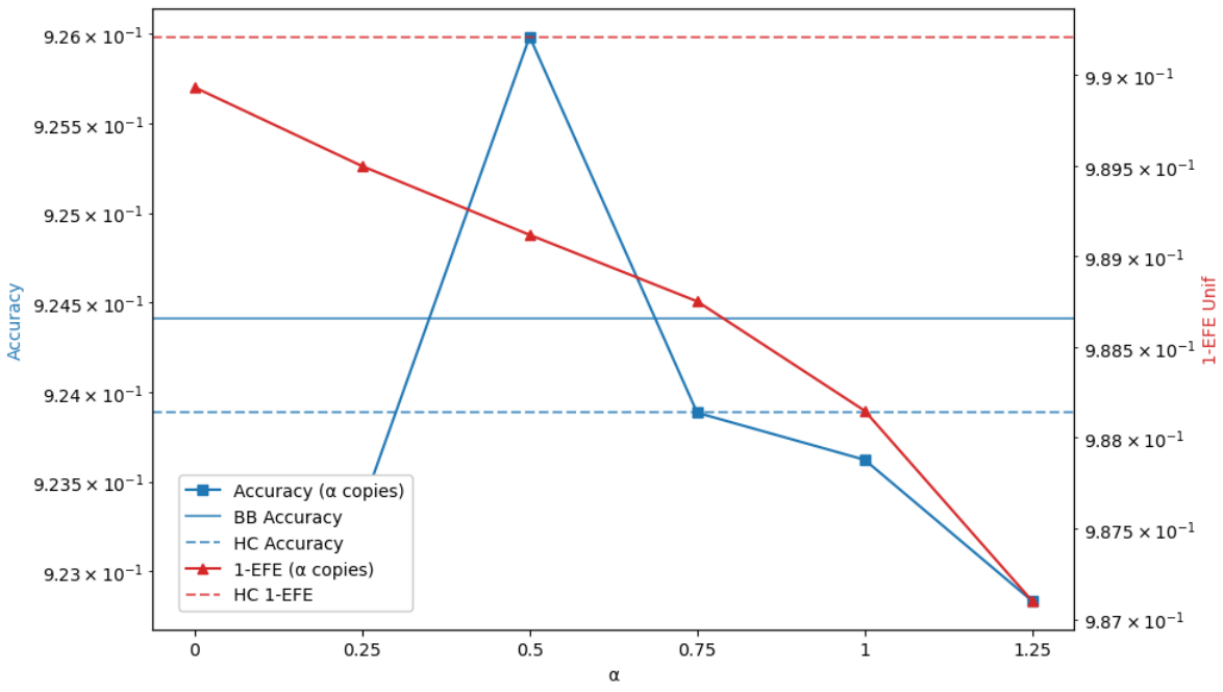}
		\caption{5-NN/SNN}
	\end{subfigure}
	\hfill
	\begin{subfigure}[t]{0.14\textheight}
		\includegraphics[width=\linewidth]{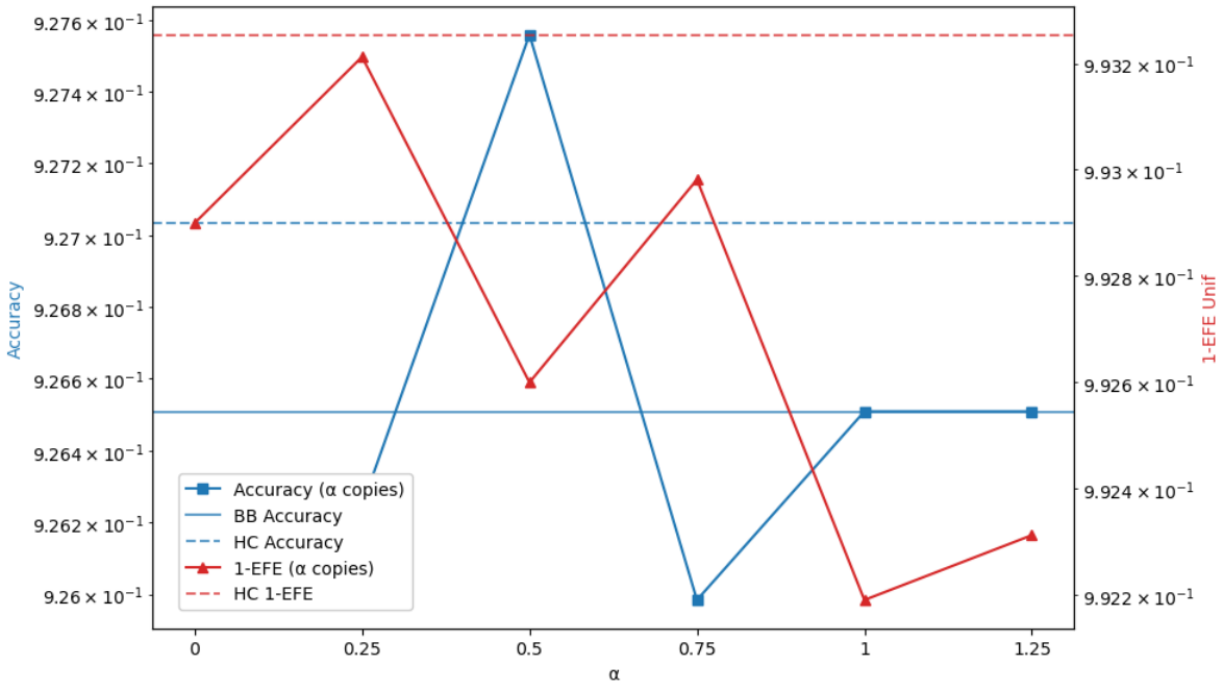}
		\caption{5-NN/MNN}
	\end{subfigure}
	\hfill
	\begin{subfigure}[t]{0.14\textheight}
		\includegraphics[width=\linewidth]{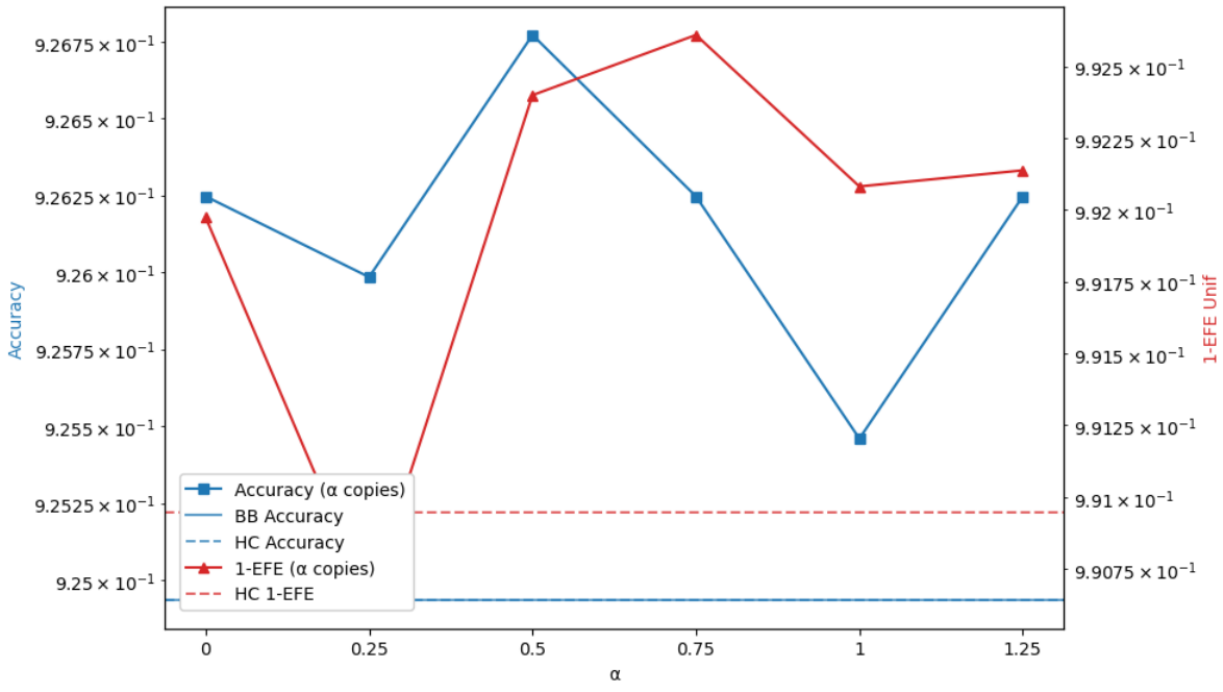}
		\caption{5-NN/LNN}
	\end{subfigure}
	\hfill
	\begin{subfigure}[t]{0.14\textheight}
		\includegraphics[width=\linewidth]{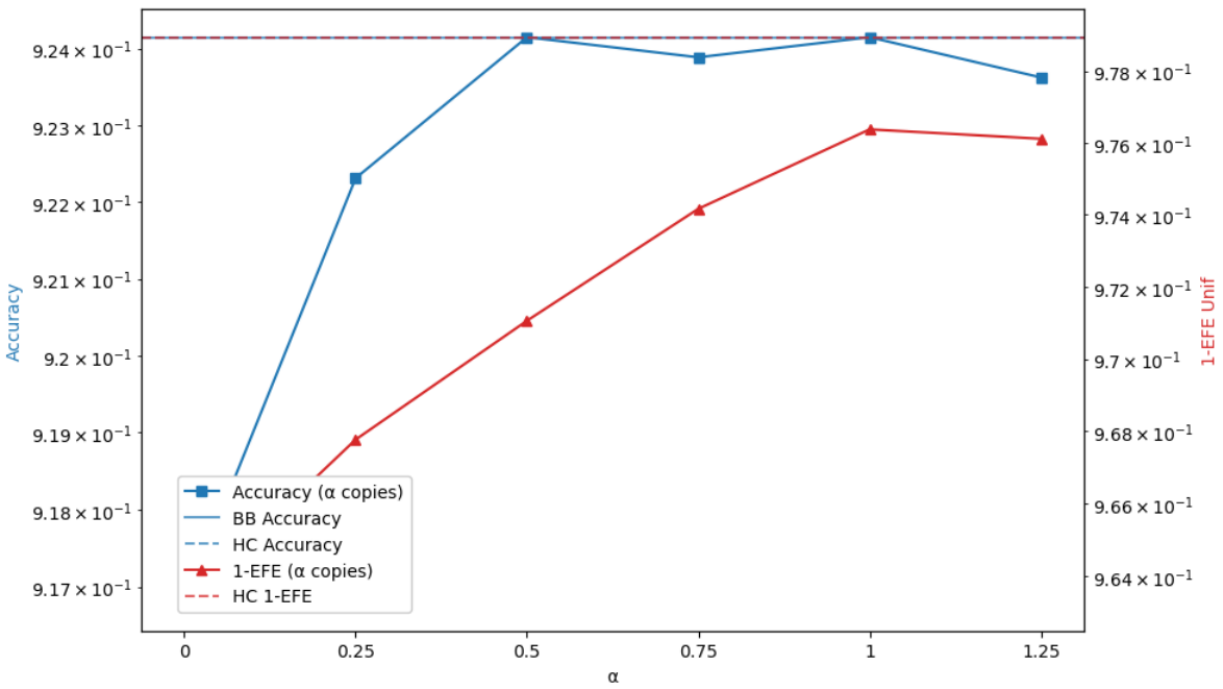}
		\caption{5-NN/GB}
	\end{subfigure}
	
	
	\begin{subfigure}[t]{0.14\textheight}
		\includegraphics[width=\linewidth]{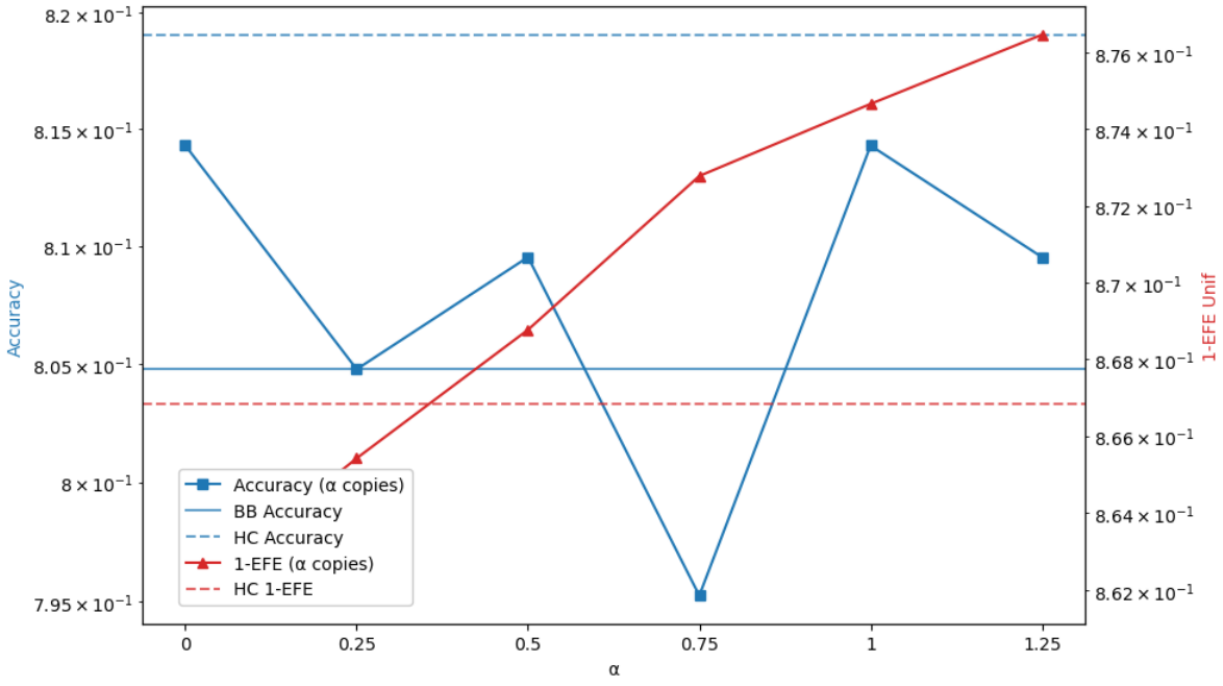}
		\caption{6-RF/SNN}
	\end{subfigure}
	\hfill
	\begin{subfigure}[t]{0.14\textheight}
		\includegraphics[width=\linewidth]{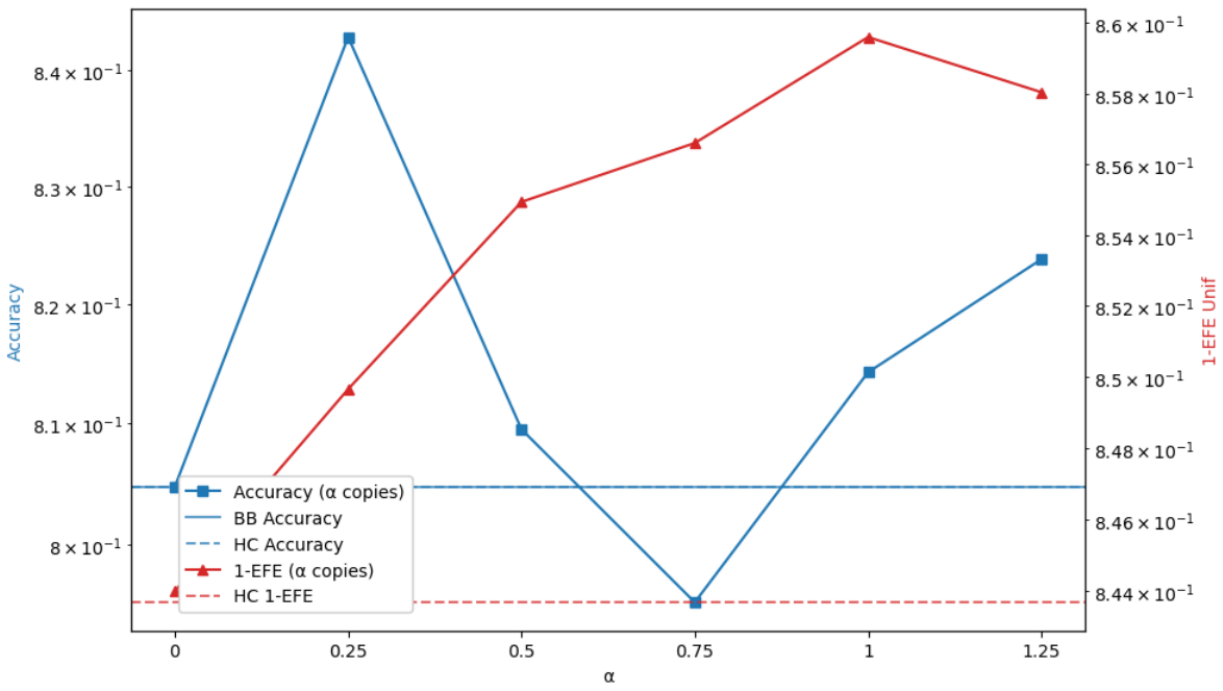}
		\caption{6-RF/MNN}
	\end{subfigure}
	\hfill
	\begin{subfigure}[t]{0.14\textheight}
		\includegraphics[width=\linewidth]{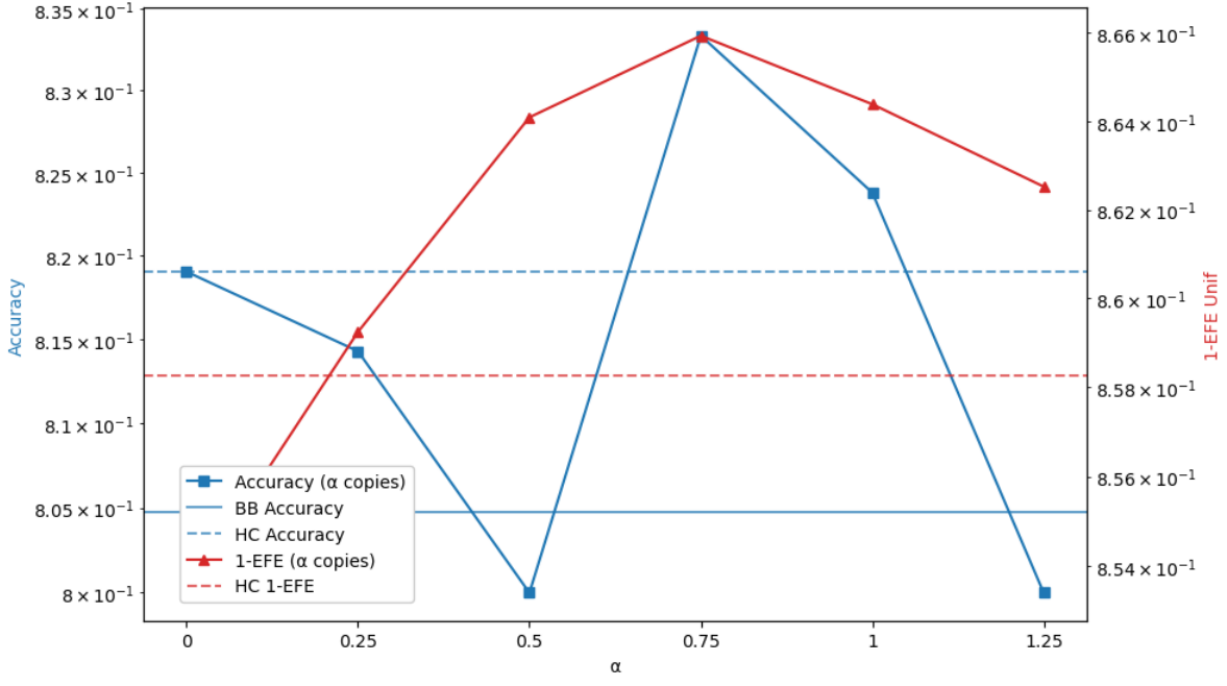}
		\caption{6-RF/LNN}
	\end{subfigure}
	\hfill
	\begin{subfigure}[t]{0.14\textheight}
		\includegraphics[width=\linewidth]{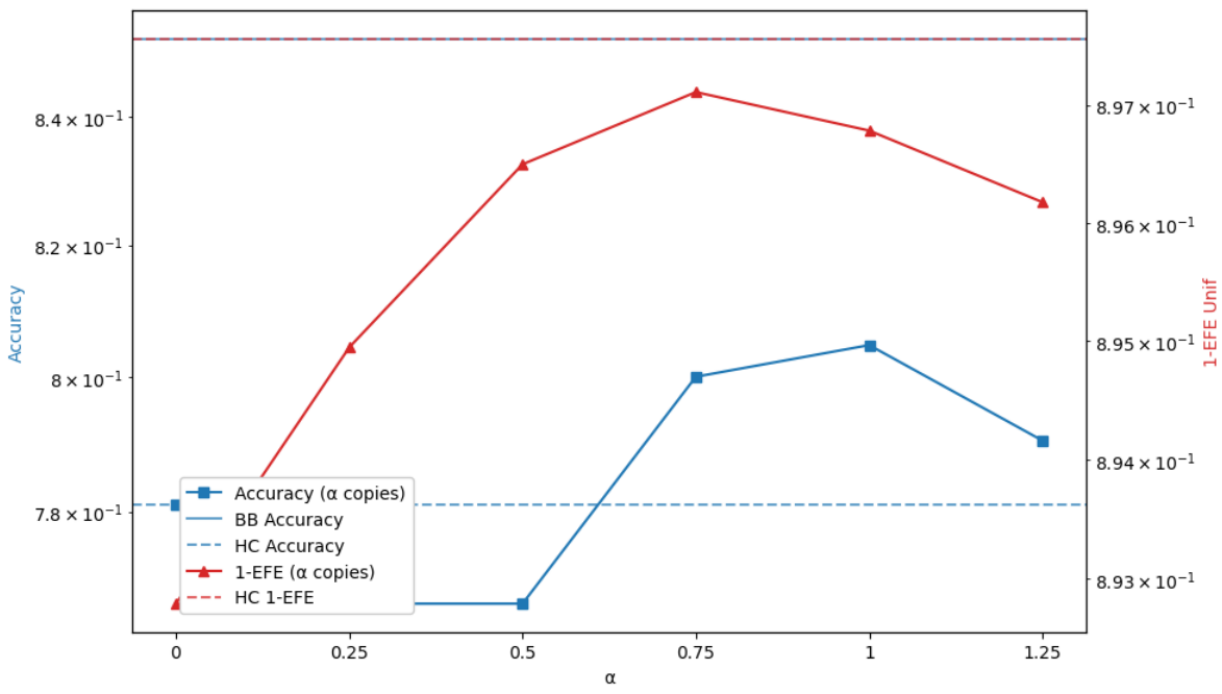}
		\caption{6-RF/GB}
	\end{subfigure}
	
	
	\begin{subfigure}[t]{0.14\textheight}
		\includegraphics[width=\linewidth]{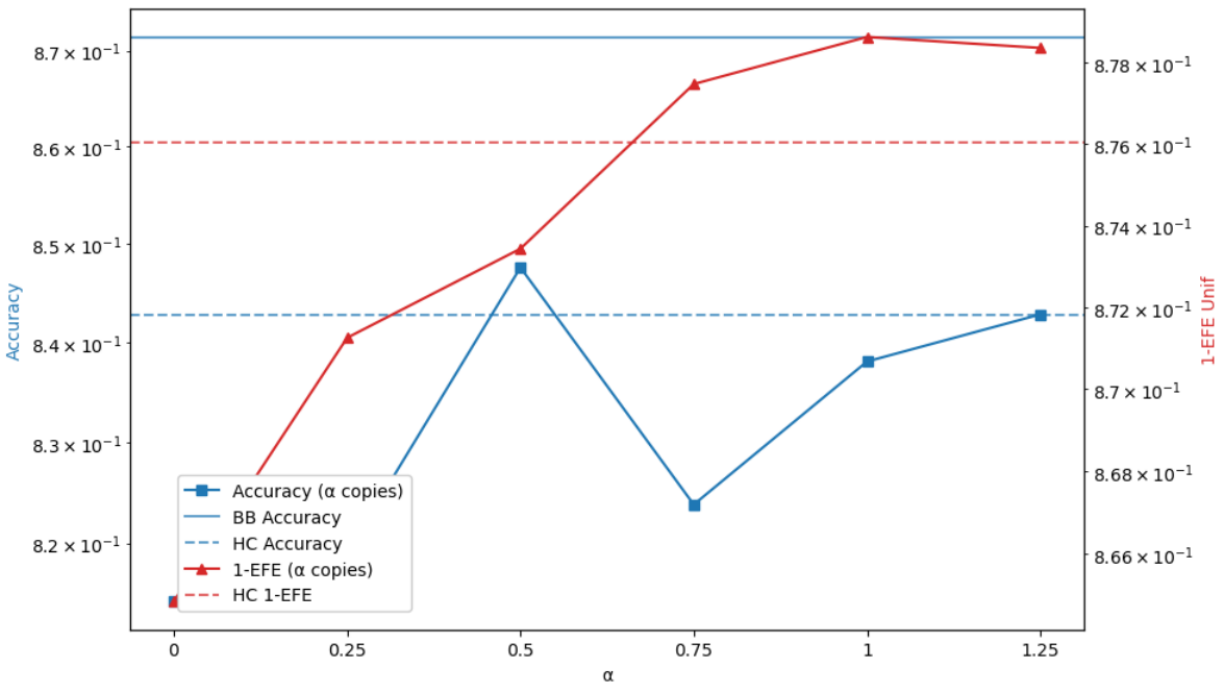}
		\caption{6-GB/SNN}
	\end{subfigure}
	\hfill
	\begin{subfigure}[t]{0.14\textheight}
		\includegraphics[width=\linewidth]{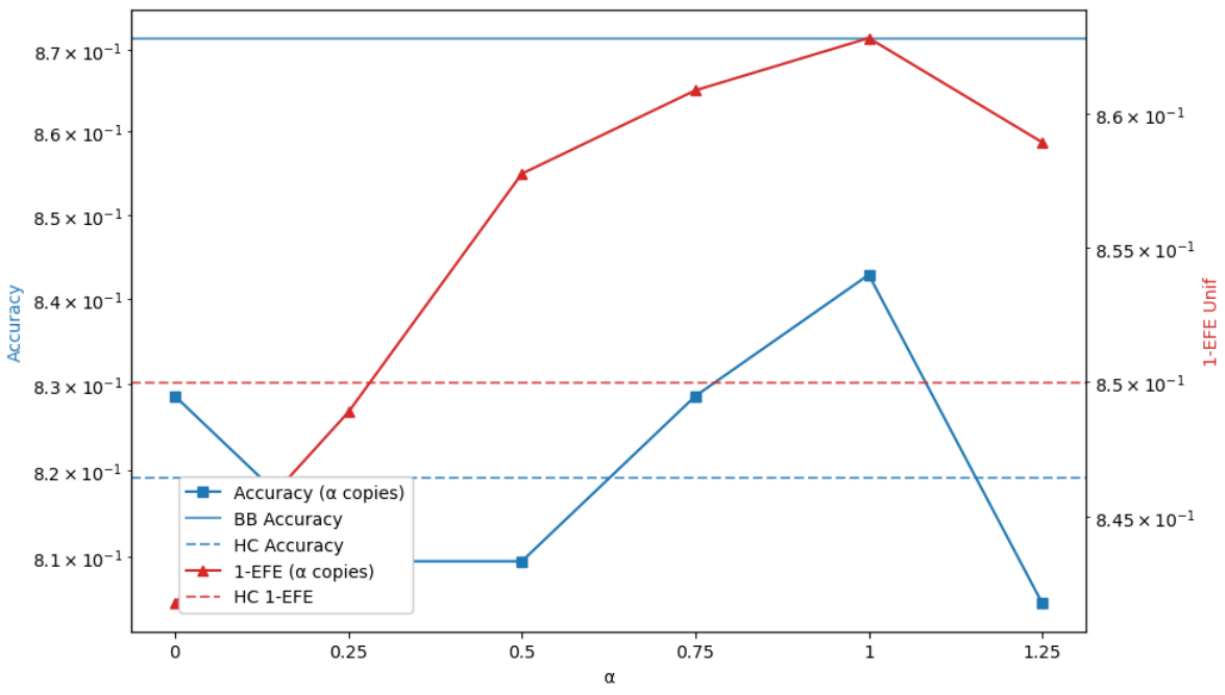}
		\caption{6-GB/MNN}
	\end{subfigure}
	\hfill
	\begin{subfigure}[t]{0.14\textheight}
		\includegraphics[width=\linewidth]{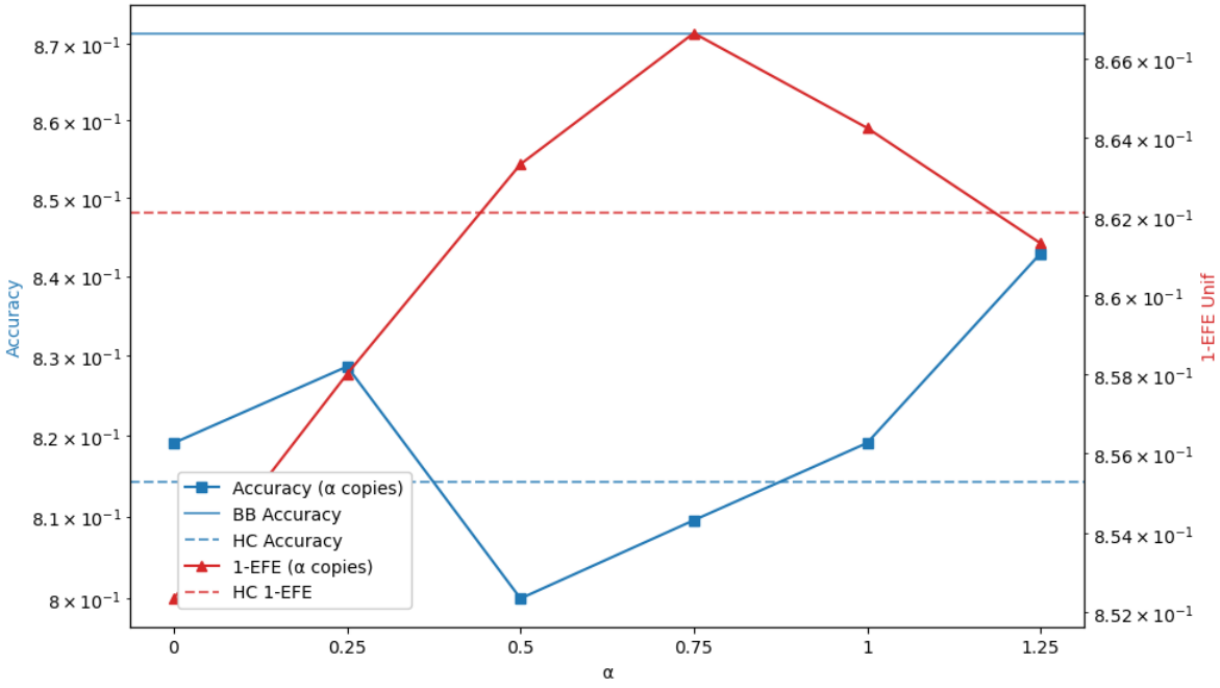}
		\caption{6-GB/LNN}
	\end{subfigure}
	\hfill
	\begin{subfigure}[t]{0.14\textheight}
		\includegraphics[width=\linewidth]{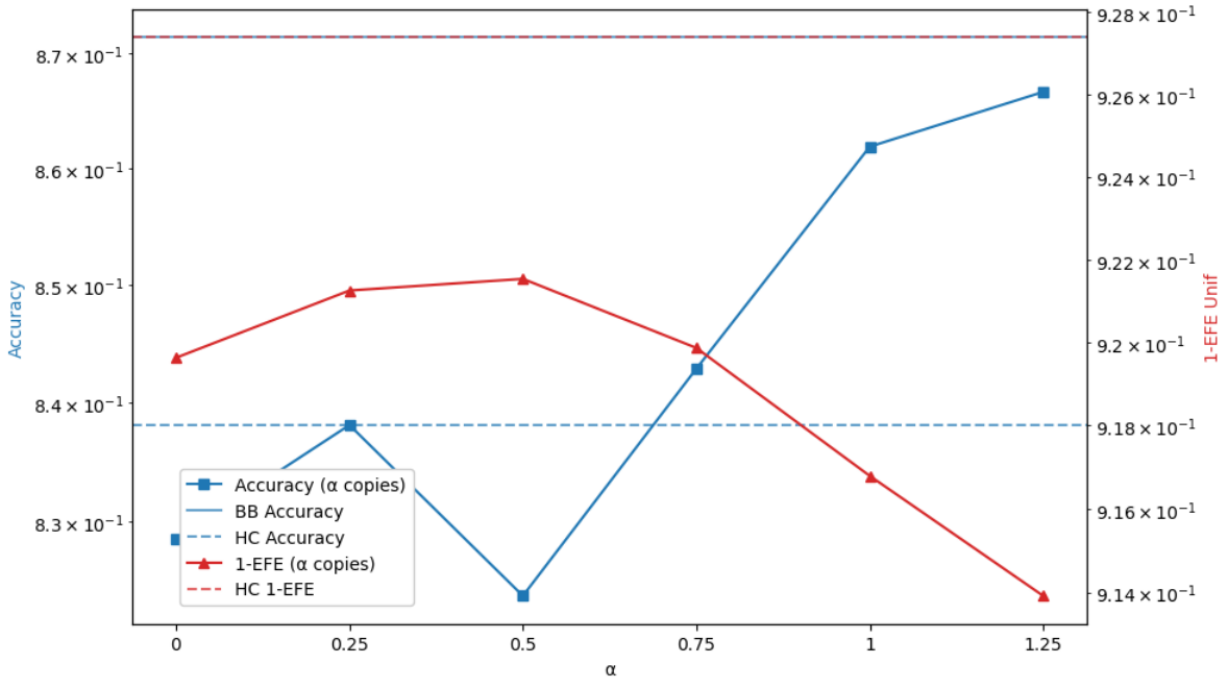}
		\caption{6-GB/GB}
	\end{subfigure}
	
	
	\begin{subfigure}[t]{0.14\textheight}
		\includegraphics[width=\linewidth]{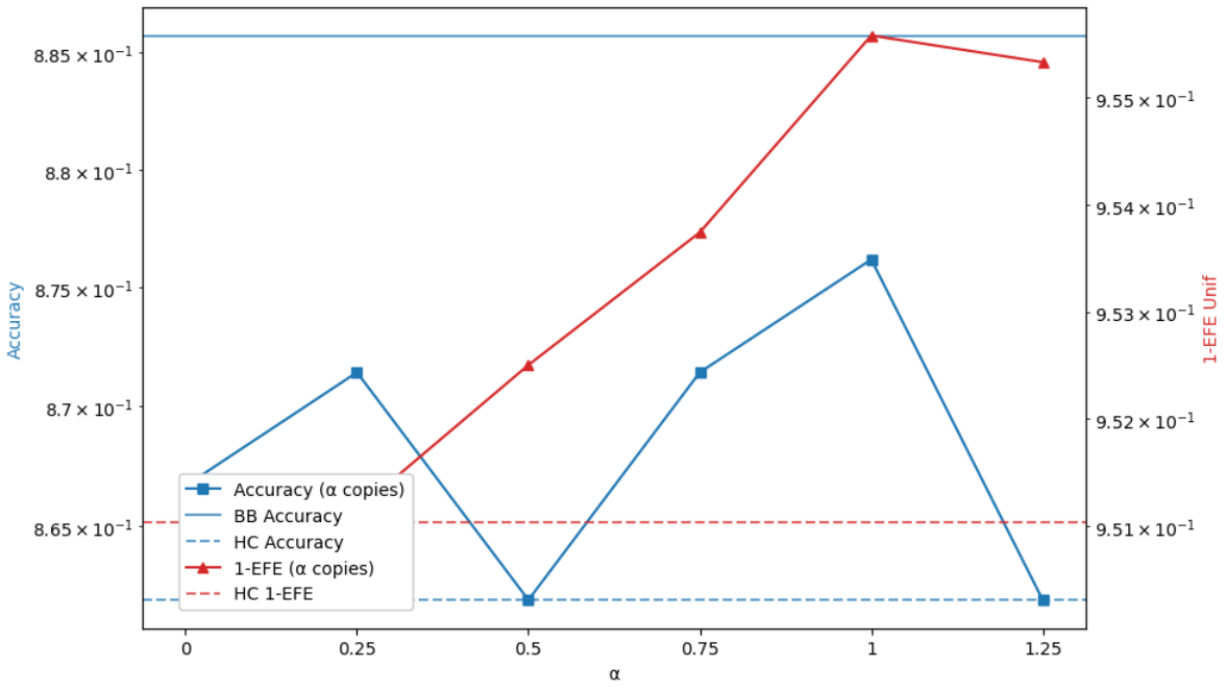}
		\caption{6-NN/SNN}
	\end{subfigure}
	\hfill
	\begin{subfigure}[t]{0.14\textheight}
		\includegraphics[width=\linewidth]{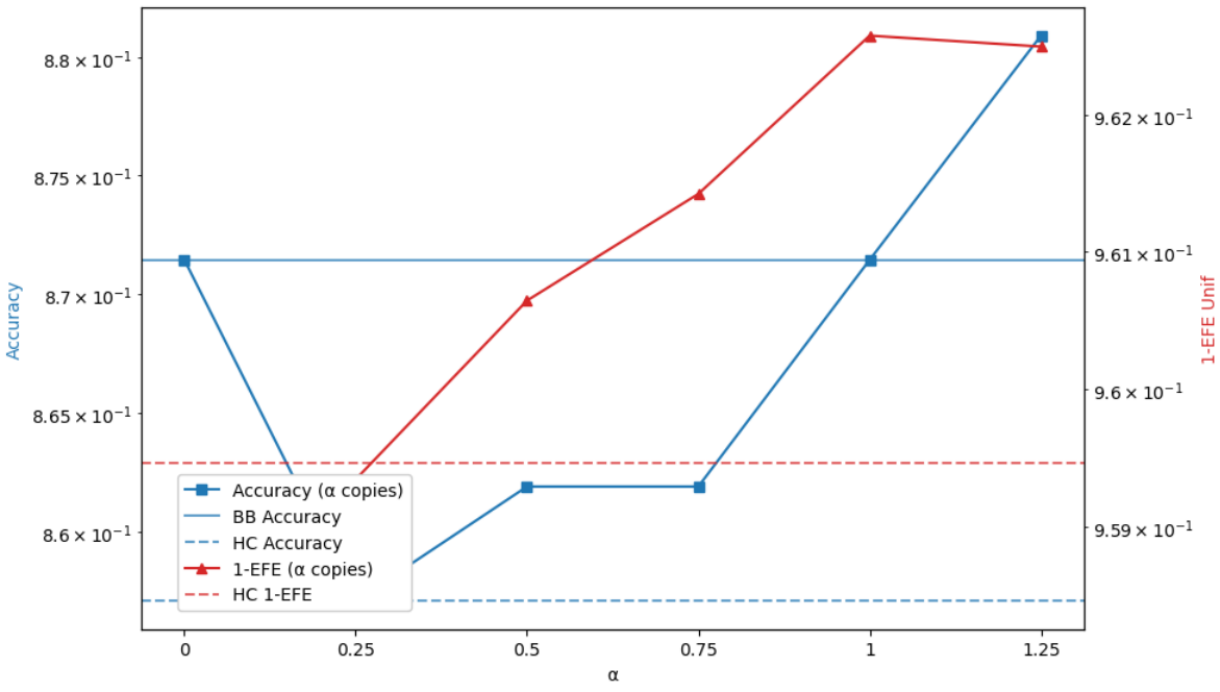}
		\caption{6-NN/MNN}
	\end{subfigure}
	\hfill
	\begin{subfigure}[t]{0.14\textheight}
		\includegraphics[width=\linewidth]{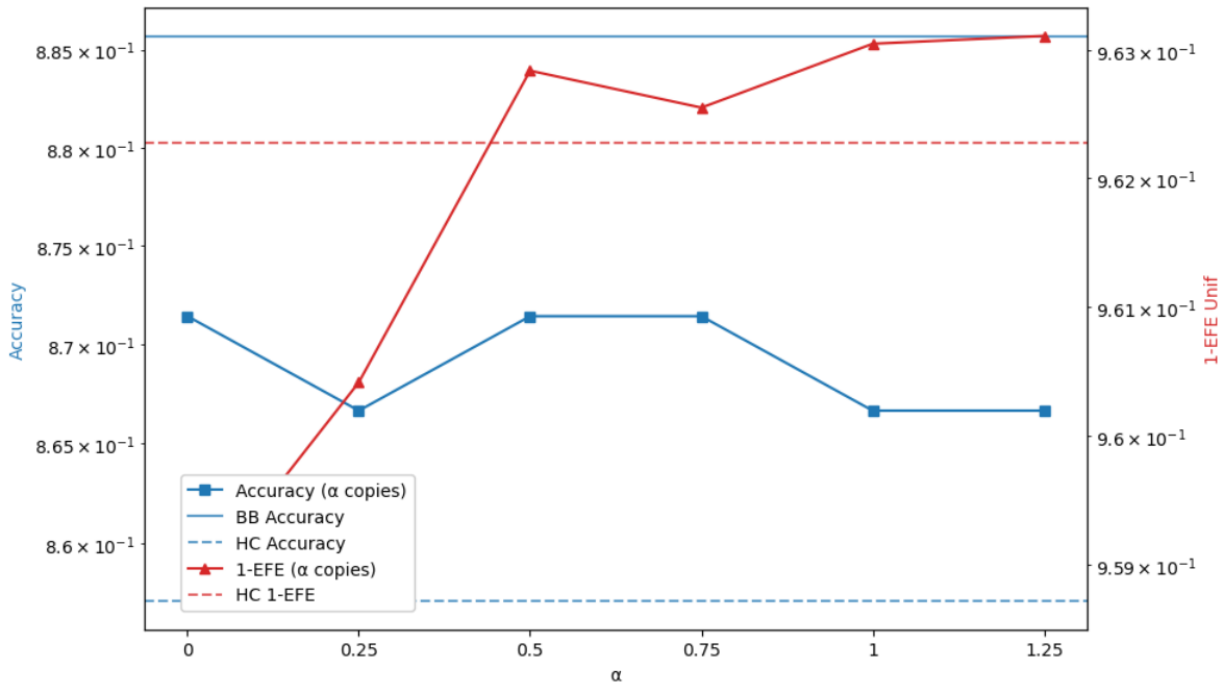}
		\caption{6-NN/LNN}
	\end{subfigure}
	\hfill
	\begin{subfigure}[t]{0.14\textheight}
		\includegraphics[width=\linewidth]{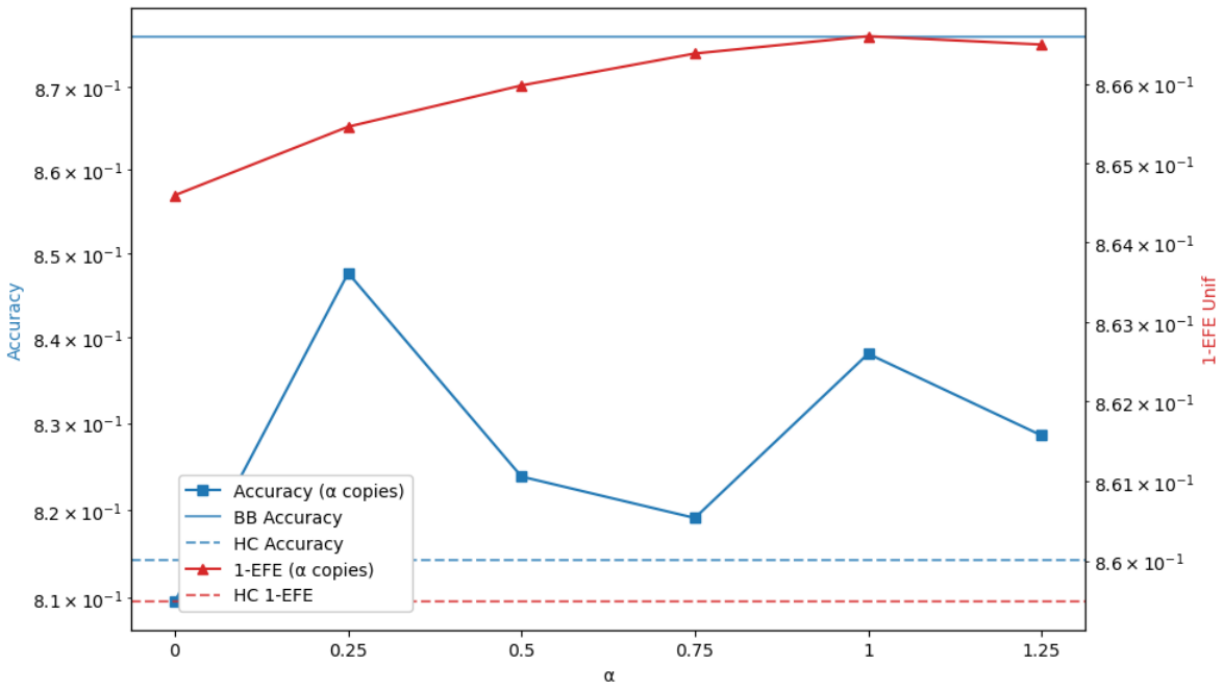}
		\caption{6-NN/GB}
	\end{subfigure}
\end{figure}

\newpage
\section{Tables corresponding to the MAE and RMSE of the approximate distances}
\label{appDistTabs}
This appendix shows the means and standard deviations for the estimated metrics for each algorithm on the different configurations and datasets over an independent synthetic sampling and test data.  

\noindent\textbf{Metric: MAE}
\vspace{3mm}

\begin{table}[!h]
\centering
\tiny
\setlength{\tabcolsep}{1.8pt}
\begin{tabular}{l c c *{6}{c} c c}
\toprule
Copy & $f_{\mathcal{O}}/f_{\mathcal{C}}$ & Metric &
\multicolumn{2}{c}{Dataset 1} &
\multicolumn{2}{c}{Dataset 2} &
\multicolumn{2}{c}{Dataset 3} &
Avg. Dat. 1-3 &
Avg. \\
\cmidrule(lr){4-9}
& & &
$\mathscr{D}_{\mathrm{te}}^{'}$ real data & $\mathcal{S}^{'}$ uniform data &
$\mathscr{D}_{\mathrm{te}}^{'}$ real data & $\mathcal{S}^{'}$ uniform data &
$\mathscr{D}_{\mathrm{te}}^{'}$ real data & $\mathcal{S}^{'}$ uniform data &
\\
\midrule
Alg. 1 copy & RF/SNN & MAE &
0.018$\pm$0.003 & 0.019$\pm$0.002 & 
0.027$\pm$0.004 & 0.024$\pm$0.002 &
0.028$\pm$0.001 & 0.031$\pm$0.003 &
0.024 & 0.163 \\
Alg. 2 copy & RF/SNN & MAE &
0.016$\pm$0.001 & 0.059$\pm$0.002 &
0.009$\pm$0.001 & 0.009$\pm$0.001 &
0.019$\pm$0.002 & 0.019$\pm$0.001 &
0.022 & 0.152\\
\midrule
Alg. 1 copy & GB/SNN & MAE &
0.043$\pm$0.003 & 0.040$\pm$0.004 & 
0.019$\pm$0.003 & 0.019$\pm$0.003 &
0.029$\pm$0.001 & 0.030$\pm$0.001 &
0.030 & 0.159\\
Alg. 2 copy & GB/SNN & MAE &
0.036$\pm$0.002 & 0.062$\pm$0.005 &
0.010$\pm$0.001 & 0.009$\pm$0.001 &
0.019$\pm$0.001 & 0.019$\pm$0.001 &
0.026 & 0.134\\
\midrule
Alg. 1 copy & NN/SNN & MAE &
0.006$\pm$0.001 & 0.007$\pm$0.001 & 
0.016$\pm$0.005 & 0.012$\pm$0.002 &
0.024$\pm$0.002 & 0.027$\pm$0.003 &
0.015 & 0.118\\
Alg. 2 copy & NN/SNN & MAE &
0.011$\pm$0.000 & 0.053$\pm$0.003 &
0.009$\pm$0.001 & 0.009$\pm$0.001 &
0.014$\pm$0.001 & 0.015$\pm$0.002 &
0.018 & 0.164\\
\midrule
Alg. 1 copy & RF/MNN & MAE &
0.017$\pm$0.001 & 0.019$\pm$0.002 & 
0.011$\pm$0.002 & 0.012$\pm$0.003 &
0.020$\pm$0.001 & 0.020$\pm$0.000 &
0.016 & 0.170\\
Alg. 2 copy & RF/MNN & MAE &
0.016$\pm$0.002 & 0.058$\pm$0.003 &
0.008$\pm$0.001 & 0.008$\pm$0.001 &
0.010$\pm$0.001 & 0.011$\pm$0.002 &
0.018 & 0.152\\
\midrule
Alg. 1 copy & GB/MNN & MAE &
0.039$\pm$0.003 & 0.037$\pm$0.004 & 
0.010$\pm$0.001 & 0.010$\pm$0.002 &
0.020$\pm$0.001 & 0.020$\pm$0.002 &
0.023 & 0.168\\
Alg. 2 copy & GB/MNN & MAE &
0.034$\pm$0.001 & 0.062$\pm$0.003 &
0.007$\pm$0.000 & 0.008$\pm$0.000 &
0.009$\pm$0.001 & 0.008$\pm$0.000 &
0.021 & 0.147\\
\midrule
Alg. 1 copy & NN/MNN & MAE &
0.006$\pm$0.002 & 0.008$\pm$0.003 & 
0.008$\pm$0.002 & 0.008$\pm$0.002 &
0.013$\pm$0.001 & 0.012$\pm$0.001 &
0.009 & 0.122\\
Alg. 2 copy & NN/MNN & MAE &
0.013$\pm$0.001 & 0.054$\pm$0.002 &
0.007$\pm$0.001 & 0.008$\pm$0.001 &
0.008$\pm$0.001 & 0.008$\pm$0.000 &
0.016 & 0.163\\
\midrule
Alg. 1 copy & RF/LNN & MAE &
0.017$\pm$0.002 & 0.019$\pm$0.005 & 
0.012$\pm$0.001 & 0.011$\pm$0.001 &
0.020$\pm$0.003 & 0.020$\pm$0.002 &
0.016 & 0.162\\
Alg. 2 copy & RF/LNN & MAE &
0.017$\pm$0.002 & 0.060$\pm$0.002 &
0.007$\pm$0.001 & 0.007$\pm$0.001 &
0.010$\pm$0.001 & 0.010$\pm$0.001 &
0.018 & 0.157\\
\midrule
Alg. 1 copy & GB/LNN & MAE &
0.037$\pm$0.003 & 0.043$\pm$0.010 & 
0.010$\pm$0.001 & 0.010$\pm$0.001 &
0.018$\pm$0.001 & 0.018$\pm$0.001 &
0.023 & 0.155\\
Alg. 2 copy & GB/LNN & MAE &
0.035$\pm$0.002 & 0.061$\pm$0.004 &
0.008$\pm$0.001 & 0.009$\pm$0.001 &
0.008$\pm$0.001 & 0.007$\pm$0.000 &
0.021 & 0.149\\
\midrule
Alg. 1 copy & NN/LNN & MAE &
0.008$\pm$0.005 & 0.009$\pm$0.003 & 
0.007$\pm$0.001 & 0.007$\pm$0.001 &
0.014$\pm$0.003 & 0.013$\pm$0.003 &
0.010 & 0.132\\
Alg. 2 copy & NN/LNN & MAE &
0.013$\pm$0.002 & 0.054$\pm$0.003 &
0.007$\pm$0.001 & 0.007$\pm$0.001 &
0.008$\pm$0.001 & 0.007$\pm$0.001 &
0.016 & 0.166\\
\midrule
Alg. 1 copy & RF/GB & MAE &
0.019$\pm$0.002 & 0.019$\pm$0.002 & 
0.011$\pm$0.000 & 0.013$\pm$0.002 &
0.021$\pm$0.001 & 0.020$\pm$0.002 &
0.017 & 0.151\\
Alg. 2 copy & RF/GB & MAE &
0.016$\pm$0.001 & 0.058$\pm$0.003 &
0.008$\pm$0.000 & 0.008$\pm$0.000 &
0.013$\pm$0.000 & 0.014$\pm$0.001 &
0.020 & 0.159\\
\midrule
Alg. 1 copy & GB/GB & MAE &
0.040$\pm$0.001 & 0.037$\pm$0.004 & 
0.012$\pm$0.001 & 0.012$\pm$0.001 &
0.018$\pm$0.001 & 0.018$\pm$0.001 &
0.023 & 0.131\\
Alg. 2 copy & GB/GB & MAE &
0.035$\pm$0.002 & 0.065$\pm$0.003 &
0.009$\pm$0.000 & 0.008$\pm$0.000 &
0.013$\pm$0.000 & 0.014$\pm$0.001 &
0.024 & 0.155\\
\midrule
Alg. 1 copy & NN/GB & MAE &
0.010$\pm$0.005 & 0.009$\pm$0.001 & 
0.017$\pm$0.001 & 0.012$\pm$0.001 &
0.015$\pm$0.001 & 0.014$\pm$0.001 &
0.013 & 0.211\\
Alg. 2 copy & NN/GB & MAE &
0.012$\pm$0.001 & 0.053$\pm$0.003 &
0.014$\pm$0.001 & 0.010$\pm$0.000 &
0.013$\pm$0.001 & 0.013$\pm$0.001 &
0.019 & 0.208\\
\bottomrule
\end{tabular}
\end{table}

\vspace{10mm}

\begin{table}[!h]
\centering
\tiny
\setlength{\tabcolsep}{1.8pt}
\begin{tabular}{l c c *{6}{c} c c}
\toprule
Copy & $f_{\mathcal{O}}/f_{\mathcal{C}}$ & Metric &
\multicolumn{2}{c}{Dataset 4} &
\multicolumn{2}{c}{Dataset 5} &
\multicolumn{2}{c}{Dataset 6} &
Avg. Dat. 4-6 &
Avg. \\
\cmidrule(lr){4-9}
& & &
$\mathscr{D}_{\mathrm{te}}^{'}$ real data & $\mathcal{S}^{'}$ uniform data &
$\mathscr{D}_{\mathrm{te}}^{'}$ real data & $\mathcal{S}^{'}$ uniform data &
$\mathscr{D}_{\mathrm{te}}^{'}$ real data & $\mathcal{S}^{'}$ uniform data &
\\
\midrule
Alg. 1 copy & RF/SNN & MAE &
0.342$\pm$0.026 & 0.295$\pm$0.023 & 
0.083$\pm$0.007 & 0.087$\pm$0.006 &
0.493$\pm$0.089 & 0.514$\pm$0.038 &
0.302 & 0.163\\
Alg. 2 copy & RF/SNN & MAE &
0.295$\pm$0.049 & 0.240$\pm$0.005 &
0.131$\pm$0.020 & 0.146$\pm$0.027 &
0.529$\pm$0.091 & 0.344$\pm$0.014 &
0.281 & 0.151\\
\midrule
Alg. 1 copy & GB/SNN & MAE &
0.246$\pm$0.030 & 0.299$\pm$0.018 & 
0.121$\pm$0.013 & 0.088$\pm$0.026 &
0.461$\pm$0.090 & 0.510$\pm$0.046 &
0.288 & 0.159\\
Alg. 2 copy & GB/SNN & MAE &
0.239$\pm$0.024 & 0.231$\pm$0.018 &
0.180$\pm$0.022 & 0.139$\pm$0.012 &
0.333$\pm$0.054 & 0.326$\pm$0.034 &
0.241 & 0.134\\
\midrule
Alg. 1 copy & NN/SNN & MAE &
0.185$\pm$0.042 & 0.177$\pm$0.035 & 
0.057$\pm$0.008 & 0.053$\pm$0.003 &
0.486$\pm$0.057 & 0.370$\pm$0.047 &
0.221 & 0.118\\
Alg. 2 copy & NN/SNN & MAE &
0.290$\pm$0.056 & 0.333$\pm$0.074 &
0.151$\pm$0.015 & 0.154$\pm$0.013 &
0.597$\pm$0.065 & 0.328$\pm$0.025 &
0.309 & 0.164\\
\midrule
Alg. 1 copy & RF/MNN & MAE &
0.410$\pm$0.033 & 0.309$\pm$0.019 & 
0.082$\pm$0.017 & 0.085$\pm$0.016 &
0.591$\pm$0.087 & 0.463$\pm$0.035 &
0.323 & 0.170\\
Alg. 2 copy & RF/MNN & MAE &
0.303$\pm$0.030 & 0.238$\pm$0.009 &
0.141$\pm$0.011 & 0.150$\pm$0.009 &
0.461$\pm$0.082 & 0.416$\pm$0.028 &
0.285 & 0.152\\
\midrule
Alg. 1 copy & GB/MNN & MAE &
0.310$\pm$0.067 & 0.327$\pm$0.027 & 
0.106$\pm$0.009 & 0.097$\pm$0.009 &
0.559$\pm$0.072 & 0.474$\pm$0.048 &
0.312 & 0.168\\
Alg. 2 copy & GB/MNN & MAE &
0.261$\pm$0.040 & 0.261$\pm$0.009 &
0.204$\pm$0.018 & 0.166$\pm$0.018 &
0.371$\pm$0.041 & 0.372$\pm$0.040 &
0.272 & 0.147\\
\midrule
Alg. 1 copy & NN/MNN & MAE &
0.171$\pm$0.019 & 0.218$\pm$0.077 & 
0.060$\pm$0.006 & 0.053$\pm$0.005 &
0.542$\pm$0.087 & 0.360$\pm$0.055 &
0.234 & 0.122\\
Alg. 2 copy & NN/MNN & MAE &
0.302$\pm$0.038 & 0.316$\pm$0.036 &
0.167$\pm$0.017 & 0.152$\pm$0.017 &
0.621$\pm$0.046 & 0.303$\pm$0.018 &
0.310 & 0.163\\
\midrule
Alg. 1 copy & RF/LNN & MAE &
0.385$\pm$0.075 & 0.310$\pm$0.028 & 
0.090$\pm$0.030 & 0.091$\pm$0.015 &
0.523$\pm$0.072 & 0.445$\pm$0.037 &
0.307 & 0.162\\
Alg. 2 copy & RF/LNN & MAE &
0.392$\pm$0.041 & 0.240$\pm$0.017 &
0.128$\pm$0.006 & 0.129$\pm$0.014 &
0.504$\pm$0.053 & 0.386$\pm$0.033 &
0.296 & 0.157\\
\midrule
Alg. 1 copy & GB/LNN & MAE &
0.258$\pm$0.027 & 0.311$\pm$0.012 & 
0.121$\pm$0.020 & 0.107$\pm$0.014 &
0.486$\pm$0.071 & 0.430$\pm$0.019 &
0.286 & 0.155\\
Alg. 2 copy & GB/LNN & MAE &
0.270$\pm$0.026 & 0.265$\pm$0.026 &
0.204$\pm$0.019 & 0.149$\pm$0.011 &
0.418$\pm$0.066 & 0.348$\pm$0.027 &
0.276 & 0.149\\
\midrule
Alg. 1 copy & NN/LNN & MAE &
0.241$\pm$0.053 & 0.225$\pm$0.038 & 
0.057$\pm$0.007 & 0.052$\pm$0.005 &
0.583$\pm$0.079 & 0.368$\pm$0.046 &
0.254 & 0.132\\
Alg. 2 copy & NN/LNN & MAE &
0.311$\pm$0.036 & 0.317$\pm$0.023 &
0.156$\pm$0.011 & 0.137$\pm$0.009 &
0.622$\pm$0.070 & 0.356$\pm$0.060 &
0.316 & 0.166\\
\midrule
Alg. 1 copy & RF/GB & MAE &
0.339$\pm$0.043 & 0.224$\pm$0.015 & 
0.086$\pm$0.006 & 0.079$\pm$0.006 &
0.561$\pm$0.078 & 0.418$\pm$0.038 &
0.285 & 0.151\\
Alg. 2 copy & RF/GB & MAE &
0.490$\pm$0.046 & 0.215$\pm$0.013 &
0.119$\pm$0.012 & 0.131$\pm$0.013 &
0.503$\pm$0.063 & 0.330$\pm$0.015 &
0.298 & 0.159\\
\midrule
Alg. 1 copy & GB/GB & MAE &
0.253$\pm$0.024 & 0.221$\pm$0.026 & 
0.111$\pm$0.010 & 0.077$\pm$0.008 &
0.395$\pm$0.068 & 0.370$\pm$0.041 &
0.238 & 0.131\\
Alg. 2 copy & GB/GB & MAE &
0.408$\pm$0.059 & 0.232$\pm$0.011 &
0.205$\pm$0.024 & 0.148$\pm$0.017 &
0.442$\pm$0.031 & 0.276$\pm$0.016 &
0.285 & 0.155\\
\midrule
Alg. 1 copy & NN/GB & MAE &
0.359$\pm$0.077 & 0.382$\pm$0.040 & 
0.075$\pm$0.010 & 0.071$\pm$0.005 &
0.943$\pm$0.125 & 0.626$\pm$0.042 &
0.409 & 0.211\\
Alg. 2 copy & NN/GB & MAE &
0.350$\pm$0.042 & 0.317$\pm$0.023 &
0.150$\pm$0.025 & 0.119$\pm$0.011 &
0.856$\pm$0.069 & 0.587$\pm$0.041 &
0.396 & 0.208\\
\bottomrule
\end{tabular}
\end{table}
\newpage

\vspace{3mm}
\noindent\textbf{Metric: RMSE}
\vspace{3mm}

\begin{table}[!h]
\centering
\tiny
\setlength{\tabcolsep}{1.8pt}
\begin{tabular}{l c c *{6}{c} c c}
\toprule
Copy & $f_{\mathcal{O}}/f_{\mathcal{C}}$ & Metric &
\multicolumn{2}{c}{Dataset 1} &
\multicolumn{2}{c}{Dataset 2} &
\multicolumn{2}{c}{Dataset 3} &
Avg. Dat. 1-3 &
Avg. \\
\cmidrule(lr){4-9}
& & &
$\mathscr{D}_{\mathrm{te}}^{'}$ real data & $\mathcal{S}^{'}$ uniform data &
$\mathscr{D}_{\mathrm{te}}^{'}$ real data & $\mathcal{S}^{'}$ uniform data &
$\mathscr{D}_{\mathrm{te}}^{'}$ real data & $\mathcal{S}^{'}$ uniform data &
\\
\midrule
Alg. 1 copy & RF/SNN & RMSE &
0.024$\pm$0.003 & 0.026$\pm$0.004 & 
0.035$\pm$0.005 & 0.033$\pm$0.003 &
0.037$\pm$0.001 & 0.039$\pm$0.004 &
0.032 & 0.211 \\
Alg. 2 copy & RF/SNN & RMSE &
0.024$\pm$0.002 & 0.118$\pm$0.001 &
0.011$\pm$0.001 & 0.012$\pm$0.001 &
0.024$\pm$0.002 & 0.025$\pm$0.001 &
0.036 & 0.211 \\
\midrule
Alg. 1 copy & GB/SNN & RMSE &
0.058$\pm$0.004 & 0.055$\pm$0.004 & 
0.025$\pm$0.004 & 0.026$\pm$0.003 &
0.037$\pm$0.001 & 0.038$\pm$0.001 &
0.040 & 0.205 \\
Alg. 2 copy & GB/SNN & RMSE &
0.050$\pm$0.003 & 0.110$\pm$0.007 &
0.012$\pm$0.001 & 0.012$\pm$0.001 &
0.024$\pm$0.001 & 0.024$\pm$0.002 &
0.039 & 0.174 \\
\midrule
Alg. 1 copy & NN/SNN & RMSE &
0.009$\pm$0.001 & 0.009$\pm$0.001 & 
0.021$\pm$0.005 & 0.017$\pm$0.003 &
0.032$\pm$0.002 & 0.034$\pm$0.003 &
0.020 & 0.165 \\
Alg. 2 copy & NN/SNN & RMSE &
0.017$\pm$0.001 & 0.111$\pm$0.005 &
0.011$\pm$0.001 & 0.011$\pm$0.001 &
0.018$\pm$0.001 & 0.019$\pm$0.002 &
0.031 & 0.249 \\
\midrule
Alg. 1 copy & RF/MNN & RMSE &
0.023$\pm$0.002 & 0.025$\pm$0.002 & 
0.015$\pm$0.003 & 0.017$\pm$0.005 &
0.027$\pm$0.001 & 0.026$\pm$0.001 &
0.022 & 0.218 \\
Alg. 2 copy & RF/MNN & RMSE &
0.024$\pm$0.002 & 0.118$\pm$0.003 &
0.011$\pm$0.001 & 0.011$\pm$0.002 &
0.013$\pm$0.001 & 0.015$\pm$0.004 &
0.032 & 0.212 \\
\midrule
Alg. 1 copy & GB/MNN & RMSE &
0.054$\pm$0.003 & 0.053$\pm$0.005 & 
0.013$\pm$0.001 & 0.013$\pm$0.002 &
0.026$\pm$0.002 & 0.025$\pm$0.002 &
0.031 & 0.215 \\
Alg. 2 copy & GB/MNN & RMSE &
0.049$\pm$0.003 & 0.111$\pm$0.005 &
0.009$\pm$0.000 & 0.010$\pm$0.000 &
0.012$\pm$0.001 & 0.011$\pm$0.001 &
0.034 & 0.192 \\
\midrule
Alg. 1 copy & NN/MNN & RMSE &
0.008$\pm$0.002 & 0.012$\pm$0.003 & 
0.010$\pm$0.002 & 0.010$\pm$0.002 &
0.017$\pm$0.002 & 0.017$\pm$0.002 &
0.012 & 0.180 \\
Alg. 2 copy & NN/MNN & RMSE &
0.018$\pm$0.002 & 0.112$\pm$0.003 &
0.008$\pm$0.001 & 0.009$\pm$0.001 &
0.010$\pm$0.001 & 0.010$\pm$0.000 &
0.028 & 0.246 \\
\midrule
Alg. 1 copy & RF/LNN & RMSE &
0.023$\pm$0.003 & 0.025$\pm$0.005 & 
0.014$\pm$0.001 & 0.015$\pm$0.001 &
0.026$\pm$0.004 & 0.026$\pm$0.004 &
0.022 & 0.210 \\
Alg. 2 copy & RF/LNN & RMSE &
0.025$\pm$0.003 & 0.121$\pm$0.004 &
0.008$\pm$0.001 & 0.009$\pm$0.001 &
0.013$\pm$0.001 & 0.014$\pm$0.002 &
0.032 & 0.227 \\
\midrule
Alg. 1 copy & GB/LNN & RMSE &
0.054$\pm$0.005 & 0.059$\pm$0.010 & 
0.013$\pm$0.001 & 0.013$\pm$0.001 &
0.024$\pm$0.001 & 0.023$\pm$0.002 &
0.031 & 0.200 \\
Alg. 2 copy & GB/LNN & RMSE &
0.051$\pm$0.003 & 0.108$\pm$0.005 &
0.009$\pm$0.001 & 0.011$\pm$0.001 &
0.010$\pm$0.001 & 0.009$\pm$0.001 &
0.033 & 0.196 \\
\midrule
Alg. 1 copy & NN/LNN & RMSE &
0.012$\pm$0.005 & 0.013$\pm$0.004 & 
0.009$\pm$0.001 & 0.009$\pm$0.001 &
0.019$\pm$0.003 & 0.018$\pm$0.005 &
0.013 & 0.190 \\
Alg. 2 copy & NN/LNN & RMSE &
0.018$\pm$0.001 & 0.111$\pm$0.003 &
0.008$\pm$0.000 & 0.008$\pm$0.001 &
0.010$\pm$0.001 & 0.009$\pm$0.001 &
0.027 & 0.256 \\
\midrule
Alg. 1 copy & RF/GB & RMSE &
0.027$\pm$0.003 & 0.027$\pm$0.003 & 
0.016$\pm$0.001 & 0.018$\pm$0.003 &
0.028$\pm$0.001 & 0.027$\pm$0.003 &
0.024 & 0.219 \\
Alg. 2 copy & RF/GB & RMSE &
0.025$\pm$0.001 & 0.115$\pm$0.004 &
0.010$\pm$0.001 & 0.010$\pm$0.000 &
0.016$\pm$0.000 & 0.018$\pm$0.002 &
0.032 & 0.206 \\
\midrule
Alg. 1 copy & GB/GB & RMSE &
0.059$\pm$0.004 & 0.051$\pm$0.007 & 
0.016$\pm$0.001 & 0.017$\pm$0.001 &
0.023$\pm$0.001 & 0.023$\pm$0.001 &
0.032 & 0.182 \\
Alg. 2 copy & GB/GB & RMSE &
0.051$\pm$0.003 & 0.115$\pm$0.005 &
0.011$\pm$0.000 & 0.010$\pm$0.001 &
0.017$\pm$0.000 & 0.018$\pm$0.000 &
0.037 & 0.193 \\
\midrule
Alg. 1 copy & NN/GB & RMSE &
0.013$\pm$0.003 & 0.012$\pm$0.001 & 
0.021$\pm$0.001 & 0.016$\pm$0.001 &
0.020$\pm$0.001 & 0.018$\pm$0.001 &
0.017 & 0.291 \\
Alg. 2 copy & NN/GB & RMSE &
0.018$\pm$0.001 & 0.110$\pm$0.005 &
0.017$\pm$0.001 & 0.013$\pm$0.001 &
0.017$\pm$0.001 & 0.016$\pm$0.001 &
0.032 & 0.273 \\
\bottomrule
\end{tabular}
\end{table}

\vspace{15mm}

\begin{table}[!h]
\centering
\tiny
\setlength{\tabcolsep}{1.8pt}
\begin{tabular}{l c c *{6}{c} c c}
\toprule
Copy & $f_{\mathcal{O}}/f_{\mathcal{C}}$ & Metric &
\multicolumn{2}{c}{Dataset 4} &
\multicolumn{2}{c}{Dataset 5} &
\multicolumn{2}{c}{Dataset 6} &
Avg. Dat. 4-6 &
Avg. \\
\cmidrule(lr){4-9}
& & &
$\mathscr{D}_{\mathrm{te}}^{'}$ real data & $\mathcal{S}^{'}$ uniform data &
$\mathscr{D}_{\mathrm{te}}^{'}$ real data & $\mathcal{S}^{'}$ uniform data &
$\mathscr{D}_{\mathrm{te}}^{'}$ real data & $\mathcal{S}^{'}$ uniform data &
\\
\midrule
Alg. 1 copy & RF/SNN & RMSE &
0.420$\pm$0.034 & 0.379$\pm$0.023 & 
0.101$\pm$0.007 & 0.112$\pm$0.008 &
0.669$\pm$0.134 & 0.662$\pm$0.064 &
0.390 & 0.211 \\
Alg. 2 copy & RF/SNN & RMSE &
0.424$\pm$0.081 & 0.302$\pm$0.013 &
0.150$\pm$0.018 & 0.172$\pm$0.026 &
0.837$\pm$0.187 & 0.430$\pm$0.021 &
0.386 & 0.211 \\
\midrule
Alg. 1 copy & GB/SNN & RMSE &
0.310$\pm$0.030 & 0.383$\pm$0.030 & 
0.151$\pm$0.016 & 0.113$\pm$0.003 &
0.622$\pm$0.117 & 0.643$\pm$0.050 &
0.370 & 0.205 \\
Alg. 2 copy & GB/SNN & RMSE &
0.296$\pm$0.023 & 0.290$\pm$0.023 &
0.224$\pm$0.026 & 0.165$\pm$0.013 &
0.458$\pm$0.090 & 0.415$\pm$0.045 &
0.308 & 0.174 \\
\midrule
Alg. 1 copy & NN/SNN & RMSE &
0.268$\pm$0.095 & 0.280$\pm$0.120 & 
0.076$\pm$0.012 & 0.069$\pm$0.007 &
0.621$\pm$0.079 & 0.543$\pm$0.108 &
0.310 & 0.165 \\
Alg. 2 copy & NN/SNN & RMSE &
0.425$\pm$0.169 & 0.571$\pm$0.228 &
0.164$\pm$0.015 & 0.169$\pm$0.015 &
0.925$\pm$0.130 & 0.545$\pm$0.107 &
0.466 & 0.249 \\
\midrule
Alg. 1 copy & RF/MNN & RMSE &
0.484$\pm$0.044 & 0.394$\pm$0.026 & 
0.107$\pm$0.023 & 0.109$\pm$0.017 &
0.766$\pm$0.128 & 0.615$\pm$0.039 &
0.413 & 0.218 \\
Alg. 2 copy & RF/MNN & RMSE &
0.496$\pm$0.090 & 0.289$\pm$0.012 &
0.159$\pm$0.011 & 0.167$\pm$0.010 &
0.712$\pm$0.145 & 0.525$\pm$0.030 &
0.391 & 0.212 \\
\midrule
Alg. 1 copy & GB/MNN & RMSE &
0.388$\pm$0.071 & 0.424$\pm$0.034 & 
0.139$\pm$0.010 & 0.122$\pm$0.007 &
0.700$\pm$0.076 & 0.619$\pm$0.077 &
0.399 & 0.215 \\
Alg. 2 copy & GB/MNN & RMSE &
0.350$\pm$0.068 & 0.316$\pm$0.014 &
0.246$\pm$0.020 & 0.192$\pm$0.019 &
0.511$\pm$0.072 & 0.478$\pm$0.043 &
0.349 & 0.192 \\
\midrule
Alg. 1 copy & NN/MNN & RMSE &
0.238$\pm$0.028 & 0.367$\pm$0.193 & 
0.077$\pm$0.007 & 0.072$\pm$0.009 &
0.738$\pm$0.169 & 0.594$\pm$0.221 &
0.348 & 0.180 \\
Alg. 2 copy & NN/MNN & RMSE &
0.443$\pm$0.079 & 0.508$\pm$0.168 &
0.182$\pm$0.017 & 0.166$\pm$0.017 &
1.016$\pm$0.091 & 0.469$\pm$0.042 &
0.464 & 0.246 \\
\midrule
Alg. 1 copy & RF/LNN & RMSE &
0.486$\pm$0.095 & 0.388$\pm$0.036 & 
0.116$\pm$0.038 & 0.118$\pm$0.019 &
0.691$\pm$0.102 & 0.584$\pm$0.035 &
0.397 & 0.210 \\
Alg. 2 copy & RF/LNN & RMSE &
0.653$\pm$0.112 & 0.292$\pm$0.022 &
0.147$\pm$0.005 & 0.148$\pm$0.014 &
0.802$\pm$0.135 & 0.486$\pm$0.043 &
0.421 & 0.227 \\
\midrule
Alg. 1 copy & GB/LNN & RMSE &
0.341$\pm$0.039 & 0.396$\pm$0.016 & 
0.160$\pm$0.023 & 0.143$\pm$0.016 &
0.606$\pm$0.074 & 0.561$\pm$0.029 &
0.368 & 0.200 \\
Alg. 2 copy & GB/LNN & RMSE &
0.372$\pm$0.047 & 0.328$\pm$0.027 &
0.243$\pm$0.022 & 0.171$\pm$0.010 &
0.583$\pm$0.119 & 0.454$\pm$0.047 &
0.358 & 0.196 \\
\midrule
Alg. 1 copy & NN/LNN & RMSE &
0.323$\pm$0.082 & 0.378$\pm$0.129 & 
0.073$\pm$0.010 & 0.071$\pm$0.009 &
0.816$\pm$0.129 & 0.536$\pm$0.092 &
0.366 & 0.190 \\
Alg. 2 copy & NN/LNN & RMSE &
0.558$\pm$0.109 & 0.468$\pm$0.058 &
0.169$\pm$0.010 & 0.150$\pm$0.010 &
1.031$\pm$0.202 & 0.525$\pm$0.121 &
0.484 & 0.256 \\
\midrule
Alg. 1 copy & RF/GB & RMSE &
0.619$\pm$0.067 & 0.296$\pm$0.020 & 
0.110$\pm$0.006 & 0.105$\pm$0.012 &
0.765$\pm$0.074 & 0.580$\pm$0.045 &
0.413 & 0.219 \\
Alg. 2 copy & RF/GB & RMSE &
0.610$\pm$0.047 & 0.269$\pm$0.017 &
0.139$\pm$0.012 & 0.154$\pm$0.015 &
0.661$\pm$0.120 & 0.443$\pm$0.009 &
0.379 & 0.206 \\
\midrule
Alg. 1 copy & GB/GB & RMSE &
0.382$\pm$0.039 & 0.297$\pm$0.034 & 
0.142$\pm$0.010 & 0.104$\pm$0.009 &
0.555$\pm$0.091 & 0.514$\pm$0.055 &
0.332 & 0.182 \\
Alg. 2 copy & GB/GB & RMSE &
0.473$\pm$0.062 & 0.289$\pm$0.017 &
0.250$\pm$0.022 & 0.178$\pm$0.016 &
0.536$\pm$0.022 & 0.366$\pm$0.028 &
0.349 & 0.193 \\
\midrule
Alg. 1 copy & NN/GB & RMSE &
0.656$\pm$0.234 & 0.487$\pm$0.056 & 
0.097$\pm$0.013 & 0.090$\pm$0.005 &
1.260$\pm$0.252 & 0.799$\pm$0.045 &
0.565 & 0.291 \\
Alg. 2 copy & NN/GB & RMSE &
0.453$\pm$0.041 & 0.444$\pm$0.038 &
0.171$\pm$0.028 & 0.140$\pm$0.012 &
1.122$\pm$0.137 & 0.745$\pm$0.052 &
0.513 & 0.273 \\
\bottomrule
\end{tabular}
\end{table}

\newpage
\section{Scatter plots for estimated distances {\small(Dataset \# - Original /Copy model(Distance algorithm) - test/ uniform data)}}
\label{appDistScatter}

\begin{figure}[!ht]
	\centering
	\begin{subfigure}[t]{0.137\textwidth}
		\includegraphics[width=\linewidth]{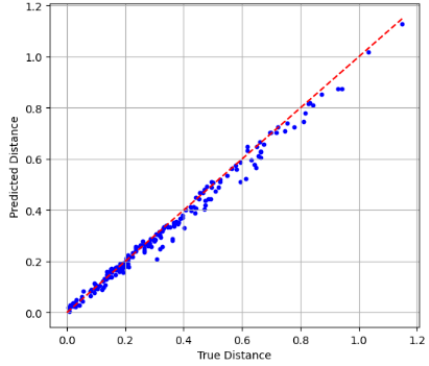}
		\caption{1-RF/SNN-1-$\mathscr{D}^{'}_{\mathrm{te}}$}
	\end{subfigure}
	\hfill
	\begin{subfigure}[t]{0.137\textwidth}
		\includegraphics[width=\linewidth]{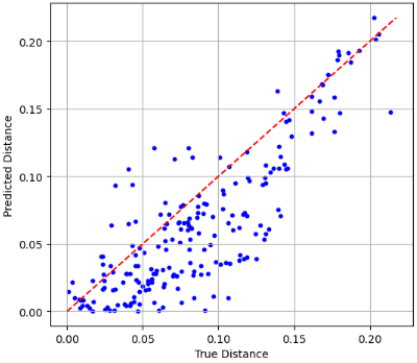}
		\caption{2-RF/SNN-1-$\mathscr{D}^{'}_{\mathrm{te}}$}
	\end{subfigure}
	\hfill
	\begin{subfigure}[t]{0.137\textwidth}
		\includegraphics[width=\linewidth]{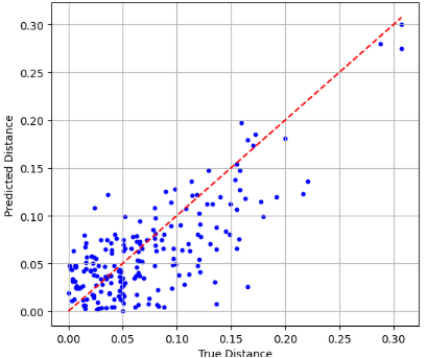}
		\caption{3-RF/SNN-1-$\mathscr{D}^{'}_{\mathrm{te}}$}
	\end{subfigure}
	\hfill
	\begin{subfigure}[t]{0.137\textwidth}
		\includegraphics[width=\linewidth]{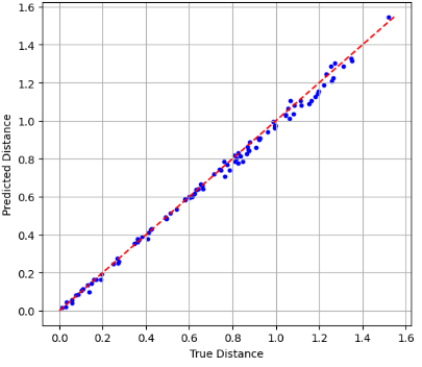}
		\caption{1-RF/SNN-1-$\mathcal{S}^{'}$}
	\end{subfigure}
	\hfill
	\begin{subfigure}[t]{0.137\textwidth}
		\includegraphics[width=\linewidth]{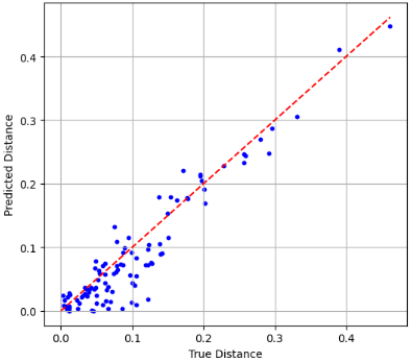}
		\caption{2-RF/SNN-1-$\mathcal{S}^{'}$}
	\end{subfigure}
	\hfill
	\begin{subfigure}[t]{0.137\textwidth}
		\includegraphics[width=\linewidth]{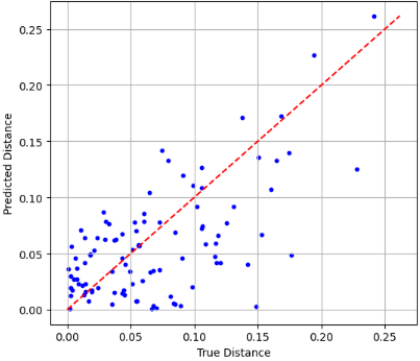}
		\caption{3-RF/SNN-1-$\mathcal{S}^{'}$}
	\end{subfigure}
	
	
	\begin{subfigure}[t]{0.137\textwidth}
		\includegraphics[width=\linewidth]{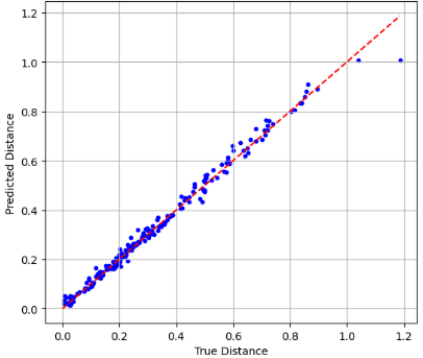}
		\caption{1-RF/SNN-2-$\mathscr{D}^{'}_{\mathrm{te}}$}
	\end{subfigure}
	\hfill
	\begin{subfigure}[t]{0.137\textwidth}
		\includegraphics[width=\linewidth]{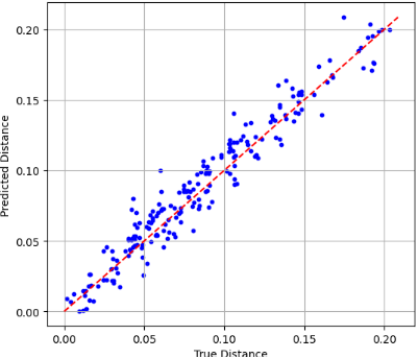}
		\caption{2-RF/SNN-2-$\mathscr{D}^{'}_{\mathrm{te}}$}
	\end{subfigure}
	\hfill
	\begin{subfigure}[t]{0.137\textwidth}
		\includegraphics[width=\linewidth]{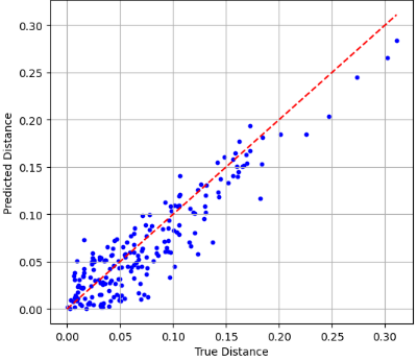}
		\caption{3-RF/SNN-2-$\mathscr{D}^{'}_{\mathrm{te}}$}
	\end{subfigure}
	\hfill
	\begin{subfigure}[t]{0.137\textwidth}
		\includegraphics[width=\linewidth]{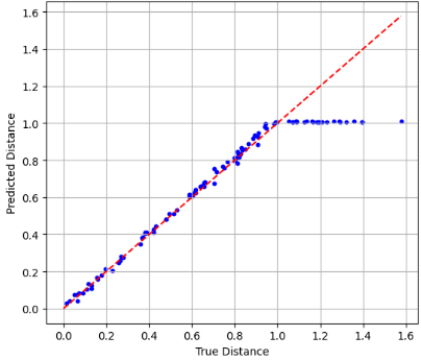}
		\caption{1-RF/SNN-2-$\mathcal{S}^{'}$}
	\end{subfigure}
	\hfill
	\begin{subfigure}[t]{0.137\textwidth}
		\includegraphics[width=\linewidth]{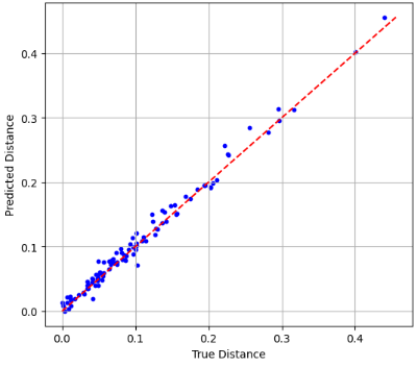}
		\caption{2-RF/SNN-2-$\mathcal{S}^{'}$}
	\end{subfigure}
	\hfill
	\begin{subfigure}[t]{0.137\textwidth}
		\includegraphics[width=\linewidth]{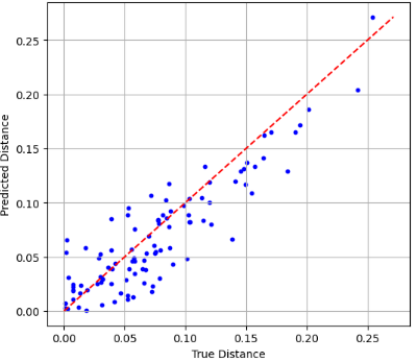}
		\caption{3-RF/SNN-2-$\mathcal{S}^{'}$}
	\end{subfigure}
	
	
	\begin{subfigure}[t]{0.137\textwidth}
		\includegraphics[width=\linewidth]{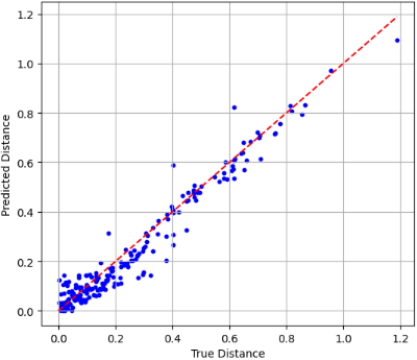}
		\caption{\scalebox{0.95}{1-GB/SNN-1-$\mathscr{D}^{'}_{\mathrm{te}}$}}
	\end{subfigure}
	\hfill
	\begin{subfigure}[t]{0.137\textwidth}
		\includegraphics[width=\linewidth]{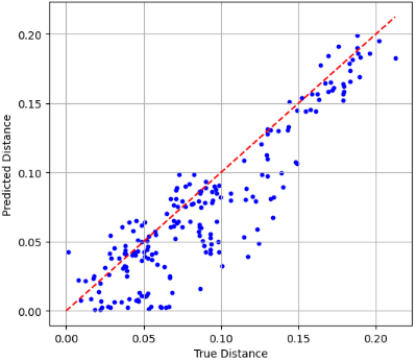}
		\caption{\scalebox{0.95}{2-GB/SNN-1-$\mathscr{D}^{'}_{\mathrm{te}}$}}
	\end{subfigure}
	\hfill
	\begin{subfigure}[t]{0.137\textwidth}
		\includegraphics[width=\linewidth]{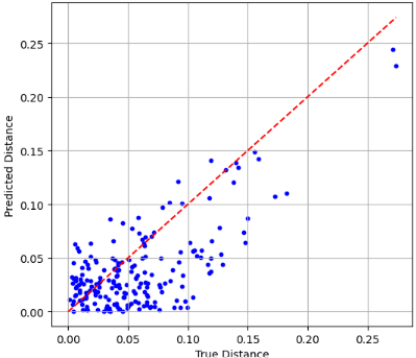}
		\caption{\scalebox{0.95}{3-GB/SNN-1-$\mathscr{D}^{'}_{\mathrm{te}}$}}
	\end{subfigure}
	\hfill
	\begin{subfigure}[t]{0.137\textwidth}
		\includegraphics[width=\linewidth]{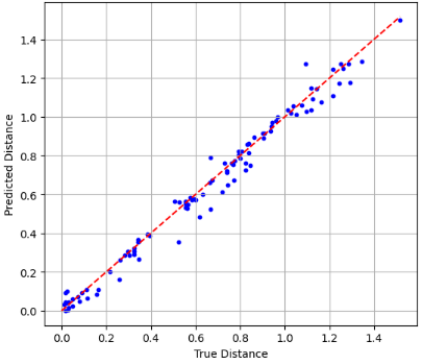}
		\caption{1-GB/SNN-1-$\mathcal{S}^{'}$}
	\end{subfigure}
	\hfill
	\begin{subfigure}[t]{0.137\textwidth}
		\includegraphics[width=\linewidth]{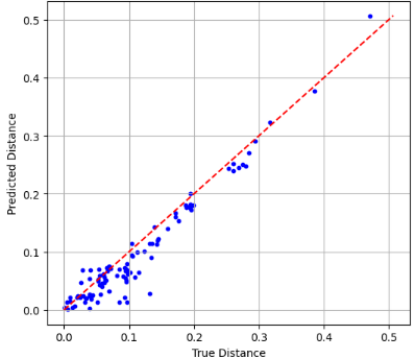}
		\caption{2-GB/SNN-1-$\mathcal{S}^{'}$}
	\end{subfigure}
	\hfill
	\begin{subfigure}[t]{0.137\textwidth}
		\includegraphics[width=\linewidth]{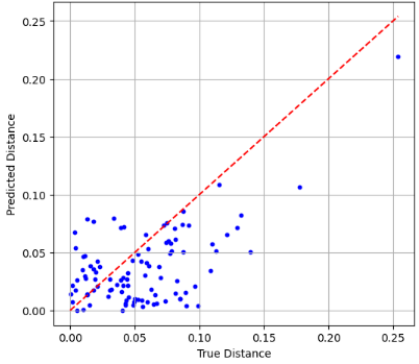}
		\caption{3-GB/SNN-1-$\mathcal{S}^{'}$}
	\end{subfigure}
	
	
	\begin{subfigure}[t]{0.137\textwidth}
		\includegraphics[width=\linewidth]{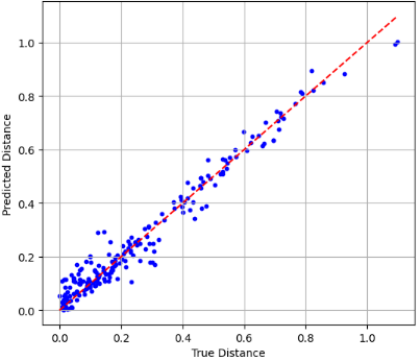}
		\caption{\scalebox{0.95}{1-GB/SNN-2-$\mathscr{D}^{'}_{\mathrm{te}}$}}
	\end{subfigure}
	\hfill
	\begin{subfigure}[t]{0.137\textwidth}
		\includegraphics[width=\linewidth]{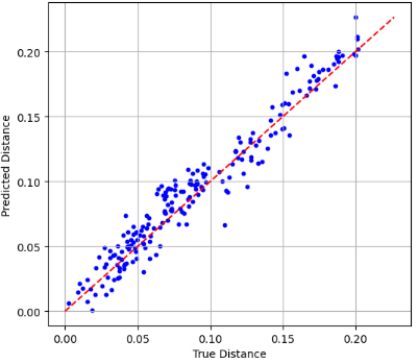}
		\caption{\scalebox{0.95}{2-GB/SNN-2-$\mathscr{D}^{'}_{\mathrm{te}}$}}
	\end{subfigure}
	\hfill
	\begin{subfigure}[t]{0.137\textwidth}
		\includegraphics[width=\linewidth]{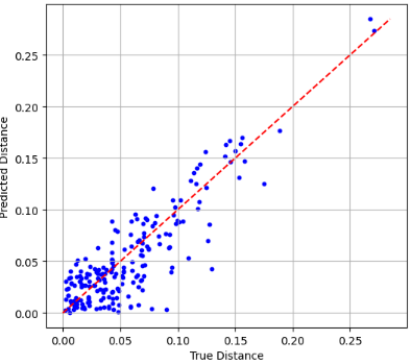}
		\caption{\scalebox{0.95}{3-GB/SNN-2-$\mathscr{D}^{'}_{\mathrm{te}}$}}
	\end{subfigure}
	\hfill
	\begin{subfigure}[t]{0.137\textwidth}
		\includegraphics[width=\linewidth]{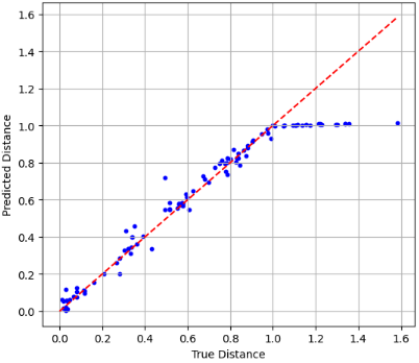}
		\caption{1-GB/SNN-2-$\mathcal{S}^{'}$}
	\end{subfigure}
	\hfill
	\begin{subfigure}[t]{0.137\textwidth}
		\includegraphics[width=\linewidth]{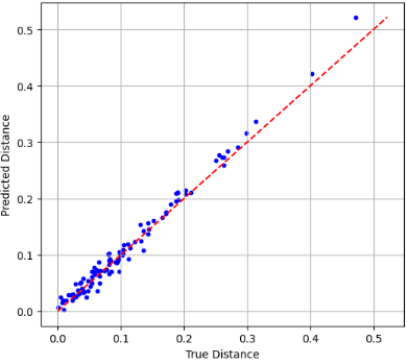}
		\caption{2-GB/SNN-2-$\mathcal{S}^{'}$}
	\end{subfigure}
	\hfill
	\begin{subfigure}[t]{0.137\textwidth}
		\includegraphics[width=\linewidth]{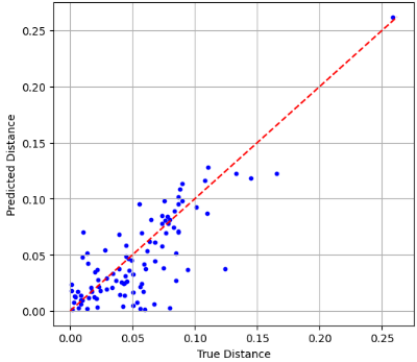}
		\caption{3-GB/SNN-2-$\mathcal{S}^{'}$}
	\end{subfigure}
	
	
	\begin{subfigure}[t]{0.137\textwidth}
		\includegraphics[width=\linewidth]{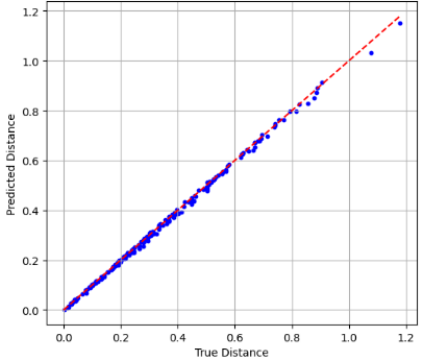}
		\caption{\scalebox{0.95}{1-NN/SNN-1-$\mathscr{D}^{'}_{\mathrm{te}}$}}
	\end{subfigure}
	\hfill
	\begin{subfigure}[t]{0.137\textwidth}
		\includegraphics[width=\linewidth]{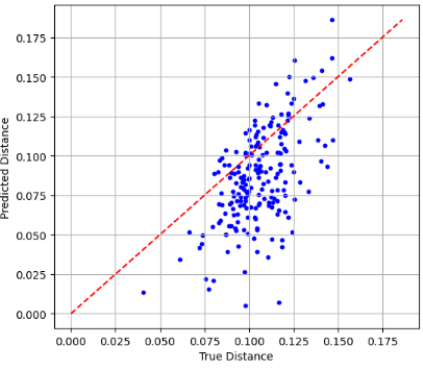}
		\caption{\scalebox{0.95}{2-NN/SNN-1-$\mathscr{D}^{'}_{\mathrm{te}}$}}
	\end{subfigure}
	\hfill
	\begin{subfigure}[t]{0.137\textwidth}
		\includegraphics[width=\linewidth]{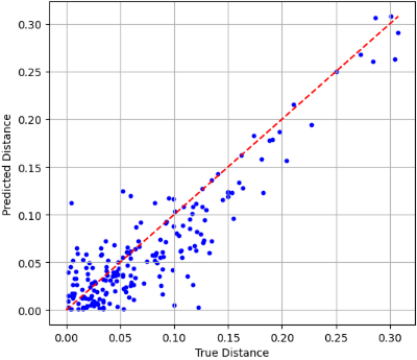}
		\caption{\scalebox{0.95}{3-NN/SNN-1-$\mathscr{D}^{'}_{\mathrm{te}}$}}
	\end{subfigure}
	\hfill
	\begin{subfigure}[t]{0.137\textwidth}
		\includegraphics[width=\linewidth]{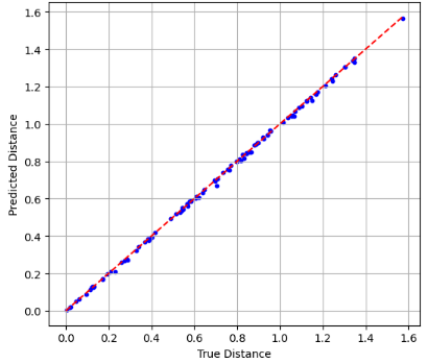}
		\caption{1-NN/SNN-1-$\mathcal{S}^{'}$}
	\end{subfigure}
	\hfill
	\begin{subfigure}[t]{0.137\textwidth}
		\includegraphics[width=\linewidth]{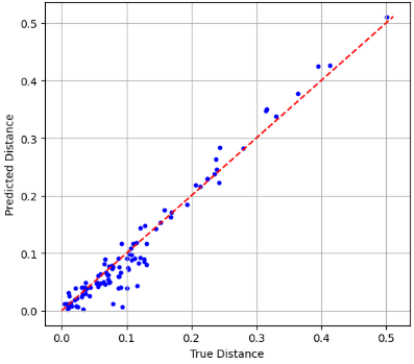}
		\caption{2-NN/SNN-1-$\mathcal{S}^{'}$}
	\end{subfigure}
	\hfill
	\begin{subfigure}[t]{0.137\textwidth}
		\includegraphics[width=\linewidth]{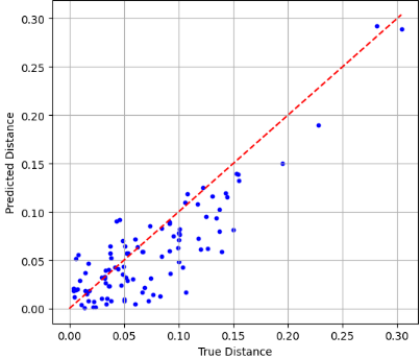}
		\caption{3-NN/SNN-1-$\mathcal{S}^{'}$}
	\end{subfigure}
	
	
	\begin{subfigure}[t]{0.137\textwidth}
		\includegraphics[width=\linewidth]{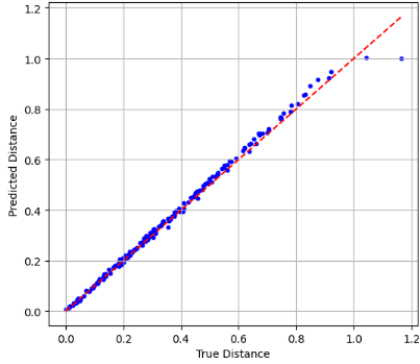}
		\caption{\scalebox{0.95}{1-NN/SNN-2-$\mathscr{D}^{'}_{\mathrm{te}}$}}
	\end{subfigure}
	\hfill
	\begin{subfigure}[t]{0.137\textwidth}
		\includegraphics[width=\linewidth]{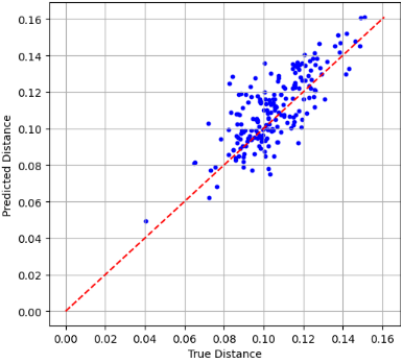}
		\caption{\scalebox{0.95}{2-NN/SNN-2-$\mathscr{D}^{'}_{\mathrm{te}}$}}
	\end{subfigure}
	\hfill
	\begin{subfigure}[t]{0.137\textwidth}
		\includegraphics[width=\linewidth]{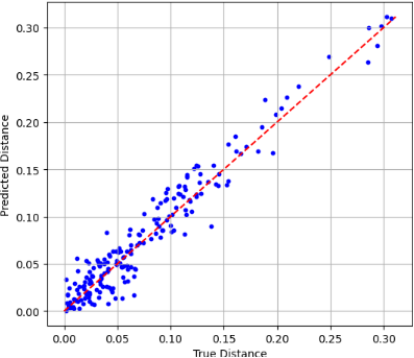}
		\caption{\scalebox{0.95}{3-NN/SNN-2-$\mathscr{D}^{'}_{\mathrm{te}}$}}
	\end{subfigure}
	\hfill
	\begin{subfigure}[t]{0.137\textwidth}
		\includegraphics[width=\linewidth]{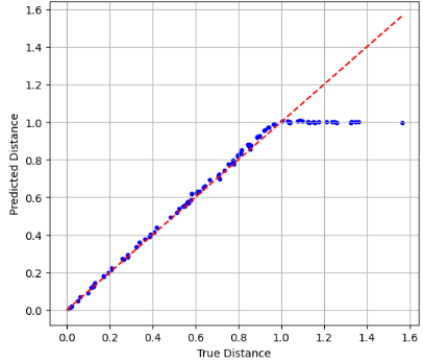}
		\caption{1-NN/SNN-2-$\mathcal{S}^{'}$}
	\end{subfigure}
	\hfill
	\begin{subfigure}[t]{0.137\textwidth}
		\includegraphics[width=\linewidth]{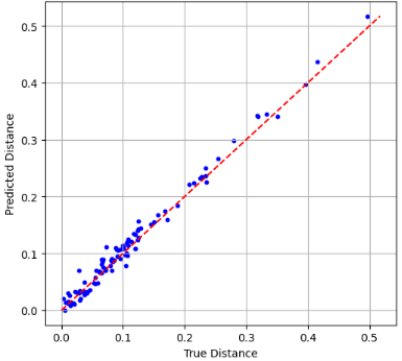}
		\caption{2-NN/SNN-2-$\mathcal{S}^{'}$}
	\end{subfigure}
	\hfill
	\begin{subfigure}[t]{0.137\textwidth}
		\includegraphics[width=\linewidth]{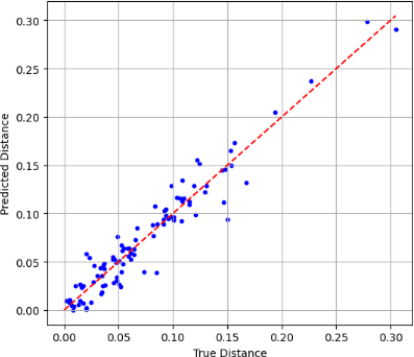}
		\caption{3-NN/SNN-2-$\mathcal{S}^{'}$}
	\end{subfigure}
	
	
	\begin{subfigure}[t]{0.137\textwidth}
		\includegraphics[width=\linewidth]{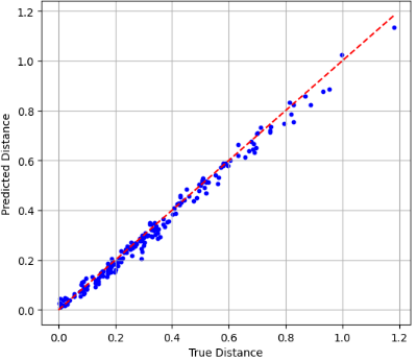}
		\caption{\scalebox{0.95}{1-RF/MNN-1-$\mathscr{D}^{'}_{\mathrm{te}}$}}
	\end{subfigure}
	\hfill
	\begin{subfigure}[t]{0.137\textwidth}
		\includegraphics[width=\linewidth]{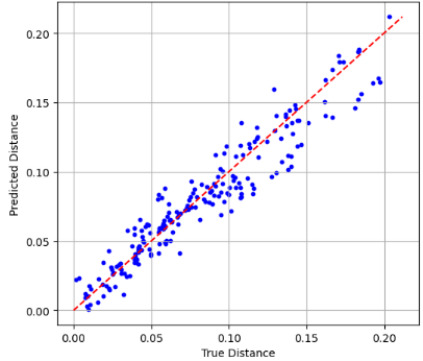}
		\caption{\scalebox{0.95}{2-RF/MNN-1-$\mathscr{D}^{'}_{\mathrm{te}}$}}
	\end{subfigure}
	\hfill
	\begin{subfigure}[t]{0.137\textwidth}
		\includegraphics[width=\linewidth]{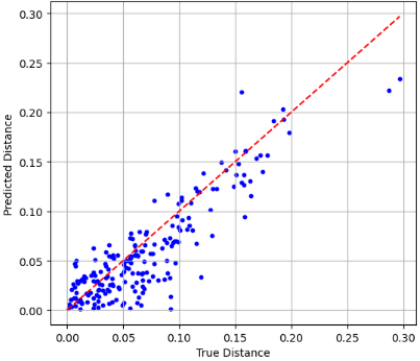}
		\caption{\scalebox{0.95}{3-RF/MNN-1-$\mathscr{D}^{'}_{\mathrm{te}}$}}
	\end{subfigure}
	\hfill
	\begin{subfigure}[t]{0.137\textwidth}
		\includegraphics[width=\linewidth]{Figures/Figure_9/1-RF-MNN-1-S.pdf}
		\caption{1-RF/MNN-1-$\mathcal{S}^{'}$}
	\end{subfigure}
	\hfill
	\begin{subfigure}[t]{0.137\textwidth}
		\includegraphics[width=\linewidth]{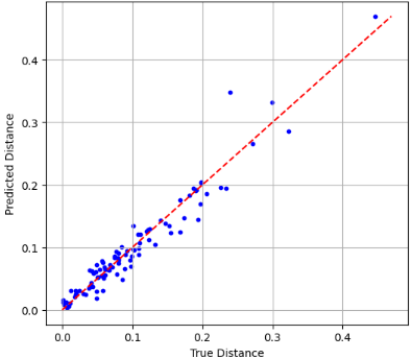}
		\caption{2-RF/MNN-1-$\mathcal{S}^{'}$}
	\end{subfigure}
	\hfill
	\begin{subfigure}[t]{0.137\textwidth}
		\includegraphics[width=\linewidth]{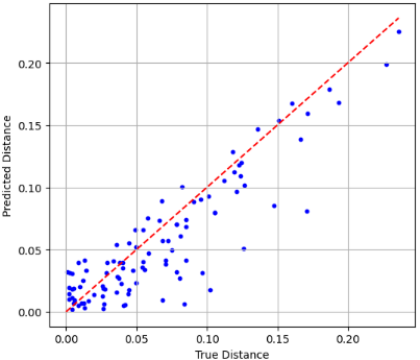}
		\caption{3-RF/MNN-1-$\mathcal{S}^{'}$}
	\end{subfigure}
	
	
	\begin{subfigure}[t]{0.137\textwidth}
		\includegraphics[width=\linewidth]{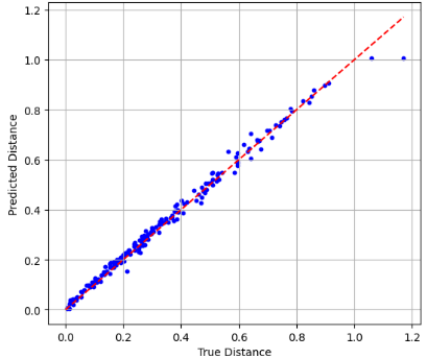}
		\caption{\scalebox{0.95}{1-RF/MNN-2-$\mathscr{D}^{'}_{\mathrm{te}}$}}
	\end{subfigure}
	\hfill
	\begin{subfigure}[t]{0.137\textwidth}
		\includegraphics[width=\linewidth]{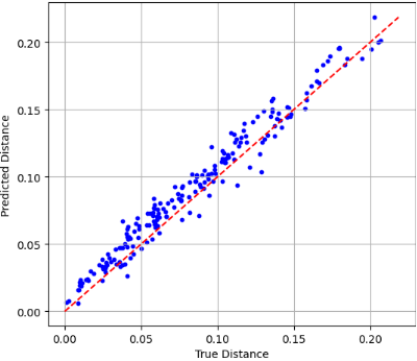}
		\caption{\scalebox{0.95}{2-RF/MNN-2-$\mathscr{D}^{'}_{\mathrm{te}}$}}
	\end{subfigure}
	\hfill
	\begin{subfigure}[t]{0.137\textwidth}
		\includegraphics[width=\linewidth]{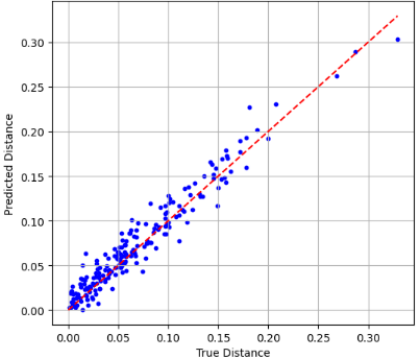}
		\caption{\scalebox{0.95}{3-RF/MNN-2-$\mathscr{D}^{'}_{\mathrm{te}}$}}
	\end{subfigure}
	\hfill
	\begin{subfigure}[t]{0.137\textwidth}
		\includegraphics[width=\linewidth]{Figures/Figure_9/1-RF-MNN-2-S.pdf}
		\caption{1-RF/MNN-2-$\mathcal{S}^{'}$}
	\end{subfigure}
	\hfill
	\begin{subfigure}[t]{0.137\textwidth}
		\includegraphics[width=\linewidth]{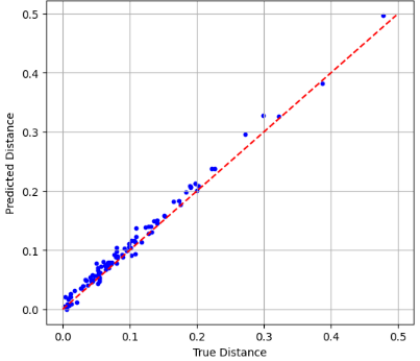}
		\caption{2-RF/MNN-2-$\mathcal{S}^{'}$}
	\end{subfigure}
	\hfill
	\begin{subfigure}[t]{0.137\textwidth}
		\includegraphics[width=\linewidth]{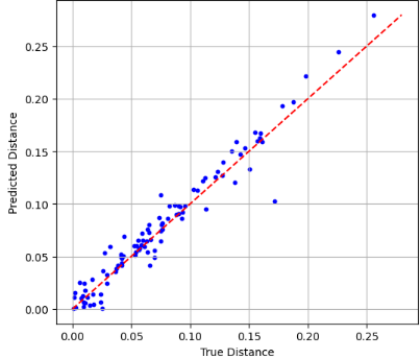}
		\caption{3-RF/MNN-2-$\mathcal{S}^{'}$}
	\end{subfigure}
\end{figure}

\begin{figure}[!ht]
	\centering
	\begin{subfigure}[t]{0.14\textwidth}
		\includegraphics[width=\linewidth]{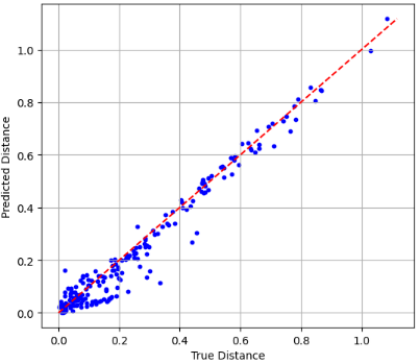}
		\caption{\scalebox{0.95}{1-GB/MNN-1-$\mathscr{D}^{'}_{\mathrm{te}}$}}
	\end{subfigure}
	\hfill
	\begin{subfigure}[t]{0.14\textwidth}
		\includegraphics[width=\linewidth]{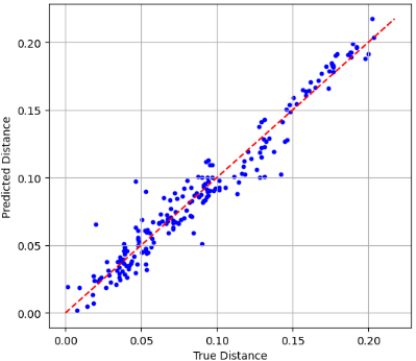}
		\caption{\scalebox{0.95}{2-GB/MNN-1-$\mathscr{D}^{'}_{\mathrm{te}}$}}
	\end{subfigure}
	\hfill
	\begin{subfigure}[t]{0.14\textwidth}
		\includegraphics[width=\linewidth]{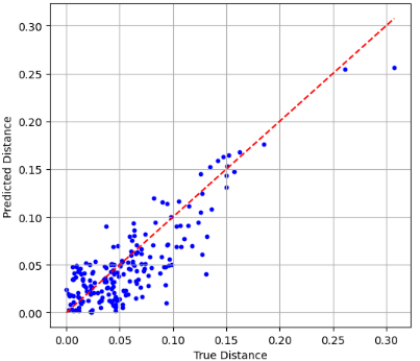}
		\caption{\scalebox{0.95}{3-GB/MNN-1-$\mathscr{D}^{'}_{\mathrm{te}}$}}
	\end{subfigure}
	\hfill
	\begin{subfigure}[t]{0.14\textwidth}
		\includegraphics[width=\linewidth]{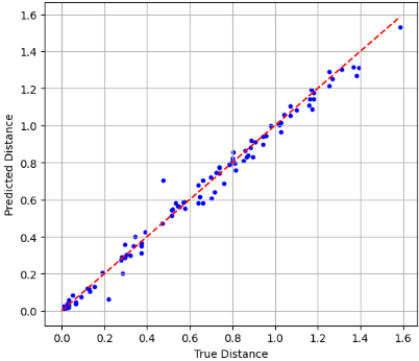}
		\caption{1-GB/MNN-1-$\mathcal{S}^{'}$}
	\end{subfigure}
	\hfill
	\begin{subfigure}[t]{0.14\textwidth}
		\includegraphics[width=\linewidth]{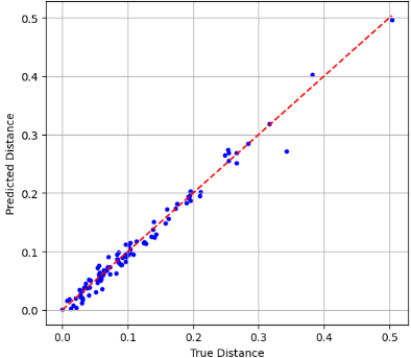}
		\caption{2-GB/MNN-1-$\mathcal{S}^{'}$}
	\end{subfigure}
	\hfill
	\begin{subfigure}[t]{0.14\textwidth}
		\includegraphics[width=\linewidth]{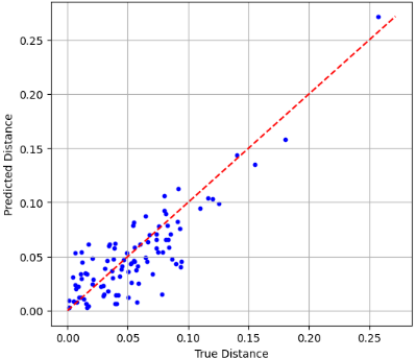}
		\caption{3-GB/MNN-1-$\mathcal{S}^{'}$}
	\end{subfigure}
	
	
	\begin{subfigure}[t]{0.14\textwidth}
		\includegraphics[width=\linewidth]{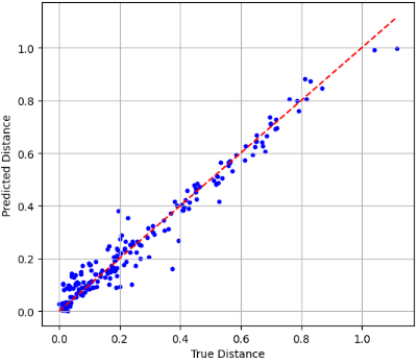}
		\caption{\scalebox{0.95}{1-GB/MNN-2-$\mathscr{D}^{'}_{\mathrm{te}}$}}
	\end{subfigure}
	\hfill
	\begin{subfigure}[t]{0.14\textwidth}
		\includegraphics[width=\linewidth]{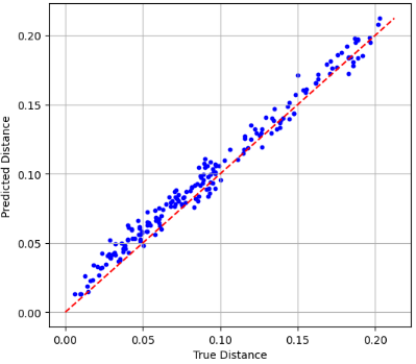}
		\caption{\scalebox{0.95}{2-GB/MNN-2-$\mathscr{D}^{'}_{\mathrm{te}}$}}
	\end{subfigure}
	\hfill
	\begin{subfigure}[t]{0.14\textwidth}
		\includegraphics[width=\linewidth]{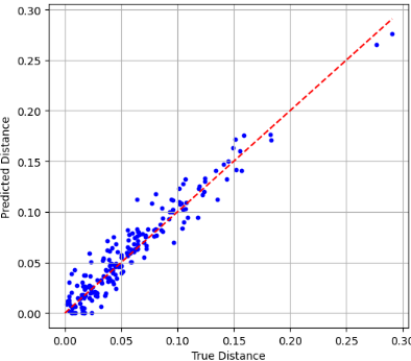}
		\caption{\scalebox{0.95}{3-GB/MNN-2-$\mathscr{D}^{'}_{\mathrm{te}}$}}
	\end{subfigure}
	\hfill
	\begin{subfigure}[t]{0.14\textwidth}
		\includegraphics[width=\linewidth]{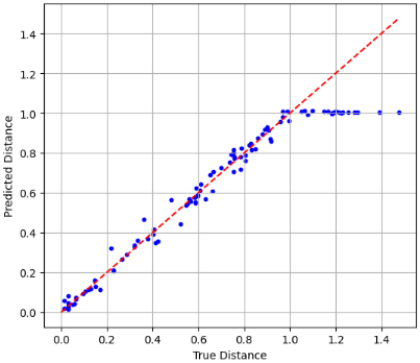}
		\caption{1-GB/MNN-2-$\mathcal{S}^{'}$}
	\end{subfigure}
	\hfill
	\begin{subfigure}[t]{0.14\textwidth}
		\includegraphics[width=\linewidth]{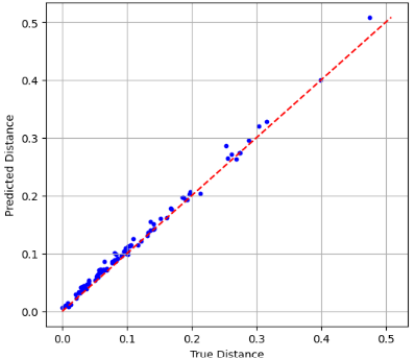}
		\caption{2-GB/MNN-2-$\mathcal{S}^{'}$}
	\end{subfigure}
	\hfill
	\begin{subfigure}[t]{0.14\textwidth}
		\includegraphics[width=\linewidth]{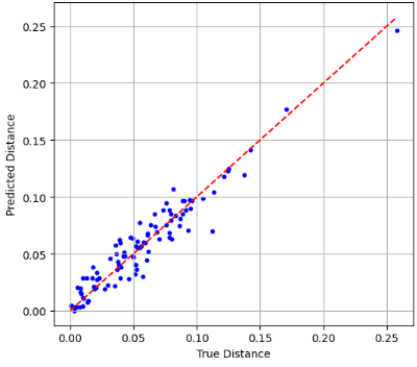}
		\caption{3-GB/MNN-2-$\mathcal{S}^{'}$}
	\end{subfigure}
	
	
	\begin{subfigure}[t]{0.14\textwidth}
		\includegraphics[width=\linewidth]{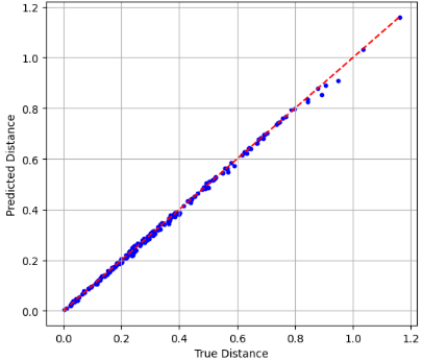}
		\caption{\scalebox{0.95}{1-NN/MNN-1-$\mathscr{D}^{'}_{\mathrm{te}}$}}
	\end{subfigure}
	\hfill
	\begin{subfigure}[t]{0.14\textwidth}
		\includegraphics[width=\linewidth]{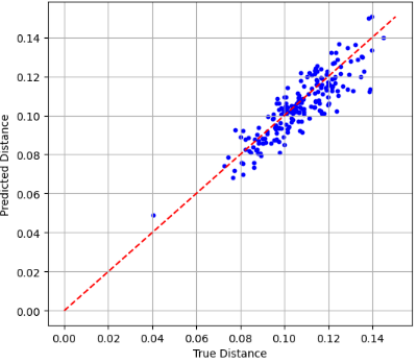}
		\caption{\scalebox{0.95}{2-NN/MNN-1-$\mathscr{D}^{'}_{\mathrm{te}}$}}
	\end{subfigure}
	\hfill
	\begin{subfigure}[t]{0.14\textwidth}
		\includegraphics[width=\linewidth]{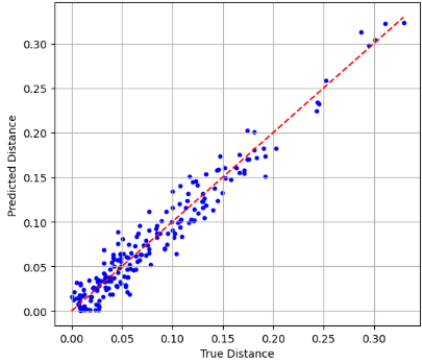}
		\caption{\scalebox{0.95}{3-NN/MNN-1-$\mathscr{D}^{'}_{\mathrm{te}}$}}
	\end{subfigure}
	\hfill
	\begin{subfigure}[t]{0.14\textwidth}
		\includegraphics[width=\linewidth]{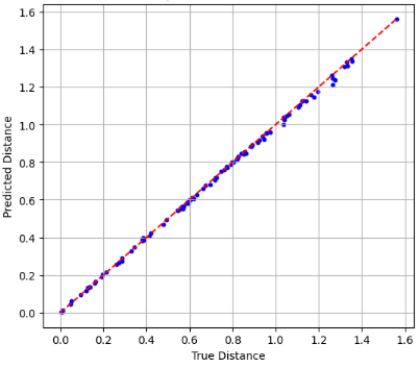}
		\caption{1-NN/MNN-1-$\mathcal{S}^{'}$}
	\end{subfigure}
	\hfill
	\begin{subfigure}[t]{0.14\textwidth}
		\includegraphics[width=\linewidth]{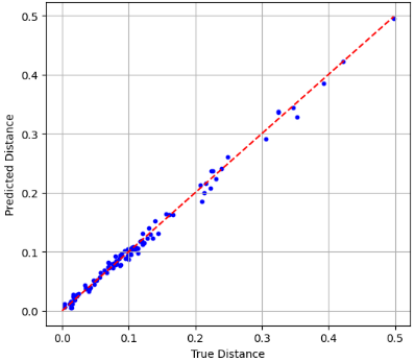}
		\caption{2-NN/MNN-1-$\mathcal{S}^{'}$}
	\end{subfigure}
	\hfill
	\begin{subfigure}[t]{0.14\textwidth}
		\includegraphics[width=\linewidth]{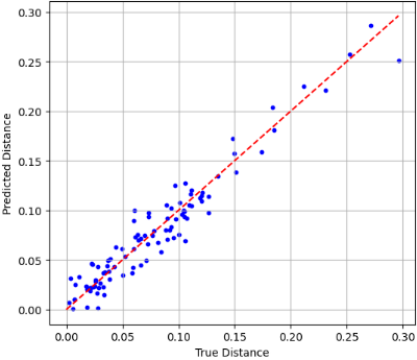}
		\caption{3-NN/MNN-1-$\mathcal{S}^{'}$}
	\end{subfigure}
	
	
	\begin{subfigure}[t]{0.14\textwidth}
		\includegraphics[width=\linewidth]{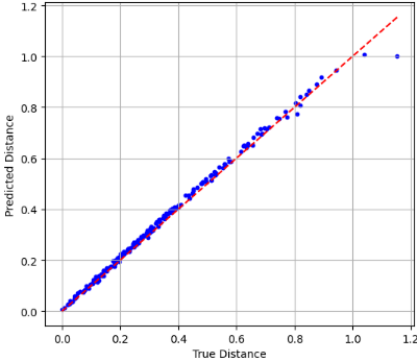}
		\caption{\scalebox{0.95}{1-NN/MNN-2-$\mathscr{D}^{'}_{\mathrm{te}}$}}
	\end{subfigure}
	\hfill
	\begin{subfigure}[t]{0.14\textwidth}
		\includegraphics[width=\linewidth]{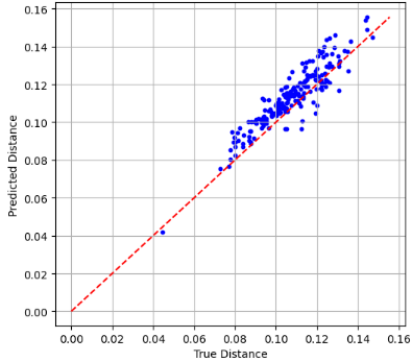}
		\caption{\scalebox{0.95}{2-NN/MNN-2-$\mathscr{D}^{'}_{\mathrm{te}}$}}
	\end{subfigure}
	\hfill
	\begin{subfigure}[t]{0.14\textwidth}
		\includegraphics[width=\linewidth]{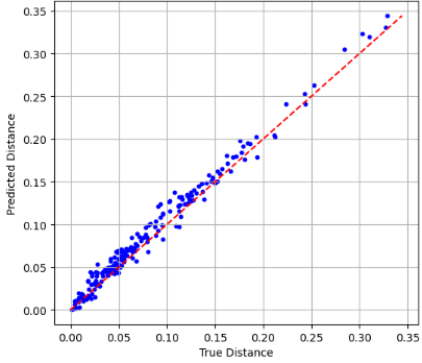}
		\caption{\scalebox{0.95}{3-NN/MNN-2-$\mathscr{D}^{'}_{\mathrm{te}}$}}
	\end{subfigure}
	\hfill
	\begin{subfigure}[t]{0.14\textwidth}
		\includegraphics[width=\linewidth]{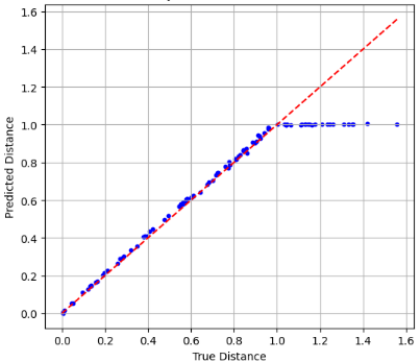}
		\caption{1-NN/MNN-2-$\mathcal{S}^{'}$}
	\end{subfigure}
	\hfill
	\begin{subfigure}[t]{0.14\textwidth}
		\includegraphics[width=\linewidth]{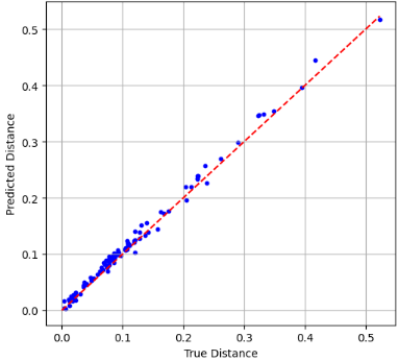}
		\caption{2-NN/MNN-2-$\mathcal{S}^{'}$}
	\end{subfigure}
	\hfill
	\begin{subfigure}[t]{0.14\textwidth}
		\includegraphics[width=\linewidth]{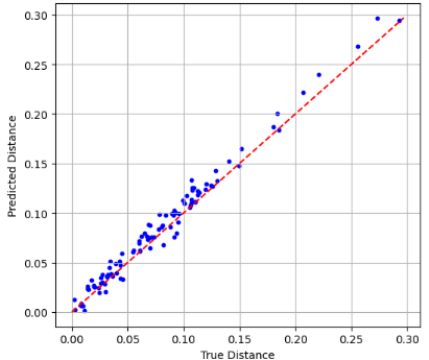}
		\caption{3-NN/MNN-2-$\mathcal{S}^{'}$}
	\end{subfigure}
	
	
	\begin{subfigure}[t]{0.14\textwidth}
		\includegraphics[width=\linewidth]{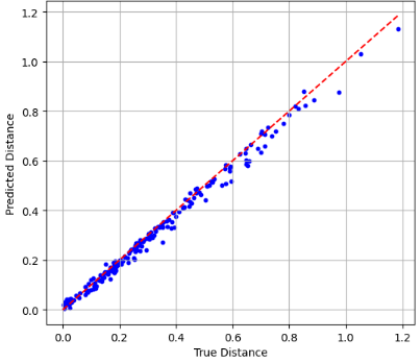}
		\caption{1-RF/LNN-1-$\mathscr{D}^{'}_{\mathrm{te}}$}
	\end{subfigure}
	\hfill
	\begin{subfigure}[t]{0.14\textwidth}
		\includegraphics[width=\linewidth]{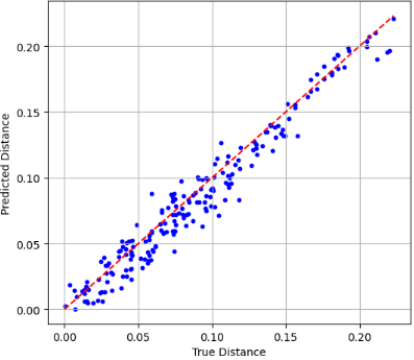}
		\caption{2-RF/LNN-1-$\mathscr{D}^{'}_{\mathrm{te}}$}
	\end{subfigure}
	\hfill
	\begin{subfigure}[t]{0.14\textwidth}
		\includegraphics[width=\linewidth]{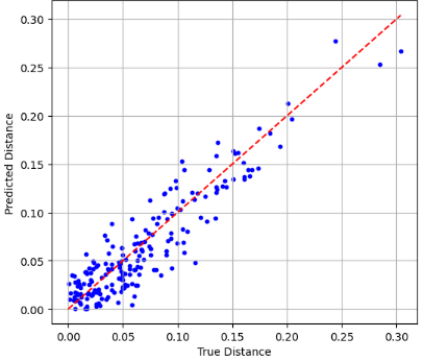}
		\caption{3-RF/LNN-1-$\mathscr{D}^{'}_{\mathrm{te}}$}
	\end{subfigure}
	\hfill
	\begin{subfigure}[t]{0.14\textwidth}
		\includegraphics[width=\linewidth]{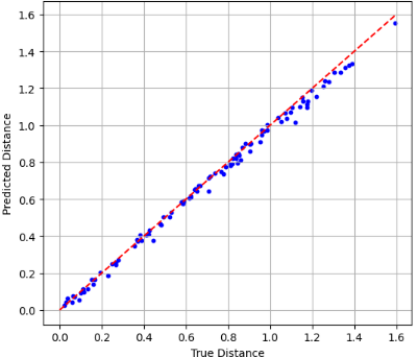}
		\caption{1-RF/LNN-1-$\mathcal{S}^{'}$}
	\end{subfigure}
	\hfill
	\begin{subfigure}[t]{0.14\textwidth}
		\includegraphics[width=\linewidth]{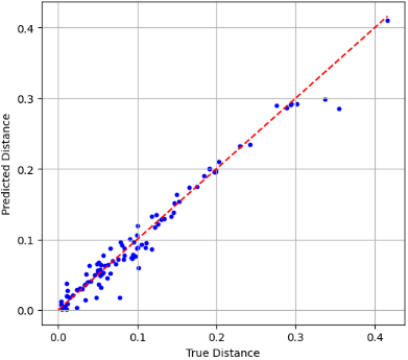}
		\caption{2-RF/LNN-1-$\mathcal{S}^{'}$}
	\end{subfigure}
	\hfill
	\begin{subfigure}[t]{0.14\textwidth}
		\includegraphics[width=\linewidth]{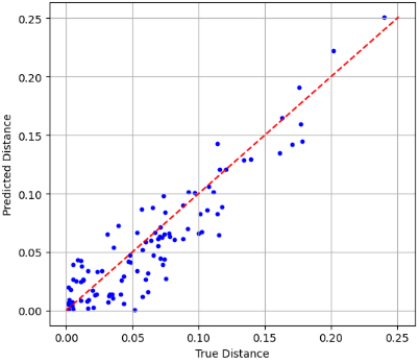}
		\caption{3-RF/LNN-1-$\mathcal{S}^{'}$}
	\end{subfigure}
	
	
	\begin{subfigure}[t]{0.14\textwidth}
		\includegraphics[width=\linewidth]{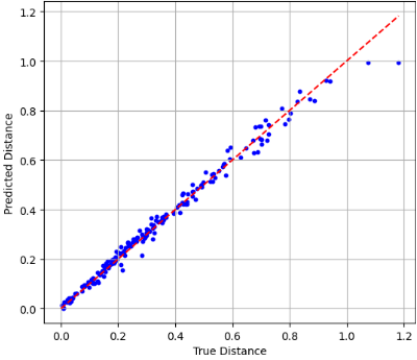}
		\caption{1-RF/LNN-2-$\mathscr{D}^{'}_{\mathrm{te}}$}
	\end{subfigure}
	\hfill
	\begin{subfigure}[t]{0.14\textwidth}
		\includegraphics[width=\linewidth]{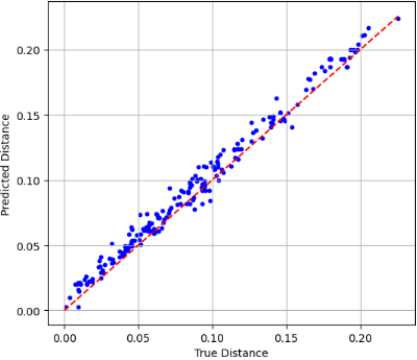}
		\caption{2-RF/LNN-2-$\mathscr{D}^{'}_{\mathrm{te}}$}
	\end{subfigure}
	\hfill
	\begin{subfigure}[t]{0.14\textwidth}
		\includegraphics[width=\linewidth]{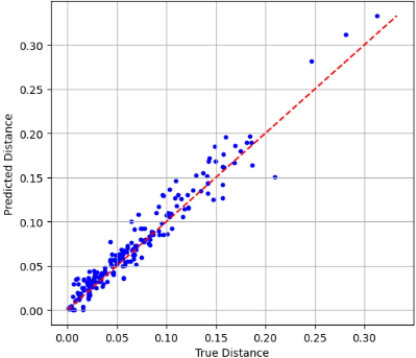}
		\caption{3-RF/LNN-2-$\mathscr{D}^{'}_{\mathrm{te}}$}
	\end{subfigure}
	\hfill
	\begin{subfigure}[t]{0.14\textwidth}
		\includegraphics[width=\linewidth]{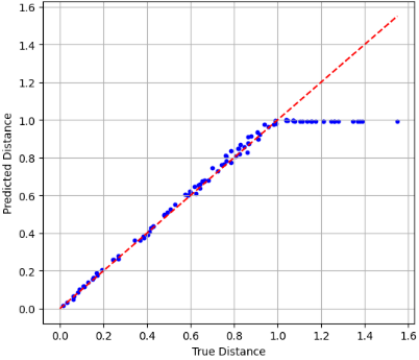}
		\caption{1-RF/LNN-2-$\mathcal{S}^{'}$}
	\end{subfigure}
	\hfill
	\begin{subfigure}[t]{0.14\textwidth}
		\includegraphics[width=\linewidth]{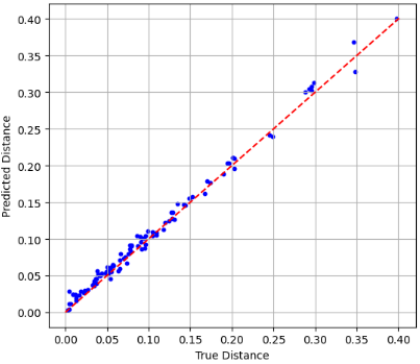}
		\caption{2-RF/LNN-2-$\mathcal{S}^{'}$}
	\end{subfigure}
	\hfill
	\begin{subfigure}[t]{0.14\textwidth}
		\includegraphics[width=\linewidth]{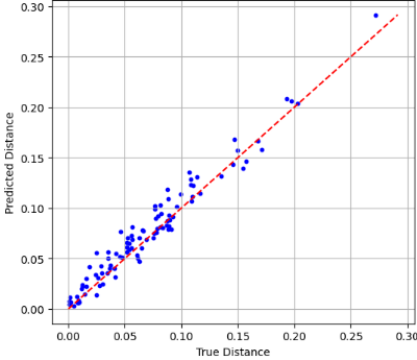}
		\caption{3-RF/LNN-2-$\mathcal{S}^{'}$}
	\end{subfigure}
		
	
	\begin{subfigure}[t]{0.14\textwidth}
		\includegraphics[width=\linewidth]{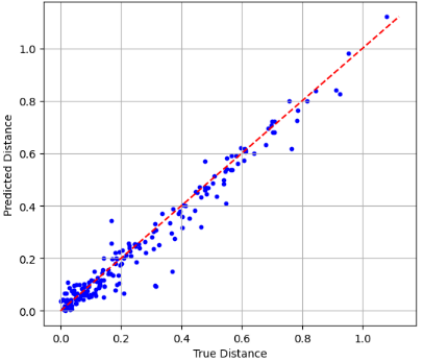}
		\caption{\scalebox{0.95}{1-GB/LNN-1-$\mathscr{D}^{'}_{\mathrm{te}}$}}
	\end{subfigure}
	\hfill
	\begin{subfigure}[t]{0.14\textwidth}
		\includegraphics[width=\linewidth]{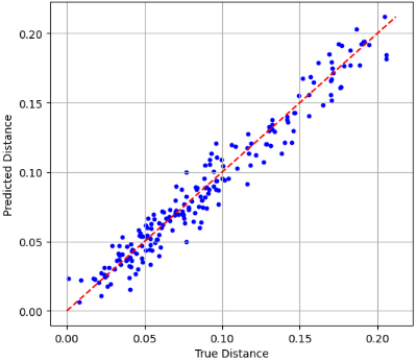}
		\caption{2-GB/LNN-1-$\mathscr{D}^{'}_{\mathrm{te}}$}
	\end{subfigure}
	\hfill
	\begin{subfigure}[t]{0.14\textwidth}
		\includegraphics[width=\linewidth]{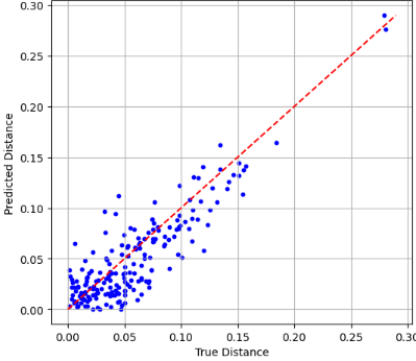}
		\caption{3-GB/LNN-1-$\mathscr{D}^{'}_{\mathrm{te}}$}
	\end{subfigure}
	\hfill
	\begin{subfigure}[t]{0.14\textwidth}
		\includegraphics[width=\linewidth]{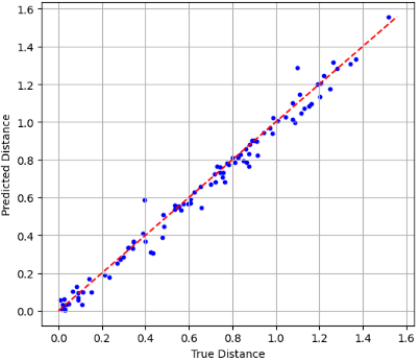}
		\caption{1-GB/LNN-1-$\mathcal{S}^{'}$}
	\end{subfigure}
	\hfill
	\begin{subfigure}[t]{0.14\textwidth}
		\includegraphics[width=\linewidth]{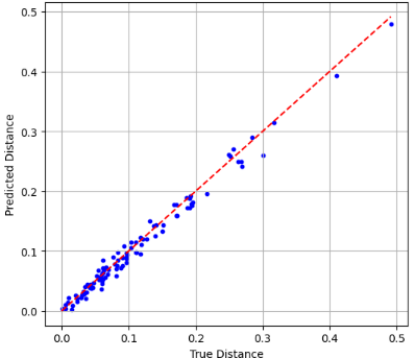}
		\caption{2-GB/LNN-1-$\mathcal{S}^{'}$}
	\end{subfigure}
	\hfill
	\begin{subfigure}[t]{0.14\textwidth}
		\includegraphics[width=\linewidth]{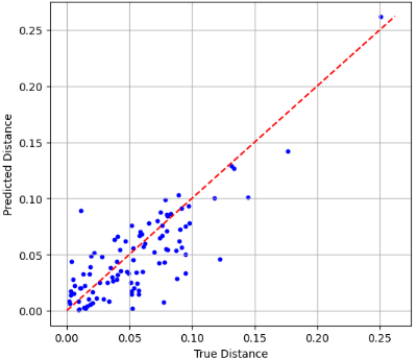}
		\caption{3-GB/LNN-1-$\mathcal{S}^{'}$}
	\end{subfigure}
	
	
	\begin{subfigure}[t]{0.14\textwidth}
		\includegraphics[width=\linewidth]{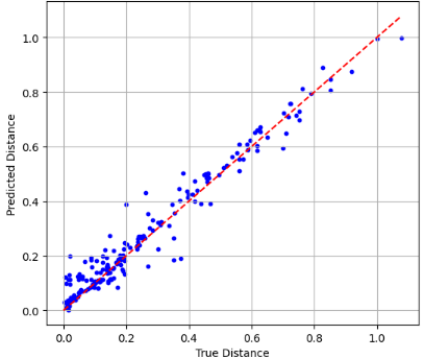}
		\caption{\scalebox{0.95}{1-GB/LNN-2-$\mathscr{D}^{'}_{\mathrm{te}}$}}
	\end{subfigure}
	\hfill
	\begin{subfigure}[t]{0.14\textwidth}
		\includegraphics[width=\linewidth]{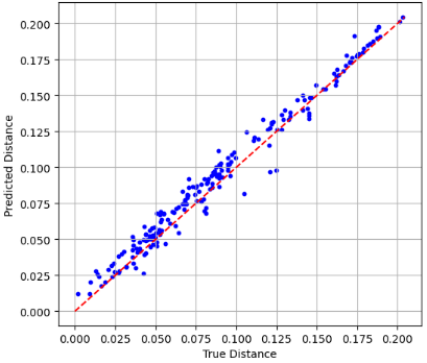}
		\caption{2-GB/LNN-2-$\mathscr{D}^{'}_{\mathrm{te}}$}
	\end{subfigure}
	\hfill
	\begin{subfigure}[t]{0.14\textwidth}
		\includegraphics[width=\linewidth]{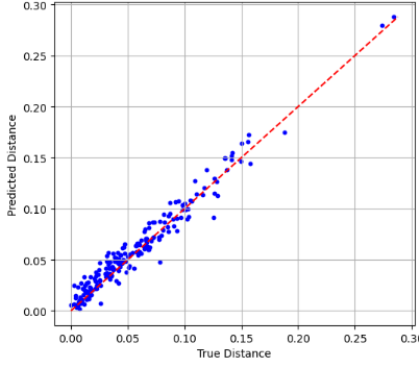}
		\caption{3-GB/LNN-2-$\mathscr{D}^{'}_{\mathrm{te}}$}
	\end{subfigure}
	\hfill
	\begin{subfigure}[t]{0.14\textwidth}
		\includegraphics[width=\linewidth]{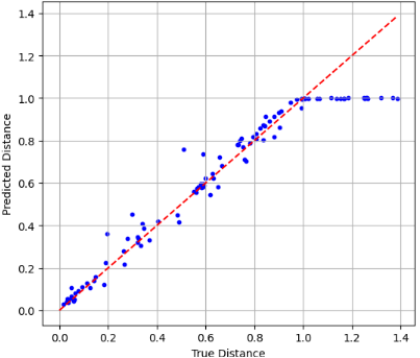}
		\caption{1-GB/LNN-2-$\mathcal{S}^{'}$}
	\end{subfigure}
	\hfill
	\begin{subfigure}[t]{0.14\textwidth}
		\includegraphics[width=\linewidth]{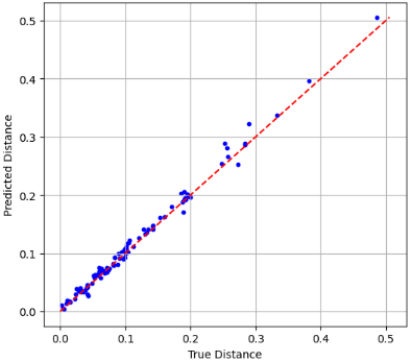}
		\caption{2-GB/LNN-2-$\mathcal{S}^{'}$}
	\end{subfigure}
	\hfill
	\begin{subfigure}[t]{0.14\textwidth}
		\includegraphics[width=\linewidth]{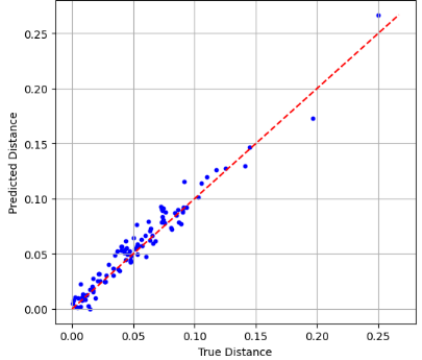}
		\caption{3-GB/LNN-2-$\mathcal{S}^{'}$}
	\end{subfigure}
\end{figure}

\begin{figure}[!ht]
	\centering
	\begin{subfigure}[t]{0.14\textwidth}
		\includegraphics[width=\linewidth]{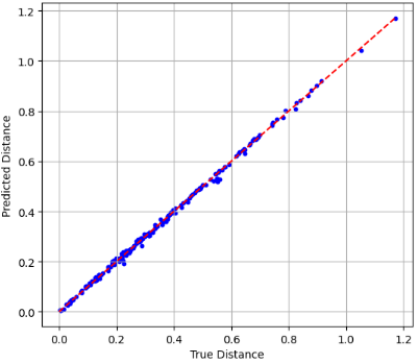}
		\caption{\scalebox{0.95}{1-NN/LNN-1-$\mathscr{D}^{'}_{\mathrm{te}}$}}
	\end{subfigure}
	\hfill
	\begin{subfigure}[t]{0.14\textwidth}
		\includegraphics[width=\linewidth]{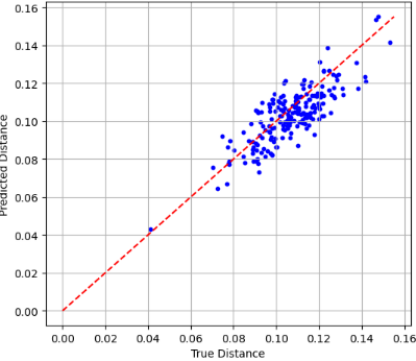}
		\caption{\scalebox{0.95}{2-NN/LNN-1-$\mathscr{D}^{'}_{\mathrm{te}}$}}
	\end{subfigure}
	\hfill
	\begin{subfigure}[t]{0.14\textwidth}
		\includegraphics[width=\linewidth]{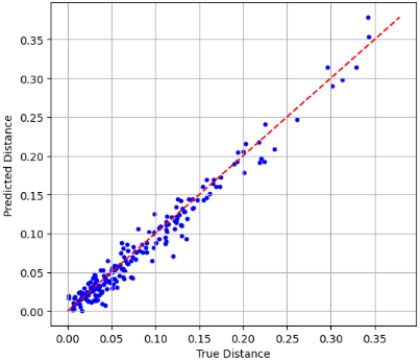}
		\caption{\scalebox{0.95}{3-NN/LNN-1-$\mathscr{D}^{'}_{\mathrm{te}}$}}
	\end{subfigure}
	\hfill
	\begin{subfigure}[t]{0.14\textwidth}
		\includegraphics[width=\linewidth]{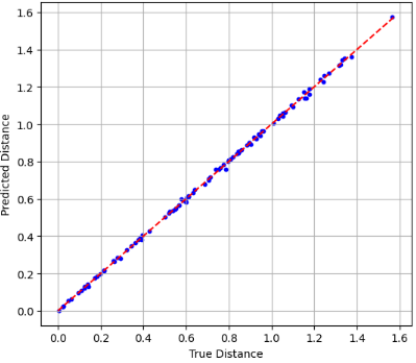}
		\caption{1-NN/LNN-1-$\mathcal{S}^{'}$}
	\end{subfigure}
	\hfill
	\begin{subfigure}[t]{0.14\textwidth}
		\includegraphics[width=\linewidth]{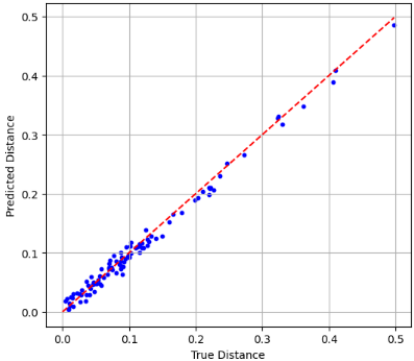}
		\caption{2-NN/LNN-1-$\mathcal{S}^{'}$}
	\end{subfigure}
	\hfill
	\begin{subfigure}[t]{0.14\textwidth}
		\includegraphics[width=\linewidth]{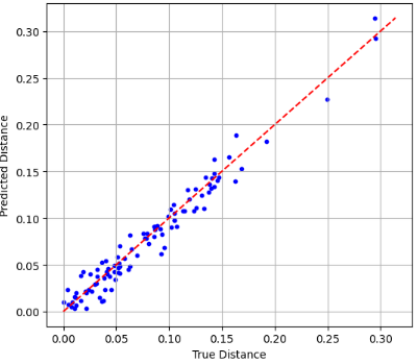}
		\caption{3-NN/LNN-1-$\mathcal{S}^{'}$}
	\end{subfigure}
	
	
	\begin{subfigure}[t]{0.14\textwidth}
		\includegraphics[width=\linewidth]{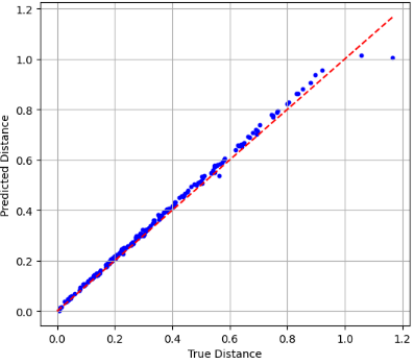}
		\caption{\scalebox{0.95}{1-NN/LNN-2-$\mathscr{D}^{'}_{\mathrm{te}}$}}
	\end{subfigure}
	\hfill
	\begin{subfigure}[t]{0.14\textwidth}
		\includegraphics[width=\linewidth]{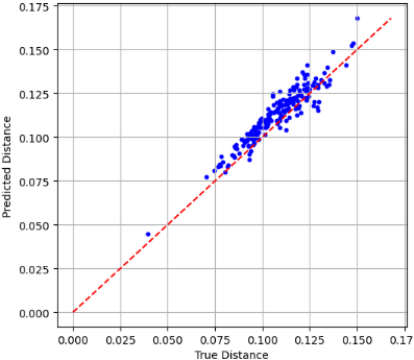}
		\caption{\scalebox{0.95}{2-NN/LNN-2-$\mathscr{D}^{'}_{\mathrm{te}}$}}
	\end{subfigure}
	\hfill
	\begin{subfigure}[t]{0.14\textwidth}
		\includegraphics[width=\linewidth]{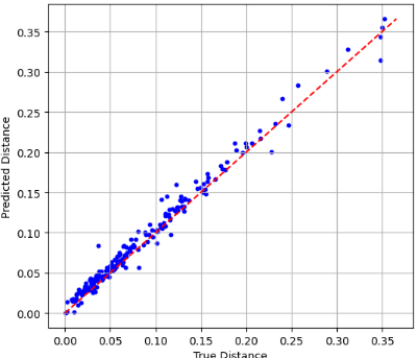}
		\caption{\scalebox{0.95}{3-NN/LNN-2-$\mathscr{D}^{'}_{\mathrm{te}}$}}
	\end{subfigure}
	\hfill
	\begin{subfigure}[t]{0.14\textwidth}
		\includegraphics[width=\linewidth]{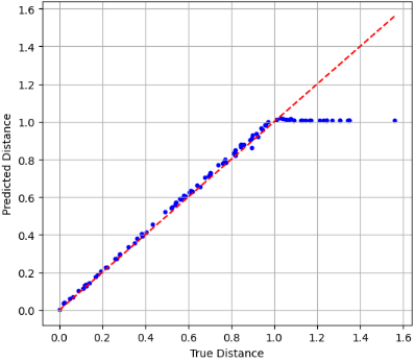}
		\caption{1-NN/LNN-2-$\mathcal{S}^{'}$}
	\end{subfigure}
	\hfill
	\begin{subfigure}[t]{0.14\textwidth}
		\includegraphics[width=\linewidth]{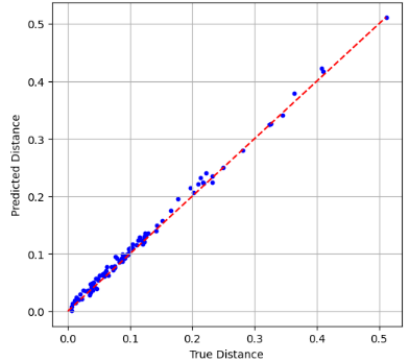}
		\caption{2-NN/LNN-2-$\mathcal{S}^{'}$}
	\end{subfigure}
	\hfill
	\begin{subfigure}[t]{0.14\textwidth}
		\includegraphics[width=\linewidth]{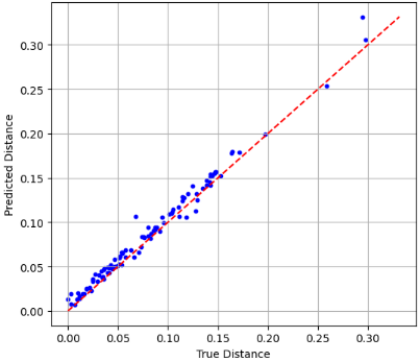}
		\caption{3-NN/LNN-2-$\mathcal{S}^{'}$}
	\end{subfigure}
	
	
	\begin{subfigure}[t]{0.14\textwidth}
		\includegraphics[width=\linewidth]{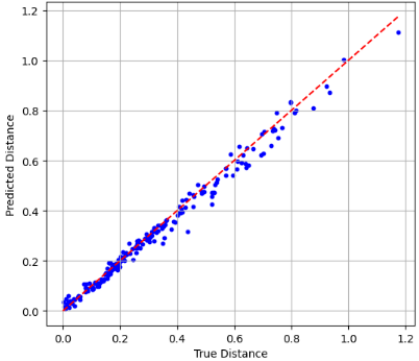}
		\caption{1-RF/GB-1-$\mathscr{D}^{'}_{\mathrm{te}}$}
	\end{subfigure}
	\hfill
	\begin{subfigure}[t]{0.14\textwidth}
		\includegraphics[width=\linewidth]{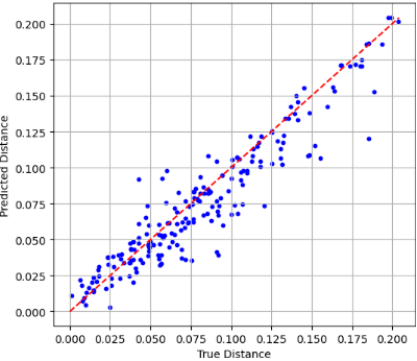}
		\caption{2-RF/GB-1-$\mathscr{D}^{'}_{\mathrm{te}}$}
	\end{subfigure}
	\hfill
	\begin{subfigure}[t]{0.14\textwidth}
		\includegraphics[width=\linewidth]{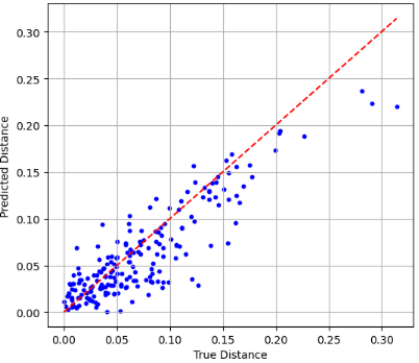}
		\caption{3-RF/GB-1-$\mathscr{D}^{'}_{\mathrm{te}}$}
	\end{subfigure}
	\hfill
	\begin{subfigure}[t]{0.14\textwidth}
		\includegraphics[width=\linewidth]{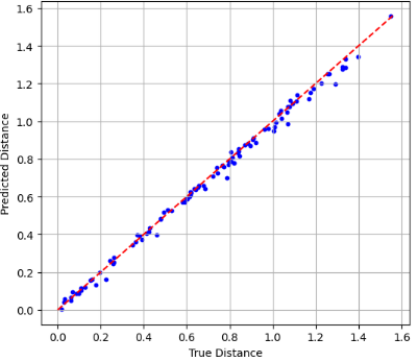}
		\caption{1-RF/GB-1-$\mathcal{S}^{'}$}
	\end{subfigure}
	\hfill
	\begin{subfigure}[t]{0.14\textwidth}
		\includegraphics[width=\linewidth]{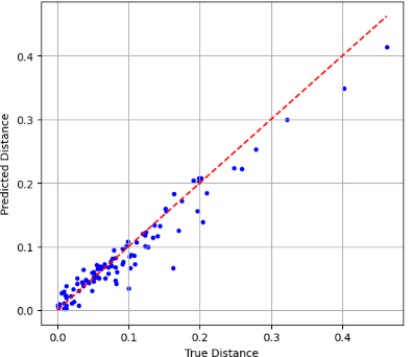}
		\caption{2-RF/GB-1-$\mathcal{S}^{'}$}
	\end{subfigure}
	\hfill
	\begin{subfigure}[t]{0.14\textwidth}
		\includegraphics[width=\linewidth]{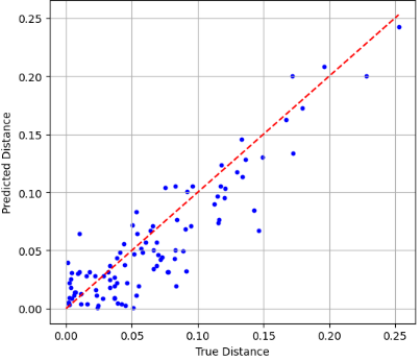}
		\caption{3-RF/GB-1-$\mathcal{S}^{'}$}
	\end{subfigure}
	
	
	\begin{subfigure}[t]{0.14\textwidth}
		\includegraphics[width=\linewidth]{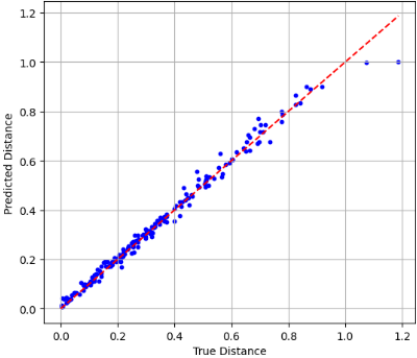}
		\caption{1-RF/GB-2-$\mathscr{D}^{'}_{\mathrm{te}}$}
	\end{subfigure}
	\hfill
	\begin{subfigure}[t]{0.14\textwidth}
		\includegraphics[width=\linewidth]{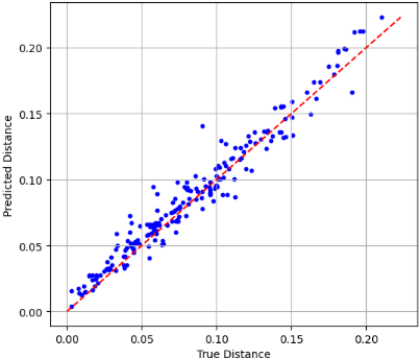}
		\caption{2-RF/GB-2-$\mathscr{D}^{'}_{\mathrm{te}}$}
	\end{subfigure}
	\hfill
	\begin{subfigure}[t]{0.14\textwidth}
		\includegraphics[width=\linewidth]{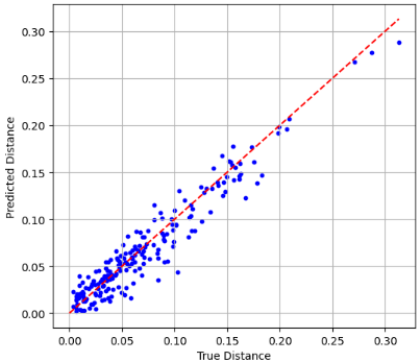}
		\caption{3-RF/GB-2-$\mathscr{D}^{'}_{\mathrm{te}}$}
	\end{subfigure}
	\hfill
	\begin{subfigure}[t]{0.14\textwidth}
		\includegraphics[width=\linewidth]{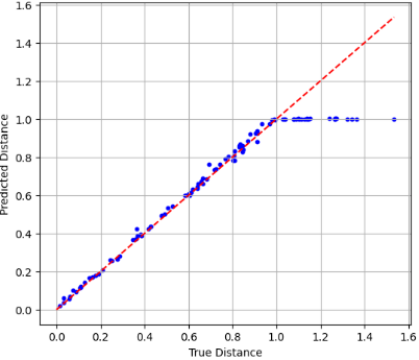}
		\caption{1-RF/GB-2-$\mathcal{S}^{'}$}
	\end{subfigure}
	\hfill
	\begin{subfigure}[t]{0.14\textwidth}
		\includegraphics[width=\linewidth]{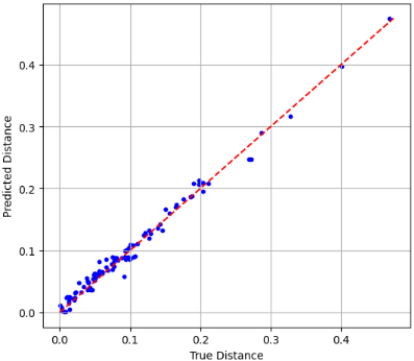}
		\caption{2-RF/GB-2-$\mathcal{S}^{'}$}
	\end{subfigure}
	\hfill
	\begin{subfigure}[t]{0.14\textwidth}
		\includegraphics[width=\linewidth]{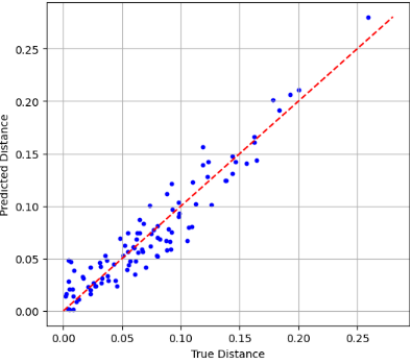}
		\caption{3-RF/GB-2-$\mathcal{S}^{'}$}
	\end{subfigure}
	
	
	\begin{subfigure}[t]{0.14\textwidth}
		\includegraphics[width=\linewidth]{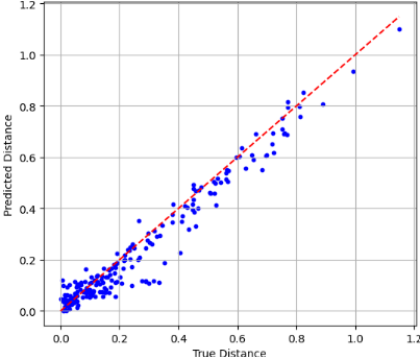}
		\caption{1-GB/GB-1-$\mathscr{D}^{'}_{\mathrm{te}}$}
	\end{subfigure}
	\hfill
	\begin{subfigure}[t]{0.14\textwidth}
		\includegraphics[width=\linewidth]{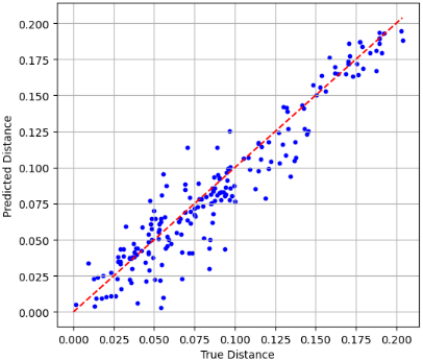}
		\caption{2-GB/GB-1-$\mathscr{D}^{'}_{\mathrm{te}}$}
	\end{subfigure}
	\hfill
	\begin{subfigure}[t]{0.14\textwidth}
		\includegraphics[width=\linewidth]{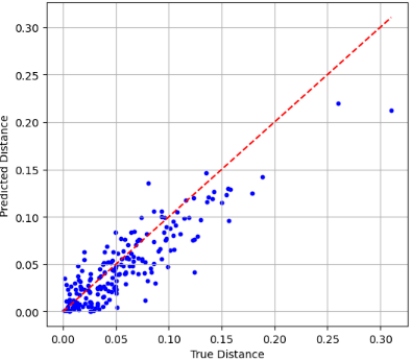}
		\caption{3-GB/GB-1-$\mathscr{D}^{'}_{\mathrm{te}}$}
	\end{subfigure}
	\hfill
	\begin{subfigure}[t]{0.14\textwidth}
		\includegraphics[width=\linewidth]{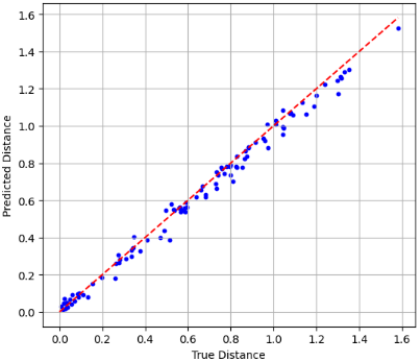}
		\caption{1-GB/GB-1-$\mathcal{S}^{'}$}
	\end{subfigure}
	\hfill
	\begin{subfigure}[t]{0.14\textwidth}
		\includegraphics[width=\linewidth]{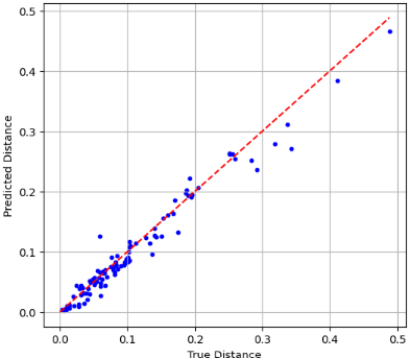}
		\caption{2-GB/GB-1-$\mathcal{S}^{'}$}
	\end{subfigure}
	\hfill
	\begin{subfigure}[t]{0.14\textwidth}
		\includegraphics[width=\linewidth]{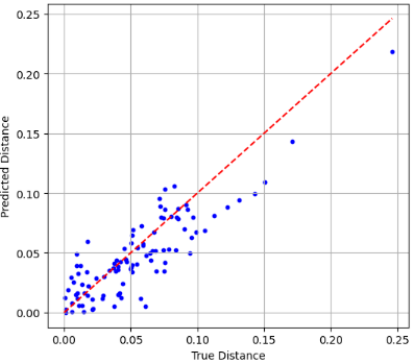}
		\caption{3-GB/GB-1-$\mathcal{S}^{'}$}
	\end{subfigure}
	
	
	\begin{subfigure}[t]{0.14\textwidth}
		\includegraphics[width=\linewidth]{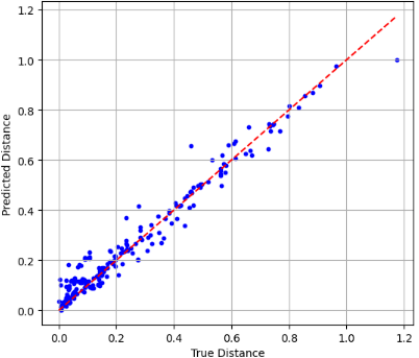}
		\caption{1-GB/GB-2-$\mathscr{D}^{'}_{\mathrm{te}}$}
	\end{subfigure}
	\hfill
	\begin{subfigure}[t]{0.14\textwidth}
		\includegraphics[width=\linewidth]{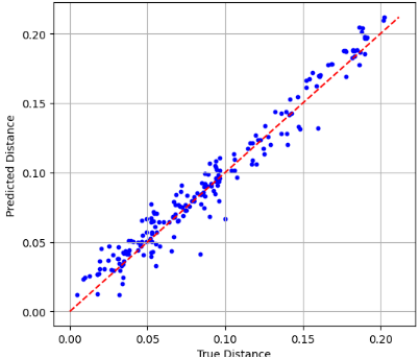}
		\caption{2-GB/GB-2-$\mathscr{D}^{'}_{\mathrm{te}}$}
	\end{subfigure}
	\hfill
	\begin{subfigure}[t]{0.14\textwidth}
		\includegraphics[width=\linewidth]{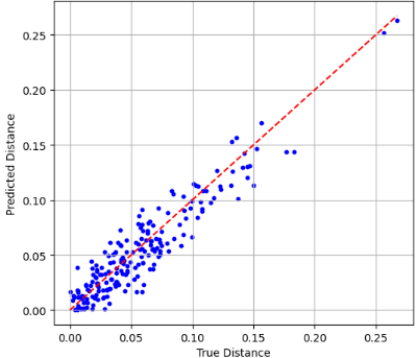}
		\caption{3-GB/GB-2-$\mathscr{D}^{'}_{\mathrm{te}}$}
	\end{subfigure}
	\hfill
	\begin{subfigure}[t]{0.14\textwidth}
		\includegraphics[width=\linewidth]{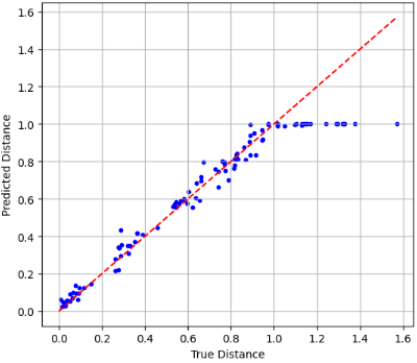}
		\caption{1-GB/GB-2-$\mathcal{S}^{'}$}
	\end{subfigure}
	\hfill
	\begin{subfigure}[t]{0.14\textwidth}
		\includegraphics[width=\linewidth]{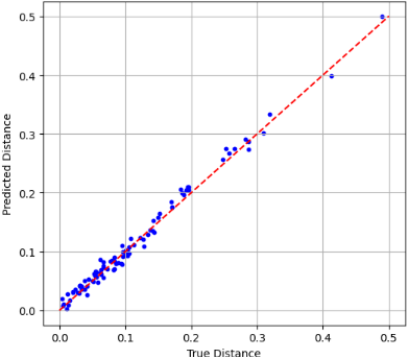}
		\caption{2-GB/GB-2-$\mathcal{S}^{'}$}
	\end{subfigure}
	\hfill
	\begin{subfigure}[t]{0.14\textwidth}
		\includegraphics[width=\linewidth]{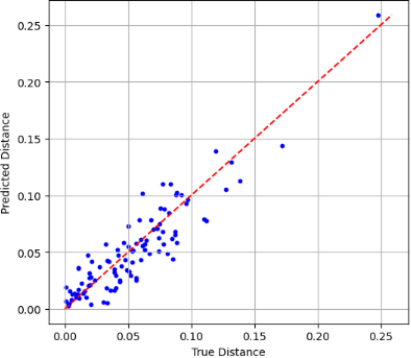}
		\caption{3-GB/GB-2-$\mathcal{S}^{'}$}
	\end{subfigure}
	
	
	\begin{subfigure}[t]{0.14\textwidth}
		\includegraphics[width=\linewidth]{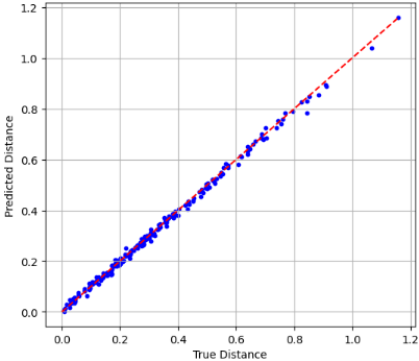}
		\caption{1-NN/GB-1-$\mathscr{D}^{'}_{\mathrm{te}}$}
	\end{subfigure}
	\hfill
	\begin{subfigure}[t]{0.14\textwidth}
		\includegraphics[width=\linewidth]{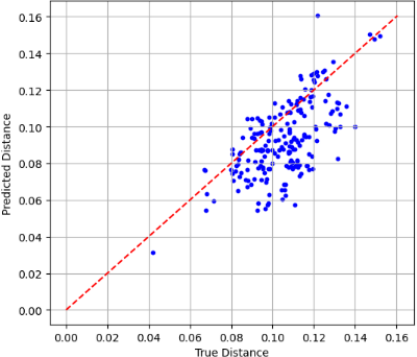}
		\caption{2-NN/GB-1-$\mathscr{D}^{'}_{\mathrm{te}}$}
	\end{subfigure}
	\hfill
	\begin{subfigure}[t]{0.14\textwidth}
		\includegraphics[width=\linewidth]{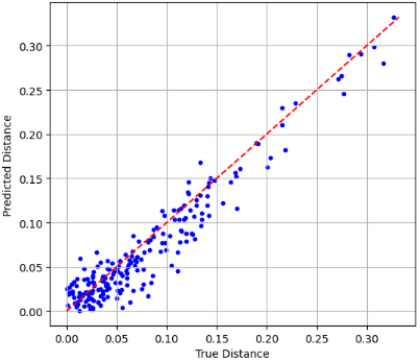}
		\caption{3-NN/GB-1-$\mathscr{D}^{'}_{\mathrm{te}}$}
	\end{subfigure}
	\hfill
	\begin{subfigure}[t]{0.14\textwidth}
		\includegraphics[width=\linewidth]{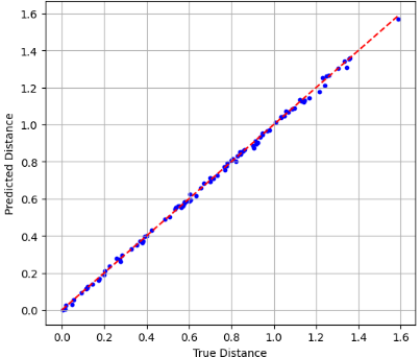}
		\caption{1-NN/GB-1-$\mathcal{S}^{'}$}
	\end{subfigure}
	\hfill
	\begin{subfigure}[t]{0.14\textwidth}
		\includegraphics[width=\linewidth]{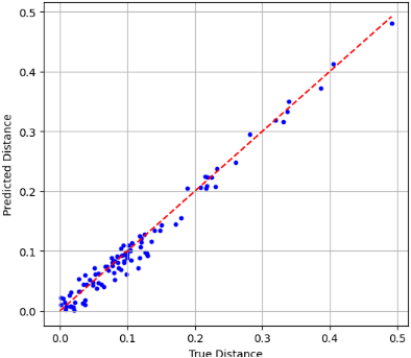}
		\caption{2-NN/GB-1-$\mathcal{S}^{'}$}
	\end{subfigure}
	\hfill
	\begin{subfigure}[t]{0.14\textwidth}
		\includegraphics[width=\linewidth]{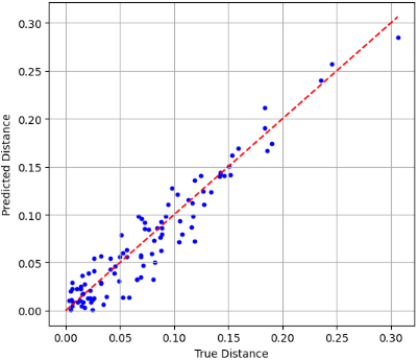}
		\caption{3-NN/GB-1-$\mathcal{S}^{'}$}
	\end{subfigure}
	
	
	\begin{subfigure}[t]{0.14\textwidth}
		\includegraphics[width=\linewidth]{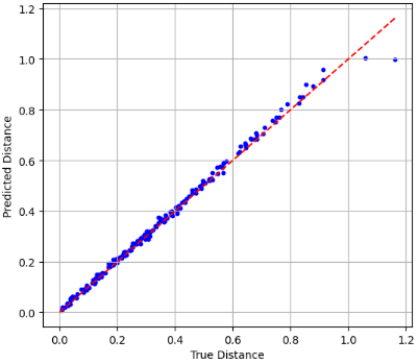}
		\caption{1-NN/GB-2-$\mathscr{D}^{'}_{\mathrm{te}}$}
	\end{subfigure}
	\hfill
	\begin{subfigure}[t]{0.14\textwidth}
		\includegraphics[width=\linewidth]{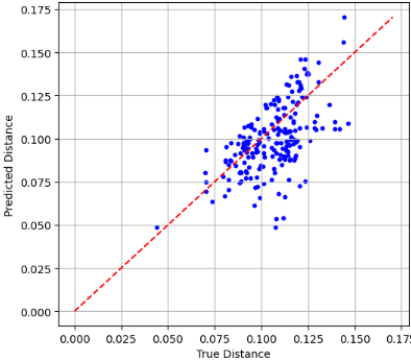}
		\caption{2-NN/GB-2-$\mathscr{D}^{'}_{\mathrm{te}}$}
	\end{subfigure}
	\hfill
	\begin{subfigure}[t]{0.14\textwidth}
		\includegraphics[width=\linewidth]{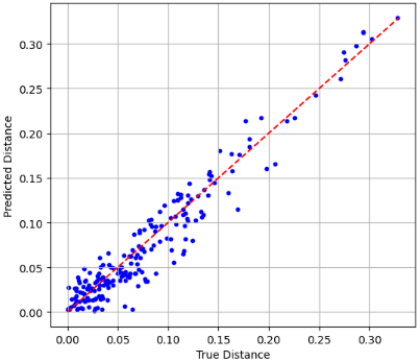}
		\caption{3-NN/GB-2-$\mathscr{D}^{'}_{\mathrm{te}}$}
	\end{subfigure}
	\hfill
	\begin{subfigure}[t]{0.14\textwidth}
		\includegraphics[width=\linewidth]{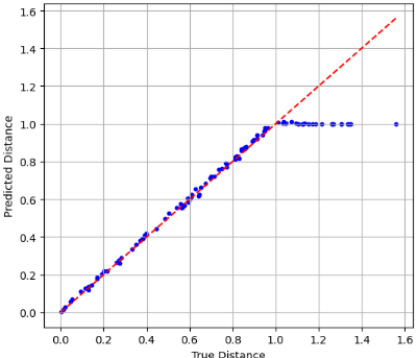}
		\caption{1-NN/GB-2-$\mathcal{S}^{'}$}
	\end{subfigure}
	\hfill
	\begin{subfigure}[t]{0.14\textwidth}
		\includegraphics[width=\linewidth]{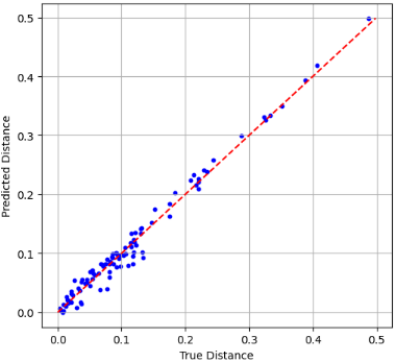}
		\caption{2-NN/GB-2-$\mathcal{S}^{'}$}
	\end{subfigure}
	\hfill
	\begin{subfigure}[t]{0.14\textwidth}
		\includegraphics[width=\linewidth]{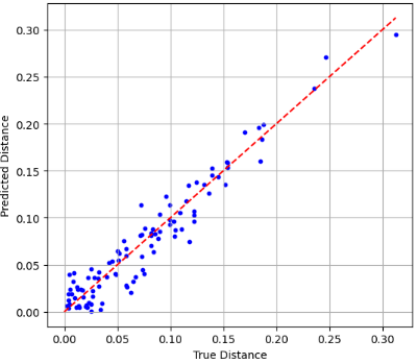}
		\caption{3-NN/GB-2-$\mathcal{S}^{'}$}
	\end{subfigure}
\end{figure}


\begin{figure}[!ht]
	\centering
	\begin{subfigure}[t]{0.14\textwidth}
		\includegraphics[width=\linewidth]{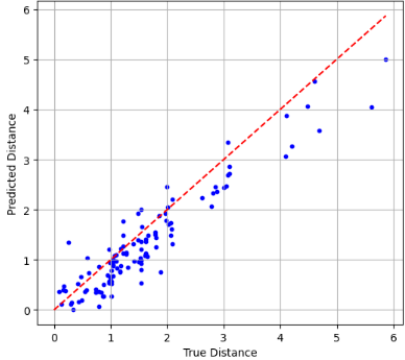}
		\caption{4-RF/SNN-1-$\mathscr{D}^{'}_{\mathrm{te}}$}
	\end{subfigure}
	\hfill
	\begin{subfigure}[t]{0.14\textwidth}
		\includegraphics[width=\linewidth]{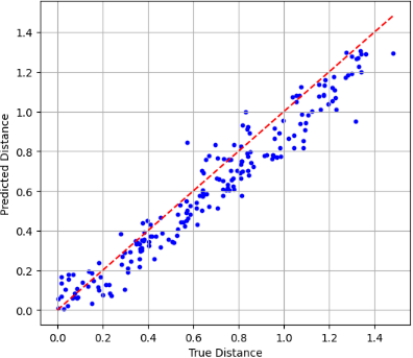}
		\caption{5-RF/SNN-1-$\mathscr{D}^{'}_{\mathrm{te}}$}
	\end{subfigure}
	\hfill
	\begin{subfigure}[t]{0.14\textwidth}
		\includegraphics[width=\linewidth]{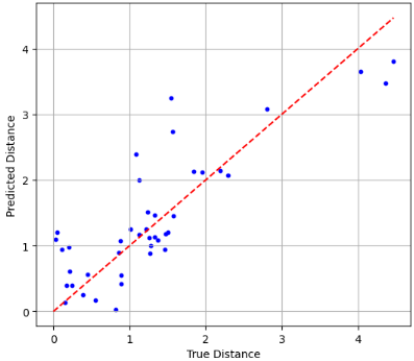}
		\caption{6-RF/SNN-1-$\mathscr{D}^{'}_{\mathrm{te}}$}
	\end{subfigure}
	\hfill
	\begin{subfigure}[t]{0.14\textwidth}
		\includegraphics[width=\linewidth]{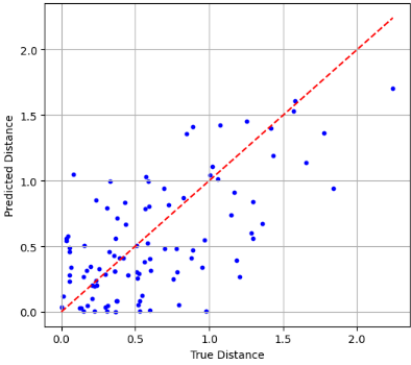}
		\caption{4-RF/SNN-1-$\mathcal{S}^{'}$}
	\end{subfigure}
	\hfill
	\begin{subfigure}[t]{0.14\textwidth}
		\includegraphics[width=\linewidth]{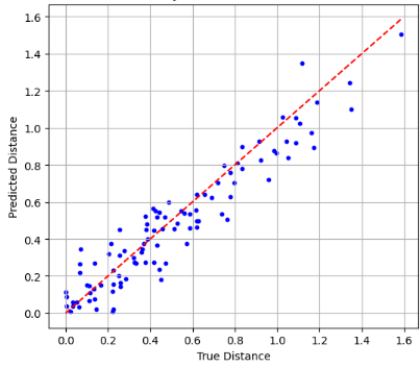}
		\caption{5-RF/SNN-1-$\mathcal{S}^{'}$}
	\end{subfigure}
	\hfill
	\begin{subfigure}[t]{0.14\textwidth}
		\includegraphics[width=\linewidth]{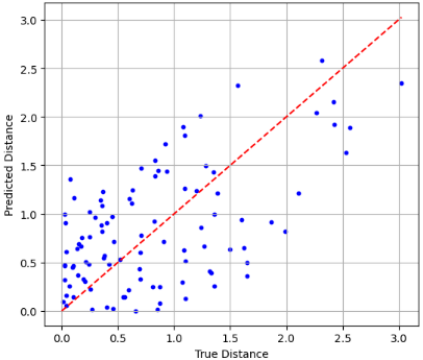}
		\caption{6-RF/SNN-1-$\mathcal{S}^{'}$}
	\end{subfigure}
	
	
	\begin{subfigure}[t]{0.14\textwidth}
		\includegraphics[width=\linewidth]{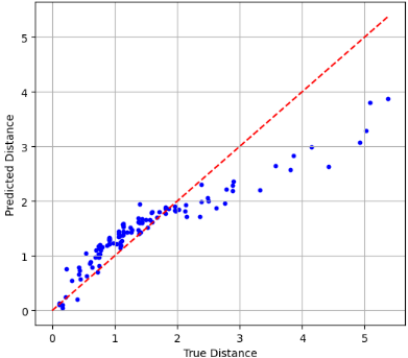}
		\caption{4-RF/SNN-2-$\mathscr{D}^{'}_{\mathrm{te}}$}
	\end{subfigure}
	\hfill
	\begin{subfigure}[t]{0.14\textwidth}
		\includegraphics[width=\linewidth]{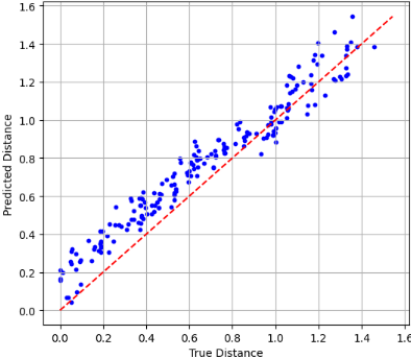}
		\caption{5-RF/SNN-2-$\mathscr{D}^{'}_{\mathrm{te}}$}
	\end{subfigure}
	\hfill
	\begin{subfigure}[t]{0.14\textwidth}
		\includegraphics[width=\linewidth]{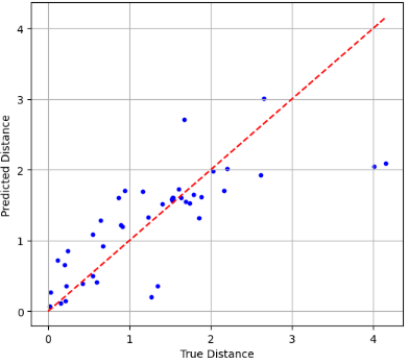}
		\caption{6-RF/SNN-2-$\mathscr{D}^{'}_{\mathrm{te}}$}
	\end{subfigure}
	\hfill
	\begin{subfigure}[t]{0.14\textwidth}
		\includegraphics[width=\linewidth]{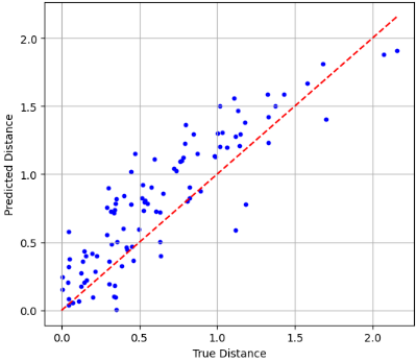}
		\caption{4-RF/SNN-2-$\mathcal{S}^{'}$}
	\end{subfigure}
	\hfill
	\begin{subfigure}[t]{0.14\textwidth}
		\includegraphics[width=\linewidth]{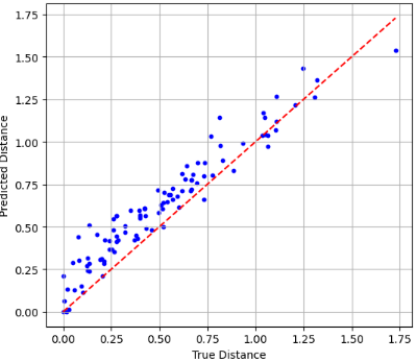}
		\caption{5-RF/SNN-2-$\mathcal{S}^{'}$}
	\end{subfigure}
	\hfill
	\begin{subfigure}[t]{0.14\textwidth}
		\includegraphics[width=\linewidth]{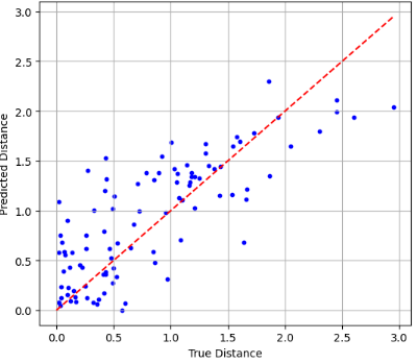}
		\caption{6-RF/SNN-2-$\mathcal{S}^{'}$}
	\end{subfigure}
	
	
	\begin{subfigure}[t]{0.14\textwidth}
		\includegraphics[width=\linewidth]{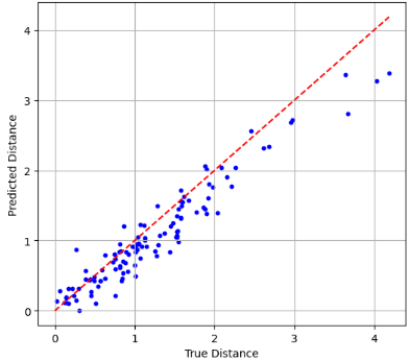}
		\caption{4-GB/SNN-1-$\mathscr{D}^{'}_{\mathrm{te}}$}
	\end{subfigure}
	\hfill
	\begin{subfigure}[t]{0.14\textwidth}
		\includegraphics[width=\linewidth]{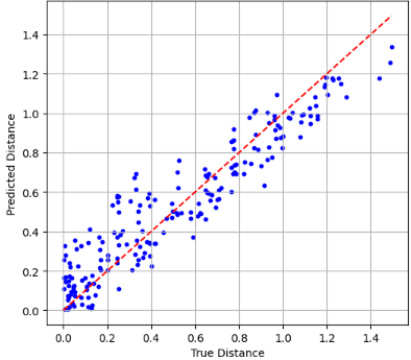}
		\caption{5-GB/SNN-1-$\mathscr{D}^{'}_{\mathrm{te}}$}
	\end{subfigure}
	\hfill
	\begin{subfigure}[t]{0.14\textwidth}
		\includegraphics[width=\linewidth]{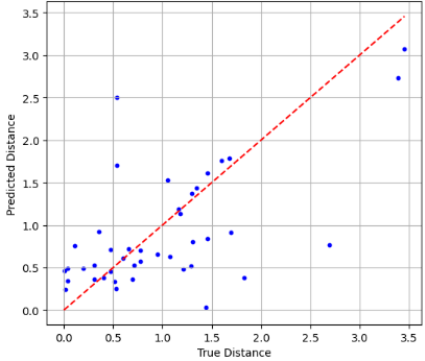}
		\caption{6-GB/SNN-1-$\mathscr{D}^{'}_{\mathrm{te}}$}
	\end{subfigure}
	\hfill
	\begin{subfigure}[t]{0.14\textwidth}
		\includegraphics[width=\linewidth]{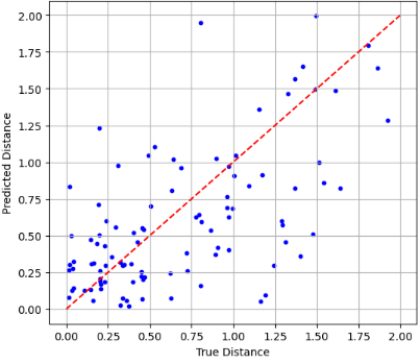}
		\caption{4-GB/SNN-1-$\mathcal{S}^{'}$}
	\end{subfigure}
	\hfill
	\begin{subfigure}[t]{0.14\textwidth}
		\includegraphics[width=\linewidth]{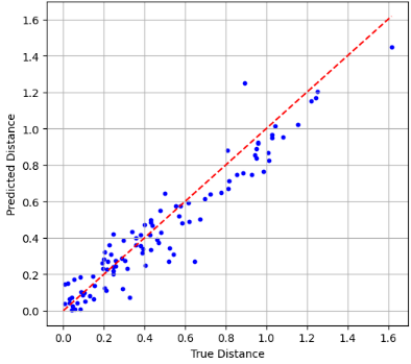}
		\caption{5-GB/SNN-1-$\mathcal{S}^{'}$}
	\end{subfigure}
	\hfill
	\begin{subfigure}[t]{0.14\textwidth}
		\includegraphics[width=\linewidth]{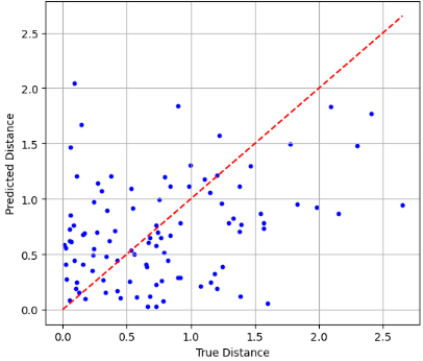}
		\caption{6-GB/SNN-1-$\mathcal{S}^{'}$}
	\end{subfigure}
	
	
	\begin{subfigure}[t]{0.14\textwidth}
		\includegraphics[width=\linewidth]{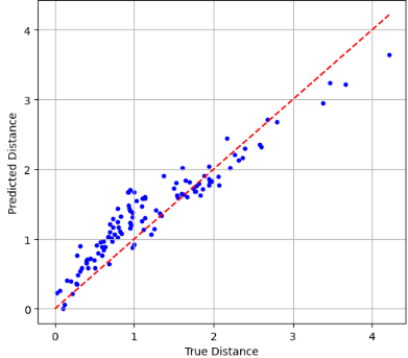}
		\caption{4-GB/SNN-2-$\mathscr{D}^{'}_{\mathrm{te}}$}
	\end{subfigure}
	\hfill
	\begin{subfigure}[t]{0.14\textwidth}
		\includegraphics[width=\linewidth]{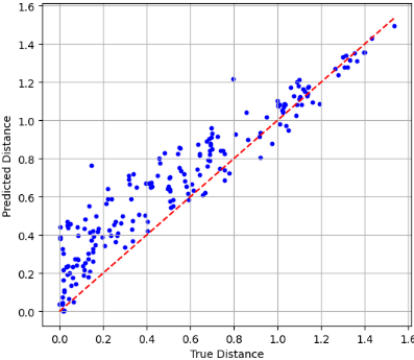}
		\caption{5-GB/SNN-2-$\mathscr{D}^{'}_{\mathrm{te}}$}
	\end{subfigure}
	\hfill
	\begin{subfigure}[t]{0.14\textwidth}
		\includegraphics[width=\linewidth]{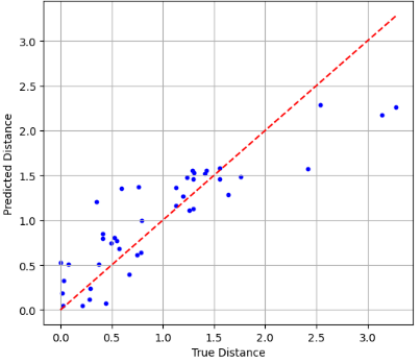}
		\caption{6-GB/SNN-2-$\mathscr{D}^{'}_{\mathrm{te}}$}
	\end{subfigure}
	\hfill
	\begin{subfigure}[t]{0.14\textwidth}
		\includegraphics[width=\linewidth]{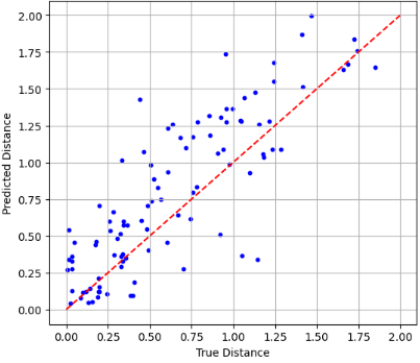}
		\caption{4-GB/SNN-2-$\mathcal{S}^{'}$}
	\end{subfigure}
	\hfill
	\begin{subfigure}[t]{0.14\textwidth}
		\includegraphics[width=\linewidth]{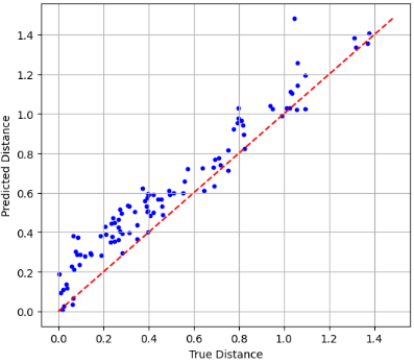}
		\caption{5-GB/SNN-2-$\mathcal{S}^{'}$}
	\end{subfigure}
	\hfill
	\begin{subfigure}[t]{0.14\textwidth}
		\includegraphics[width=\linewidth]{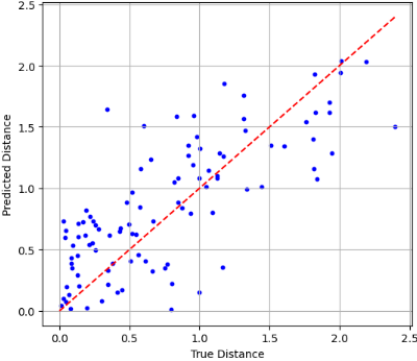}
		\caption{6-GB/SNN-2-$\mathcal{S}^{'}$}
	\end{subfigure}
	
	
	\begin{subfigure}[t]{0.14\textwidth}
		\includegraphics[width=\linewidth]{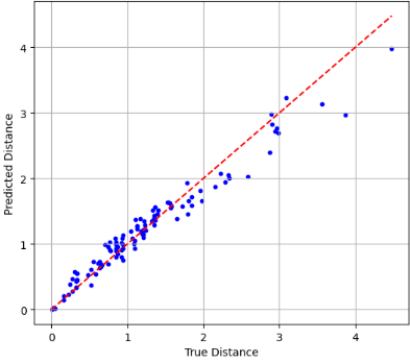}
		\caption{4-NN/SNN-1-$\mathscr{D}^{'}_{\mathrm{te}}$}
	\end{subfigure}
	\hfill
	\begin{subfigure}[t]{0.14\textwidth}
		\includegraphics[width=\linewidth]{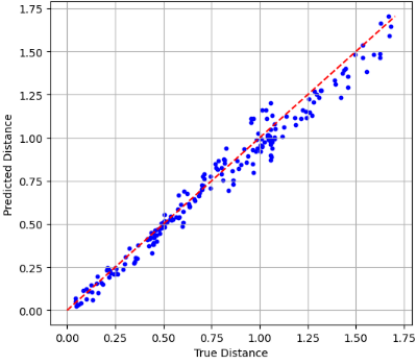}
		\caption{5-NN/SNN-1-$\mathscr{D}^{'}_{\mathrm{te}}$}
	\end{subfigure}
	\hfill
	\begin{subfigure}[t]{0.14\textwidth}
		\includegraphics[width=\linewidth]{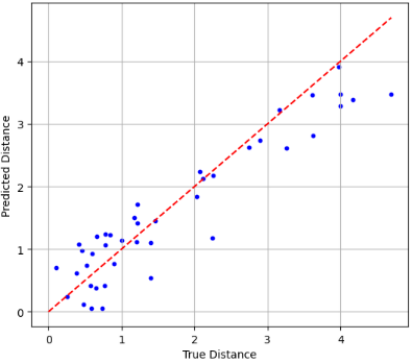}
		\caption{6-NN/SNN-1-$\mathscr{D}^{'}_{\mathrm{te}}$}
	\end{subfigure}
	\hfill
	\begin{subfigure}[t]{0.14\textwidth}
		\includegraphics[width=\linewidth]{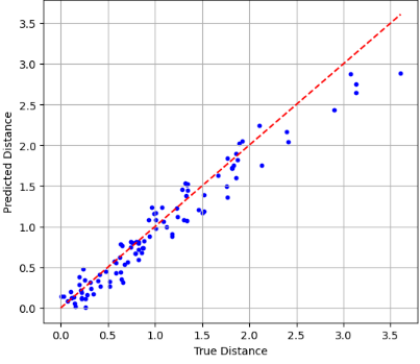}
		\caption{4-NN/SNN-1-$\mathcal{S}^{'}$}
	\end{subfigure}
	\hfill
	\begin{subfigure}[t]{0.14\textwidth}
		\includegraphics[width=\linewidth]{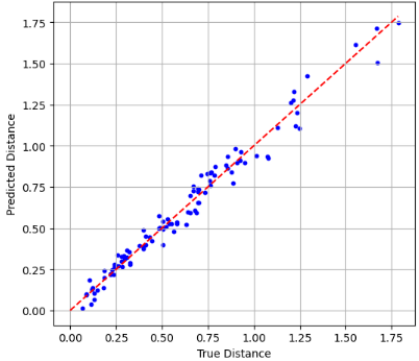}
		\caption{5-NN/SNN-1-$\mathcal{S}^{'}$}
	\end{subfigure}
	\hfill
	\begin{subfigure}[t]{0.14\textwidth}
		\includegraphics[width=\linewidth]{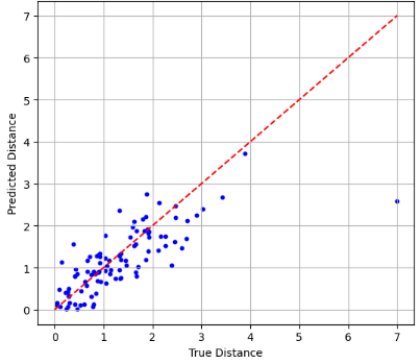}
		\caption{6-NN/SNN-1-$\mathcal{S}^{'}$}
	\end{subfigure}
	
	
	\begin{subfigure}[t]{0.14\textwidth}
		\includegraphics[width=\linewidth]{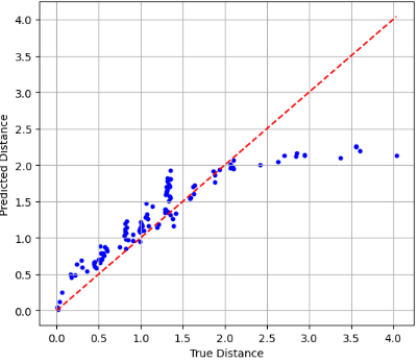}
		\caption{4-NN/SNN-2-$\mathscr{D}^{'}_{\mathrm{te}}$}
	\end{subfigure}
	\hfill
	\begin{subfigure}[t]{0.14\textwidth}
		\includegraphics[width=\linewidth]{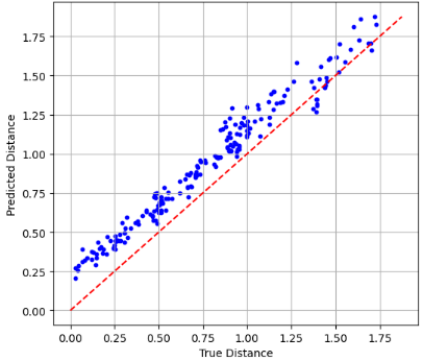}
		\caption{5-NN/SNN-2-$\mathscr{D}^{'}_{\mathrm{te}}$}
	\end{subfigure}
	\hfill
	\begin{subfigure}[t]{0.14\textwidth}
		\includegraphics[width=\linewidth]{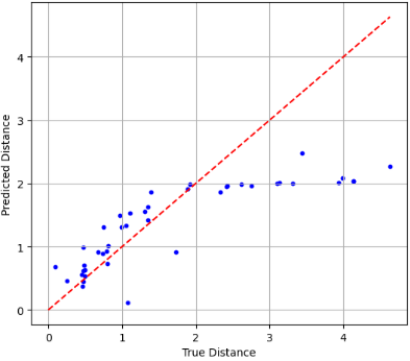}
		\caption{6-NN/SNN-2-$\mathscr{D}^{'}_{\mathrm{te}}$}
	\end{subfigure}
	\hfill
	\begin{subfigure}[t]{0.14\textwidth}
		\includegraphics[width=\linewidth]{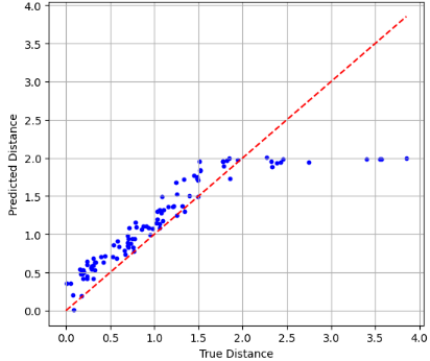}
		\caption{4-NN/SNN-2-$\mathcal{S}^{'}$}
	\end{subfigure}
	\hfill
	\begin{subfigure}[t]{0.14\textwidth}
		\includegraphics[width=\linewidth]{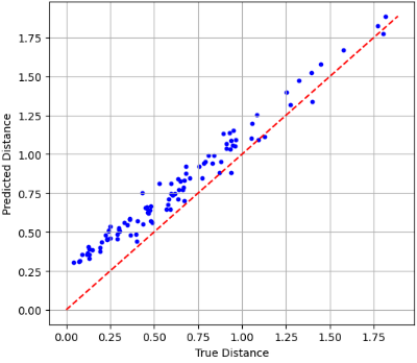}
		\caption{5-NN/SNN-2-$\mathcal{S}^{'}$}
	\end{subfigure}
	\hfill
	\begin{subfigure}[t]{0.14\textwidth}
		\includegraphics[width=\linewidth]{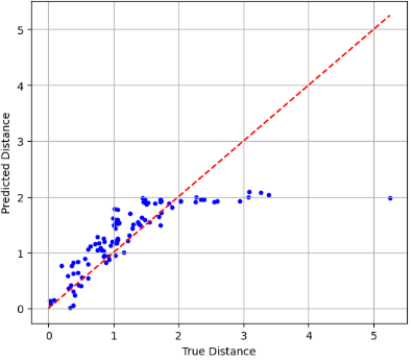}
		\caption{6-NN/SNN-2-$\mathcal{S}^{'}$}
	\end{subfigure}
	
	
	\begin{subfigure}[t]{0.14\textwidth}
		\includegraphics[width=\linewidth]{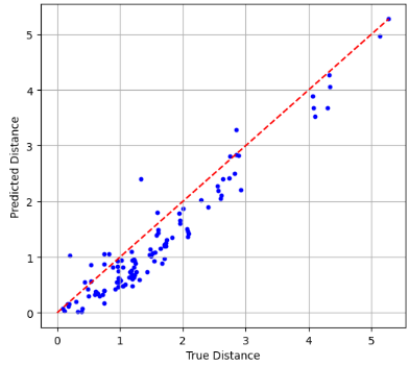}
		\caption{\scalebox{0.95}{4-RF/MNN-1-$\mathscr{D}^{'}_{\mathrm{te}}$}}
	\end{subfigure}
	\hfill
	\begin{subfigure}[t]{0.14\textwidth}
		\includegraphics[width=\linewidth]{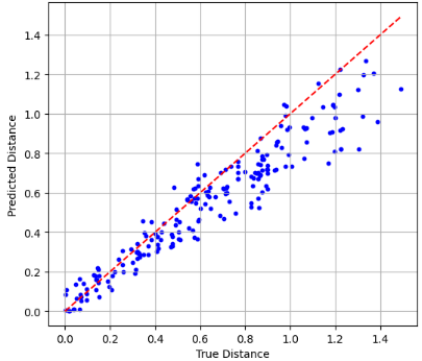}
		\caption{\scalebox{0.95}{5-RF/MNN-1-$\mathscr{D}^{'}_{\mathrm{te}}$}}
	\end{subfigure}
	\hfill
	\begin{subfigure}[t]{0.14\textwidth}
		\includegraphics[width=\linewidth]{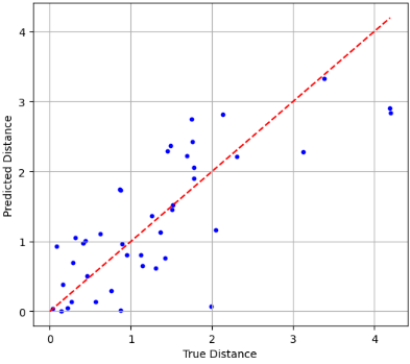}
		\caption{\scalebox{0.95}{6-RF/MNN-1-$\mathscr{D}^{'}_{\mathrm{te}}$}}
	\end{subfigure}
	\hfill
	\begin{subfigure}[t]{0.14\textwidth}
		\includegraphics[width=\linewidth]{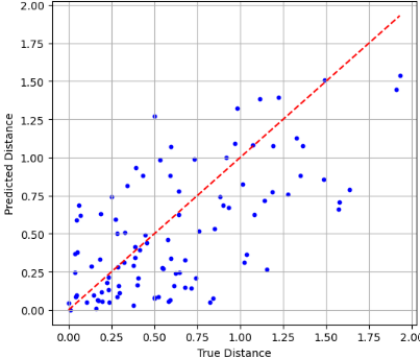}
		\caption{4-RF/MNN-1-$\mathcal{S}^{'}$}
	\end{subfigure}
	\hfill
	\begin{subfigure}[t]{0.14\textwidth}
		\includegraphics[width=\linewidth]{Figures/Figure_9/5-RF-MNN-1-S.pdf}
		\caption{5-RF/MNN-1-$\mathcal{S}^{'}$}
	\end{subfigure}
	\hfill
	\begin{subfigure}[t]{0.14\textwidth}
		\includegraphics[width=\linewidth]{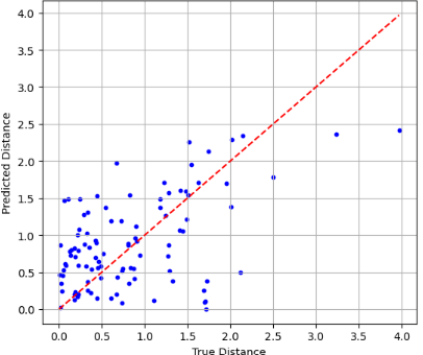}
		\caption{6-RF/MNN-1-$\mathcal{S}^{'}$}
	\end{subfigure}
	
	
	\begin{subfigure}[t]{0.14\textwidth}
		\includegraphics[width=\linewidth]{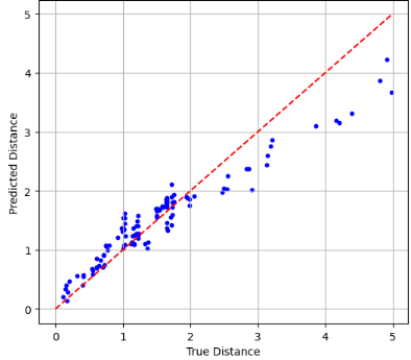}
		\caption{\scalebox{0.95}{4-RF/MNN-2-$\mathscr{D}^{'}_{\mathrm{te}}$}}
	\end{subfigure}
	\hfill
	\begin{subfigure}[t]{0.14\textwidth}
		\includegraphics[width=\linewidth]{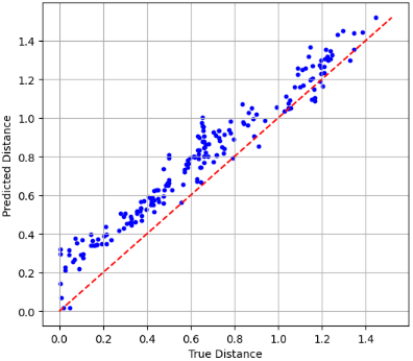}
		\caption{\scalebox{0.95}{5-RF/MNN-2-$\mathscr{D}^{'}_{\mathrm{te}}$}}
	\end{subfigure}
	\hfill
	\begin{subfigure}[t]{0.14\textwidth}
		\includegraphics[width=\linewidth]{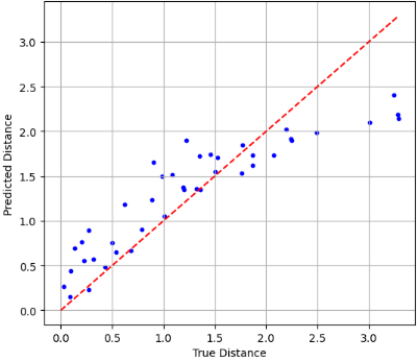}
		\caption{\scalebox{0.95}{6-RF/MNN-2-$\mathscr{D}^{'}_{\mathrm{te}}$}}
	\end{subfigure}
	\hfill
	\begin{subfigure}[t]{0.14\textwidth}
		\includegraphics[width=\linewidth]{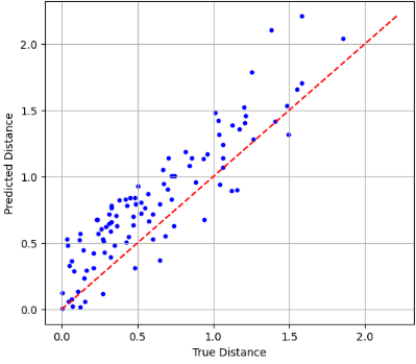}
		\caption{4-RF/MNN-2-$\mathcal{S}^{'}$}
	\end{subfigure}
	\hfill
	\begin{subfigure}[t]{0.14\textwidth}
		\includegraphics[width=\linewidth]{Figures/Figure_9/5-RF-MNN-2-S.pdf}
		\caption{5-RF/MNN-2-$\mathcal{S}^{'}$}
	\end{subfigure}
	\hfill
	\begin{subfigure}[t]{0.14\textwidth}
		\includegraphics[width=\linewidth]{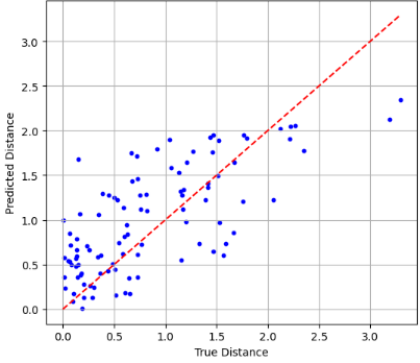}
		\caption{6-RF/MNN-2-$\mathcal{S}^{'}$}
	\end{subfigure}
\end{figure}

\begin{figure}[!ht]
	\centering
	\begin{subfigure}[t]{0.14\textwidth}
		\includegraphics[width=\linewidth]{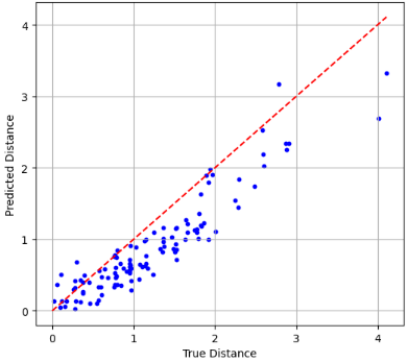}
		\caption{\scalebox{0.95}{4-GB/MNN-1-$\mathscr{D}^{'}_{\mathrm{te}}$}}
	\end{subfigure}
	\hfill
	\begin{subfigure}[t]{0.14\textwidth}
		\includegraphics[width=\linewidth]{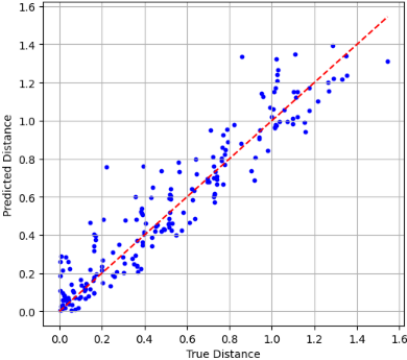}
		\caption{\scalebox{0.95}{5-GB/MNN-1-$\mathscr{D}^{'}_{\mathrm{te}}$}}
	\end{subfigure}
	\hfill
	\begin{subfigure}[t]{0.14\textwidth}
		\includegraphics[width=\linewidth]{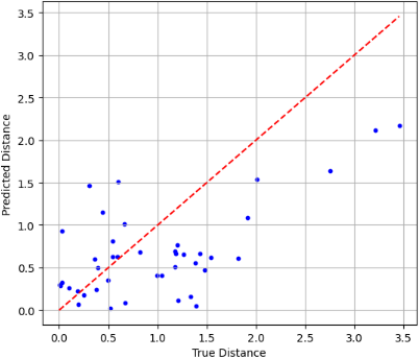}
		\caption{\scalebox{0.95}{6-GB/MNN-1-$\mathscr{D}^{'}_{\mathrm{te}}$}}
	\end{subfigure}
	\hfill
	\begin{subfigure}[t]{0.14\textwidth}
		\includegraphics[width=\linewidth]{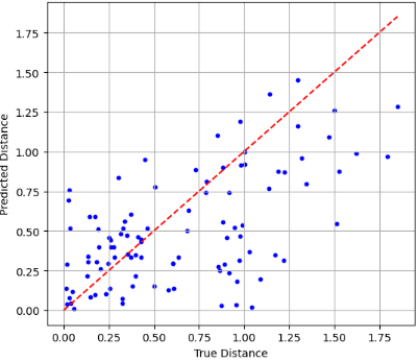}
		\caption{4-GB/MNN-1-$\mathcal{S}^{'}$}
	\end{subfigure}
	\hfill
	\begin{subfigure}[t]{0.14\textwidth}
		\includegraphics[width=\linewidth]{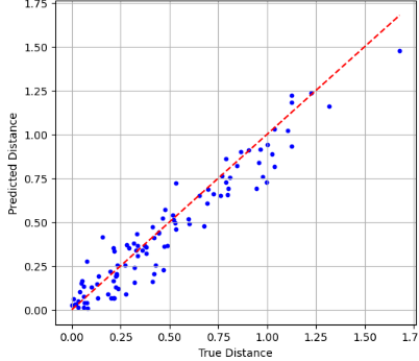}
		\caption{5-GB/MNN-1-$\mathcal{S}^{'}$}
	\end{subfigure}
	\hfill
	\begin{subfigure}[t]{0.14\textwidth}
		\includegraphics[width=\linewidth]{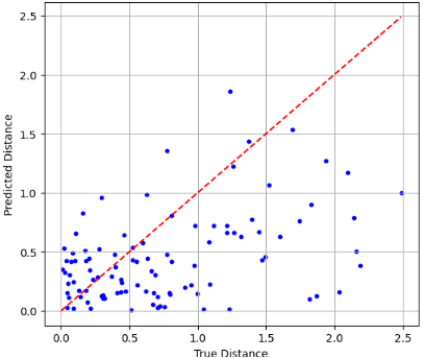}
		\caption{6-GB/MNN-1-$\mathcal{S}^{'}$}
	\end{subfigure}
	
	
	\begin{subfigure}[t]{0.14\textwidth}
		\includegraphics[width=\linewidth]{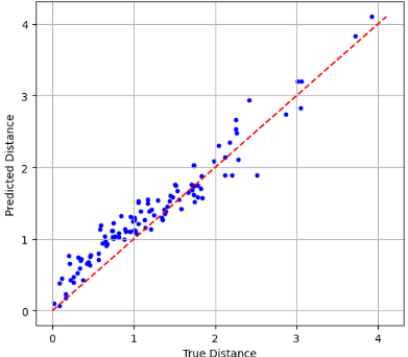}
		\caption{\scalebox{0.95}{4-GB/MNN-2-$\mathscr{D}^{'}_{\mathrm{te}}$}}
	\end{subfigure}
	\hfill
	\begin{subfigure}[t]{0.14\textwidth}
		\includegraphics[width=\linewidth]{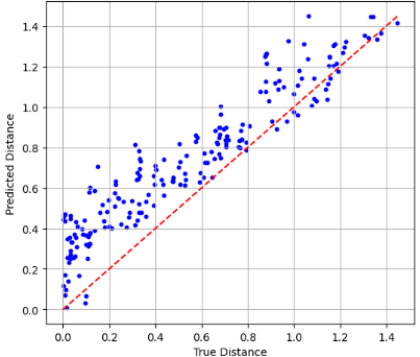}
		\caption{\scalebox{0.95}{5-GB/MNN-2-$\mathscr{D}^{'}_{\mathrm{te}}$}}
	\end{subfigure}
	\hfill
	\begin{subfigure}[t]{0.14\textwidth}
		\includegraphics[width=\linewidth]{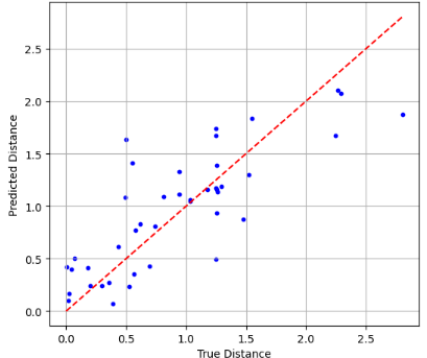}
		\caption{\scalebox{0.95}{6-GB/MNN-2-$\mathscr{D}^{'}_{\mathrm{te}}$}}
	\end{subfigure}
	\hfill
	\begin{subfigure}[t]{0.14\textwidth}
		\includegraphics[width=\linewidth]{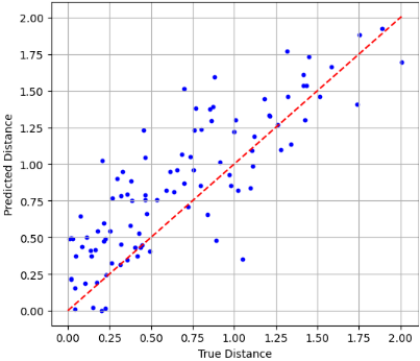}
		\caption{4-GB/MNN-2-$\mathcal{S}^{'}$}
	\end{subfigure}
	\hfill
	\begin{subfigure}[t]{0.14\textwidth}
		\includegraphics[width=\linewidth]{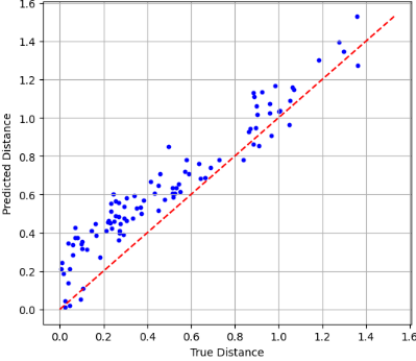}
		\caption{5-GB/MNN-2-$\mathcal{S}^{'}$}
	\end{subfigure}
	\hfill
	\begin{subfigure}[t]{0.14\textwidth}
		\includegraphics[width=\linewidth]{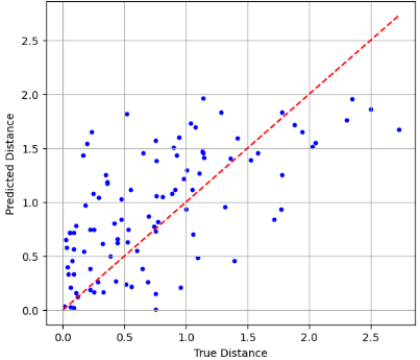}
		\caption{6-GB/MNN-2-$\mathcal{S}^{'}$}
	\end{subfigure}
	
	
	\begin{subfigure}[t]{0.14\textwidth}
		\includegraphics[width=\linewidth]{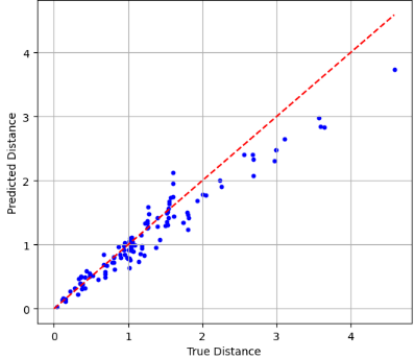}
		\caption{\scalebox{0.95}{4-NN/MNN-1-$\mathscr{D}^{'}_{\mathrm{te}}$}}
	\end{subfigure}
	\hfill
	\begin{subfigure}[t]{0.14\textwidth}
		\includegraphics[width=\linewidth]{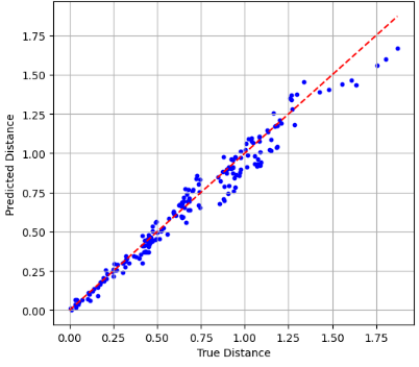}
		\caption{\scalebox{0.95}{5-NN/MNN-1-$\mathscr{D}^{'}_{\mathrm{te}}$}}
	\end{subfigure}
	\hfill
	\begin{subfigure}[t]{0.14\textwidth}
		\includegraphics[width=\linewidth]{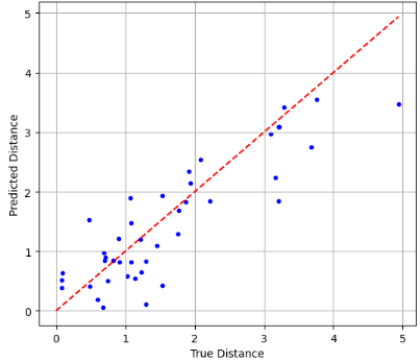}
		\caption{\scalebox{0.95}{6-NN/MNN-1-$\mathscr{D}^{'}_{\mathrm{te}}$}}
	\end{subfigure}
	\hfill
	\begin{subfigure}[t]{0.14\textwidth}
		\includegraphics[width=\linewidth]{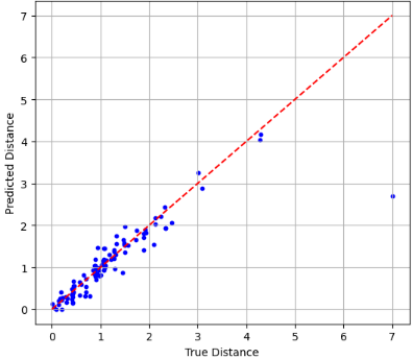}
		\caption{4-NN/MNN-1-$\mathcal{S}^{'}$}
	\end{subfigure}
	\hfill
	\begin{subfigure}[t]{0.14\textwidth}
		\includegraphics[width=\linewidth]{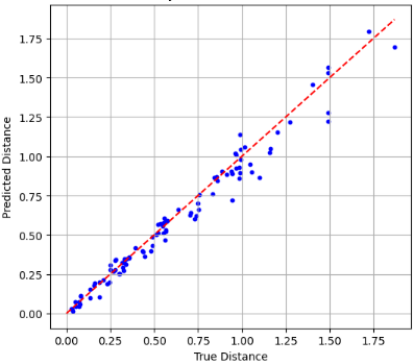}
		\caption{5-NN/MNN-1-$\mathcal{S}^{'}$}
	\end{subfigure}
	\hfill
	\begin{subfigure}[t]{0.14\textwidth}
		\includegraphics[width=\linewidth]{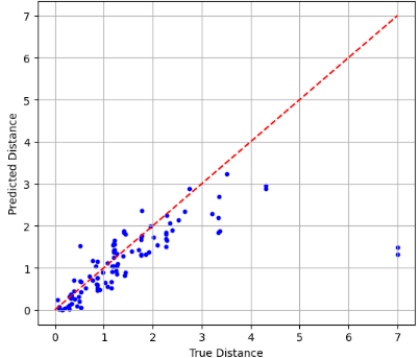}
		\caption{6-NN/MNN-1-$\mathcal{S}^{'}$}
	\end{subfigure}
	
	
	\begin{subfigure}[t]{0.14\textwidth}
		\includegraphics[width=\linewidth]{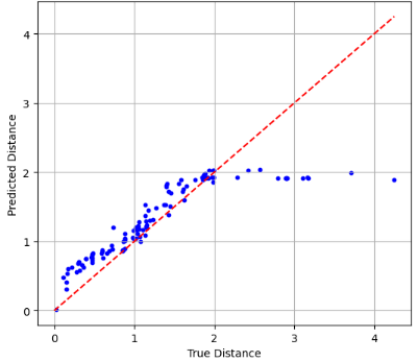}
		\caption{\scalebox{0.95}{4-NN/MNN-2-$\mathscr{D}^{'}_{\mathrm{te}}$}}
	\end{subfigure}
	\hfill
	\begin{subfigure}[t]{0.14\textwidth}
		\includegraphics[width=\linewidth]{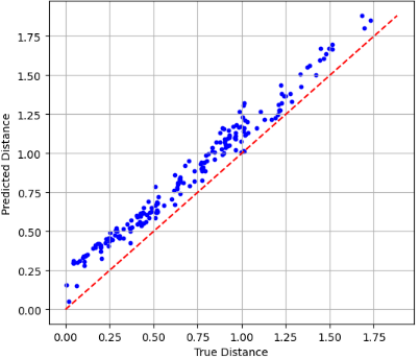}
		\caption{\scalebox{0.95}{5-NN/MNN-2-$\mathscr{D}^{'}_{\mathrm{te}}$}}
	\end{subfigure}
	\hfill
	\begin{subfigure}[t]{0.14\textwidth}
		\includegraphics[width=\linewidth]{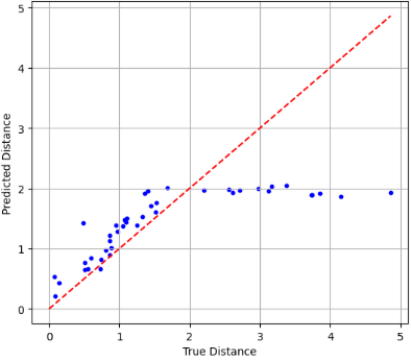}
		\caption{\scalebox{0.95}{6-NN/MNN-2-$\mathscr{D}^{'}_{\mathrm{te}}$}}
	\end{subfigure}
	\hfill
	\begin{subfigure}[t]{0.14\textwidth}
		\includegraphics[width=\linewidth]{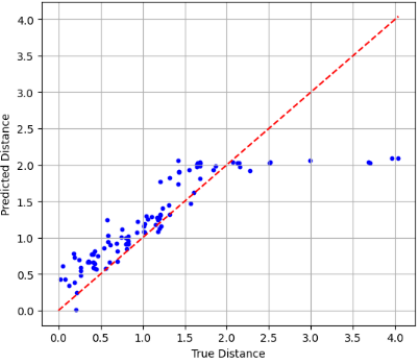}
		\caption{4-NN/MNN-2-$\mathcal{S}^{'}$}
	\end{subfigure}
	\hfill
	\begin{subfigure}[t]{0.14\textwidth}
		\includegraphics[width=\linewidth]{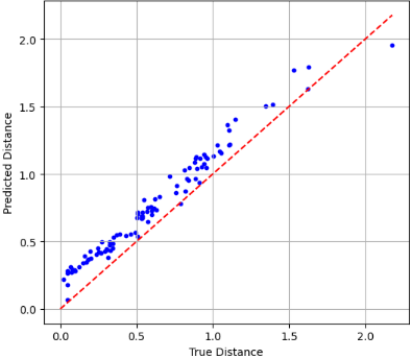}
		\caption{5-NN/MNN-2-$\mathcal{S}^{'}$}
	\end{subfigure}
	\hfill
	\begin{subfigure}[t]{0.14\textwidth}
		\includegraphics[width=\linewidth]{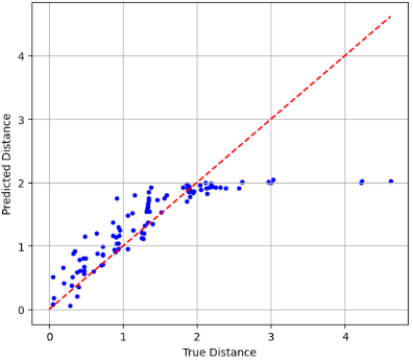}
		\caption{6-NN/MNN-2-$\mathcal{S}^{'}$}
	\end{subfigure}
	
	
	\begin{subfigure}[t]{0.14\textwidth}
		\includegraphics[width=\linewidth]{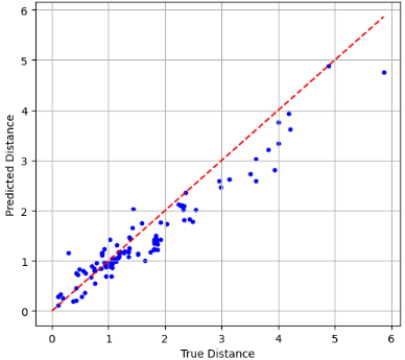}
		\caption{4-RF/LNN-1-$\mathscr{D}^{'}_{\mathrm{te}}$}
	\end{subfigure}
	\hfill
	\begin{subfigure}[t]{0.14\textwidth}
		\includegraphics[width=\linewidth]{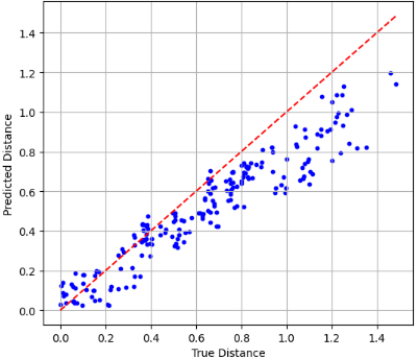}
		\caption{5-RF/LNN-1-$\mathscr{D}^{'}_{\mathrm{te}}$}
	\end{subfigure}
	\hfill
	\begin{subfigure}[t]{0.14\textwidth}
		\includegraphics[width=\linewidth]{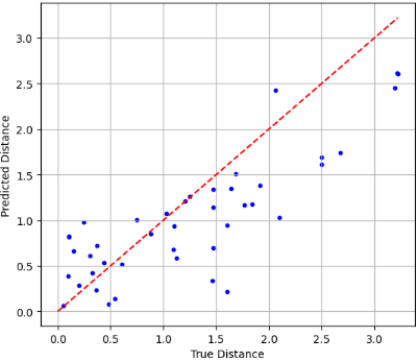}
		\caption{6-RF/LNN-1-$\mathscr{D}^{'}_{\mathrm{te}}$}
	\end{subfigure}
	\hfill
	\begin{subfigure}[t]{0.14\textwidth}
		\includegraphics[width=\linewidth]{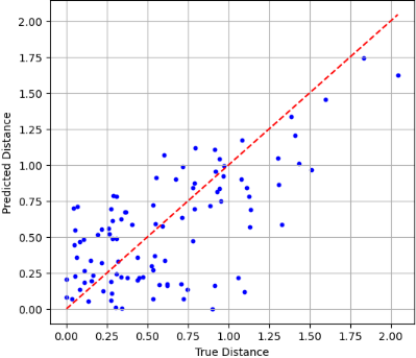}
		\caption{4-RF/LNN-1-$\mathcal{S}^{'}$}
	\end{subfigure}
	\hfill
	\begin{subfigure}[t]{0.14\textwidth}
		\includegraphics[width=\linewidth]{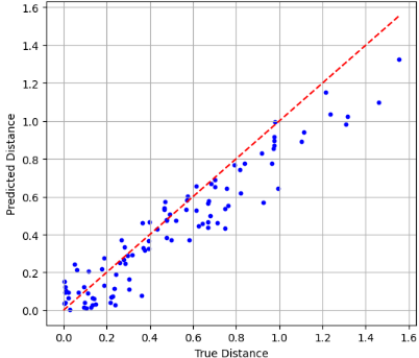}
		\caption{5-RF/LNN-1-$\mathcal{S}^{'}$}
	\end{subfigure}
	\hfill
	\begin{subfigure}[t]{0.14\textwidth}
		\includegraphics[width=\linewidth]{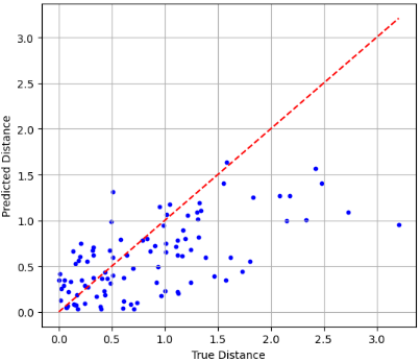}
		\caption{6-RF/LNN-1-$\mathcal{S}^{'}$}
	\end{subfigure}
	
	
	\begin{subfigure}[t]{0.14\textwidth}
		\includegraphics[width=\linewidth]{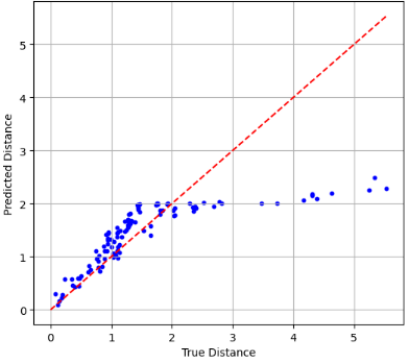}
		\caption{4-RF/LNN-2-$\mathscr{D}^{'}_{\mathrm{te}}$}
	\end{subfigure}
	\hfill
	\begin{subfigure}[t]{0.14\textwidth}
		\includegraphics[width=\linewidth]{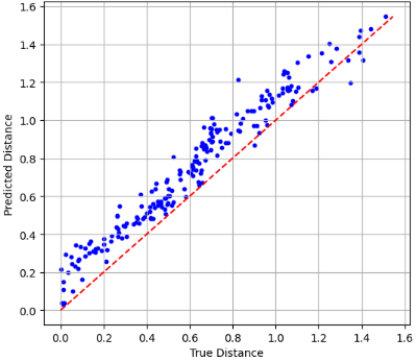}
		\caption{5-RF/LNN-2-$\mathscr{D}^{'}_{\mathrm{te}}$}
	\end{subfigure}
	\hfill
	\begin{subfigure}[t]{0.14\textwidth}
		\includegraphics[width=\linewidth]{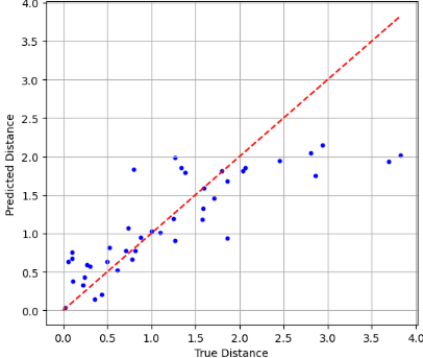}
		\caption{6-RF/LNN-2-$\mathscr{D}^{'}_{\mathrm{te}}$}
	\end{subfigure}
	\hfill
	\begin{subfigure}[t]{0.14\textwidth}
		\includegraphics[width=\linewidth]{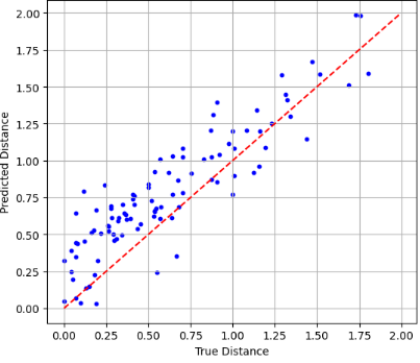}
		\caption{4-RF/LNN-2-$\mathcal{S}^{'}$}
	\end{subfigure}
	\hfill
	\begin{subfigure}[t]{0.14\textwidth}
		\includegraphics[width=\linewidth]{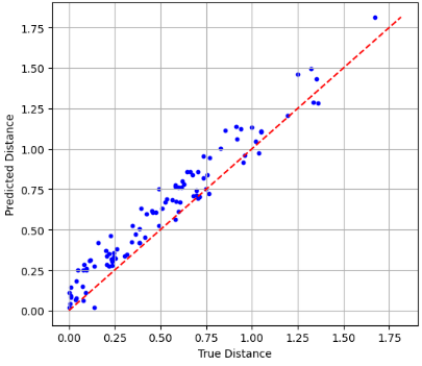}
		\caption{5-RF/LNN-2-$\mathcal{S}^{'}$}
	\end{subfigure}
	\hfill
	\begin{subfigure}[t]{0.14\textwidth}
		\includegraphics[width=\linewidth]{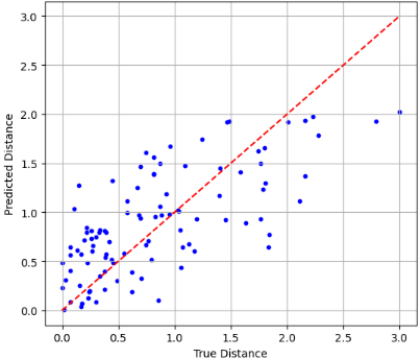}
		\caption{6-RF/LNN-2-$\mathcal{S}^{'}$}
	\end{subfigure}
	
	
	\begin{subfigure}[t]{0.14\textwidth}
		\includegraphics[width=\linewidth]{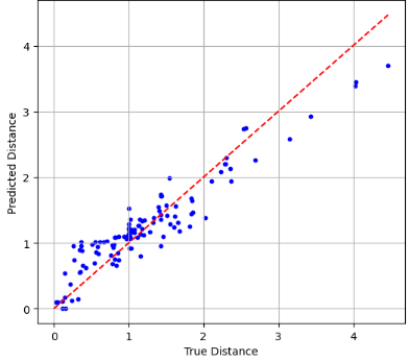}
		\caption{4-GB/LNN-1-$\mathscr{D}^{'}_{\mathrm{te}}$}
	\end{subfigure}
	\hfill
	\begin{subfigure}[t]{0.14\textwidth}
		\includegraphics[width=\linewidth]{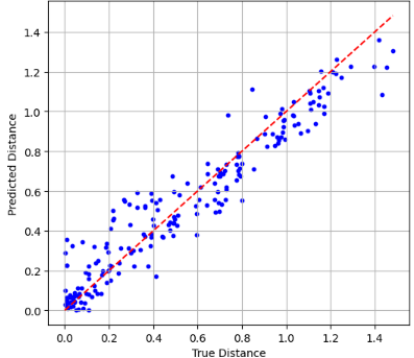}
		\caption{5-GB/LNN-1-$\mathscr{D}^{'}_{\mathrm{te}}$}
	\end{subfigure}
	\hfill
	\begin{subfigure}[t]{0.14\textwidth}
		\includegraphics[width=\linewidth]{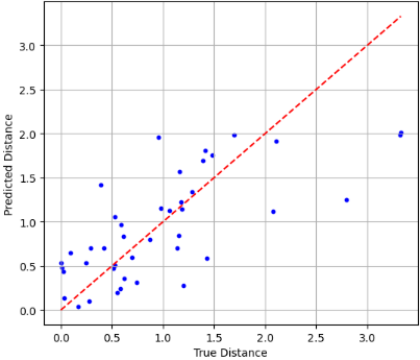}
		\caption{6-GB/LNN-1-$\mathscr{D}^{'}_{\mathrm{te}}$}
	\end{subfigure}
	\hfill
	\begin{subfigure}[t]{0.14\textwidth}
		\includegraphics[width=\linewidth]{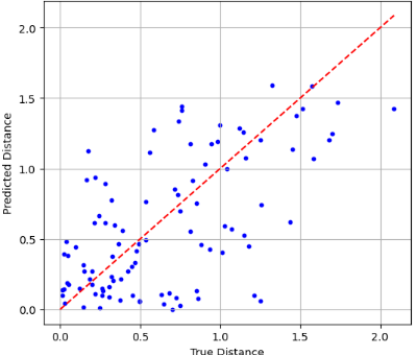}
		\caption{4-GB/LNN-1-$\mathcal{S}^{'}$}
	\end{subfigure}
	\hfill
	\begin{subfigure}[t]{0.14\textwidth}
		\includegraphics[width=\linewidth]{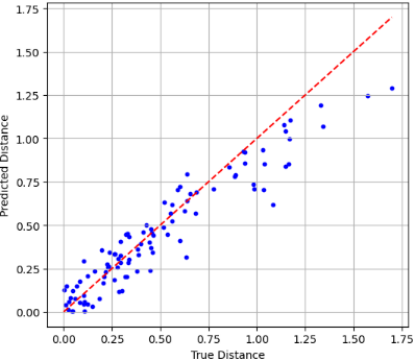}
		\caption{5-GB/LNN-1-$\mathcal{S}^{'}$}
	\end{subfigure}
	\hfill
	\begin{subfigure}[t]{0.14\textwidth}
		\includegraphics[width=\linewidth]{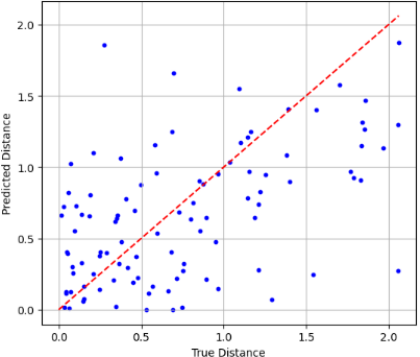}
		\caption{6-GB/LNN-1-$\mathcal{S}^{'}$}
	\end{subfigure}
	
	
	\begin{subfigure}[t]{0.14\textwidth}
		\includegraphics[width=\linewidth]{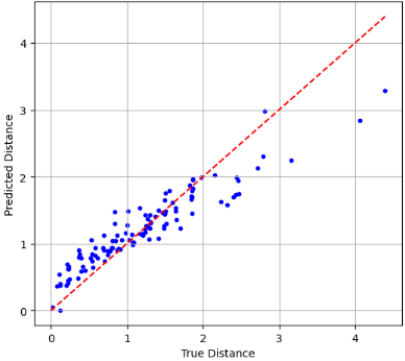}
		\caption{4-GB/LNN-2-$\mathscr{D}^{'}_{\mathrm{te}}$}
	\end{subfigure}
	\hfill
	\begin{subfigure}[t]{0.14\textwidth}
		\includegraphics[width=\linewidth]{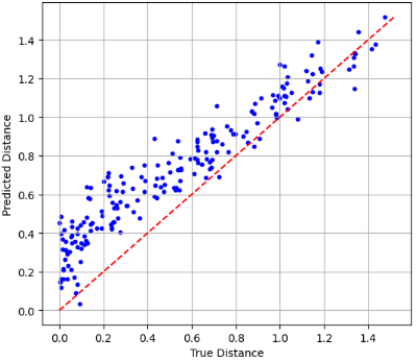}
		\caption{5-GB/LNN-2-$\mathscr{D}^{'}_{\mathrm{te}}$}
	\end{subfigure}
	\hfill
	\begin{subfigure}[t]{0.14\textwidth}
		\includegraphics[width=\linewidth]{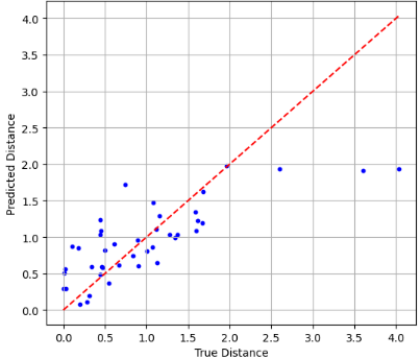}
		\caption{6-GB/LNN-2-$\mathscr{D}^{'}_{\mathrm{te}}$}
	\end{subfigure}
	\hfill
	\begin{subfigure}[t]{0.14\textwidth}
		\includegraphics[width=\linewidth]{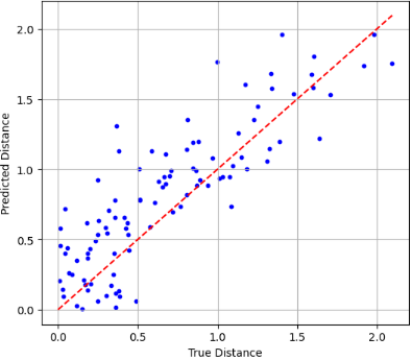}
		\caption{4-GB/LNN-2-$\mathcal{S}^{'}$}
	\end{subfigure}
	\hfill
	\begin{subfigure}[t]{0.14\textwidth}
		\includegraphics[width=\linewidth]{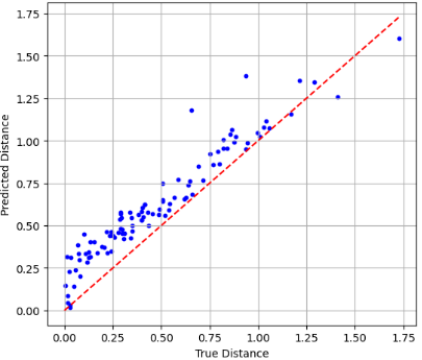}
		\caption{5-GB/LNN-2-$\mathcal{S}^{'}$}
	\end{subfigure}
	\hfill
	\begin{subfigure}[t]{0.14\textwidth}
		\includegraphics[width=\linewidth]{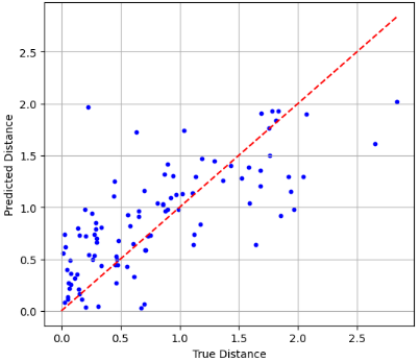}
		\caption{6-GB/LNN-2-$\mathcal{S}^{'}$}
	\end{subfigure}
\end{figure}

\begin{figure}[!ht]
	\centering
	\begin{subfigure}[t]{0.14\textwidth}
		\includegraphics[width=\linewidth]{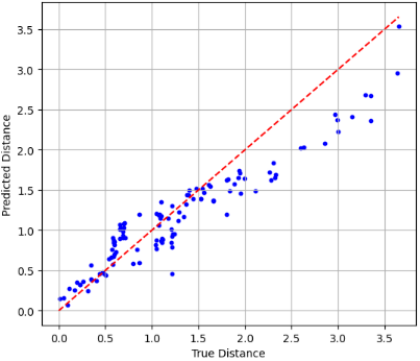}
		\caption{\scalebox{0.95}{4-NN/LNN-1-$\mathscr{D}^{'}_{\mathrm{te}}$}}
	\end{subfigure}
	\hfill
	\begin{subfigure}[t]{0.14\textwidth}
		\includegraphics[width=\linewidth]{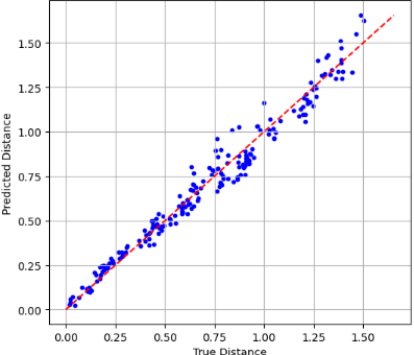}
		\caption{\scalebox{0.95}{5-NN/LNN-1-$\mathscr{D}^{'}_{\mathrm{te}}$}}
	\end{subfigure}
	\hfill
	\begin{subfigure}[t]{0.14\textwidth}
		\includegraphics[width=\linewidth]{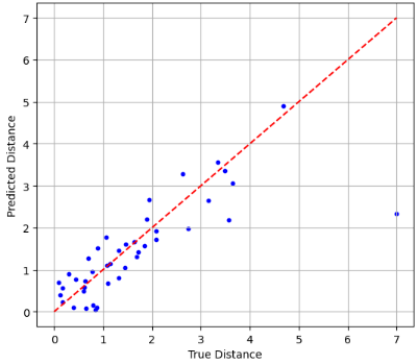}
		\caption{\scalebox{0.95}{6-NN/LNN-1-$\mathscr{D}^{'}_{\mathrm{te}}$}}
	\end{subfigure}
	\hfill
	\begin{subfigure}[t]{0.14\textwidth}
		\includegraphics[width=\linewidth]{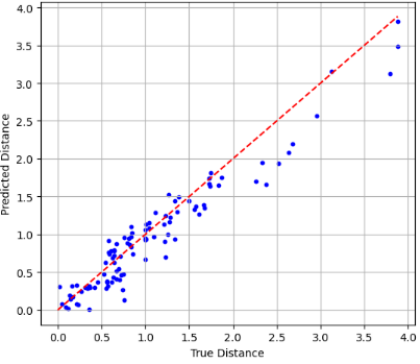}
		\caption{4-NN/LNN-1-$\mathcal{S}^{'}$}
	\end{subfigure}
	\hfill
	\begin{subfigure}[t]{0.14\textwidth}
		\includegraphics[width=\linewidth]{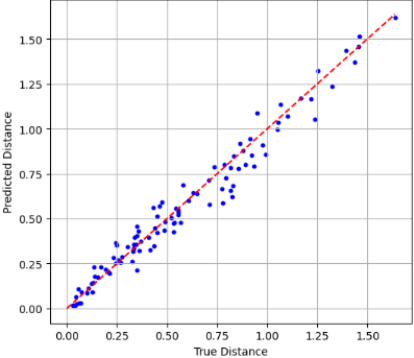}
		\caption{5-NN/LNN-1-$\mathcal{S}^{'}$}
	\end{subfigure}
	\hfill
	\begin{subfigure}[t]{0.14\textwidth}
		\includegraphics[width=\linewidth]{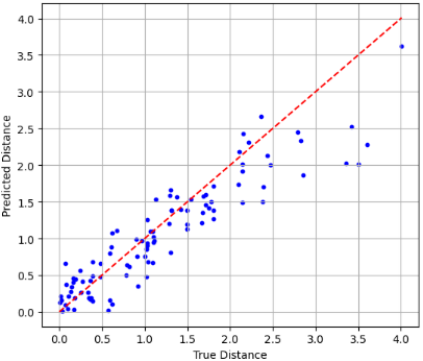}
		\caption{6-NN/LNN-1-$\mathcal{S}^{'}$}
	\end{subfigure}
	
	
	\begin{subfigure}[t]{0.14\textwidth}
		\includegraphics[width=\linewidth]{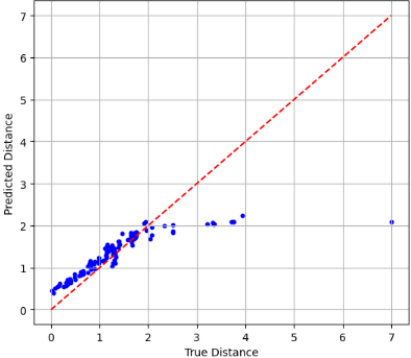}
		\caption{\scalebox{0.95}{4-NN/LNN-2-$\mathscr{D}^{'}_{\mathrm{te}}$}}
	\end{subfigure}
	\hfill
	\begin{subfigure}[t]{0.14\textwidth}
		\includegraphics[width=\linewidth]{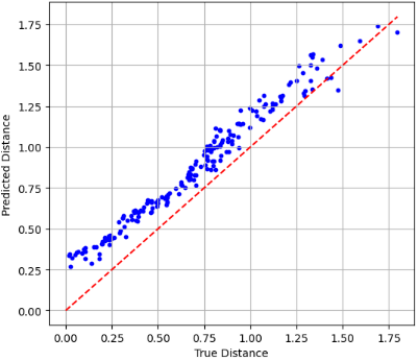}
		\caption{\scalebox{0.95}{5-NN/LNN-2-$\mathscr{D}^{'}_{\mathrm{te}}$}}
	\end{subfigure}
	\hfill
	\begin{subfigure}[t]{0.14\textwidth}
		\includegraphics[width=\linewidth]{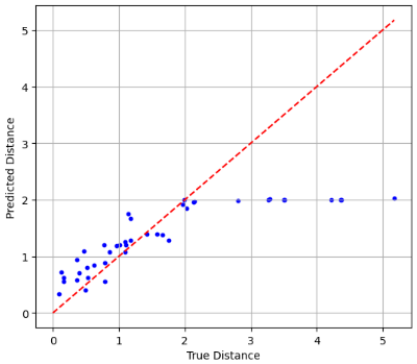}
		\caption{\scalebox{0.95}{6-NN/LNN-2-$\mathscr{D}^{'}_{\mathrm{te}}$}}
	\end{subfigure}
	\hfill
	\begin{subfigure}[t]{0.14\textwidth}
		\includegraphics[width=\linewidth]{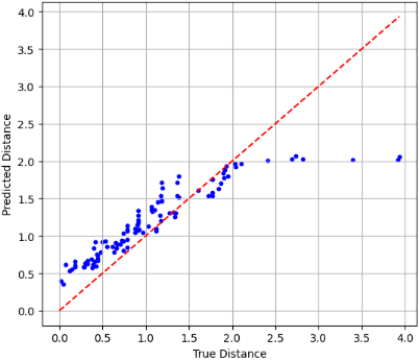}
		\caption{4-NN/LNN-2-$\mathcal{S}^{'}$}
	\end{subfigure}
	\hfill
	\begin{subfigure}[t]{0.14\textwidth}
		\includegraphics[width=\linewidth]{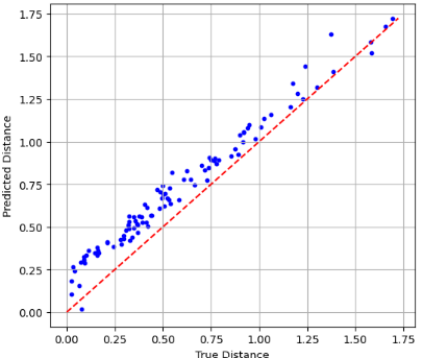}
		\caption{5-NN/LNN-2-$\mathcal{S}^{'}$}
	\end{subfigure}
	\hfill
	\begin{subfigure}[t]{0.14\textwidth}
		\includegraphics[width=\linewidth]{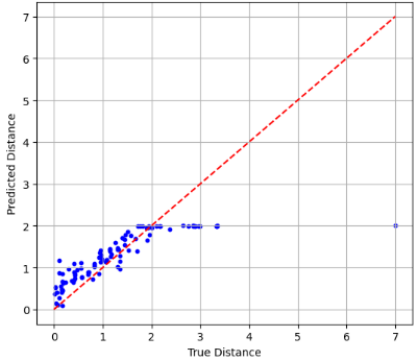}
		\caption{6-NN/LNN-2-$\mathcal{S}^{'}$}
	\end{subfigure}
	
	
	\begin{subfigure}[t]{0.14\textwidth}
		\includegraphics[width=\linewidth]{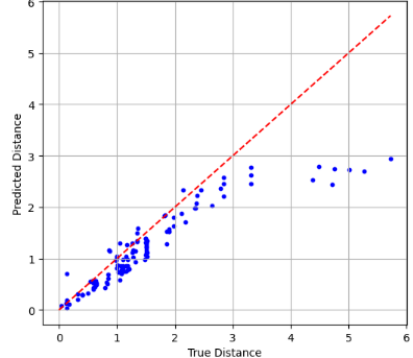}
		\caption{4-RF/GB-1-$\mathscr{D}^{'}_{\mathrm{te}}$}
	\end{subfigure}
	\hfill
	\begin{subfigure}[t]{0.14\textwidth}
		\includegraphics[width=\linewidth]{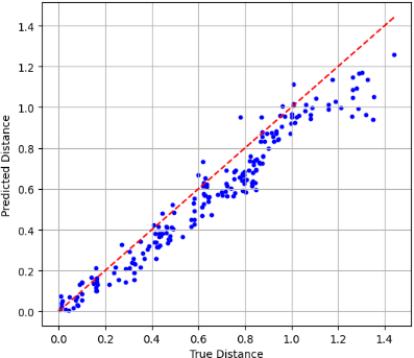}
		\caption{5-RF/GB-1-$\mathscr{D}^{'}_{\mathrm{te}}$}
	\end{subfigure}
	\hfill
	\begin{subfigure}[t]{0.14\textwidth}
		\includegraphics[width=\linewidth]{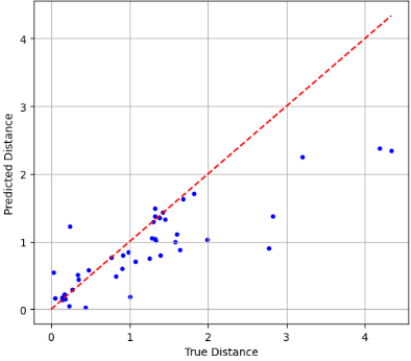}
		\caption{6-RF/GB-1-$\mathscr{D}^{'}_{\mathrm{te}}$}
	\end{subfigure}
	\hfill
	\begin{subfigure}[t]{0.14\textwidth}
		\includegraphics[width=\linewidth]{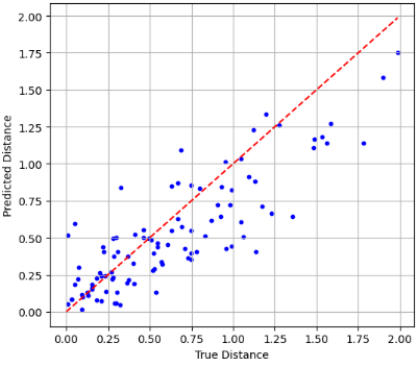}
		\caption{4-RF/GB-1-$\mathcal{S}^{'}$}
	\end{subfigure}
	\hfill
	\begin{subfigure}[t]{0.14\textwidth}
		\includegraphics[width=\linewidth]{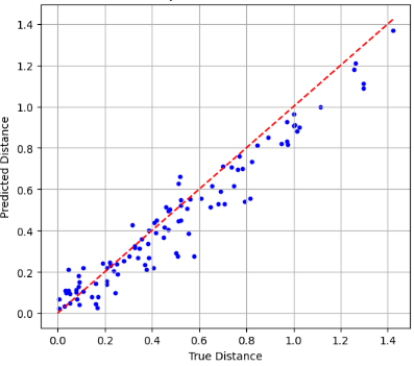}
		\caption{5-RF/GB-1-$\mathcal{S}^{'}$}
	\end{subfigure}
	\hfill
	\begin{subfigure}[t]{0.14\textwidth}
		\includegraphics[width=\linewidth]{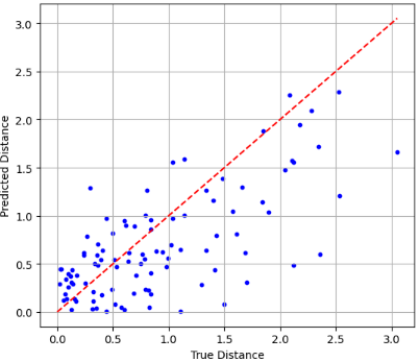}
		\caption{6-RF/GB-1-$\mathcal{S}^{'}$}
	\end{subfigure}
	
	
	\begin{subfigure}[t]{0.14\textwidth}
		\includegraphics[width=\linewidth]{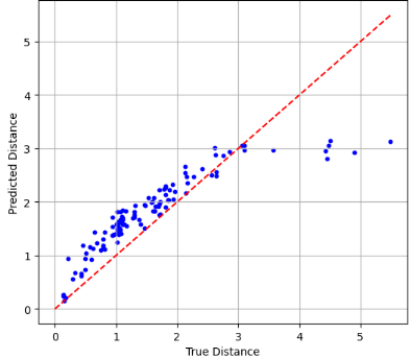}
		\caption{4-RF/GB-2-$\mathscr{D}^{'}_{\mathrm{te}}$}
	\end{subfigure}
	\hfill
	\begin{subfigure}[t]{0.14\textwidth}
		\includegraphics[width=\linewidth]{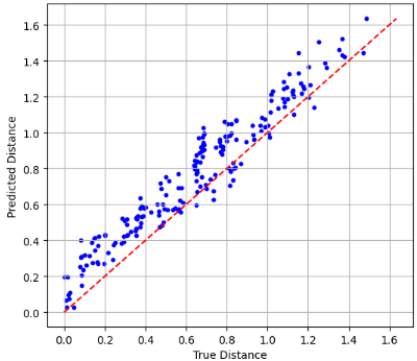}
		\caption{5-RF/GB-2-$\mathscr{D}^{'}_{\mathrm{te}}$}
	\end{subfigure}
	\hfill
	\begin{subfigure}[t]{0.14\textwidth}
		\includegraphics[width=\linewidth]{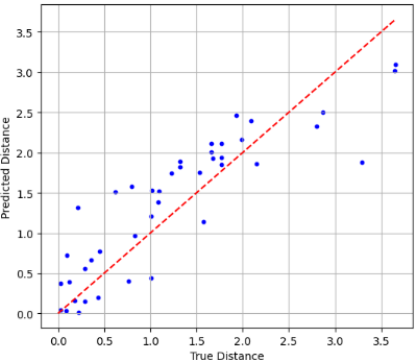}
		\caption{6-RF/GB-2-$\mathscr{D}^{'}_{\mathrm{te}}$}
	\end{subfigure}
	\hfill
	\begin{subfigure}[t]{0.14\textwidth}
		\includegraphics[width=\linewidth]{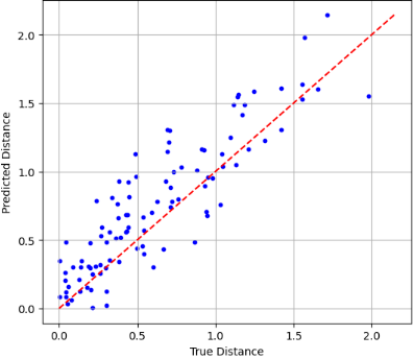}
		\caption{4-RF/GB-2-$\mathcal{S}^{'}$}
	\end{subfigure}
	\hfill
	\begin{subfigure}[t]{0.14\textwidth}
		\includegraphics[width=\linewidth]{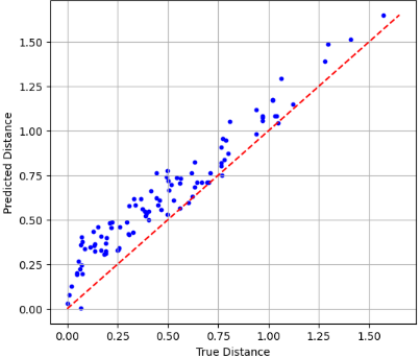}
		\caption{5-RF/GB-2-$\mathcal{S}^{'}$}
	\end{subfigure}
	\hfill
	\begin{subfigure}[t]{0.14\textwidth}
		\includegraphics[width=\linewidth]{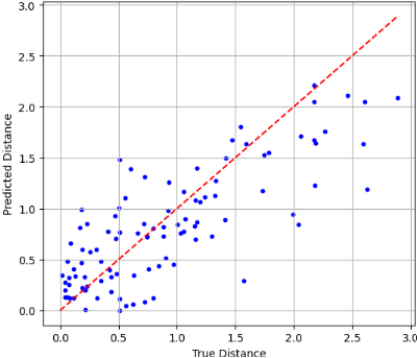}
		\caption{6-RF/GB-2-$\mathcal{S}^{'}$}
	\end{subfigure}
	
	
	\begin{subfigure}[t]{0.14\textwidth}
		\includegraphics[width=\linewidth]{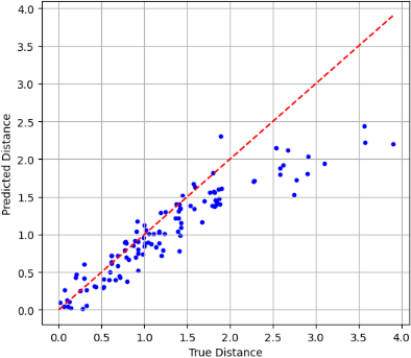}
		\caption{4-GB/GB-1-$\mathscr{D}^{'}_{\mathrm{te}}$}
	\end{subfigure}
	\hfill
	\begin{subfigure}[t]{0.14\textwidth}
		\includegraphics[width=\linewidth]{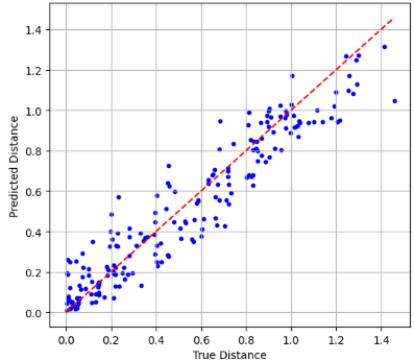}
		\caption{5-GB/GB-1-$\mathscr{D}^{'}_{\mathrm{te}}$}
	\end{subfigure}
	\hfill
	\begin{subfigure}[t]{0.14\textwidth}
		\includegraphics[width=\linewidth]{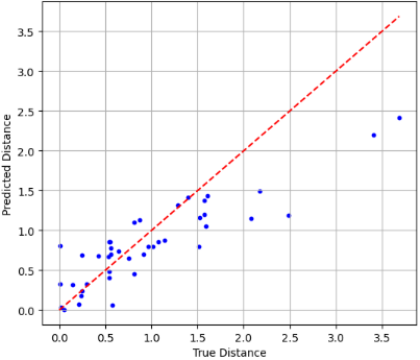}
		\caption{6-GB/GB-1-$\mathscr{D}^{'}_{\mathrm{te}}$}
	\end{subfigure}
	\hfill
	\begin{subfigure}[t]{0.14\textwidth}
		\includegraphics[width=\linewidth]{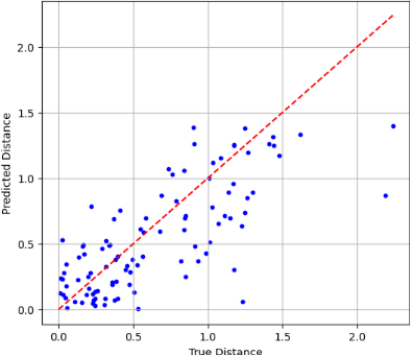}
		\caption{4-GB/GB-1-$\mathcal{S}^{'}$}
	\end{subfigure}
	\hfill
	\begin{subfigure}[t]{0.14\textwidth}
		\includegraphics[width=\linewidth]{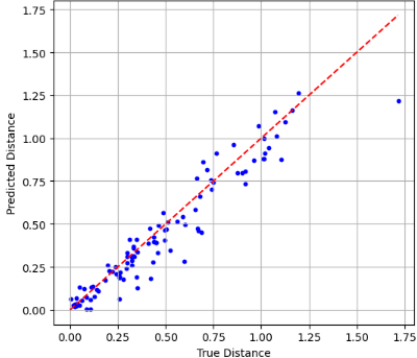}
		\caption{5-GB/GB-1-$\mathcal{S}^{'}$}
	\end{subfigure}
	\hfill
	\begin{subfigure}[t]{0.14\textwidth}
		\includegraphics[width=\linewidth]{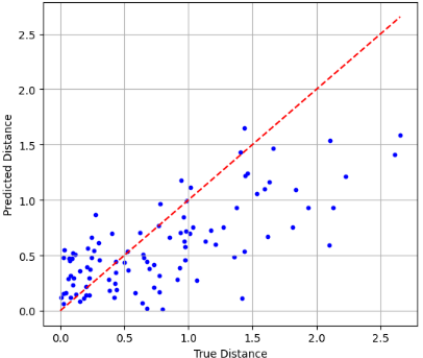}
		\caption{6-GB/GB-1-$\mathcal{S}^{'}$}
	\end{subfigure}
	
	
	\begin{subfigure}[t]{0.14\textwidth}
		\includegraphics[width=\linewidth]{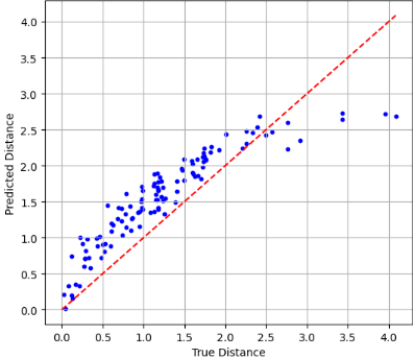}
		\caption{4-GB/GB-2-$\mathscr{D}^{'}_{\mathrm{te}}$}
	\end{subfigure}
	\hfill
	\begin{subfigure}[t]{0.14\textwidth}
		\includegraphics[width=\linewidth]{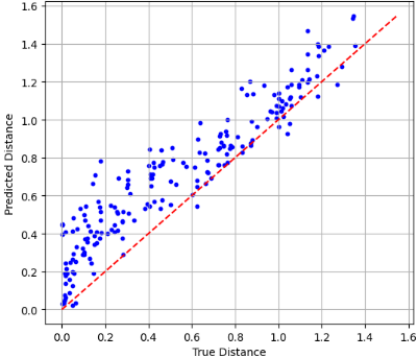}
		\caption{5-GB/GB-2-$\mathscr{D}^{'}_{\mathrm{te}}$}
	\end{subfigure}
	\hfill
	\begin{subfigure}[t]{0.14\textwidth}
		\includegraphics[width=\linewidth]{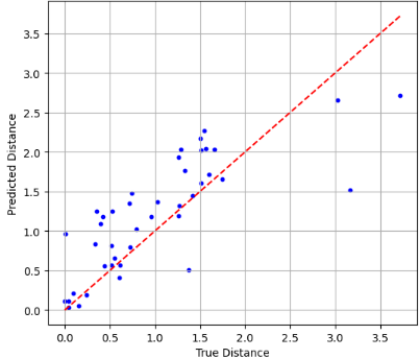}
		\caption{6-GB/GB-2-$\mathscr{D}^{'}_{\mathrm{te}}$}
	\end{subfigure}
	\hfill
	\begin{subfigure}[t]{0.14\textwidth}
		\includegraphics[width=\linewidth]{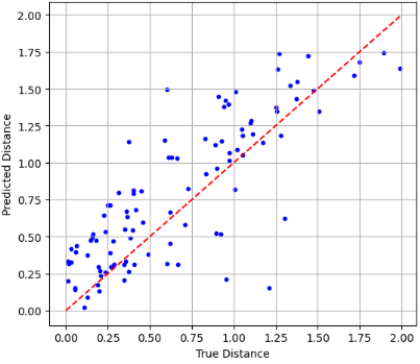}
		\caption{4-GB/GB-2-$\mathcal{S}^{'}$}
	\end{subfigure}
	\hfill
	\begin{subfigure}[t]{0.14\textwidth}
		\includegraphics[width=\linewidth]{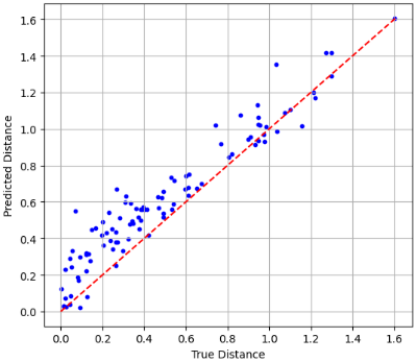}
		\caption{5-GB/GB-2-$\mathcal{S}^{'}$}
	\end{subfigure}
	\hfill
	\begin{subfigure}[t]{0.14\textwidth}
		\includegraphics[width=\linewidth]{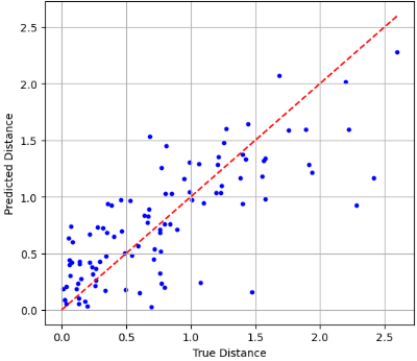}
		\caption{6-GB/GB-2-$\mathcal{S}^{'}$}
	\end{subfigure}
	
	
	\begin{subfigure}[t]{0.14\textwidth}
		\includegraphics[width=\linewidth]{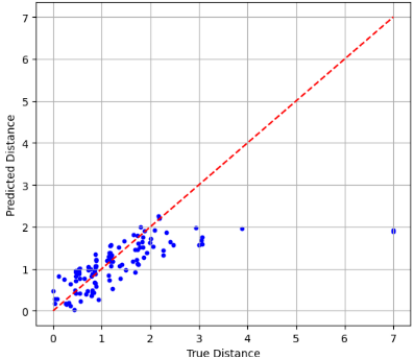}
		\caption{4-NN/GB-1-$\mathscr{D}^{'}_{\mathrm{te}}$}
	\end{subfigure}
	\hfill
	\begin{subfigure}[t]{0.14\textwidth}
		\includegraphics[width=\linewidth]{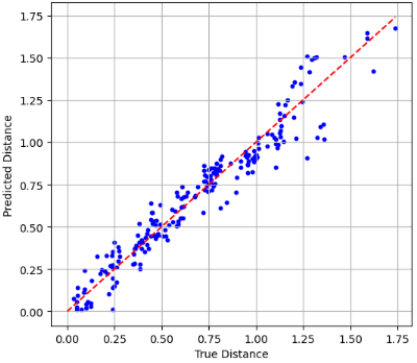}
		\caption{5-NN/GB-1-$\mathscr{D}^{'}_{\mathrm{te}}$}
	\end{subfigure}
	\hfill
	\begin{subfigure}[t]{0.14\textwidth}
		\includegraphics[width=\linewidth]{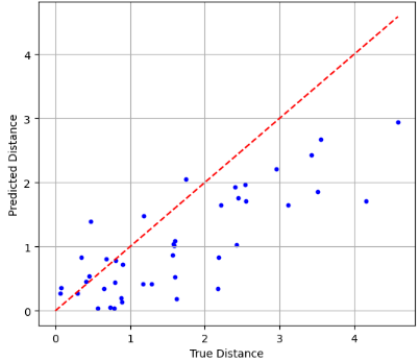}
		\caption{6-NN/GB-1-$\mathscr{D}^{'}_{\mathrm{te}}$}
	\end{subfigure}
	\hfill
	\begin{subfigure}[t]{0.14\textwidth}
		\includegraphics[width=\linewidth]{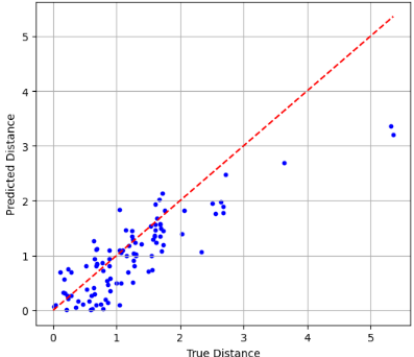}
		\caption{4-NN/GB-1-$\mathcal{S}^{'}$}
	\end{subfigure}
	\hfill
	\begin{subfigure}[t]{0.14\textwidth}
		\includegraphics[width=\linewidth]{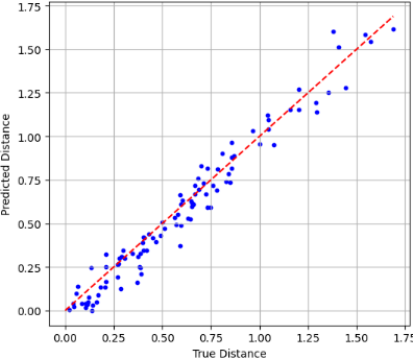}
		\caption{5-NN/GB-1-$\mathcal{S}^{'}$}
	\end{subfigure}
	\hfill
	\begin{subfigure}[t]{0.14\textwidth}
		\includegraphics[width=\linewidth]{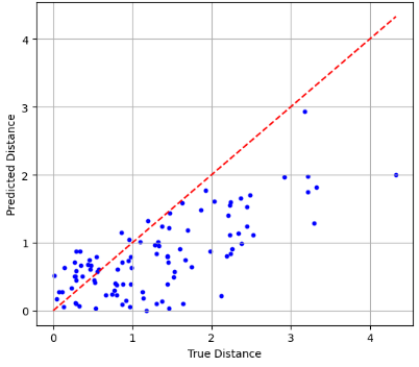}
		\caption{6-NN/GB-1-$\mathcal{S}^{'}$}
	\end{subfigure}
	
	
	\begin{subfigure}[t]{0.14\textwidth}
		\includegraphics[width=\linewidth]{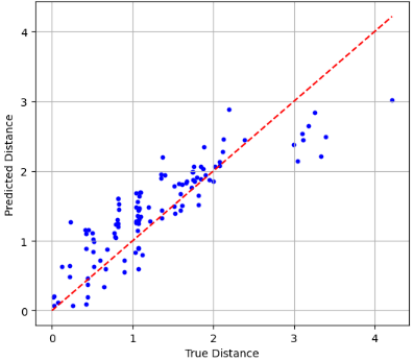}
		\caption{4-NN/GB-2-$\mathscr{D}^{'}_{\mathrm{te}}$}
	\end{subfigure}
	\hfill
	\begin{subfigure}[t]{0.14\textwidth}
		\includegraphics[width=\linewidth]{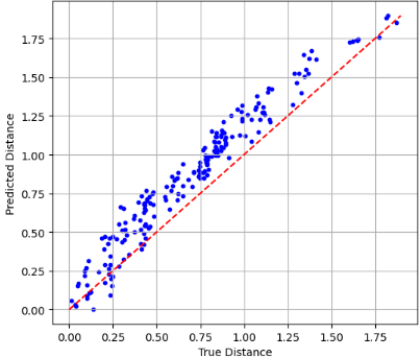}
		\caption{5-NN/GB-2-$\mathscr{D}^{'}_{\mathrm{te}}$}
	\end{subfigure}
	\hfill
	\begin{subfigure}[t]{0.14\textwidth}
		\includegraphics[width=\linewidth]{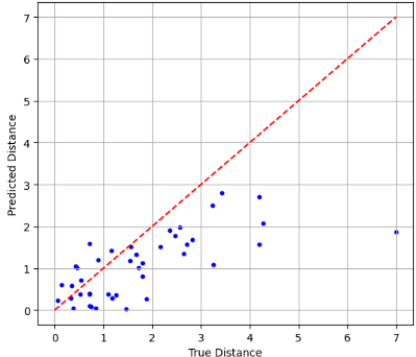}
		\caption{6-NN/GB-2-$\mathscr{D}^{'}_{\mathrm{te}}$}
	\end{subfigure}
	\hfill
	\begin{subfigure}[t]{0.14\textwidth}
		\includegraphics[width=\linewidth]{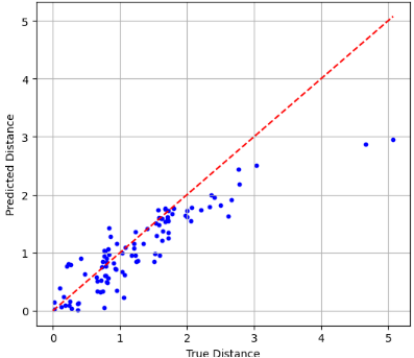}
		\caption{4-NN/GB-2-$\mathcal{S}^{'}$}
	\end{subfigure}
	\hfill
	\begin{subfigure}[t]{0.14\textwidth}
		\includegraphics[width=\linewidth]{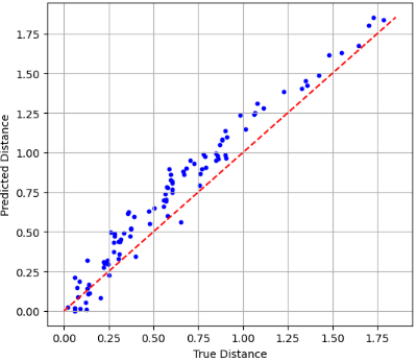}
		\caption{5-NN/GB-2-$\mathcal{S}^{'}$}
	\end{subfigure}
	\hfill
	\begin{subfigure}[t]{0.14\textwidth}
		\includegraphics[width=\linewidth]{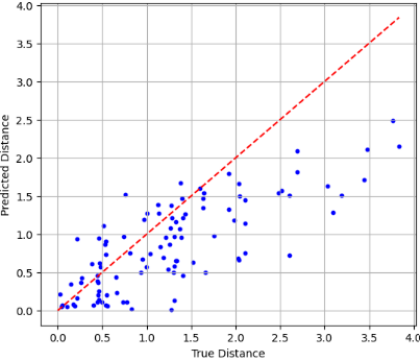}
		\caption{6-NN/GB-2-$\mathcal{S}^{'}$}
	\end{subfigure}
\end{figure}

\end{document}